\documentclass[smallcondensed,natbib]{svjour3}     
\smartqed  
\usepackage{graphicx}
\usepackage{mathptmx}      
\usepackage{latexsym}
\usepackage{bbm}
\usepackage{amsmath}

\makeatletter
\let\save@mathaccent\mathaccent
\newcommand*\if@single[3]{%
  \setbox0\hbox{${\mathaccent"0362{#1}}^H$}%
  \setbox2\hbox{${\mathaccent"0362{\kern0pt#1}}^H$}%
  \ifdim\ht0=\ht2 #3\else #2\fi
  }
\newcommand*\rel@kern[1]{\kern#1\dimexpr\macc@kerna}
\newcommand*\widebar[1]{\@ifnextchar^{{\wide@bar{#1}{0}}}{\wide@bar{#1}{1}}}
\newcommand*\wide@bar[2]{\if@single{#1}{\wide@bar@{#1}{#2}{1}}{\wide@bar@{#1}{#2}{2}}}
\newcommand*\wide@bar@[3]{%
  \begingroup
  \def\mathaccent##1##2{%
    \let\mathaccent\save@mathaccent
    \if#32 \let\macc@nucleus\first@char \fi
    \setbox\z@\hbox{$\macc@style{\macc@nucleus}_{}$}%
    \setbox\tw@\hbox{$\macc@style{\macc@nucleus}{}_{}$}%
    \dimen@\wd\tw@
    \advance\dimen@-\wd\z@
    \divide\dimen@ 3
    \@tempdima\wd\tw@
    \advance\@tempdima-\scriptspace
    \divide\@tempdima 10
    \advance\dimen@-\@tempdima
    \ifdim\dimen@>\z@ \dimen@0pt\fi
    \rel@kern{0.6}\kern-\dimen@
    \if#31
      \overline{\rel@kern{-0.6}\kern\dimen@\macc@nucleus\rel@kern{0.4}\kern\dimen@}%
      \advance\dimen@0.4\dimexpr\macc@kerna
      \let\final@kern#2%
      \ifdim\dimen@<\z@ \let\final@kern1\fi
      \if\final@kern1 \kern-\dimen@\fi
    \else
      \overline{\rel@kern{-0.6}\kern\dimen@#1}%
    \fi
  }%
  \macc@depth\@ne
  \let\math@bgroup\@empty \let\math@egroup\macc@set@skewchar
  \mathsurround\z@ \frozen@everymath{\mathgroup\macc@group\relax}%
  \macc@set@skewchar\relax
  \let\mathaccentV\macc@nested@a
  \if#31
    \macc@nested@a\relax111{#1}%
  \else
    \def\gobble@till@marker##1\endmarker{}%
    \futurelet\first@char\gobble@till@marker#1\endmarker
    \ifcat\noexpand\first@char A\else
      \def\first@char{}%
    \fi
    \macc@nested@a\relax111{\first@char}%
  \fi
  \endgroup
}
\makeatother

\usepackage{amsfonts}
\usepackage{amssymb}
\usepackage{multirow,multicol}
\usepackage{graphicx,color}
\usepackage{url}
\usepackage[pdftex]{hyperref}
\usepackage{algorithm} 
\usepackage{algorithmic}
\usepackage{comment}
\usepackage{sidecap}
\usepackage{soul}

\DeclareGraphicsExtensions{.pdf,.png,.jpg}

\spnewtheorem{fact}{Fact}{\bfseries}{\itshape}
\spnewtheorem{condition}{Condition}{\bfseries}{\itshape}

\newcommand{\Dich}{\mathrm{Dich}}

\newcommand{\Hypo}{\mathcal{H}}
\newcommand{\Hrep}{\widetilde{\mathcal{H}}}
\newcommand{\EHypo}{\widehat{\mathcal{H}}}
\newcommand{\T}{\mathcal{T}}
\newcommand{\A}{\mathcal{A}}
\newcommand{\M}{\mathbf{M}}

\newcommand{\X}{\mathcal{X}}

\newcommand{\D}{\mathcal{D}}
\newcommand{\Ncal}{\mathcal{N}}
\newcommand{\R}{\mathbb{R}}
\newcommand{\Q}{\mathbb{Q}}
\newcommand{\N}{\mathbb{N}}
\newcommand{\Nbf}{\mathbf{N}}
\newcommand{\Mcal}{\mathcal{M}}
\newcommand{\Ecal}{\mathcal{E}}
\newcommand{\Ucal}{\mathcal{U}}
\newcommand{\Wcal}{\mathcal{W}}
\newcommand{\Ccal}{\mathcal{C}}
\newcommand{\THat}{\widehat{T}}

\newcommand{\err}{\mathrm{err}}
\newcommand{\Err}{\mathrm{Err}}
\newcommand{\AdaSel}{\mathrm{AdaSelect}}

\newcommand{\margin}{\mathrm{margin}}

\newcommand{\sign}[1]{\mathrm{sign}\left( #1 \right) }
\newcommand{\loss}{\mathrm{loss}}
\newcommand{\indicator}[1]{\mathbbm{1}{\left[ {#1} \right] }}
\newcommand{\E}[1]{\mathbf{E}\left[ #1\right] }

\newcommand{\Pbf}{\mathbf{P}}
\newcommand{\Int}[1]{\mathrm{Int}{\left( {#1} \right)}}

\newcommand{\approxpi}{\widetilde{\pi}}
\newcommand{\approxphi}{\widetilde{\phi}}
\newcommand{\approxDelta}{\widetilde{\Delta}}
\newcommand{\approxOmega}{\widetilde{\Omega}}
\newcommand{\approxA}{\widetilde{\A}}
\newcommand{\Proj}{\mathrm{Proj}}
\newcommand{\Tavg}{\mathrm{Tavg}}
\newcommand{\Savg}{\mathrm{Savg}}
\DeclareMathOperator*{\argmin}{arg\,min}
\DeclareMathOperator*{\argmax}{arg\,max}

\newcommand{\closure}[1]{\overline{#1}}

\newcommand{\DichSets}{\text{\sc Dich}}
\newcommand{\HypoRep}{\text{\sc Hypo}}
\newcommand{\Tmin}{T_{\mathrm{min}}}
\newcommand{\VC}[1]{\mathrm{VC}\left( #1 \right) }

\journalname{Machine Learning}

\begin{document}

\title{On the Convergence Properties of Optimal AdaBoost}



\author{Joshua Belanich \and Luis E. Ortiz}

\institute{J. Belanich \at Google New York, New York, NY, USA\\
\email{joshuabelanich@google.com}\\
\and 
       L. E. Ortiz \at
       Department of Computer and Information Science,
       University of Michigan - Dearborn,
       4901 Evergreen Rd. Room 129 CIS Bldg.,
Dearborn, MI 48128, USA\\
       Tel.: 313-593-5239\\
       Fax: 313-632-4256\\
 \email{leortiz@umich.edu}
}

\date{January 4, 2023}

\maketitle

\begin{abstract}
AdaBoost is one of the most popular machine-learning (ML) algorithms.  It
is simple to implement and often found very effective by
practitioners, while still being mathematically elegant and
theoretically sound.
AdaBoost's interesting behavior in practice still puzzles the
ML community.  
We address the
algorithm's \emph{stability} and establish multiple convergence properties of ``Optimal AdaBoost,'' a
term coined by Rudin, Daubechies, and Schapire in 2004.  
We 
prove, in a reasonably strong computational sense, the
almost universal existence of time
averages, and with that, the convergence
of the classifier itself, 
its generalization error, and its resulting margins, among many other objects, for fixed data
sets under arguably reasonable conditions.
Specifically, we frame Optimal AdaBoost as a dynamical system 
and, employing tools from ergodic theory, prove that, under a condition that Optimal AdaBoost does not have ties for best weak classifier eventually, a condition for which we provide empirical evidence from high-dimensional real-world datasets, the algorithm's update behaves like a continuous map. 
We provide constructive proofs of several arbitrarily accurate
approximations of Optimal AdaBoost; prove that they exhibit
certain cycling behavior in finite time, and that the resulting
dynamical system is ergodic; and establish sufficient conditions for the
same to hold for the actual Optimal-AdaBoost update.
We believe that
our results provide reasonably strong 
evidence for the affirmative answer to two open conjectures, at least
from a broad computational-theory perspective: AdaBoost always cycles
and is an ergodic dynamical system. 
We
present 
empirical evidence that cycles are hard to detect while time averages
stabilize quickly.
Our results 
ground future convergence-rate analysis 
and may help optimize
generalization ability and alleviate 
a practitioner's burden of deciding how long to run the algorithm.
\end{abstract}

\vspace{2\baselineskip}

{\hfill \em LEO: Dedicated to the memory of Patrick Henry Winston}

\vspace{2\baselineskip}

\hspace{1in} 
\scalebox{0.82}{
\begin{minipage}{0.8\textwidth}
\framebox[1.1\textwidth]{
\begin{minipage}{\textwidth}
\begin{center}
{\em THE MYSTERY THICKENS}
\end{center}
AdaBoost created a big splash in machine learning and led to hundreds, perhaps thousands of papers. It was the most accurate classification algorithm available at the time.

\vspace{\baselineskip}

It differs significantly from bagging. Bagging uses the biggest trees possible as the weak learners to reduce bias.

\vspace{\baselineskip}

AdaBoost uses small trees as the weak learners, often being effective using trees formed by a single split (stumps).

\vspace{\baselineskip}

There is empirical evidence that it reduces bias as well as variance.

\vspace{\baselineskip}

It seemed to converge with the test set error gradually decreasing as hundreds or thousands of trees were added.

\vspace{\baselineskip}

On simulated data its error rate is close to the Bayes rate.

\vspace{\baselineskip}

But why it worked so well was a mystery that bothered me. For the last five years I have characterized the understanding of Adaboost as the most important open problem in machine learning.

\vspace{\baselineskip}

\end{minipage}
}

\vspace{0.5\baselineskip}

\hfill \begin{minipage}{0.88\textwidth}
{\em Leo Breiman, Machine Learning. Wald Lecture 1 speech presented at the 277th meeting of the Institute of Mathematical Statistics, held in Banff, Alberta, Canada in July 2002. Slide 29}~\footnote{\url{http://www.stat.berkeley.edu/~breiman/wald2002-1.pdf}}
\end{minipage}
\end{minipage}
}

\section{Introduction}

If one wants to place the broad impact and overall significance of
AdaBoost in perspective, the following quote is hard to beat.  It forms part of the statement from the ACM-SIGACT Awarding Committee for the 2003 G\"{o}del Prize,~\footnote{\url{http://www.sigact.org/Prizes/Godel/2003.html}} presented to Yoav Freund and Robert Schapire, the creators of AdaBoost, for their original paper in which they introduced the algorithm and established some of its theoretical foundations and properties~\citep{Freund97adecision-theoretic}:
\begin{quote}
{\em ``The algorithm demonstrated novel possibilities in analysing data and is a permanent contribution to science even beyond computer science. Because of a combination of features, including its elegance, the simplicity of its implementation, its wide applicability, and its striking success in reducing errors in benchmark applications even while its theoretical assumptions are not known to hold, the algorithm set off an explosion of research in the fields of statistics, artificial intelligence, experimental machine learning, and data mining. The algorithm is now widely used in practice.''}
\end{quote}
The last two sentences have been shown to be clear understatements
over the last two decades since the award was bestowed.

The late, eminent statistician Leo Breiman (1928-2005) once called
AdaBoost the best off-the-shelf classifier for a wide variety of
datasets~\citep{Breiman:1999:PGA:334369.334370}.  Two decades 
later, AdaBoost is still widely used because of its simplicity, speed,
and theoretical guarantees for good performance, reinforcing the
essence of the quote above.  However, despite its overwhelming
popularity, 
some mystery still surrounds its generalization
per\-for\-mance~\citep{MeaseEvidenceContrary}.  
As stated in the slide presented before the introduction here, Breiman ``characterized the understanding of Adaboost as the most important open problem in machine learning.''

In this paper we concentrate on the convergence properties of Optimal
AdaBoost. We construct several (almost-everywhere uniform)
approximations of Optimal AdaBoost and formally establish several key
theoretical 
properties:
(1) they are arbitrarily accurate; (2) they exhibit 
certain cycling behavior in finite time; and (3) their resulting dynamical system is
ergodic. We also formally establish sufficient conditions (e.g.,
non-expansion and/or no ties
eventually) for the
same to hold for the actual Optimal-AdaBoost update. We believe that our results provide a reasonably strong
evidence in favor of the affirmative answer to two open conjectures,
at least from a computational perspective. One is the so-called
``AdaBoost Always Cycles''
Conjecture~\citep{DBLP:journals/jmlr/RudinSD12}, an important open problem in computational learning
theory; that is, that the convergence of the sequence of example-weights
that the algorithm generates at every round to a cycle, or the
like, as it is often
 stated. The other is the so-called ``AdaBoost
is Ergodic'' Conjecture, which~\citet{BreimanForests} championed.
Yet, despite the theoretical guarantee of cycling behavior, we provide some
empirical evidence that suggests the following: if cycling occurs, it
may take a
long time to reach it, or the cycle may be quite long and thus hard to detect in
high-dimensional real-world datasets.  Indeed, there is no evidence of
cycling in the experiments we ran using such datasets, even within
what we believe is a
considerably large number of rounds. 

Time averages, on the other hand, do stabilize relatively quickly. Hence our emphasis in this paper: the convergence of important
objects such as the AdaBoost classifier and other quantities such as its
generalization error.  More generally, our emphasis is the empirical (time)
average, over the number of rounds of AdaBoost, of functions 
of the example weights generated at every round.  For instance, we
study the averaging behavior of some important quantities such as (1) the weak-classifier weights,
(2) the edge-weight value of any example, (3) the 
example-weight distribution/histogram, and (4) the weighted
error of the weak-classifiers selected, to name a few.~\footnote{We refer the
reader to Section~\ref{sec:adaboost_conv_prelim} for formal
definitions and description of what we mean by ``the Convergence of the AdaBoost
Classifier,'' as well as the convergence of other related objects and
important quantities.}

\begin{figure}[t]
\begin{algorithmic}

\STATE {\bf Input} training dataset of $m$ (binary) labeled examples $D = \{(x^{(1)},y^{(1)}),\ldots, (x^{(m)},y^{(m)})\}$

\STATE {\bf Initialize} $w_{1}(i) \leftarrow 1/m$ for all $i = 1,\ldots,m$

\FOR{$t=1,\ldots,T$}

\STATE $h_t \leftarrow \text{\bf WeakLearn}(D,w_t)$

\STATE $\epsilon_t \leftarrow \err(h_t ; D, w_t) \equiv \sum_{i=1, h_t(x^{(i)}) \neq y^{(i)}}^m w_t(i)$

\FOR{$i = 1,\ldots,m$}
\STATE $w_{t+1}(i) \leftarrow \frac{1}{2} \times w_t(i) \times 
\begin{cases} 
\frac{1}{\epsilon_t}, & \text{if $h_t(x^{(i)}) \neq y^{(i)}$,}\\
\frac{1}{1 - \epsilon_t}, & \text{if $h_t(x^{(i)}) = y^{(i)}$.}
\end{cases}$
\ENDFOR

\STATE $\alpha_t \leftarrow \frac 12 \ln\left(\frac{1-\epsilon_t}{\epsilon_t}\right)$

\ENDFOR

\STATE {\bf Output} final classifier: $H_T(x) \equiv \sign{F_T(x)}$ where $F_T(x) \equiv \sum_{t=1}^T \alpha_t \,  h_t(x)$.
\end{algorithmic}
\caption{\bf The AdaBoost Algorithm.}
\label{fig:adabo}
\end{figure}

\subsection{A brief introduction to AdaBoost}

AdaBoost, which stands for ``Adaptive Boosting,'' is a meta-learning algorithm
that works in rounds, sequentially combining 
\emph{``base'' or ``weak'' classifiers} at each
round~\citep{boosting_book}.  
Fig.~\ref{fig:adabo} presents an algorithmic description of
AdaBoost for \emph{binary classification}, which is the focus of this
paper.  Note that $\sign{.}$ is the
standard $\mathrm{sign}$ function: i.e., $\sign{z} = 1$ if $z > 0$;
and $-1$ otherwise.  
Implicit in the function $\text{\bf WeakLearn}(D,w_t)$ is a
\emph{weak-hypothesis class $\Hypo$} of functions from input features to binary outputs, in which
$h_t$ belongs.  In the case that the \emph{``weak'' hypothesis} $h_t$ that $\text{\bf
  WeakLearn}(D,w_t)$ returns achieves \emph{minimum
weighted-error} with respect to $D$ and $w_t$ among all hypothesis in $\Hypo$, 
the resulting
algorithm becomes the so called \emph{Optimal AdaBoost}, a term coined
by~\citet{RudinDynamics}.  That is, in Optimal AdaBoost, we have
$\err(h_t ; D, w_t) = \min_{h \in \Hypo} \err(h ; D, w_t)$, where $\err(h
; D, w_t)$ is analogous to the definition given for $\err(h_t ; D ,
w_t)$ in the description of the algorithm
in Fig.~\ref{fig:adabo}.~\footnote{In what follows, for simplicity, we
often refer to ``Optimal AdaBoost'' simply as ``AdaBoost,'' unless
stated otherwise.}

In what remains of the Introduction we discuss AdaBoost ``puzzling''
behavior (Section~\ref{sec:puzzling}), along with attempts to explain
it (Sections~\ref{sec:margin} and~\ref{sec:Breiman}).  We state our views
on the subject (Section~\ref{sec:ourview}), present a high-level summary of our contributions
(Section~\ref{sec:contr}), and provide an overview of the upcoming
sections of the paper (Section~\ref{sec:overview}).

\subsection{AdaBoost behavior in practice is ``puzzling'' }
\label{sec:puzzling}

\begin{figure}[t]
\begin{tabular}{cc}
\includegraphics[width=0.48\textwidth]
    {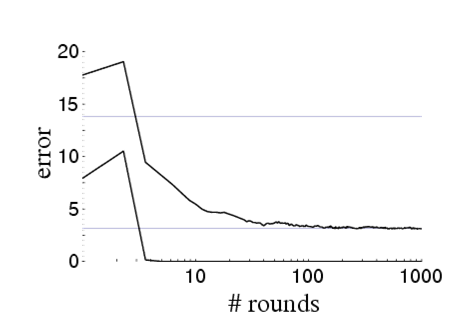} &
\includegraphics[width=0.48\textwidth]
    {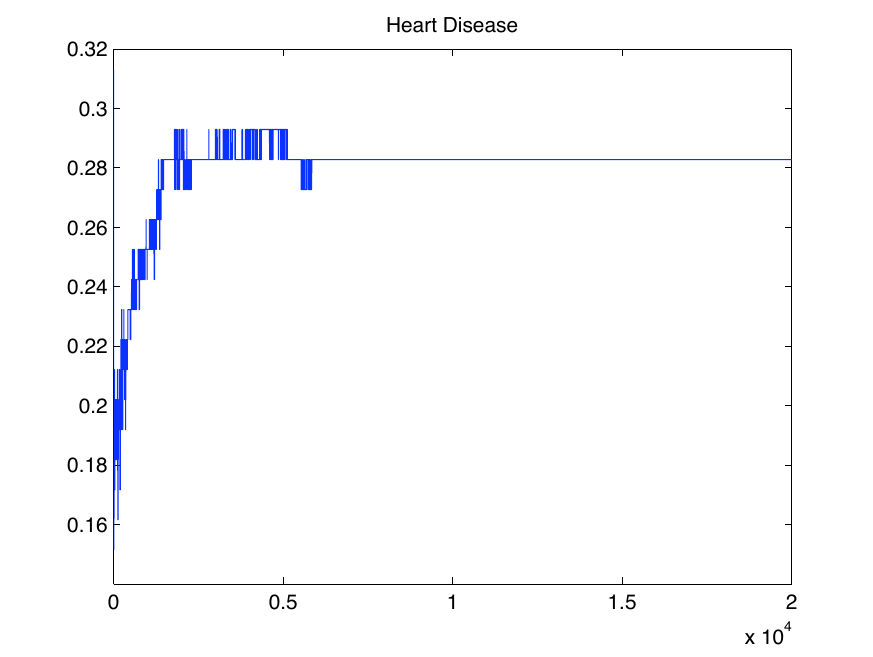} \\
(a) & (b) 
\end{tabular}
  \caption{\label{fig:c4.5} (a) {\bf Training and Test Error (y-axis) for
    AdaBoosting C4.5 on the Letter Dataset for up to 1000 Rounds
    (x-axis, in log-scale).} This plot, which originally appeared
    in~\citet{Schapire97boostingthe}, is still featured in many
    tutorials and talks, but without any definitive formal explanation.
    We refer the reader to~\citet{Breiman:1999:PGA:334369.334370}
    and~\citet{Grove:1998:BLM:295240.295766} for experimental evidence
    against the ``max-margin theory'' originally put forward as an
    explanation for this behavior by~\citet{Schapire97boostingthe}.
    We believe our contribution on the convergence properties of
    AdaBoost, discussed later, formally provides a potential
    explanation for this behavior. (b) {\bf Heart-Disease Dataset Test Error.} Test error (y-axis) of
    AdaBoosting decision-stumps on the Heart-Disease
    dataset~\citep{UCIHeart} for up to 20,000 rounds (x-axis). Note
    the same converging behavior exhibited for the Letter Dataset in
    (a). This converging behavior is typically
    observed in empirical studies of AdaBoost. Note however that this
    time, unlike the previous/canonical figure, AdaBoost seems to be
    overfitting. }
\end{figure}

As shown in Fig.~\ref{fig:adabo}, on each round, AdaBoost adds a
hypothesis, generated by the weak learning algorithm, to a
running linear combination of hypotheses.  Our common machine-learning
(ML) intuition would suggest that the complexity of this
combination of hypotheses would increase the longer the
algorithm runs.  Meanwhile, in practice,  the generalization performance of this
ensemble tends to improve or remain stationary after a large number of
iterations.  
Such behavior goes against our general theoretically-inspired
intuition and accumulated knowledge in ML.  We expect that as the complexity of the model increases, as it appears to
be the case at least on the surface for the AdaBoost classifier, the generalization error also
increases.

While this behavior does not contradict standard theoretical bounds based on the
VC-dimension of AdaBoost classifiers,
it does suggest that it seems futile to attempt to apply the standard
view to this context.  \emph{In some cases, the generalization error
  continues to decrease long after the training error of the
  corresponding AdaBoost classifier has reached
  zero~\citep{Schapire97boostingthe}.}   VC-Dimension-based bounds
cannot really explain this
behavior~\citep{BreimanArcingClassifiers,drucker:boosting,QuinlanBagging}. 
 In fact, that behavior seems generally inconsistent with the
fundamental nature of such bounds and other insights we have gained
from computational learning theory.  A common graph depicting this behavior is Fig.~\ref{fig:c4.5}(a).  Remarkably, the complicated combination of $1000$ trees generalizes better than the simpler combination of $10$.

\subsubsection{The margin theory has its limitations}
\label{sec:margin}

\hspace{1in} 
\scalebox{0.82}{
\begin{minipage}{0.8\textwidth}
\framebox[1.1\textwidth]{
\begin{minipage}{\textwidth}
\begin{center}
{\em IDEAS ABOUT WHY IT WORKS}
\end{center}
A. Adaboost raises the weights on cases previously misclassified, so
it focusses on the hard cases $i$th the easy cases just carried
along. [\emph{sic}]

\vspace{\baselineskip}

wrong: empirical results showed that Adaboost tried to equalize the misclassification rate over all cases.

\vspace{\baselineskip}

B. The margin explanation: An ingenious work by Shapire, et.al. derived an upper bound on the error rate of a convex combination of predictors in terms of the VC dimension of each predictor in the ensemble and the margin distribution.

\vspace{\baselineskip}

The margin for the $i$th case is the vote in the ensemble for the correct class minus the largest vote for any of the other classes.

\vspace{\baselineskip}

The authors conjectured that Adaboost was so poweful because it produced high margin distributions.

\vspace{\baselineskip}

I devised and published an algorithm that produced uniformly higher margin disrbutions than Adaboost, and yet was less accurate.

\vspace{\baselineskip}

So much for margins.

\vspace{\baselineskip}

\end{minipage}
}

\vspace{0.5\baselineskip}

\hfill \begin{minipage}{0.9\textwidth}
{\em Leo Breiman, Machine Learning. 2002 Wald Lecture. Slide 30}
\end{minipage}
\end{minipage}
}

\vspace{\baselineskip}

Solving the apparent paradox roughly introduced above has been a driving force behind boosting
research, and various explanations have been proposed~\citep{boosting_book}.  By far the
most popular among them is the theory of margins~\citep{Schapire97boostingthe}.  The generalization
error of any convex combination of functions can be bounded by a
function of their \emph{margins} on the training examples, independent
of the number of classifiers in the ensemble.  AdaBoost provably
produces reasonably large margins, and tends to continue to improve
the margins even after the training error has reached zero~\citep{Schapire97boostingthe}.  

The margin theory is effective at explaining AdaBoost's generalization
performance at a high level.  But it still has its downsides, as
Breiman's presentation slide indicates.  There
is evidence both for and against the power of the margin theory to
predict the quality of the generalization
performance~\citep{Breiman:1999:PGA:334369.334370,RudinDynamics,RudinPreciseStatements,rudin2007,Reyzin:2006:BMB:1143844.1143939}.  
But the most striking problem is that the margin bound is very loose:
It does not explain the precise behavior of the error.  For example,
when looking at Fig.~\ref{fig:c4.5}(a) a couple of questions arise.
  Why is the test error not fluctuating wildly underneath the bound
induced by the margin?  Or even, why is the test error not approaching the bound?  Remarkably, the error does neither of these things, and seems to converge to a stationary value.

This phenomenon is not unique to this dataset.  This convergence can
be seen on many different datasets, both natural and synthetic.  Even
in cases where AdaBoost seems to be overfitting, \emph{the generalization
performance tends to stabilize.}  Take for example
Fig.~\ref{fig:c4.5}(b).  For the first 5000 rounds it appears that the algorithm is overfitting.  Afterwards, its generalization error stabilizes.

\subsubsection{Breiman attempts to theoretically explain AdaBoost's puzzling
  behavior in practice}
\label{sec:Breiman}

\hspace{1in} 
\scalebox{0.82}{\begin{minipage}{0.8\textwidth}
\framebox[1.1\textwidth]{
\begin{minipage}{\textwidth}
\begin{center}
{\em NAGGING QUESTIONS}
\end{center}
The classification by the $T$th ensemble is defined by~\footnote{It is unfortunate that Breiman's notation is not consistent with that traditionally used in the presentation of AdaBoost in the ML community. For example, we use $m$ for the number of samples in the training datatset $D$, while he used $T$ to denote the training dataset and $|T|$ for the number of samples. Similarly, we use $T$ for the number of rounds of AdaBoost, while he used $m$. Here we are quoting his slide using \emph{our} notation in the context of our paper.}
\begin{center}
$\sign{F^D_T(x)}$
\end{center}
The most important question I chewed on
\begin{center}
\underline{Is Adaboost consistent}?
\end{center}
Does $P(Y \neq \sign{F^D_T(X)})$ converge to the Bayes
risk as $T \to \infty$ and then $m \to \infty$?

\vspace{\baselineskip}

I am not a fan of endless asymptotics, but I believe that we need to know whether predictors are consistent or inconsistent

\vspace{\baselineskip}

For five years I have been bugging my theoretical colleagues with these questions.

\vspace{\baselineskip}

For a long time I thought the answer was yes.

\vspace{\baselineskip}

There was a paper 3 years ago which claimed that Adaboost overfit after 100,000 iterations, but I ascribed that to numerical roundoff error

\vspace{\baselineskip}

\end{minipage}
}

\vspace{0.5\baselineskip}

\hfill \begin{minipage}{0.9\textwidth}
{\em Leo Breiman, Machine Learning. 2002 Wald Lecture. Slide 34}
\end{minipage}
\end{minipage}
}

\vspace{\baselineskip}

\citet{BreimanForests} conjectured that AdaBoost was an \emph{ergodic dynamical system.}  He argued that if this was the case, then the dynamics of the weights over the examples behave like selecting from some probability distribution.  Therefore, AdaBoost can be treated as a random forest.  Using the strong law of large numbers, it follows that the generalization error of AdaBoost converges for certain weak learners.

\hspace{1in} 
\scalebox{0.82}{\begin{minipage}{0.8\textwidth}
\framebox[1.1\textwidth]{
\begin{minipage}{\textwidth}
\begin{center}
{\em THE BEGINNING OF THE END}
\end{center}
In 2000, I looked at the analog of Adaboost in population space, i.e. using the Gauss-Southwell approach, minimize
\begin{center}
$EY,\mathbf{X} \exp( - Y F(\mathbf{X}))$
\end{center}
The weak classifiers were the set of all trees with
a fixed number (large enough) of terminal nodes.

\vspace{\baselineskip}

Under some compactness and continuity conditions I proved that:
\[
F_T \to F \text{  in  } L_2(P)
\]
\[
P(Y \neq \sign{F(\mathbf{X}}) = \text{Bayes Risk}
\]

\vspace{3\baselineskip}

{\em But there was a fly in the ointment}

\vspace{\baselineskip}

\end{minipage}
}

\vspace{0.5\baselineskip}

\hfill \begin{minipage}{0.9\textwidth}
{\em Leo Breiman, Machine Learning. 2002 Wald Lecture. Slide 35}
\end{minipage}
\end{minipage}
}

\hspace{1in} 
\scalebox{0.82}{\begin{minipage}{0.8\textwidth}
\framebox[1.1\textwidth]{
\begin{minipage}{\textwidth}
\begin{center}
{\em THE FLY}
\end{center}
Recall the notation~\footnote{Please read footnote about notation differences in a previously quoted slide.}
\[
F_T(x) = \sum_{t=1}^T \alpha_t h_t(x)
\]
An essential part of the proof in the population case was showing that:
\[
\sum_{t=1}^{\infty} \alpha_t^2 < \infty
\]
But in the $m$-sample case, one can show that
\[
\alpha_t \geq 2/m
\]

\vspace{3\baselineskip}

{\em So there was an essential difference between the population case and the finite sample case no matter how large $m$}

\vspace{\baselineskip}

\end{minipage}
}

\vspace{0.5\baselineskip}

\hfill \begin{minipage}{0.9\textwidth}
{\em Leo Breiman, Machine Learning. 2002 Wald Lecture. Slide 36}
\end{minipage}
\end{minipage}
}

\vspace{\baselineskip}

The following quote is from~\citet{BreimanInfinite}, Section 9
(``Discussion''), Sub-Section 9.1 (``AdaBoost'').~\footnote{A
  substantial amount of the same text appears
  in~\citet{BreimanPopTheory}, which is the article version of his
  2002 Wald Memorial Lecture.  Yet, that specific section of the technical report does not appear in the article version.}
  In that technical report, mostly superseded
  by~\citet{BreimanPopTheory}, he proved convergence of the Optimal
  AdaBoost classifier itself, and convergence to the Bayes-error risk
  (i.e., that Optimal Adaboost is Bayes-consistent, to put it in
  statistical terms), \emph{in $L_2(P)$, the class of measurable
    functions in $L_2$ with respect to probability measure $P$.}  But he did so
  under the condition of \emph{infinite amount of data.}  Unfortunately, as Breiman himself stated in that manuscript (and in his 2002 Wald Lecture), there was a \emph{fundamental flaw}, or as he put it, a ``fly in the ointment,'' in trying to transfer the result to \emph{finite-size datasets.}~\footnote{Note that, to our best interpretation,
  Breiman's description is that of Optimal AdaBoost, as considered
  here; except that in our experiments, because we can consider arbitrary
  initial example weights, we initialize them by drawing
  uniformly at random from the probability $m$-simplex. Note also that we attempted to make his notation consistent with the more traditional notation in the ML literature, as we use here.}
\begin{quote}
{\em ``The theoretical results indicate that as the sample size goes
  to infinity, the generalization error of Adaboost will converge to
  the Bayes risk. But on most data sets I have run, Adaboost does not
  converge. Instead it's behavior resembles an ergodic dynamical
  system. The mechanism producing this behavior is not understood.''
}
\end{quote}
In contrast, all our convergence results are in the case of \emph{finite}-size datasets.

Breiman continued, now informally stating his ``AdaBoost is Always Ergodic'' Conjecture:
\begin{quote}
{\em ``I also conjecture that it is this equalization property that gives Adaboost its ergodicity.	Consider a finite number of classifiers $\{h_t\}$, each one having an associated misclassification set $Q_t$ . At each iteration the $Q_t$ having the lowest weight (using the current normed weights) has its weight increased to $1/2$ while instances in the complement have their weights decreased.	Thus, the $Q_t$ selected moves to the top of the weight heap while other $Q_t$ move down until they reach the bottom of the heap, when they are bounced to the top. It is this cycling among the $Q_t$ that produces the ergodic behavior.\\
  \mbox{ } \\
However, I do not understand the connection between the finite sample size equalization and what goes on in the infinity case. Why equalization combined with ergodicity produces low generalization error is a major unsolved problem in Machine Learning.''
}
\end{quote}
We did not really see any empirical evidence in favor of
Breiman's observation in our experiments, despite our theoretical
results strongly suggesting that AdaBoost \emph{is} an ergodic dynamical
system, consistent with Breiman's Conjecture. The
\emph{logarithmic} growth on the number of decision stumps we have
empirically observed,
as partially presented in the center column of Fig.~\ref{fig:all} on pg.~\pageref{fig:all}
in Appendix~\ref{app:pac_bnds},
suggests that Breiman's observation does not 
manifest itself in real-world datasets, at least within a reasonably large
number of rounds of AdaBoost.~\footnote{Granted, Breiman's observation
  may still be consistent with our experimental results if the
  steady-state of the dynamical system induces a distribution over the
  indiviual decision stumps that, when ordered, decays exponentially.} In this paper, we show that several
approximations of Optimal AdaBoost of arbitrary accuracy are ergodic
dynamical systems, and sufficient conditions under which the same
holds for the actual Optimal-AdaBoost update. 

We suspect that we could have obtained Breiman's result about the
convergence in $L_2$ by employing a different tool from ergodic
theory, the Mean Ergodic Theorem due to~\citet{vonNeumann1932}, and that we could
have done so in the case of finite $m$. We did not pursue a formal
proof of that result in our work.
The Mean Ergodic Theorem
establishes so called ``convergence in the mean,'' a weaker form of
convergence than the one we obtain here, which is ``pointwise
convergence.'' That is, in ``convergence in the mean'' convergence is
only guarantee for the ``average point,'' and may not occur for any
``specific point.''  Indeed, this is what distinguishes the Birkhoff Ergodic
 Theorem~\citep{Birkhoff1931} from von Neumann's Mean Ergodic
 Theorem.~\footnote{As a concluding aside on this topic,
these two theorems were published almost simultaneously and share a
fascinating history in their development before publication, as~\citet{Moore2015} nicely chronicles in a recent article.}

\subsubsection{Our own view on AdaBoost's behavior: stability seems
  to be a more consistent property than resistance to over-fitting}
\label{sec:ourview}

One of our original objectives in this work was to explain the well publicized AdaBoost resistance to
over-fitting the training data.  Instead, experimental evidence suggests that the
\emph{``stability''} of the test error, and multiple other quantities, is the most \emph{consistent} behavior we see in
practice.  Hence, stability or convergence seems to be a more \emph{universal} characteristic of AdaBoost than resistance to overfitting. 
That the convergence of the AdaBoost classifier itself 
turns out to be a rather universal
characteristic, was 
surprising to us.  We note that this
universal characteristic seems independent, or at least it is not
explicitly directly dependent, of the hypothesis class that the weak
learner uses, as long as it satisfies some minimal, basic properties.

Yes, AdaBoost exhibits a tendency to resist overfitting in many
datasets.  But we and others have found that it
does overfit in several others real-world datasets (see, e.g.,
\citeauthor{Grove:1998:BLM:295240.295766},\citeyear{Grove:1998:BLM:295240.295766}). The
test error of the
Heart-Disease dataset in
Fig.~\ref{fig:c4.5}(b) is an example.~\footnote{The train-test error plot for the Parkinson
dataset presented in the second row and first column in 
Fig.~\ref{fig:all},
on pg.~\pageref{fig:all} in Appendix~\ref{app:pac_bnds}, 
provides another
example of overfitting behavior in a real-world data set.} Regardless, we still see stability of the
test error.  While there is large empirical evidence of this ``stable''
behavior,
very little theoretical understanding of that stability existed.

We believe that ``stability'' may also help explain the cases where
AdaBoost does resist overfitting.  That is, AdaBoost will appear
resistant to overfitting if it converges to stable behavior in a
relatively small number of iterations.

\subsection{Our contributions}
\label{sec:contr}

In this paper, we follow a similar general approach to that pioneered
by~\citet{RudinDynamics} in the sense that we frame AdaBoost as a
dynamical system.  However, we apply \emph{different techniques} and
mathematical tools. Also, we address \emph{different
  problems}. \citet{RudinDynamics} is primarily concerned with the
convergence to \emph{maximum margins.} Here we are primarily concerned
with convergence in a more general sense, and not margin
maximization. Using those mathematical tools from real-analysis
and measure theory (e.g., \emph{Krylov-Bogolyubov Theorem}), we
establish sufficient conditions for an invariant measure on the
dynamical system when defined only on its set of attractors.  
We do not require this measure to be ergodic, which is weaker than Breiman's requirement.  Then, using tools from ergodic theory (i.e., \emph{Birkhoff Ergodic
  Theorem}), we show that such a measure implies the convergence of
the time/per-round average of \emph{any} measurable 
function, in $L_1$,
of the weights over the examples, but only when started on its set of attractors. We then extend this convergence result to hold starting from almost any initial weight but only for a class of continous-like functions. In doing so, we provide reasonably strong evidence in favor of the 
affirmative answer to the ``AdaBoost Always Cycle'' Conjecture and the
``AdaBoost is an Ergodic Dynamical System'' Conjecture, two open
conjectures in machine learning.  We provide constructive proofs of several arbitrarily-accurate
approximations of Optimal AdaBoost, prove that the resulting
approximations exhibit
certain cycling behavior in finite time and that the resulting dynamical system is ergodic, and establish sufficient conditions for the
same to hold for the actual Optimal-AdaBoost update. 
We also prove the (global) convergence of time averages of any ``locally
continuous'' function of the example weights for all the arbitrarily-accurate
approximations of Optimal AdaBoost, and for the actual
Optimal-AdaBoost update under the sufficient condition, starting from
any $w_1$. 
We use the last result 
to formally prove that, under the respective condition, (a) the margin for every example
converges, under the respective condition (note that the last statement is not about
maximizing margins); (b) the AdaBoost classifier itself is converging; and (c) 
the generalization error is asymptotically stable.

\subsection{Overview}
\label{sec:overview}

Section~\ref{S:BackNot} begins to introduce the mathematical
preliminaries needed to state and prove our convergence
results. Section~\ref{sec:ds} formulates AdaBoost as a dynamical
system. Section~\ref{S:classifier-convergence} formally provides what
in our view are reasonably
strong affirmative answers to the ``AdaBoost Always Cycle'' Conjecture and the
``AdaBoost is an Ergodic Dynamical System'' Conjecture,
and establishes the finite-time convergence of 
time/per-round averages of functions of the example weights that Optimal AdaBoost generates,
along with many other quantities such as the margins and the
generalization error, under mild conditions.
Section~\ref{S:Exp} presents our empirical results.
Section~\ref{sec:close}
provides closing remarks, including a summary discussion of the
results and open problems.

We have already briefly presented, differentiated, and discussed the
most important and closest work to ours~\citep{RudinDynamics}. Hence, we refer the reader
to Appendix~\ref{sec:rw} for additional discussion of other work. Without further delay,
we now move on to the presentation of the main technical components
of our work.

\section{Technical preliminaries, background, and notation}
\label{S:BackNot}

In this section we 
introduce the most basic ML concepts we use throughout the article. We
refer the reader to Appendix~\ref{app:math} for a brief overview of
other, more general mathematical concepts we use for our technical
results. To help the reader keep track of the notation used in the
remaining of the article, Table~\ref{tab:not} provides a summary
legend of the general notation, while Table~\ref{tab:notab} provides a summary
legend of the notation most closely related to AdaBoost.

Let $\X$ denote the \emph{feature space} (i.e., the set of all inputs) and 
$\{-1,+1\}$ be the set of \emph{(binary) output labels}. We make the
standard assumption that $\X$ is endowed with a \emph{$\sigma$-algebra $\Sigma_\X$ over
$\X$(i.e., the set of all measurable subsets of $\X$)}. For example,
$\Sigma_\X$ may be the $\sigma$-field generated by $\X$ . Let 
$(\X,\Sigma_\X)$ be the corresponding \emph{measurable space} for the
feature space. To simplify notation, let 
$\D \equiv \X \times \{-1,+1\}$ 
be the set of possible input-output pairs. We want to learn from a given, fixed dataset of $m$ training examples $D \equiv \{(x^{(1)},y^{(1)}),(x^{(2)},y^{(2)}),\ldots,(x^{(m)},y^{(m)})\}$, where each input-output pair $(x^{(i)},y^{(i)}) \in \D$, for all examples $i = 1,\ldots,m$.~\footnote{Note that \emph{all} datasets are technically \emph{multisets} by definition because members of a dataset may appear more than once.} We make the standard assumption that each example 
in $D$ comes from a \emph{probability space} $(\D,\Sigma,P)$, where
$\D$ is the \emph{outcome space}, $\Sigma \equiv \Sigma_\X \times 2^{\{-1,+1\}}$ is the \emph{($\sigma$-algebra) set of possible events} with respect to $\D$ (i.e., subsets of $\D$), and $P$ the \emph{probability measure} mapping $\Sigma \to \R$.
For convenience, we denote the \emph{dataset of input examples in the
  training dataset $D$} by $S \equiv
\{x^{(1)},x^{(2)},\ldots,x^{(m)}\}$. Also for convenience, we denote
the \emph{set of (unique) members of (multiset) $S$} by $U \equiv
\bigcup_{i=1}^m \{x^{(i)}\} \subset \X$.

\begin{table}
\begin{center}
\begin{tabular}{|l|l|}
\hline
\emph{Symbol} & \emph{Brief Description}\\
\hline
\hline 
$\indicator{.}$ & indicator function\\
\hline
$\sign{.}$ & sign function\\
\hline 
$\X$ & feature space\\
\hline
$x$ & an element in $\X$ (an input value)\\
\hline
$j$ & index the input dimension/feature \\
& (i.e., the components of $x$)\\
\hline
$x_j$ & the $j$th component of $x$\\
\hline
$y$ & an element in $\{-1,+1\}$ (an output value)\\
\hline
$\D$ & space of input-output pairs  $\X \times \{-1,+1\}$\\
\hline
$\Sigma$ & ($\sigma$-algebra) set of possible events with respect to
           $\D$\\
\hline
$\R$ & the set of real numbers\\
\hline
$P$ & probability measure/law for measure space $(\D,\Sigma)$\\
\hline
$(\D,\Sigma,P)$ & probability measure space over possible input-output pairs\\
\hline
$D$ & training dataset\\
\hline
$m$ & number of training examples in $D$\\
\hline
$i$ & index to training examples in $D$\\
\hline 
$x^{(i)} \in \X$ & input (feature values) of $i$th training example in $D$\\
\hline
$y^{(i)} \in \{-1,+1\}$ & output (label) of $i$th training example in $D$\\
\hline 
$(x^{(i)},y^{(i)}) \in \D$ & $i$th input-output-pair training example in $D$\\
\hline
$S$ & dataset of input (feature values) examples in $D$\\
\hline
$U \subset \X$ & set of unique members of $S$\\
\hline
$\Delta_m$ & probability $m$-simplex\\
\hline
$\Delta_m^{\circ}$ & interior of $\Delta_m$\\
\hline
$w$ & example weights (probability distribution) over the indexes \\
& of the training examples in $D$ (an element of $\Delta_m$ or any of
                                                                       its subsets)\\
\hline
$w(i)$ & $i$th component of $w$\\
\hline
$\Hypo$ & the weak learner's hypothesis class\\
\hline
$h$ & weak hypothesis in $\Hypo$, of type $\X \to \{-1,+1\}$\\
\hline
$-h$ & the negative of $h$ (i.e., $-h(x)$ for all $x \in \X$)\\
\hline
$\Dich(\Hypo,S)$ & finite set of label dichotomies induced by $\Hypo$
                   on $S$\\
\hline
$\Mcal$ & finite set of error or mistake dichotomies induced by $\Hypo$
                   on $D$\\
\hline
$n$ & number of elements in $\Dich(\Hypo,S)$,\\ 
 & (also equals the
      number of elements in $\Mcal$)\\
\hline
$l$ & index to an element of $\Dich(\Hypo,S)$\\
\hline
$o \in \{-1,+1\}^m$ & a label dichotomy\\
\hline
$o^{(l)}$ & the $l$th element of $\Dich(\Hypo,S)$\\
& (equals
            $(h(x^{(1)}),\ldots,h(x^{(m)})) \in \{-1,1\}^m$ for some
            $h \in \Hypo$)\\
\hline
$\eta,\eta'$ & mistake dichotomies\\
& (equals $\left(\indicator{y^{(1)} \neq
  o(1)},\ldots,\indicator{y^{(m)} \neq o(m)}\right) \in \{0,1\}^m$\\
& for
  some $o \in \Dich(\Hypo,S)$)\\
\hline
$h^0$ & representative hypothesis in $\Hypo$ for label dichotomy $o$\\ 
\hline
$h^\eta$ & representative hypothesis in $\Hypo$ for mistake dichotomy $\eta$\\ 
\hline
\end{tabular}
\end{center}
\caption{{\bf Notation Legend.} The table summarizes some of the
  general mathematical notation used in this paper.}
\label{tab:not}
\end{table}

\begin{table}
\begin{center}
\begin{tabular}{|l|l|}
\hline
\emph{Symbol} & \emph{Brief Description}\\
\hline
\hline 
$T$ & maximum number of rounds of AdaBoost\\
\hline
$w_t$ & example weights for AdaBoost's round $t$\\
\hline 
$h_t$ & weak hypothesis that AdaBoost selects at round $t$\\
& (i.e., with respect to $w_t$ and $D$)\\
\hline
$\eta_t$ & mistake dichotomy that $h_t$ produces for $D$\\
\hline
$\err(h; D,w)$ & weighted error of weak hypothesis $h$\\
& with respect to $D$ and $w_t$\\
\hline
$\epsilon_t$ & (weighted) error of $h_t$ with respect to $w_t$\\
\hline
$\alpha_t$ & weight assigned to $h_t$ in the AdaBoost classifier\\
\hline
$F_T$ & AdaBoost classifier (proxy) function\\
\hline
$H_T$ & AdaBoost classifier\\
\hline
$\T$ & dynamical-system version of the ``hypothetical'' AdaBoost update\\
\hline
$\A$ & dynamical-system version of the ``actual'' AdaBoost update\\
\hline
\end{tabular}
\end{center}
\caption{{\bf AdaBoost Notation Legend.} The table summarizes the notation used
  in this paper that is most closely related to AdaBoost.}
\label{tab:notab}
\end{table}

Let us use an instance of a typical ``classroom example'' presented in
Fig.~\ref{fig:clex} to instantiate and help with the notation. For
that example, we have the following: $\X = \R \times \R = \R^2$, $\D = \R^2
\times \{-1,1\}$,
$\Sigma$ is the $\sigma$-algebra over the joint-space $\D$ (e.g.,
standard Borel $\sigma$-algebras over $\R^2$ for each output value in $\{-1,1\}$), $P$
is some probability measure over the measurable space $(\D,\Sigma)$
defining the distribution over input-output pairs, which in turn defines the
probability measure space $(\D,\Sigma,P)$; $m = 6$,
\[
\begin{array}{llllll}
x^{(1)} = (1,5), & x^{(2)} = (3,11), & x^{(3)} = (5,1), & x^{(4)} = (7,3),
  & x^{(5)} = (9,7), & x^{(6)} = (11,9),\\
y^{(1)} = 1, & y^{(2)} = -1, & y^{(3)} = 1, & y^{(4)} = -1, & y^{(5)} = 1, & y^{(6)} = -1,
\end{array}
\]
and
\begin{align*}
D &=
    \{((1,5),1),((3,11),-1),((5,1),1),((7,3),-1),((9,7),1),((11,9),-1)\},\\
S &= \{(1,5),(3,11),(5,1),(7,3),(9,7),(11,9)\},\\
U &= S \; .
\end{align*}

\begin{figure}
\begin{center}
\includegraphics[width=\textwidth]{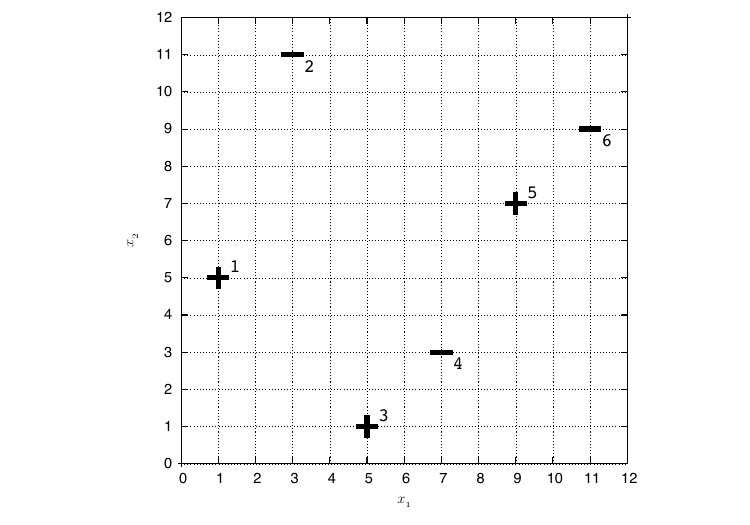}
\end{center}
\caption{{\bf A Classroom Example.} This figure shows a simple
  binary-classification classroom-like example to illustrate the
  notation and some basic concepts.}
\label{fig:clex}
\end{figure}

Denote by
\[
\Delta_m \equiv \left\{w\in\mathbb{R}^m \left| \sum_{i=1}^m w(i) = 1 \text{ and for all } i \text{, } w(i)\geq 0 \right. \right\}
\]
the \emph{standard  $m$-simplex.} Recall that $\Delta_m$ is a compact set. Denote by 
\[
\Delta_m^{\circ} \equiv \Int{\Delta_m} = \left\{w\in\mathbb{R}^m \left| \sum_{i=1}^m w(i) = 1 \text{ and for all } i \text{, } w(i) >  0 \right. \right\} \subset \Delta_m
\]
its \emph{interior} set (i.e., all positive probabilities). We will often denote elements of $\Delta_m$ or any of its subsets as $w$.  

We denote the set of \emph{hypotheses} that the \emph{weak learner} in
AdaBoost uses by $\Hypo$, often referred to within the boosting
context as the \emph{weak-hypothesis class}, and its elements as
\emph{weak hypotheses}, where each such \emph{weak hypothsesis} may or may not be selected during the
execution of the AdaBoost algorithm.~\footnote{In what follows
  we will use the terms ``weak hypothesis,'' ``weak classifier,'' and
  ``base classifier'' interchangeably. Similarly, we will use the
  terms ``weak learner'' and ``base learner'' interchangeably.}  
For instance, within the context of the example in
Fig.~\ref{fig:clex}, a simple, natural, and often-used choice for
$\Hypo$ is the set of all axis-parallel decision stumps conditioned on
every value of each of the two dimensions in Euclidean space: letting
\[
\Hypo_{\mathrm{basic}} \equiv \bigcup_{j=1}^2 \{ h : \X \to \{-1,1\} \, \mid \,  h(x_1,x_2) =
\sign{x_j - v}, \text{ for all } (x_1,x_2) \in \R^2, v \in \R \} \; ,
\]
set
\begin{align}
\nonumber
\Hypo = \bigcup_{y \in \{+1,-1\}} & \left\{ h : \X \to \{-1,1\} \, \mid
  \, h(x_1,x_2) = y, \text{ for all } (x_1,x_2) \in \R^2 \right\} \cup \\
\label{eqn:exhypo}
& \left\{ h : \X \to \{-1,1\} \, \mid \,
  h = y h_{\mathrm{basic}} \text{ for some }
  h_{\mathrm{basic}} \in \Hypo_{\mathrm{basic}} \right\} \; .
\end{align}
Finally, denote the standard \emph{indicator} by $\indicator{.}$: i.e.,
$\indicator{\text{c}} = 1$ if $c=\text{true}$; and 
$= 0$ if $c=\text{false}$. Another way of saying that $h$ makes a mistake on
example $(x,y)$ is to write $\indicator{h(x) \neq y} = 1$.

We will impose the following natural conditions on $\Hypo$. 
These conditions will prove useful in very specific parts of the analysis regarding the continuity of certain functions related to Optimal AdaBoost, most important of which is the example-weight update performed at each round of the algorithm.
The condition's main role is to avoid dealing with discontinuities at
probability distributions $w$ on the $m$-simplex for which the
weighted error of some hypothesis in $\Hypo$ is
zero. Theorem~\ref{T:lower-bound+}, stated later (in Section~\ref{sec:invmeas}),
shows that in the implementation of Optimal AdaBoost we use, which is
consistent with standard implementations, the example-weights update
stays away from such discontinuities. That theorem establishes a lower
bound on the weighted error $\epsilon_t$ generated by the algorithm
under the additional condition that the weak-learner \emph{always}
does better than random guessing, a natural condition in this context. 
\begin{condition}
({\bf Natural Weak-Hypothesis Class})
\label{cond:nathypo}
\begin{enumerate}
\item {\bf $\Hypo$ contains the constant, all-positive hypothesis}: the hypothesis $h$, such that, for all $x \in \X$, $h(x)=1$, is in $\Hypo$.
\item {\bf $\Hypo$ is closed under negation:} if $h \in \Hypo$, then $-h \in \Hypo$. (By $-h$ we mean the function $h'(x) \equiv - h(x)$).
\item {\bf No $h$ in $\Hypo$ is perfect on the training dataset $D$:}
  for all $h \in \Hypo$, there exists an $(x,y) \in D$ such that $h(x)
  \neq y$ (i.e., $h$ makes a mistake on $x$).
  \item {\bf Every $h$ in $\Hypo$ is well-behaved:} every $h \in
    \Hypo$ is $\Sigma_\X$-measurable.
\end{enumerate}
\end{condition}
The first part of the condition is easy to satisfy: just add such a hypothesis $h$ to $\Hypo$ if it is not already there. The second part of the condition is similarly easy to satisfy. Note that the first and second parts of the condition imply that (1) the constant, all-negative hypothesis is also in $\Hypo$ (i.e., the hypothesis $h$ defined as $h(x) = -1$ for all $x \in \X$ is in $\Hypo$); and (2) every example in the training dataset is incorrectly classified by some weak hypothesis in $\Hypo$ (i.e., for every $(x,y) \in D$, there exists some $h \in \Hypo$ such that $h(x) \neq y$; which we can think of as the ``converse'' of the third part of the condition).
The third part of the condition is natural because, should there be a
perfect $h \in \Hypo$, Optimal AdaBoost would stop immediately after
the first iteration given that in that case the weighted error of $h$ with respect to \emph{any}
initialization of $w_1 \in \Delta_m$ would be zero: i.e., $\sum_{i =1, h(x^{(i)}) \neq y^{(i)}}^{m} w_1(i) = \sum_{i
  =1}^{m} \indicator{h(x^{(i)}) \neq y^{(i)}} w_1(i)= 0$. 
In the context of our running example given in Fig.~\ref{fig:clex},
the hypothesis class $\Hypo$ of axis-parallel decision-stumps given in
Equation~\ref{eqn:exhypo} satisfies Condition~\ref{cond:nathypo}.
Informally speaking, the
fourth part of the condition assures that every $h \in \Hypo$ has ``well-behaved''
    decision regions~\citep[][Chapter 2, pp. 6]{BartleMT}, so that
    $h$ is a random variable
    with respect to the probability space $(\D,\Sigma,P)$ over the 
    examples~\citep[][Section 1.1, pp. 3]{DurrettProb}. It is reasonable because it holds for the
typical feature spaces considered in the machine-learning literature
and found in practice.~\footnote{For instance, the condition holds if each atribute
  inducing the feature space is either (a) \emph{discrete} (i.e., finite or
  countably infinite), and its $\sigma$-algebra is the set of all
  subsets of its domain~\citep[][Example 2.2(a)]{BartleMT}; or (b) continuous (i.e., real-valued), and the
  $\sigma$-algebra is the Borel $\sigma$-algebra~\citep[][Example
  2.2(g)]{BartleMT}. It also holds if the $\sigma$-algebra generated
  by $\Ccal \equiv \{ \{ x \in \X | h(x) = +1 \} | h \in \Hypo \}$ is a subset of
  $\Sigma_\X$, which we can guarantee by requiring that $\Sigma_\X =
  \Ccal$, of course.}  For instance, within the context of the example in
Fig.~\ref{fig:clex}, each $h \in \Hypo$ is measurable with respect to
the Borel $\sigma$-algebra on $\R^2$. In addition, if an $h \in \Hypo$ is not $\Sigma_{\X}$-measurable, then we cannot
talk about basic quantities such as the expected output or the
generalization error of $h$, with respect to $(\D,\Sigma,P)$ because
they do not exist;~\footnote{If $h \in \Hypo$ is not
  $\Sigma_\X$-measurable, then, by definition, the decision regions of
  $h$ given by the sets $\{ x \in \X |
h(x) = +1 \}$ and $\{ x \in \X |
h(x) = -1 \}$ are not $\Sigma_\X$-measurable. Hence, there is no such
thing as the expected output value $\E{h(X)}$, the generalization
error $\E{\indicator{h(X) \neq Y}} = P(h(X) \neq Y)$, or the like, with respect
to $(\D,\Sigma,P)$, as they do
not exist. So, for example, limits of empirical averages of functions of $h$ over the dataset $D$, such as
the average output, $\lim_{m \to \infty} \frac{1}{m}
\sum_{i=1}^m h(x^{(i)})$, the misclassification-error rate, $\lim_{m \to \infty} \frac{1}{m}
\sum_{i=1}^m \indicator{h(x^{(i)}) \neq y^{(i)}}$, or the like, may or may not exist, but the
typical Laws of Large Numbers do not apply, because they cannot converge to
something that does not exist.}
nor would those quantities exist for Optimal
AdaBoost if it selects $h$ at any point during its
execution.~\footnote{We do not need the condition for our results to hold if we could guarantee that Optimal AdaBoost
  never selects a non-$\Sigma_\X$-measurable $h \in \Hypo$.} \emph{This
condition only affects our results on the convergence of the Optimal AdaBoost
classifier and its generalization error.}

This set of hypotheses $\Hypo$, which may be finite or infinite,
induces a \emph{finite} set of \emph{label dichotomies} on the training dataset of input examples $S$, where each dichotomy is defined as an $m$-dimensional vector of output labels to the training examples: formally, we denote this \emph{finite set of label dichotomies} as~\footnote{Note that we do not explicitly compute such sets in practice; but they are very convenient as the only mathematical abstraction needed to characterize the actual, full behavior of Optimal AdaBoost. Said differently, the final classifier output by Optimal AdaBoost when using the mathematical abstraction implicitly provided by the finite set of label dichotomies is exactly the same as that produced by the learning algorithm when run in practice.}
\begin{align*}
\Dich(\Hypo,S) \equiv \{o^{(1)},\ldots,o^{(n)}\} = \bigcup_{h \in \Hypo} \{(h(x^{(1)}),\ldots,h(x^{(m)}))\}  \subset 
\{-1,+1\}^m \; . 
\end{align*}
Parts 1 and 2 of
Condition~\ref{cond:nathypo} imply that the vector of all +1's and the
vector of all -1's are both in $Dich(\Hypo,S)$. Hence, we have $2 \leq
n \leq 2^m$.~\footnote{Actually, the upper bound on the size is
  $2^{\min(m^*,m)}$, where $m^*$ is the \emph{VC-dimension} of
  $\Hypo$~\citep{Kearns:1994:ICL:200548}. We also refer the reader
  to
 Definition~\ref{def:vcdim} in 
Appendix~\ref{app:pac_bnds}
for a definition of the VC-dimension framed within the context of the current manuscript.} For instance, in the context of the classroom example,
for the weak-learner hypothesis class $\Hypo$ defined in
Equation~\ref{eqn:exhypo}, the set $\Dich(\Hypo,S)$ looks as follows.
\begin{align*}
\Dich(\Hypo,S) = \{ 
& (+1,+1,+1,+1,+1,+1),\\
& (-1,-1,-1,-1,-1,-1),\\
& (-1,+1,+1,+1,+1,+1),\\
& (+1,-1,-1,-1,-1,-1),\\
& (-1,-1,+1,+1,+1,+1),\\
& (+1,+1,-1,-1,-1,-1),\\
& \ldots,\\
& (+1,+1,-1,+1,+1,+1),\\
& (-1,-1,+1,-1,-1,-1),\\
& (+1,+1,-1,-1,+1,+1),\\
& (-1,-1,+1,+1,-1,-1),\\
& \ldots,\\
& (-1,+1,-1,-1,-1,-1),\\
& (+1,-1,+1,+1,+1,+1) \}
\end{align*}
For this example, we have $n=22$, but here we are only showing part of the set. We
refer the reader to Appendix~\ref{app:clex} for the full set.

For each dichotomy $o \in \Dich(\Hypo,S)$, it is convenient to
associate a (unique) \emph{representative hypothesis $h^o \in \Hypo$
  for $o$}, among any other hypothesis $h \in \Hypo$ that produces the
same dichotomy $o$. For instance, in the context of the classroom
example, given that $\Hypo$ is composed of axis-parallel decision
stumps, consider two consecutive examples in the projection along one
of the dimensions; say, for instance, the training examples indexed by
$1$ and $2$. Any $h \in \Hypo$ of the form $\sign{x_1 -v}$ with $v$ in
the open interval $(1,2)$ on the real line will produce the same label
dichotomy $o = (-1,+1,+1,+1,+1,+1)$. Hence, it is common practice to
introduce a learning bias by considering only the ``midpoint'' decision stump that results from setting
$v=1.5$, and letting that be the representative hypothesis for label dichotomy $o$. We return to this concept of a ``representative
hypothesis'' later in Section~\ref{S:classifier-convergence} when we extend our convergence results of various
functions
from the set of (unique) training examples $U$ to the whole feature
space $\X$ (see
Theorems~\ref{thm:classifier_fcn_convergence},~\ref{thm:classifier_convergence},
and~\ref{thm:gen_error_convergence},
Corollary~\ref{cor:margin_convergence}, and the discussion around them).
We call \[
\Mcal \equiv \Mcal(\Hypo,D) \equiv \bigcup_{o \in \Dich(\Hypo,S)} \left\{ \left(\indicator{y^{(1)} \neq o(1)},\ldots,\indicator{y^{(m)} \neq o(m)}\right) \right\} \subset \{0,1\}^m \; 
\]
the \emph{set of error or mistake dichotomies}
(Note that $|\Mcal| = n$).
For instance, in the context of the classroom example, the set $\Mcal$ looks as follows.
\begin{align*}
\Mcal = \{ 
& (0,1,0,1,0,1),\\
& (1,0,1,0,1,0),\\
& (1,1,0,1,0,1),\\
& (0,0,1,0,1,0),\\
& (1,0,0,1,0,1),\\
& (0,1,1,0,1,0),\\
& \ldots,\\
& (0,1,1,1,0,1),\\
& (1,0,0,0,1,0),\\
& (0,1,1,0,0,1),\\
& (1,0,0,1,1,0),\\
& \ldots,\\
& (1,1,1,0,1,0),\\
& (0,0,0,1,0,1) \}
\end{align*}
Once again, for this example, we have $n=22$, but we are only showing part of the set. We
refer the reader to Appendix~\ref{app:clex} for the full set.

AdaBoost extensively uses the weighted error of a hypothesis in its example-weight update.  The typical expression for the weighted error of any hypothesis $h$ with respect to a distribution $w$ over the examples is 
\(
\sum_{i = 1}^{m} w(i)\indicator{h(x^{(i)}) \neq y^{(i)}}.
\)
Let $h \in \Hypo$ and $\eta \in \Mcal$ be its corresponding mistake
dichotomy (i.e., for all $i$, $\eta(i) = \indicator{h(x^{(i)}) \neq
  y^{(i)}}$). We can equivalently compute the \emph{weighted error} of $h$
with respect to $w$ on $D$ as $\eta \cdot w \equiv \sum_{i=1}^m
\eta(i) w(i)$, the \emph{dot-product} of $\eta$ and $w$. Part 3 of Condition~\ref{cond:nathypo} implies $\eta \cdot w > 0$ for all $\eta \in \Mcal$ and $w \in \Delta_m^{\circ}$.

Because each \emph{mistake} dichotomy $\eta \in \Mcal$ has a
corresponding \emph{label} dichotomy $o$, which in turns has a
\emph{representative hypothesis} $h^o \in \Hypo$, it will become
convenient to denote $h^\eta \equiv h^o$ as the \emph{representative
  hypothesis for $\eta$} for the final classifier output by Optimal
AdaBoost. For instance, in the context of the classroom
example, for label
dichotomy $o = (-1,-1,-1,-1,+1,+1)$, employing the common biasing
practices for decision stumps that uses the ``midpoint'' rule, we have $h^o(x_1,x_2) = \sign{x_1 -
  8}$ as the representative hypothesis. The corresponding mistake
  dichotomy for $o$ is $\eta = (1,0,1,0,0,1)$, so that
  $h^\eta(x_1,x_2) = h^o(x_1,x_2) = \sign{x_1 -
  8}$ is the representative hypothesis for $\eta$.

Note that we are essentially producing a finite number of hypothesis selection
candidates through the process described: we are effectively reducing
the hypothesis space from $\Hypo$ to the \emph{finite set of
  representative hypotheses} $\Hrep \equiv \left\{ h^{\eta} \in \Hypo \, \mid \,
    \eta \in \Mcal \right\}$.

\section{Optimal AdaBoost as a dynamical system}
\label{sec:ds}

This paper studies Optimal AdaBoost as a \emph{dynamical system of the
  weights over the examples}, which we also refer to as the
\emph{example or sample weights}, in a way similar to previous work~\citep{RudinDynamics}.  
In this section, we show how to frame Optimal AdaBoost as such a
dynamical system. We will fix $\Hypo$ and $D$, therefore fixing $S$,
$\Dich(\Hypo,S)$,
$\Mcal$, and
$\Hrep$.

For much of our analysis we will reduce AdaBoost to only using the 
mistake di\-chot\-o\-mies in $\Mcal$, or equivalently, the
elements of 
$\Hrep$, as a proxy for the representative hypotheses in its weight
update.  Doing so is sound because of the one-to-one relationship
discussed at the end of the previous section (Section~\ref{S:BackNot}).

The following is the common condition typically assumed in the analysis of boosting algorithms, but stated in the context of our paper.
\begin{condition}{{\bf (Weak-Learning Assumption)}}\label{assume:WeakLearn}
There exists a real-value $\gamma \in (0,1/2)$ such that for all $w \in \Delta_m$, there exists an $\eta \in \Mcal$ 
that achieves weighted error $\eta \cdot w \leq \frac 12 - \gamma < \frac 12$.
\end{condition}
Said differently, 
the weak learner is guaranteed to output hypotheses whose weighted binary-classification error is strictly better than random guessing, \emph{regardless} of the dataset of examples or the weight distribution over the examples.
 The value $\gamma$ is often called the \emph{edge} of the weak
 learner. In its most general form, the assumption is sometimes
 referred to as the ``Weak-Learning Hypothesis.''~\footnote{We want to emphasize
 that, while some have attempted to further weaken or simply remove
 the Weak-Learning Assumption, this has been in the context of the
 study of other forms of convergence of AdaBoost (see Appendix~\ref{sec:mlcontext}). Recall that the main focus of this paper is the convergence, with respect to the number of rounds $T$ for a fixed, but arbitrary, datatset $D$ drawn from the probability space $(\D,\Sigma,P)$, of the generalization error of the Optimal-AdaBoost classifier, the Optimal-AdaBoost classifier itself, and other related characteristic quantities of general interest such as the margins. 
The assumption remains standard for the study of the type of convergence of Optimal AdaBoost considered in 
this paper~\citep{RudinDynamics}.
}

\subsection{Implementation details of Optimal AdaBoost}
\label{sec:details}

Before we introduce the dynamical-system view of Optimal AdaBoost, we
make a slight
generalization in the traditional initialization
of the weights over the training examples. The traditional
initialization is the uniform
distribution over the set of training examples, as presented in
Fig.~\ref{fig:adabo}.  To emphasize that almost all of our
results hold for almost every $w_1 \in \Delta_m^\circ$, which includes the
uniform distribution that is traditionally used, of course.
we replace the initialization
presented in that figure, by ``{\bf Initialize} Pick any $w_1 \in
\Delta_m^\circ$.''~\footnote{Note that picking $w_1$ in the
  \emph{boundary} of $\Delta_m$ (i.e., $\Delta_m - \Delta_m^\circ$) is
  not sensible, unless we want to effectively
  reduce the size of the data set used by AdaBoost for
  training. This is 
  because any such initial example-weights $w_1$ would have at least one
  component $i$ such that $w_1(i) = 0$, which, by the AdaBoost weight
  update being component-wise proportional to the previous weight value, implies $w_t(i) = 0$ for all rounds $t$. Hence such example with
  index $i$
  would always have zero weight throughout the execution of AdaBoost;
  said differently, essentially the learning process would not
  consider that data sample. Note also that if $w_1$ is chosen
  uniformly at random from $\Delta_m$, then $w_1 \in \Delta_m^\circ$
  with probability one.}
We note that, for
every initial
$w_1 \in \Delta_m^\circ$,  
the AdaBoost property of
driving the training error to zero holds.~\footnote{Let $w_1 \in
  \Delta_m^\circ$ be the initial example weight. 
  A minor modification of the standard derivation of Optimal AdaBoost of the upper bound on the classifier's misclassification error yields
\(
\frac{1}{m} \sum_{l=1}^m \indicator{H_T(x^{(l)}) \neq y^{(l)}} \leq \frac{\exp(-2 \gamma^2 T)}{m \times \min_{l=1,\ldots,m} w_1(l)} \; .
\)
Once we set $w_1$, 
the denominator of the upper
bound remains constant throughout the AdaBoost process (i.e., does not
depend on the number of rounds $T$); while the numerator, which
results from Condition~\ref{assume:WeakLearn} (Weak Learning), 
goes to zero exponentially fast with $T$.}  \citet{RudinDynamics} have
observed that the training behavior of AdaBoost seems sensitive to
arbitrary, but fixed initial
conditions in synthetic experiments on randomly
generated \emph{mistake matrices} corresponding to $12$ training
examples and the equivalent of $25$ mistake dichotomies. (Mistake matrices are
``isomorphic'' to the set of mistake dichotomies $\Mcal$: the mistake matrix
would have $m$ columns and each mistake dichotomy $\eta
\in \Mcal$ would form a row.) Our results formally establish that convergence
occurs almost always regardless of the initial $w_1$, at least in
theory. 
Indeed,
setting $w_1$ by drawing uniformly at random from $\Delta_m$ does not appear to have any
effect on the convergence properties of the \emph{training} process in
practice, based on our
implementation of Optimal AdaBoost in our experiments.

Any implementation of the procedure $\mathbf{WeakLearn}$ used in
Optimal AdaBoost must decide how to select and output a hypothesis
whenever there is more than one hypothesis that achieves the minimum
error on the training dataset $D$ with respect to the example weights
$w$.  We consider a typical \emph{function} implementation of
$\mathbf{WeakLearn}$, mapping elements of $(\X \times \{-1,+1\})^m
\times \Delta_m \to \Hrep$, by which we mean that $\mathbf{WeakLearn}$
has a \emph{deterministic selection scheme} to output hypotheses; said
differently, given the \emph{same} training dataset $D$ and example
weights $w$ as input, $\mathbf{WeakLearn}$ will \emph{always} map to,
or output, the \emph{same} hypothesis 
in $\Hrep$ based on
whatever selection scheme the function implementation uses.  One can
view such a deterministic selection scheme as a way to introduce
\emph{bias} into the hypothesis class $\Hypo$.~\footnote{If the bias ``matches'' the underlying process
  generating the data, then one would expect classifiers with good
  generalization error; otherwise, the quality of the classifier may
  suffer.}

Said differently, any implementation of $\mathbf{WeakLearn}$ as a function, described above, 
leads to an implementation of AdaBoost which implicitly defines a
notion of ``best'' representative weak hypothesis in $\Hrep$, or
equivalently, ``best'' mistake dichotomy in $\Mcal$, for any example weights
$w\in \Delta_m$. We use a (deterministic) function implementation of
$\mathbf{WeakLearn}$ because in general the standard notion for best
representative weak hypothesis follows from any mistake dichotomy in
$\argmin_{\eta\in \Mcal} \eta \cdot w$, which is a set, but not
necessarily a singleton: 
multiple mistake dichotomies may be in that set.  Thus, the
implementation of $\mathbf{WeakLearn}$ implicitly imposes a policy for
how to \emph{break ties} between 
mistake dichotomies with the lowest error, which is equivalent to
breaking ties between the corresponding representative weak
hypothesis. Hence, we assume that we are given a \emph{tie-breaking function} $\AdaSel : 
2^\Mcal 
\to \Mcal$, where 
$2^\Mcal \equiv \{ Z \mid Z \subset \Mcal \}$
denotes the \emph{power set} of $\Mcal$ (i.e., the set of \emph{all
  subsets} of $\Mcal$).  The tie-braking function $\AdaSel$ serves as a mathematical function proxy for the implementation of $\mathbf{WeakLearn}$.
\begin{definition}\label{D:2}
Given example weights $w \in \Delta_m$, we define our notion of \emph{the
  best representative weak hypothesis in $\Hrep$, or equivalently, the
  best mistake dichotomy in $\Mcal$, with respect to $w$} as
$h^{\eta^w} \in \Hrep$ where
\(
\eta^w \equiv \AdaSel\left(\argmin_{\eta \in \Mcal}\eta \cdot w\right)
.
\)

It is convenient to assume that $\AdaSel$ employs a strict preference relation
over the elements of $\Mcal = \{ \eta^{(1)}, \eta^{(2)},\ldots,\eta^{(n)} \}$ such
that $\eta^{(1)} \succ \eta^{(2)} \succ \cdots \succ
                                               \eta^{(n)}$.
\end{definition}
The pseudocode for the implementation of the function
$\mathbf{WeakLearn}$ used in Optimal AdaBoost is 
``\(
\eta^w \leftarrow \AdaSel\left( \argmin_{\eta \in \Mcal} \eta \cdot w \right); 
 \; h^{\eta^w} \leftarrow \mathbf{WeakLearn}(w,D) \; .
\)''

From now on, we will assume the implementation of Optimal AdaBoost just described.

Before continuing, we note that we introduce concepts and notation in
the remaining of this section that may be unfamiliar to some
readers. We refer such readers to Appendix~\ref{app:ident} for an
illustration 
within the context
of a simple set of mistake dichotomies, equivalent to the $(3
\times 3)$ identity matrix: i.e.,
$\Mcal = \{(1,0,0),(0,1,0),(0,0,1)\} =
\{\eta^{(1)},\eta^{(2)},\eta^{(3)}\}$. That section of the appendix
also includes an alternative derivation of previous results by~\citet{RudinDynamics}, but
within the context of this article.

\subsection{Preliminaries to the formal definition of the dynamical system}
\label{sec:prelimds}

The selection procedure just described naturally partitions $\Delta_m$
into regions where different 
mistake dichotomies are best, in the sense that they would be selected by $\AdaSel$.
\begin{definition}
\label{def:pistar}
For all $\eta \in \Mcal$, we define
\(
\pi^*(\eta) \equiv \left\{w \in \Delta_m |\, \eta = \eta^w \right\}.
\)
\end{definition}
Note that $\pi^*(\eta)$ may be open or closed
for different $\eta$'s, depending on how $\AdaSel$ breaks ties.  The \emph{closure} of this set, which we now formally define, will also play an important role.
\begin{definition}
\label{def:pi}
For all $\eta\in \Mcal$, we define
\(
\pi(\eta)\equiv \left\{w \in \Delta_m \left| \, \eta \in \argmin_{\eta' \in \Mcal}\eta' \cdot w \right. \right\} \; .
\)
\end{definition}
The set $\pi(\eta)$, being the closure of $\pi^*(\eta)$, is naturally
closed.  However, these sets no longer form a partition on $\Delta_m$.
Given two distinct mistake dichotomies $\eta,\eta' \in \Mcal$, it is
possible that $\pi(\eta)\cap \pi(\eta') \neq \emptyset$. We denote by
$\pi^\circ(\eta) \equiv \Int{\pi(\eta)} \subset \Delta_m^\circ$ the \emph{interior} of
$\pi(\eta)$; note that $\pi^\circ(\eta) = \Int{\pi^*(\eta)}$.

It is often convenient to consider only the subset of $\Delta_m$ where
every mistake dichotomy in $\Mcal$ has non-zero error. (If there were
a mistake dichotomy in $\Mcal$ that achieves zero error with respect
to some $w_t \in \Delta_m^\circ$ generated by the algorithm at round
$t$, Optimal AdaBoost would essentially stop at round $t$. We refer the reader to the previously-presented discussion of Condition~\ref{cond:nathypo} for more information.)
\begin{definition}\label{D:Delta_m^+}
The \emph{set of all weights in $\Delta_m$ with non-zero error on
  mistake dichotomy $\eta \in \Mcal$} is 
\(
\pi^+(\eta) \equiv \left\{w \in \pi^*(\eta) \left| \, \eta \cdot w > 0
  \right. \right\} \; ,
\)
so that \emph{set of all weights in $\Delta_m$ with non-zero error on all mistake dichotomies in $\Mcal$} is 
\(
\Delta_m^+ \equiv \bigcup_{\eta \in \Mcal} \pi^+(\eta) .
\)
\end{definition}
It is important to remind
the reader that the complement of $\Delta_m^{\circ}$ with respect to
$\Delta_m$, given by $\Delta_m - \Delta_m^{\circ}$, is the
\emph{boundary} of $\Delta_m$ and has measure zero.~\footnote{This statement is with respect to the standard definition of the Borel measure over the Borel $\sigma$-algebra
generated from all the open subsets of $\Delta_m$. The
definition of open sets depends on the standard metrizable topological
space over $\Delta_m$, which is typically defined in terms of
Euclidean distance and the usual neighborhood topology it induces.  Recall that
$\Delta_m \subset \R^m$ is really an $(m-1)$-dimensional manifold
(i.e., a topological space that locally resembles $(m-1)$-dimensional
Euclidean space near each point).} Note also that $\pi^\circ(\eta)
\subset \Delta_m^+$ for all $\eta \in \Mcal$.
\begin{proposition}\label{P:Delta_m^+}
Under Condition~\ref{cond:nathypo}
(Natural Weak-Hypothesis Class), we have that $\Delta_m^{\circ} \subset
\Delta_m^+$.  Thus, the set $\Delta_m^+$ is
not a set of measure zero; while its complement with respect to
$\Delta_m$ does have
measure zero.  In addition, the set $\Delta_m^+ - \Delta_m^{\circ}$ is a
(potentially empty) subset of the boundary of $\Delta_m$ and has
measure zero.
\end{proposition}
Because, \emph{under Condition~\ref{cond:nathypo}}, we have that
$\Delta_m^+ - \Delta_m^{\circ}$ has measure zero and the statements in
almost all of our technical results 
hold for almost every $w_1 \in \Delta_m$, it is often
safe to only consider $w_1 \in \Delta_m^\circ$ in
our analysis. However, for the sake of simplicity and convenience, we
still define the AdaBoost weight update in terms of
$\Delta_m$. A better idea is to define the update in terms of
$\Delta_m^+$ and leave any weights outside that set undefined, so that
the update is a mapping of type $\Delta_m^+ \to \Delta_m^+$. (This
relates to a simple remark that we make later in the text about dealing with non-support
vectors earlier: for the purpose of the analysis, one can assume that
every example is a support vector because, if some of the example weights
that Optimal AdaBoost generates converge to $0$, we can simply restart the algorithm, or just consider its
execution to begin right after that happens. Hence, the analysis in
terms of $\Delta_m^\circ$ would go through because we are assuming
that all examples are support vectors. Hence, the set
$\Delta_{m}^\circ(I)$, the interior of the simplex where the
probabilities of each example $i$ with weight $w(i)$ is positive if $i
\in I \subset \{1,2,\ldots,m \}$ and zero otherwise. That is the essence of our
upcoming remark.)

\subsection{Formal definition of the Optimal-AdaBoost update as a
  dynamical system}
\label{sec:fdabds}

We will depart from standard notation for the AdaBoost weight update.  The notation we use will be more convenient for the main proofs in this paper.  
First, we have a notion of a \emph{hypothetical} weight update.  That is, given $w\in \Delta_m$, if we \emph{assume} that $\eta=\eta^w$, where would the AdaBoost weight update take $w$?
\begin{definition}\label{D:1}
Given an arbitrary 
mistake dichotomy $\eta \in \Mcal$, we define $\T_\eta : \Delta_m \to \Delta_m$ component-wise as, for each component $i=1,\ldots,m$,
\[
  [\T_\eta(w)](i) \equiv \frac 12 w(i) \times \left(\frac{1}{\eta\cdot w}\right)^{\eta(i)}\left(\frac{1}{1-\eta\cdot w}\right)^{1-\eta(i)}.
\]
Implicit in this definition is that (1) for all $w \not\in \Delta_m^+$, if 
$\eta \cdot w = 0$, then
$\T_\eta(w) = w$ (i.e., the update associated with any given $\eta$ should not change $w$
if $w$ already achieves zero error with respecto to $\eta$); and that (2) for all $w \in \Delta_m$, and for all $i$,
$[\T_\eta(w)](i) = 0$ if and only if $w(i) = 0$. Also, for any set $W \subset \Delta_m$, we employ the standard abuse of notation and define $\T_\eta(W) \equiv \{ \T_\eta(w) \mid w \in W \}$.
\end{definition}
The update $\T_{\eta}$ certainly does not trace out the actual trajectory of the AdaBoost weights.  The actual update first finds the best 
mistake dichotomy $\eta^w$, and then applies $\T_{\eta^w}(w)$.
\begin{definition}
\label{D:AdaBoost_update}
The \emph{AdaBoost (example-weights) update} is $\A : \Delta_m \to \Delta_m$, defined as
\(
\A(w) \equiv \T_{\eta^w}(w).
\)
Implicit in this definition is that (1) for all $w \not\in \Delta_m^+$,
$\A(w) = w$; and that (2) for all $w \in \Delta_m$, and for all
components $i=1,\ldots,m$,
$[\A(w)](i) = 0$ if and only if $w(i) = 0$. Also, for any set $W
\subset \Delta_m$, we employ the standard abuse of notation and define
$\A(W) \equiv \{ \T_{\eta^w}(w) \mid w \in W\}$.

\end{definition}
We can now trace the trajectory of the AdaBoost example weights by
repeatedly applying $\A$ to an initial example weight picked within
$\Delta_m$.  More specifically, if $w_1\in \Delta_m$ is taken as the
initial example weight, we can rederive any $w_t$ in our original formulation of the algorithm with
\(
  w_t = \A^{(t-1)}(w_1)
\)
where $\A^{(t-1)}$ denotes \emph{composing $\A$ with itself $t-1$
  times}. Also, for any set $W \subset \Delta_m$, we employ the
standard abuse of notation and define $\A^{(t)}(W) \equiv \{
  \A^{(t)}(w) \mid w \in W \}$. 

\begin{proposition}
  \label{P:A_Delta_plus_circ}
$\A(\Delta_m) \subset \Delta_m$, $\A(\Delta_m - \Delta_m^+) = \Delta_m
- \Delta_m^+$, $\A(\Delta_m^+) \subset \Delta_m^+$, and
$\A(\Delta_m^\circ) \subset \Delta_m^\circ$.
\end{proposition}  

The following set plays an important
role in the characterization of the image of the AdaBoost weight update.
\begin{definition}\label{D:pi_half}
Given some arbitrary $\eta\in \Mcal$, we define
\(
\pi_{\frac12}(\eta)\equiv \left\{w \in \Delta_m \left| \, \eta \cdot w
    = \frac12 \right. \right\} \; .
\)
\end{definition}

In Appendix~\ref{app:abup}, we present some properties of
  $\T_\eta$ and $\A$ that will be useful in the proofs of the
  technical results.

The inverse of $\A$ plays an
important role in our application of the Birkhoff Ergodic Theorem
(Theorem~\ref{T:Birkhoff}), which we present in the next section
(Section~\ref{S:classifier-convergence}). In particular, it helps
with the proof of the existence of a measure-preserving
transformation. We denote the inverse of function $\A$ by $\A^{-1}$.
We remind the reader that $\A^{-1}$ is of
type $\Delta_m \to \Sigma_{\Delta_m}$, where $\Sigma_{\Delta_m}$ is
the Borel $\sigma$-algebra on $\Delta_m$. Gaining some insight on the properties of
$\A^{-1}$ is useful in the technical derivations, but disrupts the
presentation. Thus, we moved
the statements and
discussion of those properties to Appendix~\ref{app:Ainv}. We do so in the interest
of reaching the statements of our main technical results as early as
possible in the main body of the paper.

\subsection{Formal definition of secondary quantities}

We can also derive many of the quantities
calculated by AdaBoost solely in terms of $w_t$.
\begin{definition}\label{D:f_w}
The following are functions 
of type $\Delta_m \to \R$: $\epsilon(w) \equiv \min_{\eta\in
  \Mcal}\eta \cdot w$ and\\ $\chi_{\pi^*(\eta)}(w) \equiv \indicator{w
  \in \pi^*(\eta)}$. The following function is 
of type $\Delta_m^+ \to \R$: $\alpha(w) \equiv \frac 12 \ln\left(\frac{1-\epsilon(w)}{\epsilon(w)}\right)$.
\end{definition}
The following definition is just our way to simplify the notation of
the sequences that
the functions of $w$ just stated in Definition~\ref{D:f_w} generate with respect to the
$w_t$'s.
\begin{definition} \label{F:1}
(1) $\epsilon_t \equiv \epsilon(w_t) = \epsilon\left(\A^{(t-1)}(w_1)\right)$, (2) $\alpha_t \equiv \alpha(w_t) = \alpha\left(\A^{(t-1)}(w_1)\right)$, and (3) $\eta_t \equiv \eta^{w_t} = \eta^{\A^{(t-1)}(w_1)}$.
\end{definition}
The value sequences described in Definition~\ref{F:1} will be called
\emph{secondary quantities}, because we can derive them solely from
the example weights' trajectory.  We seek to understand the 
convergence
properties of these secondary quantities, as the number of rounds $T$
of Optimal AdaBoost increases given a fixed, arbitrary dataset $D$
drawn with respect to some probability space $(\D,\Sigma,P)$
(see Section~\ref{sec:details}). We also seek to understand the
properties of the mapping
$\A$ that causes such converging behavior.

The following properties related to the secondary quantities will
be useful in our technical proofs, and some of the upcoming
discussion. They follow directly from the respective definitions.
\begin{proposition}\label{P:secondary}
The following statements about the secondary quantities hold, under
Condition~\ref{assume:WeakLearn} (Weak Learning).
\begin{enumerate}
\item For all $t$, $\epsilon_t < \frac12 -
  \gamma$, and thus $\alpha_t > \frac12 \ln\left( \frac{\frac12 +
      \gamma}{\frac12 - \gamma} \right) = \frac12 \ln\left( \frac{1 +
      2 \gamma}{1 - 2\gamma} \right) > 0$, and $\sum_{t=1}^T
  \alpha_t > \frac{T}{2} \ln\left( \frac{1 +
     2 \gamma}{1 - 2\gamma} \right)$. Also, 
  for each $w_1 \in \Delta_m^+$, we
  have that for all $t$, $\epsilon_t > 0$ and $\alpha_t < \infty$.
\item Suppose Condition~\ref{cond:nathypo} (Natural Weak-Hypothesis Class)
  also
holds. Then, 
for each $w_1 \in \Delta_m^+$, for all $t$,
$w_{t+1} \in \pi_{\frac12}(\eta_t)$ (see Definitions~\ref{D:1},~\ref{D:AdaBoost_update} and~\ref{D:pi_half}), 
and thus $\eta_t \neq \eta_{t+1}$.
\end{enumerate}
\end{proposition}

\section{Convergence of the Optimal-AdaBoost classifier} \label{S:classifier-convergence}

As mentioned at the end of the previous section, we can express the secondary quantities of Optimal AdaBoost as functions based solely on the trajectory of $\A$ applied to some initial $w_1\in \Delta_m$.  
Empirical evidence suggests that not only are averages of this
quantities/parameters converging, where the averages are taken with respect to the number of rounds $T$ of AdaBoost, but the Optimal-AdaBoost \emph{classifier} itself is \emph{converging.}  

The study of the convergence of the AdaBoost classifier, and its implications, is the main goal of this section.

Key to our understanding of convergence, in the sense previously
discussed in Section~\ref{sec:details}, is the \emph{Birkhoff Ergodic Theorem}~\citep{Birkhoff1931}, stated as Theorem~\ref{T:Birkhoff} below.  This theorem gives us sufficient conditions for 
the (probabilistic) convergence, which we will then apply to our
secondary quantities.  Taking center stage in this theorem is the
notion of a \emph{measure} and a \emph{measure-preserving} dynamical
system.  To be able to apply the Birkhoff Ergodic Theorem, we need to
show the existence of some \emph{measure} $\mu_\Omega$ such that
$(\Omega,\Sigma_\Omega, \mu_\Omega,\A_\Omega)$ is a
\emph{measure-preserving} dynamical system, for some set
$\Omega \subset \Delta_m^+$, to be concretely defined later, as a function of
$\A$, and $\A_\Omega : \Omega \to \Omega$ is the map
consistent with $\A$ on $\Omega$: i.e., for each $w \in
\Omega$, $\A_\Omega(w) \equiv \A(w)$.  (We refer the reader to
Appendix~\ref{A:BasicMathDefn} for more details on these topics.)  The
existence of such a measure is given in Proposition~\ref{A:1}. We discuss the context
surrounding these in greater detail shortly.

We establish the existence of the measure $\mu_\Omega$ using the
\emph{Krylov-Bogolyubov Theorem} \citep{KryloffBogoliuboff1937}, formally defined as
Theorem~\ref{Krylov-Bogolyubov} in Section~\ref{sec:invmeas}, 
 and which, as it turns out, is very closely related to the Birkhoff
 Ergodic Theorem~\citep{oxtoby1952}. A couple of concepts are essential to
understand the Krylov-Bogolyubov Theorem, as well as the Birkhoff
Ergodic Theorem and the notion of convergence of sequences in
$\Delta_m$ used here.  First, as we will see,
the Krylov-Bogolyubov Theorem requires that we deal with a system of the form
$(W,\Nbf)$, formally called a \emph{topological space}, where $W$ is a
set, often called the \emph{state space}, and $\Nbf$ a
\emph{(neighborhood) topology} on it.  Furthermore, $(W,\Nbf)$ needs
to be \emph{metrizable}, meaning the topology $\Nbf$ can be induced by
some metric. In topology, a metrizable topological space $(W,\Nbf)$ is
a \emph{metric space} $(W,d)$ if the metric $d$ induces $\Nbf$.
 We note that, often, we do not use $d$ directly in our proofs, but it
 is in some sense implicit in our arguments about convergence. 
The definition of closed and open sets also implicitly uses $d$:
\emph{closed sets} are the sets in $\Delta_m$ that contain all of
their \emph{limit points}.  That is, a set $E$ is \emph{closed} if,
given any convergent sequence $(w^{(s)})$ in $E$, we have $\lim_{s \to
  \infty} w^{(s)} \in E$.  As a sub-family of closed sets we have
\emph{compact sets}, the closed sets that are also \emph{bounded}, by
the Heine-Borel Theorem~\citep[][Theorem 11.3, pp. 72]{BartleRA}.  We are only considering subsets of 
$\Delta_m^+ \subset \Delta_m$, 
so all such subsets are bounded and any closed subset will be \emph{compact}.

Equally important in the Birkhoff Ergodic Theorem is the notion of integrability, captured by the notation $f \in L^1(\mu)$.  This notation says that $f$ is \emph{integrable with respect to the measure $\mu$.}  The precise meaning of this is that, first and foremost, $f$ is \emph{measurable}.  Second, that
\(
  \int |f| d\mu < \infty.
\)
If these two conditions hold, it follows that $f \in L^1(\mu)$.  Proposition~\ref{A:3}, which we formally state later in this section, shows us that various quantities generated by Optimal AdaBoost are in $L^1(\mu)$, therefore can be analyzed using Theorem~\ref{T:Birkhoff}.

\begin{definition}{\bf (Empirical Measures, Time Averages, and State
    Averages)}
  \label{def:emp_meas_and_avgs}
  \\ Let $(W,\Sigma_W,M,\mu)$ be a dynamical system. Denote by
  $\delta_\omega \equiv \chi_{\omega}$ the Dirac-delta function,
  the point-mass probability measure with full support
  on the point $\omega \in W$. Denote by 
$\widehat{\mu}_\omega^{(T)} \equiv \widehat{\mu}_{M,\omega}^{(T)} \equiv \frac{1}{T} \sum_{t=0}^T
\delta_{M^{(t)}(\omega)}$ the \emph{empirical (probability) measure} induced by $M$ on $W$
  \emph{after $T$ time steps} starting from $\omega \in W$, and by $\widehat{\mu}_\omega \equiv \widehat{\mu}_{M,\omega} \equiv \lim_{T
  \to \infty} \widehat{\mu}_\omega^{(T)}$ the \emph{empirical (probability) measure} induced by $M$ on $W$
  \emph{in the limit} starting from $\omega \in W$, also called the {\em Birkhoff
  limit} of the point $\omega$, if the limit exists. Given a function $f : W
\to W$, denote
by $\widehat{f}_T(\omega) \equiv \widehat{f}_T^M(\omega) \equiv
\Tavg(f,\omega,M,T) \equiv \frac{1}{T} \sum_{t=0}^T
f(M^{(t)}(\omega))$ the \emph{time average} induced by $M$ on $W$, also
  called the {\em Birkhoff average}, after $T$
  time steps starting from $\omega \in W$, and by $\widehat{f}(\omega) \equiv \widehat{f}^M(\omega) \equiv
\lim_{T \to \infty} \Tavg(f,\omega,M,T)$ \emph{time average} induced by $M$ on $W$ \emph{in
  the limit} starting from $\omega \in W$, if
the limit exists. If $f \in L_1(\mu)$, denote by 
$\bar{f} \equiv \bar{f}_\mu \equiv
\Savg(f,M,\mu) \equiv \frac{1}{\mu(W)} \int f \, d\mu$ the \emph{state
  average} of $f$ with respect to $\mu$. 
\end{definition}
\begin{definition}{\bf (Ergodicity)}
  \label{def:ergodic}
Consider a dynamical system $(W,\Sigma_W,M,\mu)$ with a
measure-preserving map $M$ with respect to $(W,\Sigma_W)$ and
the finite measure $\mu$. The
system, and the measure $\mu$, is
called {\em ergodic} if for all $E \in \Sigma_W$ such that $E =
M^{-1}(E)$ (i.e., $E$ is an invariant set)
we have $\mu(E) \in \{0, \mu(W)\}$ (i.e., $E$ has full measure or
measure zero). The system is called {\em uniquely ergodic} if it admits exactly one invariant measure (i.e., $\mu$ is unique). It is called {\em strictly ergodic} if the only
invariant set satisfying that condition is $W$ (i.e., the support of
$\mu$ is $W$).
\end{definition}

We are now ready to introduce the theorem.
\begin{theorem}{\bf (Birkhoff Ergodic Theorem~\citep{Birkhoff1931,oxtoby1952})}
\label{T:Birkhoff}
Suppose $M : W \to W$ is \emph{measure-preserving} and $f \in
L^1(\mu)$ for some \emph{measure} $\mu$ on the measurable space
$(W,\Sigma_W)$.  Then, the time average $\widehat{f}(\omega)$ exists (i.e., it is
well-defined because $\widehat{f}_T$ converges as $T \to \infty$) for $\mu$-\emph{almost every} $\omega \in W$ to
a function $\widehat{f} \in L^1(\mu)$.  Also $\widehat{f} \circ M = \widehat{f}$ for
$\mu$-\emph{almost every} $\omega \in W$ and if $\mu(W) < \infty$,
then $\int \widehat{f} \, d\mu = \bar{f}$. If $(W,\Sigma_W,M,\mu)$ is
ergodic, then $\widehat{f}$ is constant $\mu$-\emph{almost everywhere} and $\widehat{f} = \bar{f}$.
\end{theorem}
Note that $\widehat{f}$ depends on $\omega$, $\mu$, $(W,\Sigma_W)$,
$f$, and $M$, while $\bar{f}$ depends on all \emph{but} $\omega$.

\subsection{Satisfying the Birkhoff Ergodic
  Theorem}\label{SatisfyingBirkhoff}

We employ an important theorem in measure theory, the
Krylov-Bogolyubov Theorem to establish the existence
of an invariant measure.
\begin{theorem}{\bf (Krylov-Bogolyubov~\citep{KryloffBogoliuboff1937,oxtoby1952})}\label{Krylov-Bogolyubov}
Let $(W,\Nbf)$ be a (non-empty) compact, metrizable topological space and $g : W \to W$ a continuous map.  Then $g$ admits an invariant Borel probability measure.
\end{theorem}
A key ingredient is establishing the continuity of the
Optimal-AdaBoost example-weights update $\A$ on some set $\Omega$ so
that we can use $\A_\Omega$ to define the dynamical system. We note that we do not need to use Theorem~\ref{Krylov-Bogolyubov} if Optimal
AdaBoost always cycles, as we discuss in Appendix~\ref{app:cycles}.

\subsubsection{Existence of an invariant measure}
\label{sec:invmeas}

We care about the asymptotic behavior of Optimal AdaBoost, and want to
disregard any of its \emph{transient states}, which for all \emph{practical} purposes in our context
means, any $w \in \Delta_m^\circ$ such that for all $w' \in
\Delta_m^\circ$ there exists a $T' \in \N$ such that for all $t > T'$,
$w \neq \A^{(t)}(w')$. The rationale for this is as follows. Speaking \emph{strictly
mathematically}, it only make sense to start AdaBoost from an initial
weight $w \in \Delta_m^+$, because any $w \in \Delta_m-
\Delta_m^+$ is a fixed-point of the update. \emph{Practically}
speaking however, it only makes sense to start the algorithm from an
initial weights $w \in \Delta_m^\circ$ because any $w \in \Delta_m^+ - \Delta_m^\circ$ must have a
component $w(l) = 0$ for which the corresponding example indexed by $l$ will have
weight $0$ and would be unaffected by the update: i.e., for any such
$w$, if $w' = \A(w)$, we have $w' \in \Delta_m^+ - \Delta_m^\circ$ and
$w'(l) = 0$. So, it is like we started the algorithm from a
\emph{subset} of the original data set. Therefore we would like to
look at a subset of its state space that the dynamics will limit
towards, or stay within, starting from $\Delta_m^\circ$ (although our
results do extend to the initial set $\Delta_m^+$ for the reasons just
stated).

The following sets are also useful to understand the
\emph{asymptotic} behavior of Optimal AdaBoost. They characterize the
set of example weights that Optimal
AdaBoost can reach for any time step $t$: i.e., the subset of
the state space of Optimal AdaBoost that the dynamics will limit
towards, or stay within. We refer to each as 
the set of \emph{transitive}, or \emph{non-transient}, states,
depending on the space of initial weights:
e.g., in contrast, the set of \emph{non-transitive}, or \emph{transient},
states in the case of $\Delta_m^+$, consists of any $w \in \Delta_m^+$
such that for all $w' \in \Delta_m^+$ there exists a $T' \in \N$ such
that for all $t > T'$, $w \neq \A^{(t)}(w')$.
\begin{definition}
\label{def:OmegaInf}
We define the sets $\Omega_{\infty}^+ \equiv  \bigcap_{t=1}^{\infty}
\A^{(t)}(\Delta_m^+)$, and $\Omega_{\infty} \equiv  \bigcap_{t=1}^{\infty}
\A^{(t)}(\Delta_m^\circ)$.
\end{definition}
Note that $(\A^{(t)}(\Delta_m^+))$ is a decreasing sequence of
sets (i.e., $(\A^{(t)}(\Delta_m^+)) \supset
(\A^{(t+1)}(\Delta_m^+))$ for all $t$); similarly for
$\A^{(t)}(\Delta_m^\circ)$. We can think of the set
$\Omega_{\infty}^+$ as a ``trapped'' attracting set in the typical sense used for dynamical systems, because $\Delta_m$ is compact and 
$\A(\Delta_m^\circ) \subset \Delta_m^\circ \subset \Delta_m^+$, $\A(\Delta_m^+) \subset
\Delta_m^+$, and $\A(\Delta_m^\circ) \subset
\A(\Delta_m^+) \subset \Delta_m^+ \subset \Delta_m$. Also note that $\bigcap_{t=1}^{\infty}
\A^{(t)}(\Delta_m) = (\Delta_m - \Delta_m^+) \cup \Omega_\infty^+$.

\paragraph{Applying the Krylov-Bogolyubov Theorem.}

The objective now is the application of
Krylov-Bogolyubov Theorem (Theorem~\ref{Krylov-Bogolyubov}) as a way
to satisfy the conditions of the Birkhoff Ergodic Theorem
(Theorem~\ref{T:Birkhoff}) within our dynamical system's view of
Optimal AdaBoost (Section~\ref{sec:ds}).
For a given dynamical system that meets certain conditions,
Krylov-Bogolyubov tells us that the system is
\emph{measure-preserving} on some Borel probability measure.  We will
apply this theorem on
$\Omega = \Omega_\infty^+$ 
to show that $\A$ admits an invariant
measure on it.

\paragraph{Continuity of $\A$.}

We will begin by studying the continuity properties of $\A$.
Theorem~\ref{T:continuity} establishes that $\A$ is continuous on most points in its
state space.
The continuity properties of AdaBoost on 
$\Omega_\infty^+$ are important to establish an invariant measure in
Section~\ref{SatisfyingBirkhoff}.  We will eventually show that, under
certain conditions, $\A$ is in fact continuous on
$\Omega = \Omega_\infty^+$.  But it is difficult to say anything
important about this set yet.  

It turns out that there are discontinuities at many points in the
state space.~\footnote{Indeed, assuming Condition~\ref{cond:nathypo} (Natural Weak-Hypothesis Class) holds, if the Optimal AdaBoost update were to
  be continuous on a compact \emph{convex} subset of $\Delta_m^\circ$, then
  it follows from Brouwer's Fixed-Point Theorem that
  Condition~\ref{assume:WeakLearn} (Weak Learning) would be violated.}
It is not difficult to see that any point $w\in \Delta_m^+$ that
yields more than one mistake dichotomy in $\argmin_{\eta\in \Mcal}
\eta \cdot w$ will be a  discontinuity. Similarly, any point that has
$\eta^w \cdot w = 0$ will also be a discontinuity. While, by definition, this type of
discontinuity does not exist in $\Delta_m^+$, we would still have to
show that $\A$ does not converge to a set \emph{outside $\Delta_m^+$in
  the limit.} This motivates the following definition.
\begin{definition}
Let $w \in \Delta_m$.
\begin{enumerate}
\item If $|\argmin_{\eta \in \Mcal} \eta \cdot w | > 1$, we then call $w$ a \emph{type-1 discontinuity}.
\item If $\eta^w \cdot w = 0$, we then call $w$ a \emph{type-2 discontinuity}.
\end{enumerate}
\end{definition}

In the following theorem, we establish that $\A$ will be continuous on any point besides \emph{type-1} and \emph{type-2} discontinuities.
\begin{theorem}\label{T:continuity}
{\bf (The Example-Weights Update of Optimal AdaBoost is Mostly Continuous.)}
Suppose Condition~\ref{assume:WeakLearn} (Weak Learning) holds.
Then Optimal AdaBoost is continuous on all points $w$ such that $w \in \bigcup_{\eta \in \Mcal}
\pi^\circ(\eta)$.
\end{theorem}
\begin{proof}
Let $W \equiv W^{\eta} \equiv \pi^\circ(\eta)$.  Take any
$w \in W$, and let $\{w^{(s)}\}$ be an arbitrary sequence in
$\Delta_m$ such that $\lim_{s \to \infty} w^{(s)} = w$.  Let
$\{w^{(s')}\}$ be the \emph{tail} of $\{w^{(s)}\}$ that is contained
within $W$, i.e., there exists a finite $T' \in \N$ such that for all
$j > T'$, $w^{(s')} \in W$.  Then, by Definitions~\ref{def:pi}
($\pi(\eta)$),~\ref{D:1} ($\T_\eta$), and~\ref{D:AdaBoost_update} ($\A$), we have the following for all $w^{(s')}$,
\begin{equation}\label{E:equality}
\A(w^{(s')}) = \T_\eta(w^{(s')}).
\end{equation}
From Definition~\ref{D:1}, for all $w^{(s')}$ it follows that, for all $i=1,\ldots,m$,
\begin{equation*}
[\T_\eta(w^{(s')})](i) = \frac12 w^{(s')}(i) \times
\left(\frac{1}{\eta\cdot
    w^{(s')}}\right)^{\eta(i)}\left(\frac{1}{1-(\eta\cdot
    w^{(s')})}\right)^{1-\eta(i)} \; .
\end{equation*}
Because $\lim_{s' \to \infty} w^{(s')} = w$, we have $\lim_{s' \to
  \infty}  w^{(s')}(i) = w(i)$ for all $i=1,\ldots,m$.  Similarly, we
have $\lim_{s' \to \infty} \eta \cdot w^{(s')} = \eta \cdot w$.
Furthermore,
by Definition~\ref{def:pi}
and Condition~\ref{assume:WeakLearn} (Weak Learning), we have $0 < \eta\cdot w < \frac 12$.  Combining these facts, we see that 
\(
\lim_{s' \to \infty} \T_\eta(w^{(s')}) = \T_\eta(w).
\)
Recalling Eqn.~\ref{E:equality}, we complete the proof. 
\qed
\end{proof}

The following corollary will be useful later in Section~\ref{sec:cycle}.
\begin{corollary}
  \label{C:continuity}
Suppose Condition~\ref{assume:WeakLearn} (Weak Learning) holds.
Then, for each $\eta \in \Mcal$, the Optimal AdaBoost update is continuous on
$\pi^+(\eta)$ when viewed as a function of type
$\pi^+(\eta) \to \Delta_m$.
\end{corollary}
\begin{proof}
  The proof is identical to that for Theorem~\ref{T:continuity},
  \emph{except} that $W =
\pi^+(\eta)$ and the arbitrary sequence
$(w^{(s)})$ converging to $w$ to be in $\pi(\eta)$, a compact superset
of $W$, so that the tail $(w^{(s')})$
is the
arbitrary sequence itself. \qed
\end{proof}

The following lemma takes a step towards establishing that AdaBoost
will not encounter type-1 discontinuities after $n+1$ rounds. The lemma shows that, given a point $w \in \Delta_m(\eta)$, if the
error of a hypothesis corresponding to a mistake dichotomy in $\Mcal$
is low on $w$, then the error of that same hypothesis on the inverse
of $w$ is not too large.  Not only that, but the error induced by the
mistake dichotomy $\eta$ on the inverse also is not too large.  We use
the following lemma to prove the next theorem, which tells us that
AdaBoost is bounded away from type-$2$ discontinuities in the limit. 
\begin{lemma}\label{L:epsilon_0+}
Given any $\eta \in \Mcal$, denote by $\Delta_m^+(\eta) \equiv
\T_\eta(\Delta_m^+)$,
the image of the set $\Delta_m^+$ with respect to the ``hypothetical''
AdaBoost example-weights update
function
$\T_\eta$ (see Definition~\ref{D:1}).
Suppose Conditions~\ref{cond:nathypo} (Natural Weak-Hypothesis Class)
and~\ref{assume:WeakLearn} (Weak Learning) hold.
Let $\eta,\eta'\in \Mcal$ and 
$w \in \Delta_m^+(\eta)$. 
If $\eta' \cdot w \leq \epsilon_0$, then for all $w' \in \A^{-1}(w)$ we have $\eta' \cdot w' \leq 2\epsilon_0$ and $\eta\cdot w' \leq 2\epsilon_0$.
\end{lemma}
\begin{proof}
Pick arbitrary $\eta \in \Mcal$, and $w \in \Delta_m^+(\eta)$, $\eta' \in
\Mcal$ such that $\eta' \cdot w \leq \epsilon_0$, and $w' \in
\A^{-1}(w)$. By the definition of the inverse of a function, we have $w = \A(w')$.

Let $\Mcal_{\frac 12}(w) \equiv \{\eta \in \Mcal | \, \eta \cdot w =
\frac 12\}$, and let $g(\rho,\eta; w) \equiv 2\rho w_\eta^- + 2(1-\rho)
w_\eta^+$, where $w_\eta^-$ and $w_\eta^+$ are as defined in
Proposition~\ref{P:Inv} in Appendix~\ref{app:Ainv}.  Let $L_{\frac
  12}(w) = \{g(\rho,\eta; w) | \, \eta \in \Mcal_{\frac 12}(w),
\rho \in (0,\frac 12)
\}$.

\begin{claim}
$\A^{-1}(w) \subset L_{\frac 12}(w)$.
\end{claim}
\begin{proof}
Let $w'' \in \A^{-1}(w)$. Then $w \in \A(w'')$. Let $\eta^* \equiv \eta^{w''} = \AdaSel\left(\argmin_{\eta'' \in \Mcal} \eta''
    \cdot w'' \right)$. Note that
  $\eta^* \cdot w'' > 0$ because $w \in \Delta_m^+$, which implies
  $w'' \in \Delta_m^+$ too (see Proposition~\ref{P:update}, and
  Propositions~\ref{P:Inv} and~\ref{pro:Ainv} in
  Appendix~\ref{app:Ainv}). The update used by Optimal AdaBoost (Definition~\ref{D:AdaBoost_update}) in this
  case is $w(i) = \frac 12 w''(i) \left( \eta^* \cdot w''
  \right)^{-\eta^*(i)} \left( 1 - \eta^* \cdot w''
  \right)^{-(1-\eta^*(i))}$. Note that by the properties of the
  AdaBoost update (see Proposition~\ref{P:update}), we have $\eta^*
  \cdot w = \frac 12$ so that $\eta^* \in \Mcal_{\frac
    12}(w)$. Rearranging the update equation using some simple algebra
  yields $w''(i) = 2 w(i) \left( \eta^* \cdot w'' \right)^{\eta^*(i)}
  \left( 1 - \eta^* \cdot w'' \right)^{1-\eta^*(i)} = 2 w(i) \left(
    \eta^*(i) (\eta^* \cdot w'') + (1-\eta^*(i)) (1 - (\eta^* \cdot w'')
  \right) = 2 (\eta^* \cdot w'') (w(i) \eta^*(i)) + 2 (1 - (\eta^* \cdot
  w'')) (w(i) (1-\eta^*(i)))$, which using the definitions of
  $w_{\eta^*}^-$ and $w_{\eta^*}^+$, and letting $\rho' \equiv \eta^*
  \cdot w''$, implies $w'' = 2 \rho' w_{\eta^*}^- + 2 (1 - \rho')
  w_{\eta^*}^+ = g(\rho',\eta^*;w)$. Invoking
  Condition~\ref{assume:WeakLearn} (Weak Learning), we have $\rho' <
  \frac 12$ (see also Proposition~\ref{P:secondary}), which yields the
  result: $w'' \in L_{\frac 12}(w)$.
  \qed
\end{proof}
So it suffices to show that the lemma holds for all elements in $L_{\frac 12}(w)$.

Pick an arbitrary real-value
$\rho \in (0,\frac 12)$.
We can
decompose $\eta' \cdot g(\rho,\eta; w)$ as
\(
  \eta' \cdot g(\rho,\eta; w) = 2\rho(\eta' \cdot w_\eta^- - \eta' \cdot w_\eta^+) + 2(\eta' \cdot w_\eta^+).
\)
To upper bound $\eta' \cdot g(\rho,\eta; w)$, we consider two cases depending on the relationship between $\eta' \cdot w_\eta^-$ and $\eta' \cdot w_\eta^+$.
\begin{enumerate}
\item If $\eta' \cdot w_\eta^- > \eta' \cdot w_\eta^+$, then 
\(
\eta'\cdot g(\rho,\eta; w) < (\eta' \cdot w_\eta^- - \eta' \cdot w_\eta^+)
+ 2(\eta' \cdot w_\eta^+) = \eta'\cdot w_\eta^- + \eta'\cdot w_\eta^+ = \eta' \cdot w \leq \epsilon_0.
\)
\item If $\eta' \cdot w_\eta^- \leq \eta' \cdot w_\eta^+$, then
\(
\eta'\cdot g(\rho,\eta; w) \leq 2\eta'\cdot w_\eta^+ \leq 2(\eta' \cdot w) \leq 2\epsilon_0.
\)
\end{enumerate}
Taking the largest of the upper bounds, we conclude that $\eta' \cdot
g(\rho,\eta; w) \leq 2\epsilon_0$.  Now, if $g(\rho,\eta; w) \in
\A^{-1}(w)$, it follows that $\eta^{g(\rho,\eta; w)} =
\AdaSel(\argmin_{\eta'' \in \Mcal} \eta'' \cdot w(\rho,\eta))$.
Therefore $\eta^{g(\rho,\eta; w)} \cdot g(\rho,\eta; w) = \min_{\eta''
  \in \Mcal} \eta'' \cdot g(\rho,\eta; w) \leq \eta' \cdot
g(\rho,\eta; w) \leq 2\epsilon_0$.
\qed
\end{proof}
Actually, had
we define the $\A$ with domain $\Delta_m^+$, instead of
$\Delta_m$, the only way we can get type-$2$ discontinuities is if the
evolution of the AdaBoost update leads to a weight in the closure of
$\Delta_m^+$, i.e., it leads to a $w$ outside $\Delta_m^+$. As we will
see, that cannot happen.

We can now apply Lemma~\ref{L:epsilon_0+}
recursively to show
that
for all $t>n+1$ the weighted
error of any hypothesis, or equivalently, mistake dichotomy, with
respect to the points/example-weights in $\A^{(t)}(\Delta_m^+)$
is bounded away from zero. (Recall that $\A^{(t)}(\Delta_m^+)$
is the set of all weight distributions $w_t$ over the examples that
Optimal AdaBoost reaches after $t$ rounds, starting from \emph{any}
initialization of the weights/distributions $w_1$ over the examples
selected from $\Delta_m^+$, the subset $\Delta_m$
where no
$\eta \in \Mcal$ has zero weighted error.)
\begin{theorem} \label{T:lower-bound+}
{\bf (Lower Bound on Weighted Errors of Optimal AdaBoost)}
Suppose Conditions~\ref{cond:nathypo} (Natural Weak-Hypothesis Class)
and~\ref{assume:WeakLearn} (Weak Learning) hold.
There exists an $\epsilon_* \geq 2^{-(n+1)}$ such that 
for all $t > n$ we have $\eta \cdot w \geq \epsilon_*$ for all $w \in \A^{(t)}(\Delta_m^+)$ and $\eta \in \Mcal$.
\end{theorem}
\begin{proof}
Set $T_0 = |\Mcal|+1 = n+1$, and $\epsilon_* < \frac{1}{2^{T_0}}$.  Take an arbitrary $\eta \in \Mcal$ and $w \in \Delta_m^+$ such that $\eta \cdot w < \epsilon_*$.  We will show that $w \notin \A^{(t)}(\Delta_m^+)$ for $t > T_0$.

Let $w^{(1)} \in \A^{-1}(w)$.  If no such $w^{(1)}$ exists, we
have already demonstrated our goal.  Otherwise, let $\eta^{(1)}
\equiv \eta^{w^{(1)}}$.  By Lemma~\ref{L:epsilon_0+}, we know that
$\eta^{(1)} \cdot w^{(1)} \leq 2\epsilon_*$ and $\eta \cdot
w^{(1)} \leq 2 \epsilon_*$. Continuing in this
way, let $w^{(2)} \in \A^{-1}(w^{(1)})$, which we can assume
exists by the same argument made for $w^{(1)}$.  Let $\eta^{(2)}
\equiv \eta^{w^{(2)}}$. By Lemma~\ref{L:epsilon_0+}, we obtain that
$\eta^{(2)} \cdot w^{(2)} \leq 2^2\epsilon_*$ and $\eta^{(1)}
\cdot w^{(2)} \leq 2^2 \epsilon_*$, or using short-hand, $\eta^{(j)}
\cdot w^{(2)} \leq 2^2 \epsilon_*$ for $j < 3$. Now note that $\eta^{(2)}
\neq \eta^{(1)}$  
because 
\(
  \eta^{(2)} \cdot w^{(1)} = \frac 12 
  > 2 \epsilon_* 
  \geq \eta^{(1)}\cdot w^{(1)}.
\)
This should be the base case of an induction argument, but we think it
may be useful to perform one more step before moving on to
the induction step, so that the pattern becomes clear. So, let $w^{(3)} \in \A^{-1}(w^{(2)})$, which we can assume
exists by the same argument made for $w^{(1)}$ and $w^{(2)}$.  Let $\eta^{(3)}
\equiv \eta^{w^{(3)}}$. By Lemma~\ref{L:epsilon_0+}, we obtain that
$\eta^{(3)} \cdot w^{(3)} \leq 2^3\epsilon_*$, $\eta^{(2)} \cdot w^{(3)} \leq 2^3\epsilon_*$, and $\eta^{(1)}
\cdot w^{(3)} \leq 2^3 \epsilon_*$, or using short-hand,
$\eta^{(j)} \cdot w^{(3)} \leq 2^3\epsilon_*$ for all $j < 4$. Now note that $\eta^{(3)}
\neq \eta^{(i)}$  for $i < 3$ because 
\(
  \eta^{(3)} \cdot w^{(2)} = \frac 12 
  > 2^2\epsilon_* 
  \geq \eta^{(i)}\cdot w^{(2)}
\)
The pattern should now be clear.

We can continue this template out to $T_0$.  Let $w^{(T_0)}
\in \A^{-1}(w^{(T_0-1)})$.  We claim that such a $w^{(T_0)}$ cannot exist.  For sake of contradiction, suppose it did.
Then, let $\eta^{(T_0)} \equiv \eta^{w^{(T_0)}}$.
From Lemma~\ref{L:epsilon_0+}, we know that $\eta^{(T_0)}
\neq \eta^{(i)}$ for all $i < T_0$ because
\(
  \eta^{(T_0)} \cdot w^{(T_0-1)} = \frac 12 
  > 2^{(T_0-1)}\epsilon_* 
  \geq\eta^{(i)}\cdot w^{(T_0-1)}.
\)
In particular, by the Principle of Mathematical Induction, all
$\eta^{(i)}$'s in the sequence are \emph{unique} by the construction.  Because $T_0 = |\Mcal|+1 = n+1$, the sequence $\{\eta^{(1)},\eta^{(2)},\ldots,\eta^{(T_0-1)}\} = \Mcal$.  But because $\eta^{(T_0)} \neq \eta^{(i)}$ for all $i < T_0$, $\eta^{(T_0)}$ is not in $\Mcal$.  As this is a contradiction, we must conclude that no such $w^{(T_0)}$ exists, and that $\A^{-1}(w^{(T_0-1)}) = \emptyset$.  Our selection of $w^{(i)}$'s was arbitrary in each step of the construction of the sequence, so we can also conclude that there does not exist any $w'$ such that $A^{(T_0)}(w') = w$, or else it would have been reached by the above procedure.  Finally, this shows $w \notin \A^{(T_0)}(\Delta_m^+)$.\qed
\end{proof}

\begin{corollary}\label{cor:non0}
  {\bf (AdaBoost Weak-Hypotheses Always Have Non-Zero Weighted Error.)} Suppose Conditions~\ref{cond:nathypo} (Natural Weak-Hypothesis Class)
and~\ref{assume:WeakLearn} (Weak Learning) hold. Then we have $\Omega_\infty^+
\subset \A^{(n+1)}(\Delta^+) \subset \Delta_m^{\epsilon_*} \equiv \{ w \in \Delta_m \mid \min_{\eta
  \in \Mcal} \eta \cdot w \geq \epsilon_* \}$ for some $\epsilon_* \geq 2^{-(n+1)}$.
\end{corollary}
\begin{proof}
The proof follows immediately from the last theorem (Theorem~\ref{T:lower-bound+}) and the definition
of $\Omega_\infty^+$ (Definition~\ref{def:OmegaInf}).
  \qed
  \end{proof}

\paragraph{Compactness of $\Omega_\infty^+$.}

Having established that the weight trajectories of Optimal AdaBoost are
bounded away from type-$2$ discontinuities after $n+1$ rounds starting
from $\Delta_m^+$, we now 
deal with
type-$1$ discontinuities. To do so, we introduce a condition
that states that any trajectory is bounded away from type-$1$
discontinuities.  This condition will be instrumental in our analysis,
and gives us a way of proving the existence of an invariant measure
that is essential for satisfying Theorem~\ref{T:Birkhoff}. 
We will provide the formal statement in Condition~\ref{C:NoTies}.
 Roughly speaking, this condition essentially says that, after a
 sufficiently long number of rounds either (1) the dichotomy
 corresponding to the optimal weak hypothesis for a round is unique
 with respect to the weights at that round, or (2) the dichotomies
 corresponding to the hypotheses that are tied for optimal are
 essentially the same with respect to the weights in
 $\Omega_{\infty}^+$.

 \begin{condition}{{\bf (Optimal AdaBoost has No Important Ties in the
       Limit.) }} \label{C:NoTies}
There exists a compact set $G$ such that
$\Omega_{\infty}^+ \subset G$ and, given any pair $\eta,\eta' \in \Mcal$, we have either
\begin{enumerate}
\item $\pi(\eta)\cap\pi(\eta')\cap G = \emptyset$; or
\item for all $w \in G$, $\sum_{i: \eta(i) \neq \eta'(i)} w(i) = 0$.
\end{enumerate}
\end{condition}
Note that Part 2 of Condition~\ref{C:NoTies} allows us to reduce the
set of label dichotomies to only those that will never become
effectively the same from the standpoint of Optimal AdaBoost when
dealing with $\Omega_{\infty}^+$, and that this condition can only
happen \emph{in the limit} if starting from $\Delta_m^\circ$ (see
Remark~\ref{R:Rerun}).

We delay further discussion on this condition until Section~\ref{sec:condrem},
where we make a brief specific remark and Section~\ref{sec:close}
(Closing remarks), where we elaborate on the condition and place it in
a broader context. But a remark is in order before we continue.
\begin{remark}
We have found that this condition always holds in all the
high-dimensional, real-world datasets we have tried in
practice. Indeed, 
we provide strong empirical evidence justifying the validity and
reasonableness of the condition in practice using high-dimension
real-world datasets in Section~\ref{S:Exp}. We should point out
that Theorem~\ref{T:continuity} in itself does not imply the
condition. Cynthia Rudin, the
Action Editor for a previous journal-submission version of this article, reports (Personal Communication), 
that in her experiments, ``AdaBoost would "walk" continuously around each continuous region (where each step is a full rotation around a cycle of weak classifiers) until it crossed a boundary where there was a tie, and then it  would change direction. So, even though the map was mostly continuous, it would hit the discontinuities occasionally and that would change the dynamics.'' We wonder whether this
experience is related to some of the example weights going to zero,
which would be consistent with the condition, or
whether it may simply be attributed to numerical instability.
(We refer the reader to Definition~\ref{def:sv} of ``non-support vector'' examples, and our discussion
on them, including their connection to this condition, in Section~\ref{sec:condrem}). Regardless, we remind the reader that our
remark concerns only about the empirical behavior of Optimal AdaBoost on high-dimensional, real-world datasets, not 
on randomly-generated or hand-tailored 
synthetic or small datasets. In fact, we conjecture that this
condition mathematically holds in AdaBoost after $T=n+1$ rounds if
started from $\Delta_m^+$, or equivalently $\Delta_m^\circ$, such that
$G =
\closure{\A^{(n+1)}(\Delta_m^+)}$.
\end{remark}

We now show the compactness of $\Omega_{\infty}^+$ given Condition~\ref{C:NoTies}.  
We first approach this by proving the following lemma, which states that any limit point $w\in G$ of $\Omega_{\infty}^+$ has a corresponding limit point $w' \in G$ of $\Omega_{\infty}^+$ such that $\A(w') = w$.
\begin{lemma} \label{L:compactness}
Suppose Conditions~\ref{cond:nathypo} (Natural Weak-Hypothesis Class),~\ref{assume:WeakLearn} (Weak Learning)
and~\ref{C:NoTies} (No Key Ties) hold.  
Let $(w^{(s)})$ be an arbitrary convergent sequence in $\Omega_{\infty}^+$, and call its limit $w$.  Then there exists a second convergent sequence $\{\omega^{(s)}\}\subset \Omega_{\infty}^+$, such that $\A(lim_{s \to \infty} \omega^{(s)}) = w$.
\end{lemma}
\begin{proof}
Let $(w^{(s)})$ be such a sequence in $\Omega_{\infty}^+$ as
described in the hypothesis, and let $w = \lim_{s \to \infty}
w^{(s)}$.  From the compactness of the set $G$, as defined in
Condition~\ref{C:NoTies} (No Key Ties), we have $w \in G$.  Additionally, as $w^{(s)} \in \Omega_{\infty}^+$, there must exist a $\omega^{(s)} \in \Omega_{\infty}^+$ such that 
\begin{equation} \label{E:def w_i}
w^{(s)} = \A(\omega^{(s)}) .
\end{equation}
Let $(\omega^{(s)})$ be a sequence in  $\Omega_{\infty}^+$ composed of
such elements.  We will now proceed to show that there exists a
subsequence of $(\omega^{(s)})$ that has a limit $w' \in \A^{-1}(w)$.

First note that $\eta^w \cdot w > 0$ by Corollary~\ref{cor:non0}
(Non-Zero Error)
and the fact that we can let $G$ in Condition~\ref{C:NoTies} (No Key
Ties) be a
subset of the set $\Delta_m^{\epsilon_*}$ also defined in 
Corollary~\ref{cor:non0}.

Consider subsets of $G$ of the form
\(
G^*(\eta) \equiv \{w \in G | \, w\in \pi^*(\eta)\} = G \cap \pi^*(\eta).
\)
Note that 
the $\pi^*(\eta)$'s form a partition of $\Delta_m$,
hence we have 
\(
G = \bigcup_{\eta \in \Mcal} G^*(\eta) \; .
\)

There exists an $\eta\in \Mcal$ such that $G^*(\eta)$ contains
infinite elements from the sequence $(\omega^{(s)})$.  Let
$(\omega^{(s_r)})$ be the subsequence of $(w^{(s)})$ that is
contained in $G^*(\eta)$.  Note that $G$ is sequentially compact
because $G$ is a compact subset of a metric space.  Therefore, there
exists a convergent subsequence $(\omega^{(s_{r_a})})$, and call its
limit $w'$.  In addition, we can always find $(\omega^{(s_{r_a})})$
such that $\lim_{a \to \infty} \eta \cdot \omega^{(s_{r_a})} = \eta \cdot w' > 0$, otherwise, the given sequence $(w^{(s)})$
cannot converge to a $w \in \Omega_\infty^+$ such that $\eta^w \cdot w >
0$. We claim that $w' \in G^*(\eta)$.

Let $G(\eta) \equiv \{w \in G | \, w \in \pi(\eta) \} = G \cap \pi(\eta)$, i.e., the \emph{closure} of $G^*(\eta)$.  
It follows that $G(\eta)$ is closed, because both sets involved in the
intersection are closed.  Also note that $G^*(\eta) \subset G(\eta)$.
The sequence $\{\omega^{(s_{r_a})} \}$ is therefore contained in $G(\eta)$, yielding $w' \in G(\eta)$.  Now, either $w' \in G^*(\eta)$ or $w' \in G(\eta) - G^*(\eta)$, the later containing only weights in which $\eta$ is tied with another 
element in $\Mcal$. The second case is impossible because 
Condition~\ref{C:NoTies} (No Key Ties) does not allow those kind of
ties,~\footnote{If Part 1 of Condition 3 holds, then the statement
  falls immediately. Otherwise, if Part 2 of Condition 3 holds, there
  is a non-key ``tie'' with another $\eta' \in \Mcal$ because $w'$ is such that
  $w'(i)=0$ wherever $\eta(i) \neq \eta'(i)$. But, at that point
  $\eta$ would have been preferred to $\eta'$; otherwise, the set
  $G^*(\eta)$ could not have been the set that contains infinite elements from the sequence $\{\omega^{(s)}\}$.} so we must conclude $w' \in G^*(\eta)$.

Now we proceed to show that $w' = \A(w)$.  From Eqn.~\ref{E:def w_i}, it is clear that there is a subsequence $(w^{(s_{r_a})})$ of $(w^{(s)})$ such that $w^{(s_{r_a})} = \A(\omega^{(s_{r_a})})$.  Whereby,
\begin{align} \label{E:w-side-left}
\lim_{a \to \infty} \A(\omega^{(s_{r_a})}) =  \lim_{a \to \infty} w^{(s_{r_a})}
= \lim_{s \to \infty} w^{(s)}
= w.
\end{align}
Because (a) $w'
\in G^*(\eta)$, (b) $\eta \cdot w' > 0$, and (c) the subsequence $(w^{(s_{r_a})})$ is also in
$G^*(\eta)$, by following a
proof akin to that of Theorem~\ref{T:continuity}, we can obtain that for all $i$,
\begin{align*} 
[\A(w')](i) \equiv & [\T_\eta(w')](i) \\
\equiv & 
\frac12 w'(i) \times
\left(\frac{1}{\eta\cdot
    w'}\right)^{\eta(i)}\left(\frac{1}{1-(\eta\cdot
    w')}\right)^{1-\eta(i)}\\
= & 
\lim_{a \to \infty} \frac12 \omega^{(s_{r_a})}(i) \times
\left(\frac{1}{\eta\cdot
    \omega^{(s_{r_a})}}\right)^{\eta(i)}\left(\frac{1}{1-(\eta\cdot
    \omega^{(s_{r_a})})}\right)^{1-\eta(i)}\\
= &
\lim_{a \to \infty} [\T_\eta(\omega^{(s_{r_a})})](i) = \lim_{a \to
    \infty} [\A(\omega^{(s_{r_a})})](i) \; .
\end{align*}
Hence, we have
\begin{align} \label{E:w-side-right}
\lim_{a \to \infty} \A(\omega^{(s_{r_a})}) = \A\left(\lim_{a \to \infty} \omega^{(s_{r_a})}\right) 
= \A(w').
\end{align}
Then, combining Eqns.~\ref{E:w-side-left} and~\ref{E:w-side-right}, we conclude that $\A(w') = w$.
\qed
\end{proof}

Given any limit point $w$ of $\Omega_{\infty}^+$, the previous lemma lets us construct an infinite orbit backwards from $w$ contained entirely in $G$, whereby $w \in \Omega_{\infty}^+$, giving us compactness.  The next theorem formalizes this.
\begin{theorem}{\bf ($\Omega_\infty^+$ is Compact.)} \label{T:compactness}
Suppose Conditions~\ref{cond:nathypo} (Natural Weak-Hypothesis Class),~\ref{assume:WeakLearn} (Weak Learning)
and~\ref{C:NoTies} (No Key Ties) hold. Then the set
$\Omega_{\infty}^+$ is compact.
\end{theorem}
\begin{proof}
Let $\{w^{(s)}\}$ be an arbitrary convergent sequence contained in
$\Omega_{\infty}^+$, and let $w = \lim_{s \to \infty} w^{(s)}$.  By
Lemma~\ref{L:compactness}, there exists a sequence
$\{w^{(s_1)}\} \subset \Omega_{\infty}^+$ converging to $w^{(1)}
\in G$ such that $\A(w^{(1)}) = w$.  However, notice that
$\{w^{(s_1)}\}$ also satisfies the hypothesis of
Lemma~\ref{L:compactness}.  Applying the lemma to
$\{w^{(s_1)}\}$, we get $\{w^{(s_2)}\} \subset
\Omega_{\infty}^+$ converging to $w^{(2)} \in G$ such that $\A(w^{(2)})
= w^{(1)}$, therefore $\A^{(2)}(w^{(2)}) = w$.  We can continue in
this way to generate $w^{(n)} \in G$ such that $\A^{(n)}(w^{(n)}) = w$
for any $n$.  Therefore, $w \in \A^{(n)}(\Delta_m)$ for all $n \in \N$
and we must conclude that $w \in \Omega_{\infty}^+$.  Because $w$ was
the limit of an arbitrary convergent sequence of $\Omega_{\infty}^+$, it
must be the case that $\Omega_{\infty}^+$ is compact.
\qed
\end{proof}

\subsection{Applying the Birkhoff Ergodic Theorem}

With Theorems~\ref{T:compactness} (Compactness of $\Omega_\infty^+$)
and~\ref{T:lower-bound+} (Lower Bound on $\epsilon_t$'s) in hand, we
can now proceed to apply the Krylov-Bugolyubov Theorem
(Theorem~\ref{Krylov-Bogolyubov}) in order to show the existence of a measure
over $\Omega_\infty^+$ under which $\A$ is 
measure-preserving (Proposition~\ref{A:1}). That in turn helps us
apply Brikhoff's Ergodic
Theorem (Theorem~\ref{T:Birkhoff}) to obtain an important technical 
result.

Proposition~\ref{A:1} helps us cover the first condition of
Theorem~\ref{T:Birkhoff}, that $\A$ is a measure-preserving dynamical
system for some measure $\mu$.  This proposition is sufficient to
yield Theorem~\ref{the:avgcvg1}, which captures one of our main results
of this paper. 
\begin{proposition}\label{A:1}
  Let $\Omega \equiv \Omega_\infty^+$ and 
   $\A_\Omega : \Omega \to \Omega$
  such that $\A_\Omega(w) \equiv \A(w)$ for all $w \in
  \Omega$.  
Under Conditions~\ref{cond:nathypo} (Natural Weak-Hypothesis
Class),~\ref{assume:WeakLearn} (Weak Learning), and~\ref{C:NoTies} (No
Ties), there exists a Borel probability
measure $\mu_\Omega$ on $\Omega$ with the property that $(\Omega, \Sigma_\Omega, \mu_\Omega, \A_\Omega)$ is a measure-preserving dynamical system.
\end{proposition}
\begin{proof}
Because Conditions~\ref{assume:WeakLearn} (Weak Learning)
and~\ref{C:NoTies} (No Key Ties) hold, we have that, by
Theorem~\ref{T:compactness}, the set $\Omega$ is a
compact and metrizable topological space.  Also, because
Conditions~\ref{cond:nathypo} (Natural Weak-Hypothesis Class) and~\ref{assume:WeakLearn} (Weak Learning)
hold, we have that, by Theorem~\ref{T:lower-bound+}, the version of the
AdaBoost
example-weights update over just $\Omega$, $\A_\Omega :
\Omega \to \Omega$, is a continuous map.  It follows
from the Krylov-Bogolyubov Theorem (Theorem~\ref{Krylov-Bogolyubov}) that $\A_\Omega$ admits an invariant Borel probability measure $\mu_\Omega$. \qed
\end{proof}

\begin{theorem}{{\bf (Averages over an AdaBoost Sequence of Example
      Weights Converge $\mu_{\Omega_\infty^+}$-Almost Everywhere.)}} 
\label{the:avgcvg1}
Suppose Conditions~\ref{cond:nathypo} (Natural Weak-Hypothesis
Class),~\ref{assume:WeakLearn} (Weak Learning), and~\ref{C:NoTies} (No
Ties) hold. Let $\Omega \equiv \Omega_\infty^+$ and $\mu \equiv \mu_\Omega$ be the (Borel) probabilistic measure on
$(\Omega,\Sigma_\Omega)$ referred to in Proposition~\ref{A:1}. For any function $f
\in L^1(\mu)$, the Optimal-AdaBoost update $\A$ has the property that
$\Tavg(f,w_1,\A,T) = \frac{1}{T} \sum_{t=0}^{T-1} f(\A^{(t)}(w_1))$ converges for $\mu$-almost
every $w_1 \in \Omega$.
\end{theorem}
\begin{proof}
The result follows from Proposition~\ref{A:1} and the Birkhoff Ergodic Theorem
(Theorem~\ref{T:Birkhoff}).
\qed
\end{proof}

\begin{definition}
  \label{D:nu_0}
For all $T$ and $w \in \Omega_\infty^+$, consider the set $\A^{(-T)}(w)$.
Note that, for each $w$, we have that $(\A^{(-T)}(w))$ is a sequence of sets that increases with $T$ to
\[ \A^{(-\infty)}(w) \equiv \bigcup_{T=0}^\infty \A^{(-T)}(w) \subset
  \Delta_m \; . \]
For all $T$, let
$(\Delta_m^+,\Sigma_{\Delta_m^+},\mu_0)$
be the \emph{uniform Borel
measure over $\Delta_m^+$.} 
For all $T$, define a measure space
$(\Delta_m^+,\Sigma_{\Delta_m^+},\nu_0^{(-T)})$ such that, for all $W \in
\Sigma_{\Delta_m^+}$,
\[
\nu_0^{(-T)}(W) \equiv \int_{w \in \Omega_{\infty}^+} \mu_0(W \cap
  \A^{(-T)}(w)) \, d\mu(w) \, 
\]
where $\mu \equiv \mu_{\Omega_\infty^+}$ is as in Theorem~\ref{the:avgcvg1}. Note that $\nu_0^{(-T)}$ is a
proper finite measure for all $T$.
Now define the
``limit'' measure space $(\Delta_m^+,\Sigma_{\Delta_m^+},\nu)$ such
that, for all $W \in \Sigma_{\Delta_m^+}$,
\begin{align*}
\nu_0(W) \equiv & \lim_T \nu_0^{(-T)}(W) = \lim_T
\int_{w \in \Omega_\infty^+} \mu_0(W \cap \A^{(-T)}(w)) \,
                  d\mu(w)  \\ = &
\int_{w \in \Omega_\infty^+} \lim_T \mu_0(W \cap \A^{(-T)}(w))
                \, d\mu(w) 
  =  \int_{w \in \Omega_\infty^+} \mu_0(W \cap \A^{(-\infty)}(w))
\, d\mu(w) \, .
\end{align*}
(The last two equalities follow from the Lebesgue Dominated
Convergence Theorem.)
\end{definition}
\begin{corollary}{{\bf (Averages over an AdaBoost Sequence of Example
      Weights Converge $\nu_0$-Almost Everywhere.)}} 
\label{cor:avgcvg_all_finite_T}
Theorem~\ref{the:avgcvg1} holds if we let 
$\Omega \equiv \Delta_m^+$ and $\mu \equiv \nu_0$, the finite measure on
$(\Delta_m^+,\Sigma_{\Delta_m^+})$ defined in terms of
$\mu_{\Omega_\infty^+}$ in Definition~\ref{D:nu_0}. 
\end{corollary}

\paragraph{Measurability of secondary quantities in $L_1$.}
The following proposition helps us cover the second condition of the Birkhoff
Ergodic Theorem (Theorem~\ref{T:Birkhoff}) in order to apply it within the context
of Optimal AdaBoost: the measurability of the
secondary quantities in $L_1$.
\begin{proposition}\label{A:3}
Let $(\Omega,\Sigma_\Omega,\mu)$ be a measure space
consisting of a Borel $\sigma$-algebra $\Sigma_\Omega$ on
$\Omega$.
Then the
following holds.
\begin{enumerate}
\item If $\Omega \subset \Delta_m$, 
  then
  the function $\epsilon(w) = \min_{\eta\in \Mcal}\eta \cdot w$ is in
  $L^1(\mu)$.
\item \label{A:3:alpha}
  If $\Omega \subset \Delta_m^{\epsilon_*}$, as defined in Corollary~\ref{cor:non0}, and Conditions~\ref{cond:nathypo} (Natural Weak-Hypothesis Class) and~\ref{assume:WeakLearn}
  (Weak Learning) hold, then the function $\alpha(w) = \frac 12
  \ln\left(\frac{1-\epsilon(w)}{\epsilon(w)}\right)$ is in $L^1(\mu)$.
\item If $\Omega \subset \Delta_m$, then the function $\chi_{\pi^*(\eta)}(w) = \indicator{w \in
    \pi^*(\eta)}$ is in $L^1(\mu)$.
\end{enumerate}
\end{proposition}
\begin{proof}
The following is the proof for each respective part of the proposition.
\begin{enumerate}
\item Because $\epsilon(w)$ is the minimum of a finite set of continuous functions, it follows that $\epsilon(w)$ is continuous as well.  In the case of a Borel algebra, continuity implies measurability.  
We also have 
\(
\int_{w\in \Omega} |\epsilon(w)| \; d\mu(w) \leq \int 1 \; d\mu = 1 \; .
\)
Therefore, $\epsilon(w) \in L^{1}(\mu)$.

\item Under Conditions~\ref{cond:nathypo} (Natural Weak-Hypothesis
  Class) and~\ref{assume:WeakLearn} (Weak Learning), because
  $\epsilon(w)$ is continuous and, by
  the definition of $\Delta_m^{\epsilon_*}$, bounded away from $0$, it follows that $\alpha(w)$ is continuous as well.  As above, this implies measurability.  From $\epsilon(w) \geq \epsilon_*>0$, where $\epsilon_*$ is as stated in Theorem~\ref{T:lower-bound+}, we have an upper bound on $\alpha(w)$ we will call $\alpha^*$.  Therefore, we then have
\(
\int_{w\in \Omega_{\infty}} |\alpha(w)| \; d\mu(w) \leq \int \alpha^* \; d\mu = 
\alpha^*.
\)
Whereby $\alpha(w) \in L^{1}(\mu)$.

\item The set $\pi^*(\eta)$, being a simple linear subset of
  $\Delta_m$, is in $\Sigma_{\Delta_m^+}$, thus measurable (i.e., it is
  composed of two disjoint measurable sets: $\pi^\circ(\eta)$ being an open subset
  of $\Delta_m^+$ and $\pi^*(\eta) -
  \pi^\circ(\eta)$ which is ether empty or a closed subset of $\Delta_m^+$. Hence, the characteristic function $\chi_{\pi^*(\eta)}(w)$ is measurable and bounded above by $1$.  Therefore, it is in $L^1(\mu)$.
\end{enumerate}
\qed
\end{proof}
Note that Part~\ref{A:3:alpha} of the proposition is a non-trivial statement
because there was no reason to believe \emph{a priori} that
$\epsilon(w)$ is bounded from below in the context of Optimal
AdaBoost, which we showed here in Theorem~\ref{T:lower-bound+}. 

\subsection{The flies in \emph{our} ointment}

As an anonymous
reviewer of a previous version of this paper pointed out, one issue with the result above is that it does not preclude
$\Omega_\infty^+ = \emptyset$.
Although we think this is just a
red-herring, because it would be counterintuitive otherwise, the reality is that we do not have a
formal mathematical proof showing that $\Omega_\infty^+ \neq
\emptyset$. We start by stating this as a condition, followed by
studying its implication.
\begin{condition}
  \label{C:NonEmptyOmega}
  $\Omega_\infty^+ \neq \emptyset$.
\end{condition}
\begin{proposition} 
  \label{P:T_tau}
  Suppose Condition~\ref{C:NonEmptyOmega} holds.
  Then, we have 
  \[
    \lim_{T \to \infty} \sup_{(w',w) \in \A^{(T)}(\Delta_m^+) \times
      \Omega_\infty^+} d(w',w) = 0 \; .
  \]
\end{proposition}
\begin{proof}
  By Condition~\ref{C:NonEmptyOmega}, we have that $\A^{(T)}(\Delta_m^+) \times
      \Omega_\infty^+$ is non-empty for all $T$, so that $d(w',w)$ exists for all
    $(w,w') \in \A^{(T)}(\Delta_m^+) \times
      \Omega_\infty^+$, and in turn, $\sup_{(w',w) \in \A^{(T)}(\Delta_m^+) \times
      \Omega_\infty^+} d(w',w)$ is a well defined non-increasing
    sequence, which converges (by the Monotone Convergence Theorem for
    sequences).
    We have
  \begin{align*}
    \lim_{T \to \infty} \sup_{(w',w) \in \A^{(T)}(\Delta_m^+) \times
    \Omega_\infty^+} d(w',w) & \leq \lim_{T \to \infty} \sup_{(w',w) \in \closure{\A^{(T)}(\Delta_m^+)} \times
                              \lim_t \closure{\A^{(t)}(\Delta_m^+)}} d(w',w) = 0
  \end{align*}
   because $A^{(T)}(\Delta_m^+) \subset \closure{A^{(T)}(\Delta_m^+)}$,
   $\Omega_\infty^+ \subset \lim_T \closure{A^{(T)}(\Delta_m^+)}$, and
   $\closure{A^{(T)}(\Delta_m^+)} \searrow \lim_T \closure{A^{(T)}(\Delta_m^+)}$.
  \qed
\end{proof}
\begin{definition}
  \label{D:nu}
Suppose Condition~\ref{C:NonEmptyOmega} holds. For all $T$, let $\A_T$ be a function from $w \in
\Omega_\infty^+$ to $\Sigma_{\Omega_\infty^+}$ such that, for all $w \in
\Omega_\infty^+$, 
\[ A_T(w) \equiv \A^{(-T)}\left( \left\{ w' \in \closure{\A^{(T)}(\Delta_m^+)} \, \left| \,
d(w',w) = \inf_{w'' \in \Omega_\infty^+} d(w',w'') \right. \right\} \right) \;
. \]
Note that, for each $w$, we have that $(A_T(w))$ is a sequence of sets that increases with $T$ to
$A_\infty(w) \equiv \bigcup_{T=0}^\infty A_T(w) \subset \Delta_m$. Let $(\Delta_m^+,\Sigma_{\Delta_m^+},\mu_0)$ be the \emph{uniform Borel
measure over $\Delta_m^+$.} 
For all $T$, define a measure space
$(\Delta_m^+,\Sigma_{\Delta_m^+},\nu_T)$ such that, for all $W \in
\Sigma_{\Delta_m^+}$,
\[
\nu_T(W) \equiv \int_{w \in \Omega_{\infty}^+} \mu_0(W \cap A_T(w)) \, d\mu(w) \, 
\]
where $\mu \equiv \mu_{\Omega_\infty^+}$ is as in Theorem~\ref{the:avgcvg1}. Note that $v_T$ is a
proper finite measure for all $T$. Now define the
``limit'' measure space $(\Delta_m^+,\Sigma_{\Delta_m^+},\nu)$ such
that, for all $W \in \Sigma_{\Delta_m^+}$,
\begin{align*}
\nu(W) \equiv & \lim_T \nu_T(W) = \lim_T
\int_{w \in \Omega_\infty^+} \mu_0(W \cap A_T(w)) \, d\mu(w)  = 
\int_{w \in \Omega_\infty^+} \lim_T \mu_0(W \cap A_T(w))
                \, d\mu(w) \\
  = & \int_{w \in \Omega_\infty^+} \mu_0(W \cap A_\infty(w))
\, d\mu(w) \, .
\end{align*}
(The last two equalities follow from the Lebesgue Dominated
Convergence Theorem.)
\end{definition}
\begin{condition} {{\bf (AdaBoost is Sufficiently Non-Expansive Globally)}}
  \label{C:NonExpansiveGlobal}
  For all $w \in \Delta_m^+$ there exists a $\tau' > 0$ such that for
  all $\tau \in (0,\tau')$, if $d(\A^{(T)}(w), w') < \tau$ for
  some $T$ and $w' \in \Delta_m^+$
  then $d(\A^{(T+t)}(w),\A^{(t)}(w')) < \tau$ for all $t > 0$.
\end{condition}
\begin{definition}
  \label{def:ContFncs}
For any set $W \subset \Delta_m$, denote by $C(W)$ the set of all continuous functions $f$ of type $\Delta_m \to
\R$ that are continuous in $W$.
\end{definition}
Note that all the functions in $C(\Delta_m)$ are also \emph{absolutely}
continuous because $\Delta_m$ is compact.
\begin{lemma}
  \label{L:mirror}
Suppose Conditions~\ref{C:NonEmptyOmega} and~\ref{C:NonExpansiveGlobal} hold. Then we have that for
all $w' \in \Delta_m^+$, there exists $w \in \Omega_\infty^+$ such
that \[
  \lim_{T \to \infty} \left| \Tavg(f,w,\A,T) -
  \Tavg(f,w',\A,T) \right| = 0 \] for all $f \in C(\Delta_m)$.
\end{lemma}
\begin{proof}
  Let $w_t \equiv \A^{(t)}(w)$. By Proposition~\ref{P:T_tau},  we can
  consider an arbitrary $\tau > 0$, and with that a round $T_0 \equiv
  T_0(\tau)$ and an example weight $w^{(T_0)} \in \Omega_\infty^+$ such that
  $d(w_{T_0}, w^{(T_0)}) < \tau$. Let $w^{(t)} \equiv
  \A^{(t-T_0)}(w^{(T_0)})$ for all $t > T_0$; and for all $t <
  T_0$  let it be such that
  $w^{(t)} \in \A^{(t-T_0)}(w^{(T_0)}) \cap \Omega_\infty^+$. Then by Condition~\ref{C:NonExpansiveGlobal} we have
  $d(w_t,w^{(t)}) < \tau$ for all $t > T_0$. Let $w'
  \equiv w^{(0)}$. For all $T > T_0$ we have
  \begin{align*}
  & \left| \Tavg(f,w,\A,T) -
    \Tavg(f,w',\A,T) \right| \leq \frac{1}{T} \sum_{t=0}^{T-1} \left| f(w_t) -
                               f(w^{(t)}) \right| \\
    & = \frac{1}{T} \left( \sum_{t=0}^{T_0-1} \left| f(w_t) -
                               f(w^{(t)}) \right| + \sum_{t=T_0}^{T-1} \left| f(w_t) -
                               f(w^{(t)}) \right| \right) \; .
  \end{align*}
  The result follows by noting that for all $\tau' > 0$ we can
  always pick $\tau < \tau'$ such that 
  \begin{align*}
  \sum_{t=T_0}^{T-1} \left| f(w_t) -
    f(w^{(t)}) \right| < \tau' (T - T_0)
  \end{align*}
  because $f$ is absolutely continuous.
  \qed
  
\end{proof}
\begin{corollary}{{\bf (Averages over an AdaBoost Sequence of Example
      Weights Converge $\nu$-Almost Everywhere.)}} 
\label{cor:avgcvg_all}
Suppose Conditions~\ref{cond:nathypo} (Natural Weak-Hypothesis
Class),~\ref{assume:WeakLearn} (Weak Learning), and~\ref{C:NoTies} (No
Ties) hold. Suppose, in addition, that
Conditions~\ref{C:NonEmptyOmega} ($\Omega_\infty^+ \neq \emptyset$) and~\ref{C:NonExpansiveGlobal}
(Non-Expansive) hold.  Let
$\Omega \equiv \Omega_\infty^+$ and $\nu$ be the finite measure on
$(\Delta_m^+,\Sigma_{\Delta_m^+})$ defined in terms of
$\mu_\Omega$ in Definition~\ref{D:nu}. For any function $f
\in C(\Delta_m)$, the Optimal-AdaBoost update $\A$ has the property that
$\widehat{f}_T^\A(w_1) = \frac{1}{T} \sum_{t=0}^{T-1} f(\A^{(t)}(w_1))$ converges for $\nu$-almost
every $w_1 \in \Delta_m^+$.
\end{corollary}
\begin{proof}
Let $\Ncal \equiv \{ w_1 \in \Omega_\infty^+ \, \mid \, \widehat{f}_T^\A(w_1)
\text{ diverges} \}$ and $\Ncal' \equiv \{ w'_1 \in \Delta_m^+ \, \mid \, \widehat{f}_T^\A(w'_1)
\text{ diverges} \}$. We can express 
  \begin{align*}
\nu(\Ncal') = & \int_{w \in \Omega} \chi_\Ncal(w) \mu_0(\Ncal'
    \cap A_\infty(w)) 
\, d\mu(w) + \int_{w \in \Omega} \chi_{\Omega-\Ncal} (w) \mu_0(\Ncal'
    \cap A_\infty(w))
                \, d\mu(w) \; .
    \end{align*}
The second term equals $0$ because $\Ncal'
    \cap A_\infty(w) = \emptyset$ for all $w \in \Omega-\Ncal$ by
    Lemma~\ref{L:mirror}. Hence, considering only the first term, we
    obtain
  \begin{align*}
\nu(\Ncal') \leq & \int_{w \in \Omega} \chi_\Ncal(w) \mu_0(\Ncal')
                \, d\mu(w) = \mu_0(\Ncal') \int_{w \in \Omega} \chi_\Ncal(w)
                   \, d\mu(w) \\
    = &  \mu_0(\Ncal') \mu(\Ncal) = \mu_0(\Ncal')
                   \cdot 0 = 0\; .
    \end{align*}    
  \qed 
\end{proof}

One way to avoid Condition~\ref{C:NonEmptyOmega} ($\Omega_\infty^+ \neq \emptyset$) is to replace Condition~\ref{C:NoTies} (no ties \emph{in the
  limit}), with the following seemingly weaker (no ties \emph{after a finite time}).
\begin{condition}{{\bf (Optimal AdaBoost has No Important Ties Eventually.) }} \label{C:NoTies2}
There exists a compact set $G$ and a round $T$ such that
$\A^{(T)}(\Delta_m^+) \subset G$ and, given any pair $\eta,\eta' \in \Mcal$, we have either
\begin{enumerate}
\item $\pi(\eta)\cap\pi(\eta')\cap G = \emptyset$; or
\item for all $w \in G$, $\sum_{i: \eta(i) \neq \eta'(i)} w(i) = 0$.
\end{enumerate}
\end{condition}
While, on the surface, Condition~\ref{C:NoTies2} seems weaker than
Condition~\ref{C:NoTies}, they might turn out to be equivalent in our
context, but we currently lack a formal mathematical proof of that.
Note that, in this case, we can take $G =
\closure{\A^{(T)}(\Delta_m^+)}$, the \emph{closure} of
$\A^{(T)}(\Delta_m^+)$. A specific example of when the condition holds
is the case of mistake dichotomies isomorphic to an $(m \times m)$
identity matrix, discussed in Appendix~\ref{app:ident}.

This variant of the no-ties condition allows an exact
characterization of $\Omega_\infty^+$ with simpler proofs.
\begin{theorem}{\bf ($\Omega_\infty^+$ is the Limiting Closure of the
    AdaBoost Update.)} \label{T:compactness2}
Suppose Conditions~\ref{assume:WeakLearn} (Weak Learning)
and~\ref{C:NoTies2} (No Key Ties Eventually) hold. Then the set
$\Omega_{\infty}^+ = \lim_{t \to \infty} \closure{\A^{(t)}(\Delta_m^+)} =
\bigcap_{t=1}^\infty \closure{\A^{(t)}(\Delta_m^+)}$, thus compact and non-empty.
\end{theorem}
\begin{proof}
  Let $\Omega_r^+
\equiv \A^{(r)}(\Delta_m^+) = \bigcap_{t=1}^r \A^{(t)}(\Delta_m^+)$ and $E_r \equiv
\A^{(r)}(\closure{\Omega_T^+})$ for $t = 1,2\ldots$, where
$T$ is as in Condition~\ref{C:NoTies}, and note that, under
the given conditions, we have that
$E \equiv \bigcap_{r=1}^\infty E_r$ is a non-empty subset
of $G$ also containing $\Omega_\infty^+$.
We also have $E_t \subset
\A^{(t)}(\Delta_m^+)$, so that $E \subset \Omega_\infty^+$, and hence,
$\Omega_\infty^+ = E$ and non-empty.  \qed
\end{proof}
\begin{condition} {{\bf (AdaBoost is Sufficiently Non-Expansive Locally)}}
  \label{C:NonExpansive}
  For all $w \in \Delta_m^+$ there exists a $\tau' > 0$ such that for
  all $\tau \in (0,\tau')$, if $d(\A^{(T)}(w), w') < \tau$ for
  some $T$ and $w' \in \Delta_m^+ \cap \pi(\eta^w)$
  then $d(\A^{(T+t)}(w),\A^{(t)}(w')) < \tau$ for all $t > 0$.
\end{condition}
\begin{definition}
  \label{def:C_McalPlus}
Denote by $C_\Mcal^+$ the set of all functions $f$ of type $\Delta_m \to
\R$
that are 
continuous on each
$\pi^\circ(\eta)$ individually when
viewed as a function of type
$\pi^{\circ}(\eta) \to \R$
for all $\eta \in \Mcal$.~\footnote{Technically, the functions only
  need to be continuous on each $\pi^\circ(\eta) \cap \cup_{\eta'}
  \pi_{\frac12}(\eta')$ individually.}
\end{definition}
\begin{corollary}{{\bf (Averages over an AdaBoost Sequence of Example
      Weights Converge $\nu$-Almost Everywhere if No Ties Eventually)}} 
\label{cor:avgcvg_all_ft}
Corollary~\ref{cor:avgcvg_all} holds if
Conditions~\ref{C:NoTies} (No Ties in the Limit)
and~\ref{C:NonEmptyOmega} ($\Omega_\infty^+ \neq \emptyset$) are replaced
with Condition~\ref{C:NoTies2} (No Ties Eventually);
Condition~\ref{C:NonExpansiveGlobal} (Globally Non-Expansive) is replaced by
Condition~\ref{C:NonExpansive} (Locally Non-Expansive); and the set of functions $C(\Delta_m)$
(Definition~\ref{def:ContFncs}) is replaced by $C_\Mcal^+$
(Definition~\ref{def:C_McalPlus}). 
\end{corollary}

But we have another potentially more critical ``fly in the ointment.'' We would like to establish that the convergence of the time average
$\Tavg(f,w_1,\A,T)$ of any measurable function, in $L_1$,
$C(\Delta_m)$, or $C_\Mcal^+$, of the weights
generated by AdaBoost (see Definition~\ref{def:emp_meas_and_avgs}) converges as $T \to \infty$ starting from
\emph{Lebesgue}-almost every initial weight $w_1 \in \Delta_m^+$. Unfortunately, as pointed out by an anonymous
reviewer of a previous version of this paper, the results stated in Theorem~\ref{the:avgcvg1} and Corollaries~\ref{cor:avgcvg_all} and~\ref{cor:avgcvg_all_ft}) only
imply the convergence of time averages starting from
$\mu$-almost or $\nu$-almost every initial weight $w_1 \in \Delta_m^+$. Under
Condition~\ref{C:NoTies2} (No Key Ties Eventually), it is
relatively straightforward to tie the convergence of any average of
any function that is continuous on the interior of each $\pi(\eta)$
starting from any $w_1 \in \Delta_m^+$ to
that of an initial weight in $\Omega_\infty^+$ as we have done above
in Lemma~\ref{L:mirror}. But doing so does not
really solve the
problem because, at least base on a straight, ``on the surface''
application of Birkoff's, we would still have a better 
understanding of the \emph{support} of the
measure on $\Omega_\infty^+$ in order to be able to say something meaningful about the starting
points for which the AdaBoost
generated averages will converge.

A more in-depth study of the literature on dynamical systems and
ergodic theory~\citep[see, e.g.,][]{OxtobyUlam1939,OxtobyUlam1941,oxtoby1952,sigmund74},
 especially some of the most recent literature~\citep{Abdenur2013,Blank_2017,2018arXiv181104805C,dong_oprocha_tian_2018}, led us to an alternative
approach to this problem, and eventually an alternative understanding of the
behavior of Optima AdaBoost, including connections to the existing conjectures
on cycling~\citep{DBLP:journals/jmlr/RudinSD12} and
ergodicity~\citep[Section 9.1]{BreimanInfinite}.  
We adapted relatively recent results in dynamical
systems about when one can extend the results of the Birkhoff Ergodic
Theorem, from simply almost-sure convergence with respect to the
invariant measure to that with respect to the standard Lebesgue/Borel
measure~\citep{Abdenur2013,2018arXiv181104805C}, to our context.

But before we continue, let us summarize what we know up to this point
and describe the challenges we still face moving forward with the
approach we have used so far: applying directly and non-constructively
the Birkhoff Ergodic Theorem
(Theorem~\ref{T:Birkhoff}) via the Krylov-Bogolyubov Theorem (Theorem~\ref{Krylov-Bogolyubov}).  Ideally, we would
like to say that the support of $\mu_\Omega$ (Proposition~\ref{A:1}) includes almost every
$w \in \Omega$
with respect to the standard Lebesgue/Borel
measure. That way, using the result of the last proposition and
applying the Birkhoff Ergodic Theorem for $\A_\Omega$ on $\Omega$ would
show that the time average $\Tavg(f,w_1,\A_\Omega,T)$ of any
measurable
function $f \in L_1$ of the example weights would converge as $T \to \infty$.
The Birkhoff Ergodic Theorem establishes the
equivalence of time and space averages with respect to a given measure
under some conditions.

In the literature on dynamical systems and ergodic theory, roughly
speaking, an
empirical measure for which the time average of every continuous
function $f$ converges is called a {\em pseudo-physical measure}, and if that convergence is to its corresponding space average, a {\em
  physical measure}. A physical measure is pseudo-physical but the
opposite does not always hold. Our main interest from the Birkhoff
Ergodic Theorem is the convergence of time averages, and thus,
pseudo-physical measures. In fact, it is very common to find dynamical systems
defined by a continuous self-maps on a \emph{compact connected} domain
$W$ are
\emph{weird}, a formal term defined
by~\citet{Abdenur2013}, which essentially means
that time averages exists for Lebesgue-almost every $\omega \in W$,
but no physical measure exists (i.e., roughly, time and state averages never match).

The following simple example~\citep[pg. 4650]{Blank_2017} beautifully illustrates the distinction between time and space averages. Consider the
discontinuous map
$M : [0,1] \to [0,1]$ defined as $M(\omega) = \omega/2$ if $\omega \in (0,1]$ and $ =
1$ if $\omega = 0$. The empirical 
measure $\widehat{\mu}_\omega^{(T)}$
converges to
$\delta_{0}$, the Dirac-delta measure having only support at $0$, for
any $\omega \in [0,1]$: i.e., the Birkhoff limit $\widehat{\mu}_\omega$ exists and is
$M$-invariant for all
$\omega \in [0,1]$, and the dynamical system $([0,1],\Sigma_{[0,1]},M,\widehat{\mu}_\omega)$ is uniquely ergodic. Consider the identity function $f(\omega) =
\omega$, which is 
continuous and $L_1$ measurable.  
The time
average $\widehat{f}_\omega$ converges to $2 \omega$ for all $\omega \in (0,1]$ and to $2$ for $\omega =
0$, while the space average $\bar{f}$ is $1$. (Note that, in contrast,
the Birkhoff
Ergodic Theorem would only guarantee the convergence of the time
average to the space average for $\omega = 0$.) Hence, the empirical measure is
pseudo-physical but not physical, because the time average converges to
a value equal to the
space average only when started at $\omega \in \{0,1\}$. 

So, while intuition may suggest that the time averages converge from any
starting $w_1 \in \Omega_\infty^+$ in the case of the Optimal AdaBoost
update $\A$, proving this formally seems difficult. Even if that is possible, we would still have to formally show
that this convergence extends to Lebesgue-almost every $w_1 \in \Delta_m^+$. Once again,
while intuition may suggest that this is true, a formal proof also
seems difficult. For instance, even under Condition~\ref{C:NoTies2}, we would have to show that for any
measurable set $\Wcal \in \Sigma_{\Delta_m^+}$ of positive Borel
measure, we have $\closure{\A^{(T)}(\Wcal)}$ being a subset of
$\Omega$ of positive measure: i.e., under Condition~\ref{C:NoTies2},
and where $T$ is as in that condition, 
$\mu_\Omega(\closure{\A^{(T)}(\Wcal)}) > 0$,  for all measurable sets
$\Wcal \in \Sigma_{\Delta_m^+}$. While that seems relatively easier, it
still seems difficult overall.
In what follows, we will provide
a constructive proof of the existence of $\mu_\Omega$ that essentially
proves just that in Section~\ref{sec:cycle}.

The Krylov-Bugolyubov Theorem
is in fact strongly connected to the Birkhoff
Ergodic Theorem~\citep{oxtoby1952}. Its proof typically consists of two parts: (1)
establishing the existence of a Borel probability measure such that the
time average of any continuous function equals its space
average with respect to the measure; and (2) showing that the measure
is invariant with respect to the map. It turns out that only the second part requires the
continuity of the map. The typical argument for the first part, using notation in the context of the Optimal
AdaBoost update, begins
by defining a sequence of time-indexed empirical measures $\widehat{\mu}^{(T)}_{w_1}$
starting from an
arbitrary initial example-weight $w_1$ in the compact set $\Omega$. Using sophisticated
mathematical tools (e.g., the Riezs Theorem, also known as the Riezs-Fischer
Theorem, as well as properties of the sets of 
Borel probability measures and weak-*
convergence), the proof would continue by showing that for any map $\A_\Omega$ on a compact set
$\Omega$ to
itself there exists
a subsequence $(\mu^{(T_s)}_{w_1})$ converging to
$\widehat{\mu}_{w_1}$ such that for any
continuous function $f$, we have $\Savg(f,\A_\Omega,\widehat{\mu}_{w_1}) =
\lim_{s \to \infty} \Tavg(f,w_1,\A_\Omega,T_s)$ (see Definition~\ref{def:emp_meas_and_avgs}).
Note that this means that a
{\em subsequence} of the time averages converges, {\em not} that the original full
sequence of time averages does; otherwise, we would have been done: because
we could take an arbitrarily large compact subset of $\Delta_m^+$, and
then have $w_1$ in that set,
instead of $\Omega$.

Because the~\citet{Birkhoff1931} Ergodic Theorem is close to a century old, one would expect
that the question of establishing conditions under which one can extend convergence from the set of
positive measure with respect to the invariant measure to all or
almost all points in the original space of the map has been addressed 
long ago. Yet, to our surprise, it appears that this
question was actually addressed only rather
recently~\citep{Abdenur2013,Blank_2017,2018arXiv181104805C,dong_oprocha_tian_2018}. This
seems remarkable at first, given the importance of this theorem to ergodic
theory and dynamical systems~\citep{Moore2015}. The sense one gets from the
early literature following the publication of the ergodic theorem may
explain why. Indeed, some of the early results essentially establish
that, for many functions (i.e., uniformly continuous, or bounded
functions for which the set of discontinuities has Borel measure zero), the set of points for which
convergence of time averages exists is large, measure-theoretically
speaking: i.e., convergence occurs for almost all initial
points~\citep[see, e.g.,][]{oxtoby1952}. However, perhaps because of the apparent interest at the time, most results
specifically concern maps that are continuous homeomorphisms~\citep[see, e.g.,][]{OxtobyUlam1939,OxtobyUlam1941,oxtoby1952,sigmund74}. None of
those conditions hold for the Optimal AdaBoost update $\A$.

Even the recent results have limitations as they only concern continuous
maps and homeomorphisms on compact \emph{connected}
manifolds~\citep{Abdenur2013,2018arXiv181104805C}. Note that, under
Condition~\ref{C:NoTies2}, Optimal AdaBoost is continuous on $\Omega \equiv
\closure{\A^{(T)}(\Delta_m^+)}$, a compact set, that set is clearly no
connected because the ties are not there. 
In addition, they do not establish convergence of the time
averages with respect to every continuous map. Instead, they show that for
every continuous map there exists another continuous function that
uniformly approximates it and for which convergence of time averages
occurs. (Or in the parlance of ergodic theory and dynamical systems,
they hold for so called ``typical,'' or ``generic,'' continuous maps.) This appears to be a fundamental limitation in general. Yet, it
turns out that the main ideas behind the proof could be adapted
specifically for the Optimal AdaBoost update, as we will show. In
fact, 
this approach
essentially establishes the no ties condition, albeit only by
construction,
and
leads to what seems like a proof of the conjectures that AdaBoost always
converges to (almost) a cycle~\citep{DBLP:journals/jmlr/RudinSD12} and that it is
ergodic~\citep[Section 9.1]{BreimanInfinite}, to our very pleasant surprise.  Indeed, we show that the Optimal AdaBoost update
can be uniformly approximated by a sequence of continuous maps, each
converging in finite time to a cycle of sets arbitrarily close to a cycle of length at
most $n$, and whose time averages of essentially any function with
certain type of local continuity properties also always converge.
  The construction we design here provides a constructive proof of
  exitence of an ergodic invariant measure and reveals that the actual Optimal
AdaBoost update exhibits the same cycling and ergodic properties if $\A$ satisfies certain ``non-expansion''
property.

\subsection{Optimal AdaBoost cycling behavior}
\label{sec:cycle}

In this subsection, we consider approximations of Optimal AdaBoost and
show that they all exhibit cycling behavior.

\subsubsection{Finite-precision (Discretized) Optimal AdaBoost converges to a cycle}

Let us begin by using a simple approximation of Optimal AdaBoost
based on a discretization of $\Delta_m$ induced by a discretization of
each $\pi(\eta)$ for all $\eta \in \Mcal$. Despite its simplicity,
this construction reveals key insights into establishing cycling
behavior in the other more sophisticated constructions we pursue later
in this section. It also serves as a nice warmup to those more
sophisticated constructions.

\begin{definition}{\bf (Finite-Precision Optimal AdaBoost)}
  \label{def:FPAB}
  Let $\tau > 0$. Define \[ \approxpi^\tau(\eta) \equiv \{ w^{\tau,\eta,1},
w^{\tau,\eta,2}, \ldots, w^{\tau,\eta,N_{\tau,\eta}}\} \subset
\pi^\circ(\eta) \] as a finite set of points, of
minimal cardinality $N_{\tau,\eta}$, that are ``uniformly
distributed'' over $\pi(\eta)$ such that the sets \[ R_{\tau,\eta,j}
  \equiv \left\{ w \in \pi(\eta) \, \left| \, 
d(w,w^{\tau,\eta,j}) = \min_{j' \in [N_{\tau,\eta}]}
d(w,w^{\tau,\eta,j'}) \right. \right\} \] form a covering of $\pi(\eta)$ (i.e., $\pi(\eta) =
\bigcup_j R_{\tau,\eta,j}$) with the following property.
$\mathrm{diam}(R_{\tau,\eta,j}) \equiv \sup \{ d(w,w^{\tau,\eta,j}) |
  w \in R_{\tau,\eta,j} \} < \tau$.  Let $R^\circ_{\tau,\eta,j} \equiv
\mathrm{Int}(R_{\tau,\eta,j})$. Note from the construction of the
covering that $R^\circ_{\tau,\eta,j} \cap
R^\circ_{\tau,\eta,j} = \emptyset$ for all $(j,j'), j\neq j'$. 
Define \[ \approxDelta_m^\tau \equiv \bigcup_{\eta \in \Mcal}
\approxpi^\tau(\eta) \] the resulting discretization of
$\Delta_m$. For every $(\tau,\eta)$, impose a fixed, but arbitrary preference
order on the sets $(R_{\tau,\eta,j})$ such that $j \succ j'$, for $j
\neq j'$, indicates that the set with index $j$ is preferred to that
with index $j'$, and let 
$R^*_{\tau,\eta,j} \equiv R_{\tau,\eta,j} - \bigcup_{j' \succ j}
\left( R_{\tau,\eta,j'} \bigcap R_{\tau,\eta,j} \right)$.
Note
that the sets $R^*_{\tau,\eta,j}$, for $j =
1,2,\ldots,N_{\tau,\eta}$, form a partition of $\pi(\eta)$ (i.e.,
$R^*_{\tau,\eta,j} \bigcap R^*_{\tau,\eta,j'} = \emptyset$, for all $(j,j'), j\neq
j'$, and $\pi(\eta) = \bigcup_j R^*_{\tau,\eta,j}$).  Because we
will be considering a sequence of approximations produced by a strictly
monotonically decreasing sequence
$(\tau_k)$, for technical
reasons, we require that $\pi^{\tau_k}(\eta) \subset
\pi^{\tau_{k+1}}(\eta)$. We can achieve that by employing a kind of hierarchical
discretization: i.e., to obtain a finer discretization, discretize each
$R_{\tau_k,\eta,j}$ individually, by adding new points in its interior
to the discretization, so that we obtain 
$\pi^{\tau_{k+1}}(\eta)$ after inserting those new points leading to the finer-grain
discretization into $\pi^{\tau_k}(\eta)$.
Now define a projection operator $\Proj_{\approxDelta_m^\tau} : \Delta_m \to
\approxDelta_m^\tau$ such that $\Proj_{\approxDelta_m^\tau}(w) \equiv w^{\tau,\eta,j}$
if $w \in R^*_{\tau,\eta,j}$,
and define the map 
$\approxA_{\tau,\mathrm{disc}} : \Delta_m \to \Delta_m$ such
that $\approxA_{\tau,\mathrm{disc}}(w) \equiv \Proj_{\approxDelta_m^\tau}(\A(w))$. Let $w_1 = \Proj_{\approxDelta_m^\tau}(w)$
for some $w \in \Delta_m^\circ$, so that effectively
$\approxA_{\tau,\mathrm{disc}}$ is a map from $\approxDelta_m^\tau$ to
itself. Let us call the corresponding
algorithm {\em $\tau$-Finite-Precision Optimal AdaBoost}, given that this is
how AdaBoost would essentially behave in a finite-precision computer.
\end{definition}
The following follows from the construction
just described.
\begin{theorem}{\bf (Properties of Finite-Precision Optimal
    AdaBoost.)}
  \label{thm:FPAB}
Suppose Conditions~\ref{cond:nathypo} (Natural Weak-Hypothesis
Class) and ~\ref{assume:WeakLearn} (Weak Learning) hold.
\begin{enumerate}
\item {\bf (Approximates Optimal AdaBoost)} The sequence
  of functions $(\approxA_{\tau,\mathrm{disc}})$ converges to $\A$ on $\Delta_m$.
\item For every
 $\tau > 0$, starting from any $w_1 = \Proj_\tau(w)
 \in
\approxDelta_m^\tau$ for some $w \in \Delta_m^\circ$, the following
holds. 
Let $\approxA_\tau \equiv \approxA_{\tau,\mathrm{disc}}$. Let $w_{t+1} \equiv
\approxA_\tau^{(t)}(w_1)$ and $\eta_t \equiv
\eta^{w_t}$ for all $t$, and consider the sequence of
example weights $(w_t)$.
\begin{enumerate}
\item {\bf (Never has Ties)} $w_t \in
\pi^\circ\left(\eta_t \right)$ for all
$t$.
 \item {\bf (Always Converges to a Cycle)} The sequence $(w_t)$
   converges in finite time to a
cycle in $\approxDelta_m^\tau$, with the precise cycle depending on $w_1$.
\item {\bf (Is Always Ergodic)} Let $T_\tau \equiv T_1(w_1,\tau)$ be the first time
  that the sequence $(w_t)$ enters a cycle of period $p_\tau \equiv
  p(w_1,\tau)$, $1 < p_\tau \leq n$, and define
  \begin{align*}
    \widehat{\mu}^\tau \equiv \widehat{\mu}_{\mathrm{disc},w_1}^\tau \equiv \frac{1}{p_\tau} \sum_{s=0}^{p_\tau-1}
    \delta_{w_{T_\tau+s}} = \widehat{\mu}_{\approxA_\tau,w_1} \; . 
  \end{align*}
 The dynamical system $(\approxDelta_m^\tau,
 2^{\approxDelta_m^\tau}, \approxA_\tau,\widehat{\mu}^\tau)$ corresponding to $\tau$-Finite-Precision
 Optimal AdaBoost is ergodic.
 \item {\bf (Time Averages 
    Always Converge)} The time average $\Tavg(f,w_1,\approxA_{\tau},T)$ of any function $f$ over the elements of the
  sequence $(w_t)$ converges to \(  \frac{1}{p_\tau} \sum_{s=0}^{p_\tau-1}
    f(w_{T_\tau+s}) \; . \)
\end{enumerate}
\end{enumerate}
\end{theorem}
\begin{proof}
Part 1 follows because $(\approxDelta_m^\tau)$ converges to $\Delta_m$
and $\approxA_{\tau,
    \mathrm{disc}} = \A$ on $\approxDelta_m^\tau$. Part 2.a follows
  immediately from the construction: $\approxDelta_m^\tau \subset
  \pi^\circ(\eta)$ and only points in $\approxDelta_m^\tau$ are
  visited. Part 2.b follows by the Pigeonhole Principle: the map is
  deterministic and only a finite number of points can be visited so
  that $(w_t)$ enters a cycle as soon as a point is revisited. Part
  2.c follows from 2.b and the definition of ergodicity
  (Definition~\ref{def:ergodic}): the empirical measure is a probability mass
  function with positive mass only on the elements of the cycle and 
  the cycle is an invariant set. Part 2.d follows from 2.b and
  Proposition~\ref{pro:cycles}.\qed
\end{proof}

\subsubsection{Almost uniform approximations of Optimal AdaBoost by step
  functions that converge to a cycle}

The approximation leading to the $\tau$-Finite-Precision Optimal
AdaBoost update is not uniform. An alternative is to consider an
approximation of Optimal AdaBoost that is \emph{Lebesgue-almost} uniform and consistent with
an infinite precision computational model. In such a model, we would
be considering example weights in the set $\Delta_m \bigcap \Q^m$, where
$Q$ denotes the set of rational numbers.
\begin{definition}
We say an example weight $w$ is \emph{rational} if all its components $w(i)$
are rational numbers: i.e., $w \in \Delta_m \bigcap \Q^m$. We say that $w$ is \emph{irrational} if at least one
component is an irrational number: i.e., $w \in \Delta_m - \Q^m$. 
\end{definition}
Starting from a rational $w_1$, each example weight $w_t$ that 
Optimal AdaBoost generates
would also be rational.
Note
however that this may not be true in the limit: i.e., the
sequence of example-weights $(w_t)$ may converge to a set of
irrational weights; said differently, the set $\Omega_\infty^+$ may consist
of irrational weights entirely. In fact, this always happens for $m=3$
examples, where every mistake matrix is isomorphic to the $3\times 3$
identity matrix: the system globally converges to one of two cycles of
period $3$, each cycle consisting of $3$ irrational weights, with two
of their components involving the golden ratio (see Appendix~\ref{app:ident}).

In the infinite-precision computational model, the strictly
monotonically decreasing sequence
$(\tau_k)$ used for the covering of $\Delta_m$ should consist of a rational number. For example, we can let
$\tau_k = 1/k$.
\begin{definition}{\bf (Step-Function Optimal AdaBoost)}
  \label{def:SFAB}
  Let $(\Delta_m,d)$ be a metric space and $\Sigma_{\Delta_m}$ be the
  Borel $\sigma$-algebra, with respect to $(\Delta_m,d)$.  Consider
  the measure space
  $(\Delta_m,\Sigma_{\Delta_m},\bar{\mu}_{\mathrm{Leb}})$, where
  $\bar{\mu}_{\mathrm{Leb}}$ is the uniform Borel probability measure:
  i.e., if $\mu_{\mathrm{Leb}}$ is the (standard) Lebesgue/Borel
  measure,
  $\bar{\mu}_{\mathrm{Leb}}(\Wcal) \equiv \mu_{\mathrm{Leb}}(\Wcal) /
  \mu_{\mathrm{Leb}}(\Delta_m)$ for each measurable set
  $\Wcal \in \Sigma_{\Delta_m}$. Assume, without loss of generality,
  that $\pi^*(\eta) \neq \emptyset$; otherwise $\eta$ is ``dominated''
  or ``consistently non-preferred'' and therefore Optimal AdaBoost
  would never select it.  For each $\eta \in \Mcal$, let
  $\omega^\eta \in \pi^\circ(\eta)$ be a ``centroid'' of $\pi(\eta)$: i.e.,
  \[ \omega^\eta \in \argmax_{w \in \pi(\eta)}
  d(w,\mathrm{Bnd}(\pi(\eta))) \equiv \sup \{ d(w,w') | w' \in
  \mathrm{Bnd}(\pi(\eta)) \} \; . \] 
Let $\theta > 0$ and define, for each $\eta \in \Mcal$, the set $\pi_{1-\theta}(\eta) \equiv
  \{ (1-\theta) w + \theta \omega^\eta | w \in \pi(\eta) \}$, which is a compact and
  convex subset of $\pi^*(\eta)$, such that
$\bar{\mu}_{\mathrm{Leb}}(\pi(\eta) - \pi_{1-\theta}(\eta)) <
\frac{\theta}{n}$ (i.e., $\pi_{1-\theta}(\eta)$ is a ``slight shrinking''
of $\pi^*(\eta)$). The reasoning behind this slight shrinking is that
the Optimal AdaBoost update is absolutely continuous on each $
\pi_{1-\theta}(\eta)$, something that may not be possible on each
$\pi^+(\eta)$. aAs we will see later, it turns that we do not need this under
Condition~\ref{C:NoTies2} (No Ties Eventually) if we start the
construction after $T$ rounds of AdBoost, where $T$ is the finite
round when ties are guarantee not to appear anymore as described in
that condition. 
  Let $\pi_{1-\theta}^\circ \equiv
\Int{\pi_{1-\theta}(\eta)}$ and $\Delta_m^{1-\theta} \equiv \bigcup_{\eta \in \Mcal}
\pi_{1-\theta}(\eta)$.
  The construction will now
  proceed as in that given on the previous definition
  (Definition~\ref{def:FPAB}, Finite-Precision Optimal AdaBoost), \emph{except}
  that it is done over $\pi_{1-\theta}(\eta)$, \emph{instead of}
  $\pi(\eta)$. Thus, we slightly abuse notation throughout the
  construction so that the construction parallels that given in
  Definition~\ref{def:FPAB}. Let $\tau > 0$. To simplify notation, and
  improve the presentation, let $w^{\eta,j} \equiv
  w^{(\tau,1-\theta),\eta,j}$ and $N_\eta \equiv N_{(\tau,1-\theta),\eta}$. Define
  \[ \approxpi^\tau_{1-\theta}(\eta) \equiv \{ w^{\eta,1},
    w^{\eta,2}, \ldots, w^{\eta,N_\eta}\} \subset
    \pi^\circ_{1-\theta}(\eta) \] as a finite set of points, of minimal
  cardinality $N_\eta$, that are ``near-uniformly distributed'' over
  $\pi_{1-\theta}(\eta)$ such that the sets
  \[ R_{\eta,j} \equiv R_{(\tau,1-\theta),\eta,j} \equiv \left\{w \in
      \pi_{1-\theta}(\eta) \, \left| \, 
      d(w,w^{\eta,j}) = \min_{j' = 1,2,\ldots, N_\eta}
      d(w,w^{\eta,j'}) \right. \right\} \] form a covering of
  $\pi_{1-\theta}(\eta)$ (i.e., $\pi_{1-\theta}(\eta) = \bigcup_{j =
    1}^{N_\eta} R_{\eta,j}$)
  satisfying 
\[ \mathrm{diam}(R_{\eta,j}) < \tau . \]
  In an attempt to simplify the notation in the presentation of the
  construction, let 
  $R^\circ_{\eta,j} \equiv R^\circ_{(\tau,1-\theta),\eta,j}  \equiv
  \mathrm{Int}(R_{\eta,j})$ and note that $R^\circ_{\eta,j} \cap R^\circ_{\eta,j} = \emptyset$
  for all $(j,j'), j\neq j'$, by construction. 
  Define \[ \approxDelta_m^{1-\theta} \equiv \approxDelta_m^{(\tau,1-\theta)} \equiv \bigcup_{\eta \in \Mcal}
\approxpi^\tau_{1-\theta}(\eta) \] the resulting discretization of
$\Delta_m^{1-\theta} \equiv \bigcup_{\eta \in \Mcal} \pi_{1-\theta}(\eta)$. For every $(\theta,\tau,\eta)$, impose a fixed, but arbitrary preference
order on the sets $(R_{(\tau,1-\theta),\eta,j})$ such that $j \succ j'$, for $j
\neq j'$, indicates that the set with index $j$ is preferred to that
with index $j'$, and let 
\[
R^*_{\eta,j}  \equiv R^*_{(\tau,1-\theta),\eta,j} \equiv R_{(\tau,1-\theta),\eta,j} - \bigcup_{j' \succ j}
\left( R_{(\tau,1-\theta),\eta,j'} \bigcap R_{(\tau,1-\theta),\eta,j}
\right). \]
Note
that the sets $R^*_{\eta,j}$, for $j =
1,2,\ldots,N_\eta$, form a partition of
$\pi_{1-\theta}(\eta)$ (i.e.,
$R^*_{\eta,j} \bigcap R^*_{eta,j'} = \emptyset$, for all $(j,j'), j\neq
j'$, and $\pi_{1-\theta}(\eta) = \bigcup_{j = 1}^{N_\eta} R^*_{\eta,j}$).
Because we
will be considering a sequence of approximations produced by related strictly
monotonically decreasing sequences $(\theta_k)$ and $(\tau_k) \equiv (\tau_k(\theta_k))$, for technical
reasons, we require that $\approxpi^{\tau_k}_{1-\theta_k}(\eta) \subset
\approxpi^{\tau_{k+1}}_{1-\theta_{k+1}}(\eta)$. Note from the
construction that for each
$\eta \in \Mcal$, we have $\pi_{1-\theta_k}(\eta) \subset
\pi_{1-\theta_{k+1}}(\eta)$, so that $\pi_{1-\theta_k}^\circ(\eta) \subset
\pi_{1-\theta_{k+1}}^\circ(\eta)$.  We can achieve the desired construction by employing a kind of hierarchical
discretization: i.e., to obtain a finer discretization, discretize each
$R_{(\tau_k,1-\theta_k),\eta,j}$ individually, by adding new points in its interior
to the discretization, so that we obtain 
$\pi^{\tau_{k+1}}_{1-\theta_{k+1}}(\eta)$ after inserting those new points leading to the finer-grain
discretization into $\pi^{\tau_k}_{1-\theta_k}(\eta)$.
Define $\Proj_{\pi_{1-\theta}(\eta)} : \pi^*(\eta) \to
\pi_{1-\theta}(\eta)$ such that, for all $w \in \pi^*(\eta)$,
\[
  \Proj_{\pi_{1-\theta}(\eta)}(w) \equiv \begin{cases}
    w, & \text{if $w \in \pi_{1-\theta}(\eta)$,}\\
    (1-\theta) w + \theta \omega^\eta, & \text{otherwise.}
    \end{cases}
  \]
  Define $\approxA_{(\tau,1-\theta),\mathrm{step}} : \Delta_m \to
  \Delta_m$ such that for all $w
  \in \Delta_m$ we have
  \[ \approxA_{(\tau,1-\theta),\mathrm{step}}(w) \equiv
    \A(\Proj_{\pi_{1-\theta}(\eta^w)}(w)) \]
(i.e., $\approxA_{(\tau,1-\theta),\mathrm{step}}$ is a ``slightly squeezed'' version of $\A$ on each $\pi^*(\eta)$).
\end{definition}
\begin{proposition}
  \label{pro:unifcont}
For each $\theta > 0$, $\A$ is uniformly
continuous on $\pi_{1-\theta}(\eta)$ for each $\eta \in \Mcal$. Thus,
the same holds for $\approxA_{(\tau,1-\theta),\mathrm{step}}$ on each $\pi^*(\eta)$.
\end{proposition}
\begin{proof}
The first statement follows from the Uniform Continuity Theorem
because $\A$ is continuous on $\pi^\circ(\eta)$ and 
$\pi_{1-\theta}(\eta)$ is a compact subset of $\pi^\circ(\eta)$. The
result for $\approxA_{(\tau,1-\theta),\mathrm{step}}$ follows
immediatly from that for $\A$ by noting that
$\Proj_{\pi_{1-\theta}(\eta)}$ is continuous on $\pi^*(\eta)$ for each $\eta$.
  \qed
\end{proof}
\begin{theorem}{\bf (Properties of Step-Function Optimal AdaBoost.)}
  \label{thm:SFAB}
Suppose Conditions~\ref{cond:nathypo} (Natural Weak-Hypothesis
Class) and ~\ref{assume:WeakLearn} (Weak Learning) hold. 
\begin{enumerate}
\item {\bf (Almost Uniformly Approximates Optimal AdaBoost)}
  The Optimal AdaBoost update $\A$ can be Lebesgue-almost uniformly
  approximated on $\Delta_m$
  by step functions: i.e., for each $\tau > 0$, there exists $\tau_*
  \equiv \tau_*(\tau)$, $0 < \tau_* < \tau'$, such that for 
 each function in the sequence $(\approxA_{(\tau_*,1-\tau),\mathrm{step}})$
 uniformly approximates $\A$ on $\Delta_m^{1-\tau}$ to within $\tau$ and 
 $\bar{\mu}_{\mathrm{Leb}}(\Delta_m - \Delta_m^{1-\tau}) < \tau$.
\item Starting from
  any $w_1 \in \Delta_m^\circ$ and for every
 $\tau > 0$, there exists $\tau_*
  \equiv \tau_*(\tau)$, $0 < \tau_* < \tau'$ such that the following holds. 
Let $\approxA_\tau \equiv \approxA_{(\tau_*,1-\tau),\mathrm{step}}$, $w_{t+1} \equiv
\approxA_\tau^{(t)}(w_1)$ and $\eta_t \equiv
\eta^{w_t}$ for all $t$, and consider the sequence of
example weights $(w_t)$. 
\begin{enumerate}
\item {\bf (Never Has Ties)} $w_{t+1} \in
\pi^\circ\left(\eta_{t+1}\right)$ for all
$t$.
\item {\bf (Always Converges to a Cycle)} The sequence $(w_t)$
  converges in finite time to a
 cycle in $\approxDelta_m^{(\tau_*,1-\tau)} \subset \Delta_m^\circ$, the precise cycle depending
 on $w_1$.
\item {\bf (Is Always Ergodic)} Let $T_\tau \equiv T_1(w_1,\tau)$ be the first time
  that the sequence $(w_t)$ enters a cycle of period $p_\tau \equiv
  p_{(\tau_*,1-\tau)}(w_1)$, $1 < p_\tau \leq n$, and define
  \begin{align*}
    \widehat{\mu}^\tau \equiv \widehat{\mu}_{\mathrm{step},w_1}^{(\tau_*,1-\tau)} \equiv \frac{1}{p_\tau} \sum_{s=0}^{p_\tau-1}
    \delta_{w_{T_\tau+s}} = \widehat{\mu}_{\approxA_\tau,w_1} \; . 
  \end{align*}
 The dynamical system $(\Delta_m,
 \Sigma_{\Delta_m},
 \approxA_\tau, \widehat{\mu}^\tau)$
 is ergodic.
 \item {\bf (Time Averages Always Converge)} The time average $\Tavg(f,w_1,\approxA_\tau,T)$ of
 any function $f$ based on the sequence $(w_t)$ converges \( \frac{1}{p_\tau} \sum_{s=0}^{p_\tau-1}
    f(w_{T_\tau+s}) \; . \)
\end{enumerate}
\end{enumerate}
\end{theorem}
\begin{proof}
For Part 1, pick $\tau' > 0$. Set $\theta = \tau'$. By properties of
Lebesgue measure and the construction, we
have
\begin{align*}
\bar{\mu}_{\mathrm{Leb}}(\Delta_m - \Delta_m^{1-\tau'}) &= \bar{\mu}_{\mathrm{Leb}}(\cup_{\eta \in \Mcal} (\pi^*(\eta)
  - \pi_{1-\tau'}(\eta)) = \sum_{\eta \in \Mcal} \bar{\mu}_{\mathrm{Leb}}(\pi^*(\eta)
           - \pi_{1-\tau'}(\eta)) \\
  &\leq \sum_{\eta \in \Mcal} \bar{\mu}_{\mathrm{Leb}}(\pi(\eta)
  - \pi_{1-\tau'}(\eta)) < \sum_{\eta \in \Mcal} \frac{\tau'}{n} =
  \tau' \; .
\end{align*}
By Proposition~\ref{pro:unifcont}, for each $\eta$, we can find 
$\tau_\eta \equiv \tau_\eta(\tau') > 0$ such that, letting $N_\eta
\equiv N_{(\tau_\eta,1-\tau'),\eta}$, for each $j = 1, 2, \ldots,
N_\eta$, and letting $R^*_{\eta,j} \equiv
R^*_{(\tau_\eta,1-\tau'),\eta,j}$, for each $w,w' \in
R^*_{\eta,j}$, we have \[ d(\A(w),\A(w')) < \tau' \; . \] Let
$\tau_* \equiv \tau_*(\tau') \equiv \min(\tau',\min_{\eta \in \Mcal} \tau_\eta)$,
so that for each $j$ and each $w,w' \in
R^*_{\eta,j}$, we still have \[ d(\A(w),\A(w')) <
\tau' \] because $d(w,w') < \mathrm{diam}(R^*_{\eta,j}) < \tau_* \leq \tau_\eta$.
For each $w \in \Delta_m^{1-\tau'}$, letting $(\eta,j)$ 
such that $w \in R^*_{\eta,j}$, $w^{\eta,j} \equiv
w^{(\tau_*,1-\tau'),\eta,j}$, and $\approxA \equiv \approxA_{(\tau_*,1-\tau'),\mathrm{step}}$, we have
\begin{align*}
d(\approxA(w),\A(w)) 
&=
                                                        d(\A(\Proj_{\pi_{1-\tau'}(\eta)}(w)),\A(w))
  \\
&=
  d(\A(w^{\eta,j}),\A(w)) < \tau' \; ,
\end{align*}
which completes the proof of Part 1. The proof for Part 2 is
essentially identical to that of Theorem~\ref{thm:FPAB}.
  \qed
\end{proof}

\subsubsection{Almost uniform approximations of Optimal AdaBoost by continuous
  functions that converge to a cycle of arbitrarily small sets}

We now consider Lebesgue-almost uniform approximations of $\A$ on $\Delta_m$ by continuous
functions. This would yield a sufficient condition for the convergence
of the time averages induced by $\A$ itself.
\begin{definition}{\bf (Continuous-Function Optimal AdaBoost)}
  \label{def:CFAB}
  Let $\rho \in (0,1)$.
  Define
  \[ \bar{\A} \equiv \bar{\A}_{(\tau,1-\theta, 1-\rho)} \] of type $\Delta_m \to \Delta_m$ as a ``slightly squeezed'' version
  of $\A$ whenever \[ \A(w^{\eta,j}) \in
  \mathrm{Bnd}(R_{\eta',j'}) \] for some $(\eta,j)$ and $(\eta',j')$:
i.e., for each $w \in \Delta_m$, letting $\eta \equiv \eta^w$, define
\begin{align*}
\bar{\A}(w) & \equiv
                       \A((1-\rho) w + \rho
                       w^{\eta,j})\\
  & \text{ if } w \in
        R_{\eta,j}^*,\\ &  w' \equiv \A(w^{\eta,j}) \in
                                            R_{\eta',j'}^*,                                            
                                            \eta' =
                                            \eta^{w'}, \text{ and }\\
  & w' \in
    \mathrm{Bnd}(R_{\eta',j'}) \; ;
\end{align*}
and $\bar{\A}(w) \equiv
  \A(w)$ otherwise.
  The rationale behind this slight squeeze is to make sure that we can
  continuously map every element in $R^*_{\eta,j}$ to exactly one
  $R^*_{\eta',j'}$. We can avoid this slight squeeze if we could
  guarantee that every $A(w^{\tau,j})$ always falls in the interior of
  some $R^*_{\eta',j'}$. 
 (As an aside, before continuing with the construction, we note that while $\rho$ plays a very important conceptual role in the
 construction, it turns out that the exact value of $\rho$ does not
 really matter for
  the proofs to go through, we want it to be as small as
  possible, and ideally set to zero. For the purpose of presenting the
  construction, we can simply set $\rho = \tau$, for example. Hence,
  from now on, we avoid adding $\rho$ to our notation, except on those
  cases where it seems important not to forget its dependency.)
   Note that now, for $(\eta,j)$, we have
   \begin{align}
     \label{eqn:intball}
     w' \equiv \bar{\A}(w^{\eta,j})
  \in R^\circ_{\eta',j'}
\end{align}
for $\eta' = \eta^{w'}$  
and some $j'$.  For any $\lambda \in [0,1)$, denote
by \[ \lambda R_{\eta,j}\equiv \lambda R_{(\tau,1-\theta),\eta,j} \equiv \{ \lambda w + (1-\lambda)
w^{\eta,j} | w  \in R_{\eta,j} \} \; .
\]
For completeness, for
$\lambda = 1$, let $\lambda R_{\eta,j} \equiv \lambda R_{(\tau,1-\theta),\eta,j} \equiv 1 R_{\eta,j}
\equiv R_{\eta,j}^*$. Note that \[ \lambda
R_{\eta,j}  \subset R_{\eta,j} \] because $R_{\eta,j}$ is
compact and convex; it is also non-empty. By
Proposition~\ref{pro:unifcont} and Equation~\ref{eqn:intball}, for
each $\theta > 0$, $\rho > 0$, and $\tau > 0$ there exists $\lambda \equiv
\lambda(\theta,\rho,\tau) > 0$
with the following property: for all $(\eta,j)$,  
there exists an
open neighborhood ``ball'' $B(w^{\eta,j}, \lambda)\equiv B(w^{(\tau,1-\theta),\eta,j}, \lambda)
\equiv \{ w \in R^*_{\eta',j'} | d(w,w') < \lambda
\}$ of $w'$ of ``radius'' $\lambda > 0$ such that 
$\bar{\A}_{(\tau,1-\theta,1-\rho)}(w) \in B(w^{\eta,j}, \lambda)
\subset R^\circ_{\eta',j'}$ for all $w \in \lambda
R_{\eta,j}$. Define a non-negative real-valued function $r_{\eta,j}^{\mathrm{in}}  \equiv
r_{(\tau,1-\theta),\eta,j}^{\mathrm{in}}$ of type $R^*_{\eta,j} \to
[0,\infty)$ such that $r_{\eta,j}^{\mathrm{in}}(w) \equiv
d(w,w^{\eta,j})$. Similarly, define \[
  r_{\eta,j}^{\mathrm{out}} \equiv r_{(\tau,1-\theta),\eta,j}^{\mathrm{out}} : R^*_{\eta,j} \to
[0,\infty) \] such that \[ r_{\eta,j}^{\mathrm{out}}(w) \equiv
\sup \{ d(w,w') | w' = \lambda w  + (1-\lambda) w^{\eta,j} \in
R^*_{\eta,j} \text{ for some } \lambda > 0 \}\; . \] Let
\[ r_{\eta,j}^{\mathrm{prop}} (w) \equiv r_{(\tau,1-\theta),\eta,j}^{\mathrm{prop}} (w) \equiv \frac{r_{\eta,j}^{\mathrm{in}}(w)}{r_{\eta,j}^{\mathrm{in}}(w)
  + r_{\eta,j}^{\mathrm{out}}(w)} \; . \]
Define $\approxA_{(\tau,1-\theta, 1-\rho),\mathrm{cont}} : \Delta_m \to \Delta_m$ such that for all $w
\in \Delta_m$, letting $(\eta,j)$ be such that $w \in
R^*_{(\tau,1-\theta),\eta,j}$ and $\lambda' \equiv \lambda'(w,\theta,\tau) \equiv
r_{(\tau,1-\theta),\eta,j}^{\mathrm{prop}} (w) \lambda(\theta,\tau)$,
we have \[ \approxA_{(\tau,1-\theta, 1-\rho),\mathrm{cont}}(w) \equiv \bar{\A}_{(\tau,1-\theta,1-\rho)}(\lambda' w
+ (1-\lambda') w^{(\tau,1-\theta),\eta,j}) \; . \]
\end{definition}

\begin{condition}{\bf (AdaBoost is Sufficiently Non-Expansive, With
    Respect to the Discretization.)}
  \label{cond:containment}
Suppose we allow $\theta = 0$ in the construction so that
$\pi_{1-\theta}(\eta) = \pi(\eta)$, $w^{(\tau,1-\theta),\eta,j} = w^{\tau,\eta,j}$, $\approxpi^\tau_{1-\theta}(\eta)
= \approxpi^\tau(\eta)$, and $R^*_{(\tau,\theta),
  \eta,j} = R^*_{\tau,
  \eta,j}$ for each $(\eta,j)$, as in the discretization for $\tau$-Finite-Precision Optimal
  AdaBoost (Definition~\ref{def:FPAB}). There exists $\tau^* > 0$ with
  the following property: 
for all $\tau > 0, \tau < \tau^*$, there exists a positioning of each
point $w^{\tau,\eta,j}$ in the
  discretization induced by $\approxpi^\tau(\eta)$ on
  $\pi(\eta)$ such that for each pair
$(\eta,j)$, we have $\A(R^*_{\tau,
  \eta,j}) \in R^*_{\tau,\eta',j'}$, where $\eta' =
\eta^{w^{\tau,\eta,j}}$ and $j'$ is such that $\A(w^{\tau,\eta,j}) \in
R^*_{\tau,\eta',j'}$.
\end{condition}

\begin{definition}
  \label{def:C_Mcal}
Denote by $C_\Mcal$ the set of all functions $f$ of type $\Delta_m \to
\R$
that are 
continuous on each
$\pi^*(\eta)$ individually when
viewed as a function of type
$\pi^*(\eta) \to \R$
for all $\eta \in \Mcal$.~\footnote{Technically, the functions only
  need to be continuous on each $\pi^*(\eta) \cap \cup_{\eta'}
  \pi_{\frac12}(\eta')$ individually.}
\end{definition}

\begin{theorem}{\bf (Properties of Continuous-Function Optimal AdaBoost.)}
 \label{thm:CFAB}
Suppose Conditions~\ref{cond:nathypo} (Natural Weak-Hypothesis
Class) and ~\ref{assume:WeakLearn} (Weak Learning) hold. 
\begin{enumerate}
\item {\bf (Uniformly Approximates Optimal AdaBoost)}
 The Optimal AdaBoost update $\A$ can be Lebesgue-almost uniformly
  approximated on $\Delta_m$
  by continuous functions: i.e., for each $\tau > 0$, there exists $\tau_*
  \equiv \tau_*(\tau)$, $0 < \tau_* < \tau'$, such that for 
 each function in the sequence $(\approxA_{(\tau_*,1-\tau,1-\tau),\mathrm{cont}})$
 uniformly approximates $\A$ on $\Delta_m^{1-\tau}$ to within $\tau$ and 
  $\bar{\mu}_{\mathrm{Leb}}(\Delta_m - \Delta_m^{1-\tau}) < \tau$.

\item Starting from
  any $w_1 \in \Delta_m^\circ$ and for every
 $\tau > 0$, there exists $\tau_*
  \equiv \tau_*(\tau)$, $0 < \tau_* < \tau'$ such that the following holds. 
Let $\approxA_\tau \equiv \approxA_{(\tau_*,1-\tau,1-\tau),\mathrm{cont}}$, $w_{t+1} \equiv
\approxA_\tau^{(t)}(w_1)$ and $\eta_t \equiv
\eta^{w_t}$ for all $t$, and consider the sequence of
example weights $(w_t)$.  
\begin{enumerate}
\item {\bf (Never Has Ties)} $w_{t+1} \in
\pi^\circ\left(\eta_{t+1}\right)$ for all
$t$.
\item {\bf (Converges to a Cycle of Arbitrarily Small Sets Containing a Cycle)}
The sequence $(w_t)$
 converges in finite time to a
 cycle of sets
 $(R_{(\tau_*,1-\tau'),\eta^{(s)},j^{(s)}})_{s=1,2,\ldots,p}$
of period $p_\tau \equiv p_{(\tau_*,1-\tau',1-\tau)}(w_1)$, $1 <
p_\tau \leq n$, 
 containing a cycle $(\omega^{(s)})_{s=1,2,\ldots,p}$, in the
 partition of $\approxDelta_m^{1-\tau}$ induced by
 $\approxDelta_m^{(\tau_*,1-\tau)}$, so that
 each set indexed by $s$ in the cycle has Lebesgue measure 
 \[ \bar{\mu}_{\mathrm{Leb}}(R_{(\tau_*,1-\tau'),\eta^{(s)},j^{(s)}})
   < \tau_*^{m-1} \; . \] 
The partition, the precise cycle of sets, the cycle it contains and its period, and the
Lebesgue measure of each set in the cycle depend on $w_1$ and $\tau$.
 \item {\bf (Is Always Ergodic)}  Let $\widehat{\mu}^\tau_{w_1,T} \equiv
  \widehat{\mu}_{\approxA_{(\tau_*,1-\tau,1-\tau),\mathrm{cont}},w_1}^{(T)}$ for all $T$. 
\begin{enumerate}
\item The sequence of empirical measures
 $(\widehat{\mu}^\tau_{w_1,T})_T$ converges to
$\widehat{\mu}^\tau_{w_1} \equiv \widehat{\mu}_{\approxA_\tau,w_1}$
 (i.e.,
 the Birkhoff limit exists) and
\item The dynamical system $(\Delta_m,
 \Sigma_{\Delta_m},
 \approxA_\tau, \widehat{\mu}_{w_1}^\tau)$
 is ergodic. 
\end{enumerate}
\item {\bf (Time Averages Converge)}
The time average $\Tavg(f,w_1,\approxA_{(\tau_*,1-\tau,1-\tau),\mathrm{step}},T)$ of
 any function $f \in C_\Mcal$ based on the sequence converges to \( \frac{1}{p_\tau} \sum_{s=0}^{p_\tau-1}
    f(\omega^{(s)}) \; . \) 
\end{enumerate}
The same
holds for the Optimal AdaBoost update $\A$ if, in addition,
 Condition~\ref{cond:containment} (Non-Expansive) holds.
\end{enumerate}
\end{theorem}
\begin{proof}
  The proof for Part 1 is almost identical to that of Theorem~\ref{thm:SFAB},
except for the last step: for each $w \in \Delta_m^{1-\tau'}$, letting $(\eta,j)$ 
such that $w \in R^*_{\eta,j}$, $w' \equiv
\lambda' w + (1-\lambda') w^{\eta,j} \in
R^*_{\eta,j}$, $\approxA \equiv
\approxA_{(\tau_*,1-\tau',1-\tau'),\mathrm{cont}}$, and $\bar{\A}
\equiv \bar{\A}_{(\tau_*,1-\tau',1-\tau')}$, we have
\begin{align*}
d(\approxA(w),\A(w)) =
  d(\bar{\A}(w'),\A(w)) \; .
\end{align*}
Let $\eta' \equiv \eta^{w'}$ and $j'$ such that $w' \in
R^*_{\eta',j'}$. If $\A(w') \not\in
\mathrm{Bnd}(R^*_{\eta',j'})$, then
\begin{align*}
d(\approxA(w),\A(w)) =
  d(\A(w'),\A(w)) < \tau' \; .
\end{align*}
Otherwise, $\A(w') \in
\mathrm{Bnd}(R^*_{\eta',j'})$ and, because $w''
\equiv (1-\rho) w' + \rho w^{(\tau_*,1-\tau'),\eta,j} \in
R^*_{\eta,j}$ too, we have
\begin{align*}
d(\approxA(w),\A(w)) =
  d(\A(w''),\A(w)) < \tau' \; .
\end{align*}

Part 2.a follows immediately from the construction.

For Part 2.b, we first note a few properties of the construction. Consider a
pair 
$(\eta,j)$ and $w \in R^*_{\eta,j}$.  
Let $\lambda
\equiv \lambda(\tau',\tau_*)$, $\lambda' \equiv \lambda'(w,\tau',\tau_*) \equiv
r_{(\tau_*,1-\tau'),\eta,j}^{\mathrm{prop}} (w)
\lambda \leq \lambda$. Thus, $w' \equiv \lambda' w + (1-\lambda')
w^{\eta,j} \in \lambda R_{\eta,j} \equiv \lambda
R_{(\tau_*,1-\tau'),\eta,j}$ 
and $\approxA(w) =
\bar{\A}(w') \in 
R^\circ_{\eta',j'}$. Hence, each pair $(\eta,j)$ maps
to 
exactly one $(\eta',j')$: i.e.,
$\approxA$ maps every $w \in
R^*_{\eta,j}$ maps
$R^\circ_{\eta',j'}$ only. By the Pigeonhole
Principle, $\approxA$ will
enter a cycle of closed sets
$(R_{\eta^{(s)},j^{(s)}})_{s=1,2,\ldots,p}$
of period $p \equiv p_{(\tau_*,1-\tau',1-\rho)}(w_1)$, $1 < p \leq n$. Hence, we can
view
$\approxA^{(p)}$
as a continuous mapping from $R_{\eta^{(1)},j^{(1)}}$
to itself. Because $R_{\eta^{(1)},j^{(1)}}$ is compact and convex, by the Brouwer Fixed-Point Theorem,
$\approxA^{(p)}$
has fixed point $\omega^{(1)} \equiv \omega^{(1)}_{(\tau_*,1-\tau')} \in
R_{\eta^{(1)},j^{(1)}}$, anchoring a cycle
$(\omega^{(s)})_{s=1,2,\ldots,p}$ such that $\omega^{(s)} \in
R_{\eta^{(s)},j^{(s)}}$ for all $s =
1,2,\ldots,p$. Because $\mathrm{diam}(R_{(\tau_*,1-\tau'),\eta,j}) <
\tau_*$, we have $\bar{\mu}_{\mathrm{Leb}}(R_{(\tau_*,1-\tau'),\eta,j})
\leq \tau_*^{m-1}$, because $m > 2$ for Optimal
AdaBoost to be consistent with Condition~\ref{assume:WeakLearn} (Weak Learning).

For Part 2.c, we note that a sufficient condition
for the convergence of the empirical measure is that the time
average of any continuous function exists (i.e., converges). 
Consider an arbitrary continuous real-valued function $f
: \Delta_m \to \R$. By the Uniform Continuity Theorem, $f$ is
uniformly continuous because $\Delta_m$ is compact. Pick $\tau' >0$
and set $\theta = \tau'$. Consider an $\eta
\in \Mcal$. Because $f$ is
uniformly continuous, we can find
$\tau_\eta \equiv \tau_\eta(\tau') > 0$, such that if $w,w' \in
\pi_{1-\tau'}(\eta)$ and $d(w,w') < \tau_\eta$ then
$|f(w)-f(w')| < \tau'$. Let $\tau_* \equiv \min(\tau',\min_{\eta \in \Mcal}
\tau_\eta)$. Let $w_1 \in \Delta_m^\circ$ be an arbitrary initial
point, $w_{t+1} \equiv \approxA^{(t)}(w_1)$,
and $p \equiv p_{(\tau_*,1-\tau',1-\rho)}(w_1)$, $1 < p \leq n$ be the period of the cycle
of sets
$(R_{\eta^{(T_1+s)},j^{(T_1+s)}})_{s=0,1,\ldots,p-1}$ 
that the sequence $(w_t)$ enters first
at time $T_1 \equiv T_1(w_1,\tau',\rho)$, so that
$w_{(T_{\tau}+t')} \in
R^*_{\eta^{(T_1+ (t' \mod p))},j^{(T_1+ (t' \mod p))}}$
for all $t' = 0,1,2,\dots$, which implies that
$|f(w_{(T_1+t')}) -
f(w_{(T_1+(t' \mod p_\tau))}) | < \tau'$ because
$d(w_{(T_1+t')} w_{(T_1+(t' \mod p))}) < \tau'$. Denote by
$\widehat{f}_{w_1} \equiv \widehat{f}_{(\tau_*,1-\tau',1-\rho)}(w_1) \equiv \frac{1}{p} \sum_{s=0}^{p-1}
f(w_{(T_1+s)})$.
Let
$L \equiv \lfloor \frac{T-T_1+1}{p} \rfloor$, $T' \equiv p L$, 
and $r \equiv T - (T_1 + T')$. 
The
time average of $f$ can be decomposed as follows.
\begin{align*}
\frac{1}{T} \sum_{t=1}^T f(w_t) &= \frac{1}{T}
                                                \sum_{t=1}^{T_1-1}
                                                f(w_t) +
                                                \frac{1}{T}
                                                \sum_{t=T_1}^T
                                                f(w_t)
\end{align*}
The first term can be upper-bounded as
\begin{align*}
\frac{1}{T} \sum_{t=1}^{T_1-1} f(w_t) 
  \leq \frac{T_1-1}{T} \max_{t=1,2,\ldots,T_1-1} f(w_t)
    \; ;
\end{align*}
the second term can be further decomposed as
\begin{align*}
  \frac{1}{T}
  \sum_{t=T_1}^T
                                                f(w_t) = \frac{1}{T}
                                                \sum_{t=T_1}^{T_1
  + T' - 1}
                                                f(w_t) + \frac{1}{T}
                                                \sum_{t=T_1+T'}^T
                                                f(w_t)
  \; .
\end{align*}
The first term in the last expression is 
\begin{align*}
\frac{1}{T}
                                                \sum_{t=T_1}^{T_1 + T' - 1}
  f(w_t)
  & = \frac{1}{T}
                                                \sum_{t'=0}^{T' - 1}
                                                f(w_{T_1
  + t'}) \\
  &< \frac{1}{T}
                                                \sum_{t'=0}^{T' - 1}
                                                \left( f(w_{T_1
  + (t' \mod p)}) + \tau' \right) \\
  & = \frac{L}{T}
                                                \sum_{s=0}^{p-1}
                                                \left(
    f(w_{T_1+s}) + \tau' \right) \\ &=
  \frac{p L}{T} \left( \widehat{f}_{w_1} + \tau' \right) 
\end{align*}
and 
the second term is
\begin{align*}
\frac{1}{T} \sum_{t=T_1+T'}^T f(w_t)
  &= \frac{1}{T} \sum_{s=0}^r
    f(w_{T_1+T'+s}) \\
  & < \frac{1}{T} \sum_{s=0}^r
    \left( f(w_{T_1+((T'+s) \mod p)}) + \tau' \right) \\
  & = \frac{1}{T} \sum_{s=0}^r
  \left( f(w_{T_1+s}) + \tau' \right)
  \\ &< \frac{p}{T}
  \left( \widehat{f}_{w_1}+ \tau' \right) \; ,
\end{align*}
so that
\begin{align*}
  \frac{1}{T}
  \sum_{t=T_1}^T
  f(w_t) & < \frac{p (L+1)}{T}
                                            \left(
                             \widehat{f}_{w_1} + \tau' \right) \\
                                          &< \frac{p(\frac{T-T_1+1}{p}+1)}{T}
                                            \left( \widehat{f}_{w_1} + \tau' \right) \\
  &= \widehat{f}_{w_1} + \tau' + \frac{1+p-T_1}{T}
  \left( \widehat{f}_{w_1} + \tau' \right) \; .
\end{align*}
Putting everything back together we obtain the following upper bound
on the time averages of $f$ starting from any $w_1 \in \Delta_m^\circ$:
\begin{align*}
\frac{1}{T} \sum_{t=1}^T f(w_t) &<
                                                                   \widehat{f}_{w_1} + \tau' 
                                                                   + \frac{1+p-T_1}{T}
  \left( \widehat{f}_{w_1} + \tau'  \right)+ 
                                                                   \frac{T_1-1}{T}
                                                                   \max_{t=1,2,\ldots,T_1-1}
                                                                   f(w_t)
                                                                   \; ,
\end{align*}
which implies
\begin{align*}
\limsup_{T \to \infty} \frac{1}{T} \sum_{t=1}^T
  f(w_t) \leq \widehat{f}_{w_1} + \tau'
  \; .
\end{align*}
An analogus derivation leads to the following lower bound:
\begin{align*}
\frac{1}{T} \sum_{t=1}^T f(w_t) &>
                                                                   \widehat{f}_{w_1}
                                                                   - \tau'
                                                                   + \frac{1-T_1}{T}
  \left( \widehat{f}_{w_1} - \tau' \right) + 
                                                                   \frac{T_1-1}{T}
                                                                   \min_{t=1,2,\ldots,T_1-1}
                                                                   f(w_t)
                                                                   \; ,
\end{align*}
which implies
\begin{align*}
\liminf_{T \to \infty} \frac{1}{T} \sum_{t=1}^T
  f(w_t) \geq \widehat{f}_{w_1} - \tau'
  \; .
\end{align*}
Hence, 
we have
\begin{align*}
\limsup_{T \to \infty} \frac{1}{T} \sum_{t=1}^T
  f(w_t) - \liminf_{T \to \infty} \frac{1}{T} \sum_{t=1}^T
  f(w_t) \leq 2 \tau'
  \; .
\end{align*}
The result for Part 2.c.i immediately follows because the above is a
sufficient condition for the empirical measure (i.e., the Berkhoff
limit) to exist~\citep[see,
e.g.,][]{Abdenur2013,2018arXiv181104805C}.~\footnote{As an aside, note
  that the last statement does not immediately imply that
  $\sum_{t=1}^T
  f(\A^{(t-1)}(w_1))$ converges because we would have to show that,
  for any
  $\tau'$, we can make the
  $d(\approxA^{(t-1)}(w_1),\A^{(t-1)}(w_1))$ remain arbitrarily small:
  while we know that $d(\approxA(w_1),\A(w_1)) <
  \tau_*$, this is not sufficient.
   Sufficient would be to show that
   $d(\approxA^{(t)}(w_1),\A^{(t)}(w_1))
   < \tau_*$ for all $t > 1$ and that 
   $\sum_{t=1}^T
   d(\approxA^{(t-1)}(w_1),\A^{(t-1)}(w_1))$
   is $o(T \tau_*(\tau'))$.}

For Part 2.c.ii, note that $\widehat{\mu} \equiv
\widehat{\mu}_{w_1}$ assigns measure
zero to any subset outside the cycle of sets: i.e., letting $\Wcal^{\mathrm{cycle}}
\equiv \Wcal^{\mathrm{cycle}}_{(\tau_*,1-\tau',1-\rho)} \equiv \bigcup_{s=1}^p
R_{\eta^{(s)},j^{(s)}}$, any subset of
$\Delta_m - \Ccal$. The rest of the proof is by
contradiction. Suppose that there exists two $\approxA$-invariant sets
in $\Wcal_1,\Wcal_2 \in \Wcal^{\mathrm{cycle}}$ with positive measure $\widehat{\mu}(\Wcal_1) > 0$
and $\widehat{\mu}(\Wcal_2) > 0$. Then that would mean that example
weights generated by the update that are in $\Wcal_1$ can reach those
in $\Wcal_2$, and \emph{vice versa}, which contradicts the fact that
the sets are $\approxA$-invariant. Hence, there is only one
$\approxA$-invariant set in $\Wcal^{\mathrm{cycle}}$ and it has full
measure. The results follows immediately from the definition of
ergodicity (Definition~\ref{def:ergodic}).

The proof of Part 2.d is essentially identical to that for Part 2.b,
except that now the function $f$ is continuous on each
$\pi^*(\eta)$ only when viewed as a mapping from $\pi^*(\eta)$ to $\R$
only. But the proof would not change if we could find a corresponding
$\tau_*$ as function of $\tau'$ for
this case. We can find such a value by noting that $f$ is uniformly
continuous on each $\pi_{1-\tau'}(\eta)$, by the Uniform Continuity
Theorem, because $\pi_{1-\tau'}(\eta)$ is a compact subset of
$\pi^*(\eta)$. So we can perform the same process to find such a
$\tau_*$, and the proof continues exactly as that for Part 2.b.

For the last statement in Part 2 of the theorem, first note that the
only reason we use $\pi_{1-\theta}(\eta)$ instead of $\pi(\eta)$ in
the construction is
to make sure that $\A$ is uniformly continuous, so that we can find an
appropriate $\tau_* > 0$ for any $\tau' > 0$ in the proof. But there
is another way to 
achieve this. By Theorem~\ref{T:lower-bound+}, we can run $\A$ for $n$ rounds
first, and then continue the process in a discretization over each
closed set
$\pi_{\epsilon_*}(\tau) \equiv \{ w \in \pi(\eta) | \eta \cdot w
\geq \epsilon_* \}$ for the lower bound on the error $\epsilon_* \geq
\frac{1}{2^{n+1}}$ guaranteed by that theorem.~\footnote{In fact, we
  could have performed the whole construction in that same way. We did
not do so because (a) we wanted to maintain certain degree of fidelity to the
construction performed by~\citet{Abdenur2013} and~\citet{2018arXiv181104805C}; and (b) the exact
value of $\theta$ has little to no effect in the proofs.} Now note that if Part 1
of Condition~\ref{cond:containment} holds, then $\bar{A}_{1-\tau'} = \A$. If
Part 2 of the condition holds, then the construction holds for
$\lambda \equiv \lambda(\tau_*,\tau') = 1$, so that $\lambda R_{\eta,j}
= R^*_{\eta,j}$. In that case, every weight in $R^*_{\eta,j}$ maps to
exactly the same set $R^*_{\eta',j'}$, where $\eta' \equiv
\eta'_{(\tau_*,1-\tau')} = \eta^{w^{\eta,j}}$ and $j'$ is such that
$\A(w^{\eta,j}) \in R^*_{\eta',j'}$. The cycling
results then follow immediately by applying the Pigeonhole Principle, while the
rest of the proof is as in the case without Condition~\ref{cond:containment}.
  \qed
\end{proof}

\subsubsection{Uniform approximations of Optimal AdaBoost by continuous
  functions that converge to a cycle of arbitrarily small sets}

The objective of the next construction is to eliminate the dependence
on $\theta$. Under Condition~\ref{C:NoTies2} (No Ties Eventually), we do that by starting the approximation after perform
$T$ rounds of AdaBoost, for some finite $T$, after which we know the algorithm behaves
like a continuous map on a compact set $\closure{\A^{(T)}(\Delta_m^+)}$ under the no-ties condition, and thus it is absolutely continuous.  
We perform the construction on the next definiton based on 
pairs of mistake dichotomies that are related via the Optimal AdaBoost
update. Although we can obtain the same results using a construction
based on each $\eta$ individually, we want to highlight characteristics of the Optimal-AdaBoost
update, at the expense of a slight increase in the complexity of the
presentation.

\begin{definition}{\bf (A Tailored Version of Continuous-Function Optimal AdaBoost)}
  \label{def:DVAB}
Recall that, under Condition~\ref{C:NoTies2} (No Ties Eventually),
$\Omega \equiv \closure{\A^{(T)}(\Delta_m^+)}$ for some finite $T$. Let
$\pi_\Omega(\eta) \equiv \pi^*(\eta) \cap \Omega = \pi(\eta) \cap
\Omega$ and
\[
\phi_\Omega(\eta,\eta') \equiv \{ w \in \pi_\Omega(\eta) \; \mid \;  \A(w) =
\T_\eta(w) \in
\pi_\Omega(\eta') \} \; .
\]
Let $\tau > 0$. To simplify notation slightly, and
  improve the presentation, let $w^{(\eta,\eta'),j} \equiv
  w^{\tau,(\eta,\eta'),j}$ and $N_{\eta,\eta'} \equiv N_{\tau,(\eta,\eta')}$. Define
  \[ \approxphi^\tau_\Omega(\eta,\eta') \equiv \{ w^{(\eta,\eta'),1},
    w^{(\eta,\eta'),2}, \ldots, w^{(\eta,\eta'),N_{\eta,\eta'}}\} \subset
    \phi_\Omega(\eta,\eta') \] as a finite set of points, of minimal
  cardinality $N_{\eta,\eta'}$, that are ``near-uniformly distributed'' over
  $\phi_\Omega(\eta,\eta')$ such that the sets
  \[ R_{(\eta,\eta'),j} \equiv R_{\tau,(\eta,\eta'),j} \equiv \left\{w
      \in \phi_\Omega(\eta,\eta') \; \left| \;
      d(w,w^{(\eta,\eta'),j}) = \min_{j' = 1,2,\ldots, N_{\eta,\eta'}}
      d(w,w^{(\eta,\eta'),j'}) \right. \right\} \] form a covering of
  $\phi_\Omega(\eta,\eta')$ (i.e., $\phi_\Omega(\eta) = \cup_{j =
    1}^{N_{\eta,\eta'}} R_{(\eta,\eta'),j}$)
  satisfying 
\[ \mathrm{diam}(R_{(\eta,\eta'),j}) < \tau . \]
  Let 
  $R^\circ_{(\eta,\eta'),j} \equiv R^\circ_{\tau,(\eta,\eta'),j}  \equiv
  \mathrm{Int}(R_{(\eta,\eta'),j})$ and note that $R^\circ_{(\eta,\eta'),j} \cap R^\circ_{(\eta,\eta'),j} = \emptyset$
  for all $(j,j'), j\neq j'$, by construction. 
  Define \[ \approxOmega \equiv \approxOmega^\tau \equiv
    \bigcup_\eta \bigcup_{\eta' \succ \eta}
\approxpi^\tau_\Omega(\eta,\eta') \] the resulting discretization of
$\Omega \equiv \cup_\eta \cup_{\eta' \succ \eta} \phi_\Omega(\eta,\eta')$. For every $(\tau,(\eta,\eta'))$, impose a fixed, but arbitrary preference
order on the sets $(R_{\tau,(\eta,\eta'),j})$ such that $j \succ j'$, for $j
\neq j'$, indicates that the set with index $j$ is preferred to that
with index $j'$, and let 
\[
R^*_{(\eta,\eta'),j}  \equiv R^*_{\tau,(\eta,\eta'),j} \equiv R_{\tau,(\eta,\eta'),j} - \bigcup_{j' \succ j}
\left( R_{\tau,(\eta,\eta'),j'} \bigcap R_{\tau,(\eta,\eta'),j}
\right) \; . \]
Note
that the sets $R^*_{(\eta,\eta'),j}$, for $j =
1,2,\ldots,N_{\eta,\eta'}$, form a partition of
$\phi_\Omega(\eta,\eta')$; that is, 
$R^*_{(\eta,\eta'),j} \bigcap R^*_{(\eta,\eta'),j'} = \emptyset$, for all $(j,j'), j\neq
j'$, and $\phi_\Omega(\eta,\eta') = \bigcup_{j = 1}^{N_{\eta,\eta'}} R^*_{(\eta,\eta'),j}$.
Because we
will be considering a sequence of approximations produced by a strictly
monotonically decreasing sequence $(\tau_k)$, for technical
reasons, we require that $\approxphi^{\tau_k}_\Omega(\eta,\eta') \subset
\approxphi^{\tau_{k+1}}_\Omega(\eta,\eta')$.
We can achieve the desired construction by employing a kind of hierarchical
discretization: i.e., to obtain a finer discretization, discretize each
$R_{\tau_k,(\eta,\eta'),j}$ individually, by adding new points in its interior
to the discretization, so that we obtain 
$\phi^{\tau_{k+1}}_\Omega(\eta,\eta')$ after inserting those new points leading to the finer-grain
discretization into $\phi^{\tau_k}_\Omega(\eta,\eta')$.

Let $\rho \in (0,1)$.
  Define
  \[ \bar{\A} \equiv \bar{\A}_{(\tau,1-\rho)} \] of type $\Omega \to \Omega$ as a ``slightly squeezed'' version
  of $\A$ whenever \[ \A(w^{(\eta,\eta'),j}) \in
  \mathrm{Bnd}(R_{(\eta',\eta''),j'}) \] for some $(\eta,j)$,
$(\eta',j')$, and $\eta''$:
i.e., for each $w \in \Omega$, letting $\eta \equiv \eta^w$, define
\begin{align*}
\bar{\A}(w) & \equiv
                       \A((1-\rho) w + \rho
                       w^{(\eta,\eta'),j})\\
  & \text{ if } w \in
        R_{(\eta,\eta'),j}^*,\\ &  w' \equiv \A(w^{(\eta,\eta'),j}) \in
                                            R_{(\eta',\eta''),j'}^*,
                                            \eta' =
                                            \eta^{w'}, \text{ and }\\
    & w' \in
      \mathrm{Bnd}(R_{(\eta',\eta''),j'}) \; ;
\end{align*}
and $\bar{\A}(w) \equiv
  \A(w)$ otherwise.
   Note that now, for $((\eta,\eta'),j)$, we have
   \begin{align}
     \label{eqn:intball2}
     w' \equiv \bar{\A}(w^{(\eta,\eta'),j})
  \in R^\circ_{(\eta',\eta''),j'}
\end{align}
for $\eta' = \eta^{w'}$  
and some $j'$ and $\eta''$.  For any $\lambda \in [0,1)$, denote
by \[ \lambda R_{(\eta,\eta'),j}\equiv \lambda R_{\tau,(\eta,\eta'),j} \equiv \{ \lambda w + (1-\lambda)
w^{(\eta,\eta'),j} | w  \in R_{(\eta,\eta'),j} \} \; .
\]
For completeness, for
$\lambda = 1$, let $\lambda R_{(\eta,\eta'),j} \equiv \lambda R_{\tau,(\eta,\eta'),j} \equiv 1 R_{(\eta,\eta'),j}
\equiv R_{(\eta,\eta'),j}^*$. Note that \[ \lambda
R_{(\eta,\eta'),j}  \subset R_{(\eta,\eta'),j} \] because $R_{(\eta,\eta'),j}$ is
compact and convex; it is also non-empty. By
Proposition~\ref{pro:unifcont} and Equation~\ref{eqn:intball2}, for each $\rho > 0$ and $\tau > 0$ there exists $\lambda \equiv
\lambda(\rho,\tau) > 0$
with the following property: for all $((\eta,\eta'),j)$,  
there exists an
open neighborhood ``ball'' $B(w^{(\eta,\eta'),j}, \lambda)\equiv B(w^{\tau,(\eta,\eta'),j}, \lambda)
\equiv \{ w \in R^*_{(\eta',\eta''),j'} | d(w,w') < \lambda
\}$ of $w'$ of ``radius'' $\lambda > 0$ such that 
$\bar{\A}_{(\tau,1-\rho)}(w) \in B(w^{(\eta,\eta'),j}, \lambda)
\subset R^\circ_{(\eta',\eta''),j'}$ for all $w \in \lambda
R_{(\eta,\eta''),j}$. Define a non-negative real-valued function $r_{(\eta,\eta''),j}^{\mathrm{in}}  \equiv
r_{\tau,(\eta,\eta'),j}^{\mathrm{in}}$ of type $R^*_{(\eta,\eta'),j} \to
[0,\infty)$ such that $r_{(\eta,\eta'),j}^{\mathrm{in}}(w) \equiv
d(w,w^{(\eta,\eta'),j})$. Similarly, define \[
  r_{(\eta,\eta'),j}^{\mathrm{out}} \equiv r_{\tau,(\eta,\eta'),j}^{\mathrm{out}} : R^*_{(\eta,\eta'),j} \to
[0,\infty) \] such that \[ r_{(\eta,\eta'),j}^{\mathrm{out}}(w) \equiv
\sup \{ d(w,w') | w' = \lambda w  + (1-\lambda) w^{(\eta,\eta'),j} \in
R^*_{(\eta,\eta'),j} \text{ for some } \lambda > 0 \}\; . \] Let
\[ r_{(\eta,\eta'),j}^{\mathrm{prop}} (w) \equiv r_{\tau,(\eta,\eta'),j}^{\mathrm{prop}} (w) \equiv \frac{r_{(\eta,\eta'),j}^{\mathrm{in}}(w)}{r_{(\eta,\eta'),j}^{\mathrm{in}}(w)
  + r_{(\eta,\eta'),j}^{\mathrm{out}}(w)} \; . \]
Define $\approxA_{(\tau,1-\rho)} : \Omega \to \Omega$ such that for all $w
\in \Omega$, letting $(\eta,j)$ is such that $w \in
R^*_{\tau,(\eta,\eta'),j}$ and $\lambda'_w \equiv \lambda'_w(\theta,\tau) \equiv
r_{\tau,(\eta,\eta'),j}^{\mathrm{prop}} (w) \lambda(\theta,\tau)$,
we have \[ \approxA_{(\tau,1-\rho)}(w) \equiv \bar{\A}_{(\tau,1-\rho)}(\lambda'_w w
+ (1-\lambda'_w) w^{\tau,(\eta,\eta'),j}) \; . \]
\end{definition}
\begin{definition}
  \label{def:C_Mcal_Omega}
Denote by $C_\Mcal(\Omega)$ the set of all functions $f$ of type $\Omega \to
\R$
that are 
continuous on each
$\pi_\Omega(\eta)$ individually when
viewed as a function of type
$\pi_\Omega(\eta) \to \R$
for all $\eta \in \Mcal$.
\end{definition}

\begin{condition}{\bf (AdaBoost is Sufficiently Non-Expansive, With
    Respect to the Tailored Discretization.)}
  \label{cond:containment_tailored}
There exists $\tau^* > 0$ with
  the following property: 
for all $\tau > 0, \tau < \tau^*$, there exists a positioning of each
point $w^{\tau,(\eta,\eta'),j}$ in the
  discretization induced by $\approxphi_\Omega^\tau(\eta,\eta')$ on
  $\phi_\Omega(\eta,\eta')$ such that for each 
$((\eta,\eta'),j)$, we have $\A(R^*_{\tau,(\eta,\eta'),j}) \in
R^*_{\tau,(\eta',\eta''),j'}$, where $j'$ and $\eta''$ are such that $\A(w^{\tau,(\eta,\eta'),j}) \in
R^*_{\tau,(\eta',\eta''),j'}$.
\end{condition}
\begin{theorem}{\bf (Properties of the Tailored Version of Continuous-Function Optimal AdaBoost.)}
 \label{thm:TCFAB}
Suppose Conditions~\ref{cond:nathypo} (Natural Weak-Hypothesis
Class),~\ref{assume:WeakLearn} (Weak Learning), and~\ref{C:NoTies2} (No
Ties Eventually) hold. 
\begin{enumerate}
\item {\bf (Uniformly Approximates Optimal AdaBoost)}
 The Optimal AdaBoost update $\A$ can be uniformly
  approximated on $\Delta_m^+$
  by continuous functions: i.e., for each $\tau > 0$, there exists $\tau_*
  \equiv \tau_*(\tau)$, $0 < \tau_* < \tau'$, such that for 
 each function in the sequence $(\approxA_{(\tau_*,1-\tau)})$
 uniformly approximates $\A$ on $\Delta_m^+$ to within $\tau$.

\item Starting from
  any $w_1 \in \Delta_m^+$ and for every
 $\tau > 0$, there exists $\tau_*
  \equiv \tau_*(\tau)$, $0 < \tau_* < \tau'$ such that the following holds. 
Let $\approxA_\tau \equiv \approxA_{(\tau_*,1-\tau)}$, $w_{t+1} \equiv
\approxA_\tau^{(t)}(w_1)$ and $\eta_t \equiv
\eta^{w_t}$ for all $t$, and consider the sequence of
example weights $(w_t)$.  
\begin{enumerate}
\item {\bf (Never Has Ties)} $w_{t+1} \in
\pi^\circ\left(\eta_{t+1}\right)$ for all
$t > T$, where $T$ is sufficiently large, as prescribed by
Condition~\ref{C:NoTies2} (No Ties Eventually).
\item {\bf (Converges to a Cycle of Arbitrarily Small Sets Containing a Cycle)}
The sequence $(w_t)$
 converges in finite time to a
 cycle of sets
 $(R_{\tau_*,(\eta^{(s)},\eta^{(s+1)}),j^{(s)}})_{s=0,1,2,\ldots,p_\tau-1}$,
 with $\eta^{(p_\tau)} = \eta^{(0)}$, 
 of period $p_\tau \equiv p_{(\tau_*,1-\tau)}(w_1)$, $1 <
p_\tau \leq n$, 
 containing a cycle $(\omega^{(s)})_{s=0,1,2,\ldots,p_\tau-1}$, in the
 partition of $\Omega$ induced by
 $\approxOmega^{(\tau_*,1-\tau)}$, so that
 each set indexed by $s$ in the cycle has Lebesgue measure 
 \[ \bar{\mu}_{\mathrm{Leb}}(R_{\tau_*,(\eta^{(s)}, \eta^{(s+1)}),j^{(s)}})
   < \tau_*^{m-2} \; . \] 
The partition, the precise cycle of sets, the cycle it contains and its period, and the
Lebesgue measure of each set in the cycle depend on $w_1$ and $\tau$.
\item {\bf (Is Always Ergodic)}  Let $\widehat{\mu}^\tau_{w_1,T} \equiv
  \widehat{\mu}_{\approxA_{(\tau_*,1-\tau),w_1}}^{(T)}$ for all $T$. 
\begin{enumerate}
\item The sequence of empirical measures
 $(\widehat{\mu}^\tau_{w_1,T})_T$ converges to
 $\widehat{\mu}^\tau_{w_1} \equiv \widehat{\mu}_{\approxA_\tau,w_1}$
 (i.e.,
 the Birkhoff limit exists) and
\item The dynamical system $(\Delta_m^+,
 \Sigma_{\Delta_m^+},
 \approxA_\tau, \widehat{\mu}_{w_1}^\tau)$
 is ergodic. 
\end{enumerate}
\item {\bf (Time Averages Converge)}
The time average $\Tavg(f,w_1,\approxA_{(\tau_*,1-\tau)},T)$ of
 any function $f \in C_\Mcal(\Omega)$ based on the sequence converges to \( \frac{1}{p_\tau} \sum_{s=0}^{p_\tau-1}
    f(\omega^{(s)}) \; . \) 
\end{enumerate}
The same
holds for the Optimal AdaBoost update $\A$ if, in addition,
 Condition~\ref{cond:containment_tailored} (Non-Expansive, Tailored Version) holds.
\end{enumerate}
\end{theorem}

\subsubsection{Birkhoff averages and the non-expansive condition}

We now discuss why we believe we have hit the limit of the current state
of knowledge on dynamical systems and ergodic theory, even if the
no-ties condition holds, and argue why further
applications require extension of the current state of knowledge in
pure mathematics or a more detailed understanding of the specific AdaBoost
dynamical system. 

Consider an (arbitrary) continuous map $M : G \to G$ on a compact set
$G$. Note that $M$ is also uniformly continuous on $G$, by the
compactness of $G$~\citep[][Uniform Continuity Theorem 23.3, pp. 160]{BartleRA}. Let $w_1
\in G$, and $w_{t+1} \equiv M(w_t) =  M^{(t-1)}(w_1)$ for all $t \geq
1$. Because $(w_t)$ is a sequence in a compact set, it has a
convergence subsequence $(w_{t_s}) \in G$, by the 
Bolzano-Weierstrass Theorem~\citep[][Theorem 16.4,
pp. 108]{BartleRA}. Let $w^{(1)} \equiv \lim_{s \to \infty}
w_{t_s}$. 

Let $M^{(0)} : G \to G$ be the identity function (i.e.,
$M^{(0)}(w) = w$ for all $w$ in $G$), and, for all natural numbers
$k > 0$, let $M^{(k)} \equiv M \circ M^{(k-1)} \equiv \circ_{j=1}^k M$
be the composition of $M$ with itself $k$ times, which is also
(uniformly) continuous on $G$ by the (uniform) continuity of $M$ on
$G$~\citep[][Theorem 20.8, pp. 143]{BartleRA}. Note that for all
$k \geq 0$, we have $w_{t_s + k} = M^{(k)}(w_{t_s})$, so that
$\lim_{s \to \infty} w_{t_s + k} = \lim_{s \to \infty}
M^{(k)}(w_{t_s}) = M^{(k)}\left(\lim_{s \to \infty} w_{t_s} \right) =
M^{(k)}(w^{(0)})$. Hence, for all $k \geq 0$, let
$w^{(k)} \equiv M^{(k)}(w^{(0)}) = \lim_{s \to \infty} w_{t_s + k}$,
so that $(w^{(0)},w^{(1)},w^{(2)},\ldots)$ is the trajectory of the dynamical system starting
from $w^{(0)} \in G$.

We now show that $w^{(k)} \in \Omega_\infty(G) \equiv \bigcap_{t=0}^{\infty} M^{(t)}(G)$ for
all $k \geq 0$. For all $k \geq 0$, because the sequence
$(w_{t_s-k})$ is in $G$, by the Bolzano-Weierstrass Theorem~\citep[][Theorem 16.4,
pp. 108]{BartleRA}, there exists a convergent subsequence
$(w_{t_{s_l}-k})$, also in $G$, of  $(w_{t_s-k})$. Let $w^{(-k)} \equiv \lim_{l \to \infty} w_{t_{s_l}
  -k}$. Note that we have
\begin{align*}
M^{(k)}(w_{t_{s_l} -
  k}) &= w_{t_{s_l}}\\
\lim_{l \to \infty} M^{(k)}(w_{t_{s_l} -
  k}) &= \lim_{l \to \infty} w_{t_{s_l}}\\
M^{(k)}\left(\lim_{l \to \infty} w_{t_{s_l} -
  k} \right) &= w^{(0)}\\
M^{(k)}(w^{(-k)}) &= w^{(0)}  
\end{align*}
which implies that $w^{(0)} \in \Omega_\infty(G)$ and, in turn, that $w^{(k)} \in \Omega_\infty(G)$ for all $k \geq 0$.

Let $p_s \equiv t_{s+1} - t_s$. For $s$ large enough, we can think of
$p_s$ as  the ``return time'' of the original trajectory of the
sequence $(w_t)$ near $w^{(0)}$.

Consider the subsequence $(w_{t_s})$ that converges to $w^{(0)}$. For
notational convenience, let $t_0 \equiv 1$ so that $p_0 \equiv t_1 -
t_0$. For all positive natural numbers $T$,
let $S^{(T)} = \sup \{ s | t_{s-1} \leq T < t_s \}$ and $p_s^{(T)} =
\min(t_{s+1},T+1) - t_s$ for all $s=0,1,\ldots,S^{(T)}-1$ (i.e.,
$p_{S^{(T)}-1}^{(T)} = T-t_{S^{(T)}-1}+1$ and, for $0 \leq s < S^{(T)}-1$, $p_s^{(T)} = t_{s+1} -
t_s = p_s$).

Let $f$ be a continuous real-valued function on $G$. Because $f$ is
continuous and $G$ is compact, there exist
real-valued constants $u^*$ and $v^*$
such that $u^* = \sup_{w \in G} f(w)$ and $v^* = \inf_{w \in G}
f(w)$~\citep[][Maximum and Minimum Value Theorem 22.6, pp. 154]{BartleRA}, which implies that $|f(w)-f(w')| \leq \tau^*
\equiv u^* - v^*$ for all $w,w' \in G$.

We can express the Birkhoff average of $f$, with respect to $w_1$ and
$M$, as a sequence of weighted averages:
\begin{align*}
\frac{1}{T} \sum_{t=1}^T f(w_t) &= \sum_{s = 0}^{S^{(T)}-1}
                              \frac{p_s^{(T)}}{T} \left(
                              \frac{1}{p_s^{(T)}}
                              \sum_{k=0}^{p_s^{(T)} - 1}
                              f(w_{t_s + k}) \right)
\end{align*}
To simplify notation, let
\[
  \widebar{F}_T \equiv \widebar{F}_T(w_1) \equiv \frac{1}{T}
  \sum_{t=1}^T f(w_t) \; , 
\]
\[
  \widetilde{F}^{\left(p_s^{(T)}\right)}_s((w_{t_s+k})) \equiv \frac{1}{p_s^{(T)}}
  \sum_{k=0}^{p_s^{(T)} - 1} f(w_{t_s + k})
\]
and \[ \widetilde{F}_T((w_{t_s+k})) \equiv \sum_{s = 0}^{S^{(T)}-1}
\frac{p_s^{(T)}}{T} \widetilde{F}^{\left(p_s^{(T)}\right)}_s((w_{t_s+k})) \; , \]
so that we can
then express the last equality simply as
$\widebar{F}_T =
\widetilde{F}_T((w_{t_s+k}))$. To make some sense of our notational
choices, first note that $w_t
= M(w_{t-1})$ for all $t>1$, so that the entire sequence
$(w_t)$ is solely a function of the
initial element $w_1$. Note also that $\widetilde{F}_T$ operates on
the sequence $(w_{t_s+k})$ because 
implicit in the notation of the sequence $(w_{t_s+k})$ we
have $k \in \{0,1,\ldots,p_s-1\}$ for
all $s$ so that the sequence $(w_t)$ and
$(w_{t_s + k})$ are the same. The separate
notation for $\widetilde{F}_T$ will
become clearer soon once we let
$\widetilde{F}_T$ operate on the 
sequence $(w^{(k)})$ instead of the
sequence $(w_{t_s+k})$ 
to obtain
\begin{align*}
  \widetilde{F}_T((w^{(k)})) = \sum_{s = 0}^{S^{(T)}-1}
                              \frac{p_s^{(T)}}{T} \widetilde{F}^{\left(p_s^{(T)}\right)}_s((w^{(k)})) \; .
\end{align*}     
This highlights that the weighted averages expressed in
$\widetilde{F}_T((w_{t_s+k}))$ (i.e., each average
$\widetilde{F}^{\left(p_s^{(T)}\right)}_s((w^{(k)}))$ is weighted by a factor $\frac{p_s^{(T)}}{T}$) are related to the sequence $(w^{(k)})$
in $\Omega_\infty(G)$ because
$\lim_{s\to \infty} f(w_{t_s + k}) = f( \lim_{s\to \infty} w_{t_s +
  k}) = f(w^{(k)})$ for all $k$. Intuitively, we expect the average
$\widetilde{F}^{\left(p_s^{(T)}\right)}_s((w_{t_s+k}))$ to be
close to the average
$\widetilde{F}^{\left(p_s^{(T)}\right)}_s((w^{(k)}))$ for large
enough $s$.

Now, let $p' \equiv \liminf_{s \to \infty} p_s$. If the sequence
$(p_s)$ is bounded from above (note that it is already bounded from
below by $0$), then we have $p' < +\infty$ exists. Consider a
subsequence $(w_{t_{s_l}})$ of $(w_{t_s})$ such that $p_{s_l} =
p'$. Note that such a subsequence exists because $(p_s)$ is a sequence
of integers and $p'$ is an integer.  Note also that
$M(w_{t_{s_l}+p' - 1}) = M(w_{t_{s_l}+p_{s_l} - 1}) =
M(w_{t_{s_l+1}-1}) = w_{t_{s_l+1}}$, so that we have
\[ M(w^{(p' - 1)}) = M(\lim_{l \to \infty} w_{t_{s_l}+p' - 1}) = \lim_{l
  \to \infty} M(w_{t_{s_l}+p' - 1}) = \lim_{l \to \infty}
w_{t_{s_l+1}} = w^{(0)} \; . \] Hence, if $(p_s)$ is bounded, then the
sequence $(w^{(0)},w^{(1)},w^{(2)},\ldots)$ is a cycle of period
$p \leq p'$. In addition, this suggests that we could have selected
the subsequence $(w_{t_s})$ such that $\lim_{s \to \infty} p_s = p$.

Pick $\tau > 0$. Let $S_p$ such that $p_s = p$ for all $s \geq S_p$ and let
$S_{p,\tau} \geq S_p$ such that \[ \max_{k=0,\ldots,p-1}
|f(w_{t_s +k}) - f(w^{(k)})| < \tau \] for all $s \geq
S_{p,\tau}$.  For  sufficiently large $T$, we have 
\begin{align*}
\left| \widebar{F}_T - \widetilde{F}_T((w^{(k)}))
                 \right| = & \left| \widetilde{F}_T((w_{t_s + k}))
            -  \widetilde{F}_T((w^{(k)})) \right|\\
  = & \left| \sum_{s = 0}^{S^{(T)}-1}
                              \frac{p_s^{(T)}}{T} \left(
                              \frac{1}{p_s^{(T)}}
                              \sum_{k=0}^{p_s^{(T)} - 1}
                                     (f(w_{t_s + k}) - f(w^{(k)})) \right)
    \right|\\
  \leq &\sum_{s = 0}^{S^{(T)}-1}
                              \frac{p_s^{(T)}}{T} \left(
                              \frac{1}{p_s^{(T)}}
                              \sum_{k=0}^{p_s^{(T)} - 1}
                                     \left|  f(w_{t_s + k}) -
    f(w^{(k)}) \right| \right) \\
  = & \sum_{s = 0}^{S_{p,\tau}-1}
                              \frac{p_s^{(T)}}{T} \left(
                              \frac{1}{p_s^{(T)}}
                              \sum_{k=0}^{p_s^{(T)} - 1}
                                     \left|  f(w_{t_s + k}) -
      f(w^{(k)}) \right| \right) +\\
  & \sum_{s = S_{p,\tau}}^{S^{(T)}-1}
                              \frac{p_s^{(T)}}{T} \left(
                              \frac{1}{p_s^{(T)}}
                              \sum_{k=0}^{p_s^{(T)} - 1}
                                     \left|  f(w_{t_s + k}) -
    f(w^{(k)}) \right| \right) 
\end{align*}
The first term on the right-hand-side of the bound
\begin{align*}
\sum_{s = 0}^{S_{p,\tau}-1}
                              \frac{p_s}{T} \left(
                              \frac{1}{p_s}
                              \sum_{k=0}^{p_s - 1}
                                     \left|  f(w_{t_s + k}) -
    f(w^{(k)}) \right| \right) &= \sum_{s = 0}^{S_{p,\tau}-1}
                              \frac{p_s}{T} \left(
                              \frac{1}{p_s}
                              \sum_{k=0}^{p_s - 1}
                                     \left|  f(w_{t_s + k}) -
                                 f(w^{(k)}) \right| \right) \\
                               &\leq  
                              \frac{1}{T} \sum_{s = 0}^{S_{p,\tau}-1} \sum_{k=0}^{p_s - 1}
                                     \tau^* \\
  &= \frac{t_{S_{p,\tau}}}{T} \tau^* 
\end{align*}
goes to zero with $T$. We can further decompose the second term as
\begin{align*}
& \sum_{s = S_{p,\tau}}^{S^{(T)}-1}
                              \frac{p_s^{(T)}}{T} \left(
                              \frac{1}{p_s^{(T)}}
                              \sum_{k=0}^{p_s^{(T)} - 1}
                                     \left|  f(w_{t_s + k}) -
                 f(w^{(k)}) \right| \right) \\
  = & \sum_{s = S_{p,\tau}}^{S^{(T)}-2}
                              \frac{p_s^{(T)}}{T} \left(
                              \frac{1}{p_s^{(T)}}
                              \sum_{k=0}^{p_s^{(T)} - 1}
                                     \left|  f(w_{t_s + k}) -
                                 f(w^{(k)}) \right| \right) + \\
  & \frac{p_{S^{(T)}-1}^{(T)}}{T} \left(
                              \frac{1}{p_{S^{(T)}-1}^{(T)}}
                              \sum_{k=0}^{p_{S^{(T)}-1}^{(T)} - 1}
                                     \left|  f(w_{t_{S^{(T)}-1} + k}) -
                                 f(w^{(k)}) \right| \right)\\
  = & \sum_{s = S_{p,\tau}}^{S^{(T)}-2}
                              \frac{p_s}{T} \left(
                              \frac{1}{p_s}
                              \sum_{k=0}^{{p_s} - 1}
                                     \left|  f(w_{t_s + k}) -
    f(w^{(k)}) \right| \right) +\\
  & \frac{1}{T} 
                              \sum_{k=0}^{p_{S^{(T)}-1}^{(T)} - 1}
                                     \left|  f(w_{t_{S^{(T)}-1} + k}) -
                                 f(w^{(k)}) \right| 
\end{align*}
The second term in the last expression
\begin{align*}
\frac{1}{T} \sum_{k=0}^{p_{S^{(T)}-1}^{(T)} - 1}
                                     \left|  f(w_{t_{S^{(T)}-1} + k}) -
                                 f(w^{(k)}) \right| &< \frac{1}{T}
                                                      \sum_{k=0}^{p_{S^{(T)}-1}^{(T)}
                                                      - 1} \tau = \frac{1}{T}
                                                      p_{S^{(T)}-1}^{(T)}
                                                      \tau \leq \frac{1}{T} p \tau
\end{align*}
goes to zero with $T$, while the first term
\begin{align*}
\sum_{s = S_{p,\tau}}^{S^{(T)}-2}
                              \frac{p_s}{T} \left(
                              \frac{1}{p_s}
                              \sum_{k=0}^{p_s - 1}
                                     \left|  f(w_{t_s + k}) -
  f(w^{(k)}) \right| \right) &= \sum_{s = S_{p,\tau}}^{S^{(T)}-2}
                              \frac{p}{T} \left(
                              \frac{1}{p}
                              \sum_{k=0}^{p- 1}
                                     \left|  f(w_{t_s + k}) -
                               f(w^{(k)}) \right| \right)\\
  &< \sum_{s = S_{p,\tau}}^{S^{(T)}-2}
                              \frac{p}{T} \left(
                              \frac{1}{p}
                              \sum_{k=0}^{p - 1}
    \tau \right) \\
  &= \sum_{s = S_{p,\tau}}^{S^{(T)}-2}
    \frac{p}{T} \tau \\
                             &= \frac{p (S^{(T)}- S_{p,\tau} - 1)}{T} \tau \\
  &\leq \tau \; ,
\end{align*}
where the last inequality follows because
\begin{align*}
T \geq & t_{S^{(T)}-1} = t_{S_{p,\tau}} +
  \sum_{s=S_{p,\tau}}^{S^{(T)}-2} t_{s+1} - t_s = t_{S_{p,\tau}} +
  \sum_{s=S_{p,\tau}}^{S^{(T)}-2} p_s = t_{S_{p,\tau}} +
                       \sum_{s=S_{p,\tau}}^{S^{(T)}-2} p \\
  = & t_{S_{p,\tau}} + p (S^{(T)}-
  S_{p,\tau} -1) \\ \geq & p (S^{(T)}-
  S_{p,\tau} -1)
\end{align*}

Putting everything together, we obtain that for all $\tau > 0$
\begin{align*}
\lim_{T \to \infty} \left| \widebar{F}_T - \widetilde{F}_T((w^{(k)})) 
  \right| < \tau 
\end{align*}
which implies that
\[
  \lim_{T \to \infty} \widebar{F}_T - \widetilde{F}_T((w^{(k)}))  = 0
\]
This also implies that 
\[
  \lim_{T \to \infty} \widebar{F}_T = \widetilde{F}^{(p)} \equiv
  \frac{1}{p} \sum_{k=0}^{p-1} f(w^{(k)}) 
\]
because, as we will show next,
\[
 \lim_{T \to \infty} \widetilde{F}_T((w^{(k)})) =
 \widetilde{F}^{(p)}
\]
so that
\begin{align*}
\lim_{T \to \infty} \widebar{F}_T = & \lim_{T \to \infty} (\widebar{F}_T
- \widetilde{F}_T((w^{(k)}))) + \widetilde{F}_T((w^{(k)})) \\
= & \left( \lim_{T \to \infty} \widebar{F}_T
- \widetilde{F}_T((w^{(k)})) \right)  + \left( \lim_{T \to \infty}
\widetilde{F}_T((w^{(k)})) \right)
= 0 + \widetilde{F}^{(p)} \\
= & \widetilde{F}^{(p)} \; ,
\end{align*}
as claimed.

To prove our remaining claim, for sufficiently large $T$, we can
decompose $\widetilde{F}_T((w^{(k)}))$ as follows.
\begin{align*}
\widetilde{F}_T((w^{(k)}))= & \sum_{s = 0}^{S_{p,\tau}-1}
                              \frac{p_s}{T} \left(
                              \frac{1}{p_s}
                              \sum_{k=0}^{p_s - 1}
                                     f(w^{(k)}) \right) + \\
& \sum_{s = S_{p,\tau}}^{S^{(T)}-2}
                              \frac{p}{T} \left(
                              \frac{1}{p}
                              \sum_{k=0}^{p - 1}
                                     f(w^{(k)}) \right) +\\
&
                              \frac{1}{T} 
                              \sum_{k=0}^{p_{S^{(T)}-1}^{(T)} - 1}
                                     f(w^{(k)}) 
\end{align*}
The first and third term in the decomposition goes to zero with $T$ because
\begin{align*}
 \frac{t_{S_{p,\tau}}}{T} v^* \leq \sum_{s = 0}^{S_{p,\tau}-1}
                              \frac{p_s}{T} \left(
                              \frac{1}{p_s}
                              \sum_{k=0}^{p_s - 1}
                                     f(w^{(k)}) \right) \leq
  \frac{t_{S_{p,\tau}}}{T} u^*
\end{align*} 
and
\begin{align*}
\frac{1}{T} v^* \leq \frac{1}{T} \sum_{k=0}^{p_{S^{(T)}-1}^{(T)} - 1}
                                     f(w^{(k)}) \leq  \frac{p}{T} u^*
  \; .
\end{align*}
For the second term in the composition, we have
\begin{align*}
\sum_{s = S_{p,\tau}}^{S^{(T)}-2}
                              \frac{p}{T} \left(
                              \frac{1}{p}
                              \sum_{k=0}^{p - 1}
                                     f(w^{(k)}) \right) &= \frac{p
                                                          (S^{(T)}-S_{p,\tau}
                                                          -1)}{T} \frac{1}{p}
                              \sum_{k=0}^{p - 1}
                                     f(w^{(k)}) \; .
\end{align*}
Noting that
\begin{align*}
T \leq t_{S^{(T)}} &= t_{S_{p,\tau}} + t_{S^{(T)}} - t_{S^{(T)}-1} +
                     \sum_{s=S_{p,\tau}}^{S^{(T)}-2} t_{s+1} - t_s \\
  &=
                     t_{S_{p,\tau}} + p_{S^{(T)}-1} + 
  \sum_{s=S_{p,\tau}}^{S^{(T)}-2} p_s = t_{S_{p,\tau}} + p +
  \sum_{s=S_{p,\tau}}^{S^{(T)}-2} p = t_{S_{p,\tau}} + p + p (S^{(T)}-
  S_{p,\tau} -1) \; ,
\end{align*}
we obtain
\begin{align*}
\frac{p (S^{(T)}-
  S_{p,\tau} -1)}{T} \geq 1 - \frac{t_{S_{p,\tau}} + p}{T} \; .
\end{align*}
Hence, we have $\lim_{T \to \infty} \frac{p (S^{(T)}-
  S_{p,\tau} -1)}{T} = 1$, and
\begin{align*}
\lim_{T \to \infty} \sum_{s = 0}^{S^{(T)}-1}
                              \frac{p_s^{(T)}}{T} \left(
                              \frac{1}{p_s^{(T)}}
                              \sum_{k=0}^{p_s^{(T)} - 1}
                                     f(w^{(k)}) \right) = \frac{1}{p} \sum_{k=0}^{p - 1}
                                     f(w^{(k)})
\end{align*}
Putting everything together, we obtain that if the sequence $(p_s)$ is bounded, then the sequence $(w_t)$ converges to a cycle and
Birkhoff averages always exist (i.e., converge); that is, we have $\lim_{T \to \infty} \frac1T \sum_{t=1}^T f(w_t)$
\begin{align*}
= & \lim_{T \to
                                                  \infty} \frac1T
                                                  \sum_{t=1}^T f(w_t)
                                                  - \sum_{s = 0}^{S^{(T)}-1}
                              \frac{p_s^{(T)}}{T} \left(
                              \frac{1}{p_s^{(T)}}
                              \sum_{k=0}^{p_s^{(T)} - 1}
                                                    f(w^{(k)}) \right)
                                                  + \sum_{s = 0}^{S^{(T)}-1}
                              \frac{p_s^{(T)}}{T} \left(
                              \frac{1}{p_s^{(T)}}
                              \sum_{k=0}^{p_s^{(T)} - 1}
                                                    f(w^{(k)}) \right) 
  \\
  = & \left( \lim_{T \to
                                                  \infty} \frac1T
                                                  \sum_{t=1}^T f(w_t)
                                                  - \sum_{s = 0}^{S^{(T)}-1}
                              \frac{p_s^{(T)}}{T} \left(
                              \frac{1}{p_s^{(T)}}
                              \sum_{k=0}^{p_s^{(T)} - 1}
                                                    f(w^{(k)}) \right) \right)
    + \\
  & 
    \left( \lim_{T \to
                                                  \infty} \sum_{s = 0}^{S^{(T)}-1}
                              \frac{p_s^{(T)}}{T} \left(
                              \frac{1}{p_s^{(T)}}
                              \sum_{k=0}^{p_s^{(T)} - 1}
                                                    f(w^{(k)}) \right) \right)
  \\
  = & 0 +
\frac{1}{p} \sum_{k=0}^{p - 1}
    f(w^{(k)}) \\
  = & 
\frac{1}{p} \sum_{k=0}^{p - 1}
    f(w^{(k)})
\end{align*}

Intuitively, one would think that if $(p_s)$ is unbounded, then the
existence of the Birkhoff average would be covered by Birkhoff's
Convergence Theorem and its value would be $\lim_{K \to \infty}
\frac{1}{K} \sum_{k=0}^{K-1} f(w^{(k)})$. But the fact that no such
general result seems to exist in the current state of knowledge of the
literature on dynamical system and ergodic theory suggests that that
intuition is wrong in general. So, even if Condition~\ref{C:NoTies2}
(No Ties Eventually) holds, we cannot rely on general results from that
theory to provide us with universal convergence for Optimal
AdaBoost from every initial weight in $\Delta_m^\circ$. We would have to look more closely at the specific
properties of Optimal AdaBoost.

The non-expansion conditions, Conditions~\ref{cond:containment} and~\ref{cond:containment_tailored}, stated relatively generally previously,
may be one way to address the particular universal convergence for
Optimal AdaBoost. Assuming $\ell_1$ distance (i.e., $d(w,w') \equiv \| w -
w' \|_1 \equiv \sum_i | w(i) - w'(i) |$), a sufficient condition for
those non-expansion conditions to hold
is that for every pair of weights $\omega_1, \omega_2 \in \Delta_m^+$ such that
$\eta^{\omega_1} = \eta^{\omega_2} = \eta \in \Mcal$ (i.e., $\omega_1,
\omega_2 \in \pi^*(\eta)$),
we have \[ \|
\omega_1 - \omega_2 \|_1 \geq \|
\A(\omega_1) - \A(\omega_2) \|_1 \; . \]
This condition may also be necessary for non-expansion. (Non-expansion
may not be generally necessary for the existence of Birkhoff averages though, as
simple examples intuitively suggest.) Using properties of the AdaBoost inverse
discussed in Appendix~\ref{app:Ainv}, we can show that this condition is
equivalent to \[ (1 - \eta) \cdot | \omega_1 - \omega_2 | \geq \eta
\cdot | \A(\omega_1) - \A(\omega_2) | \; , \] which under Condition~\ref{assume:WeakLearn} (Weak
Learning) is in turn equivalent to \[ \frac{1}{2} \| w_1 - w_2 \|_1 \geq
\eta \cdot | \A(\omega_1) - \A(\omega_2) | \; . \] Hence, the condition 
depends on specific characteristics of the mistake matrices induced by
the data and the class of weak classifiers. Having said that, recall
however that, by the properties of the Optimal AdaBoost update, $\eta
\cdot \A(\omega_1) = \eta \cdot \A(\omega_2) = \frac{1}{2}$, which
may suggest that  the condition may hold more generally.

\subsubsection{Putting everything together: The time averages of
  Optimal AdaBoost (essentially) converge}

The following theorem states a key technical result of this
paper from which most of the results stated in the next section
follow. It serves as a \emph{template} that summarizes the results derived in this
section on the time averages of several approximating versions of
Optimal AdaBoost. This is because the theorem holds if
we replace the phrases ``the Optimal  AdaBoost update $\A$'' and ``for
any function $f$'' with any of the following:
\begin{enumerate}
\item ``the $\tau$-Finite-Precision Optimal AdaBoost update
  $\approxA_{\tau,\mathrm{disc}}$'' and ``for any function $f$''
\item ``the step function $\approxA_{(\tau_*,1-\tau),\mathrm{step}}$
  Lebesgue-almost uniformly approximating Optimal AdaBoost to
  within an arbitrarily small error $\tau$'' and ``for any function $f$''
\item ``the continuous function $\approxA_{(\tau_*,1-\tau,1-\tau),\mathrm{cont}}$
  Lebesgue-almost uniformly approximating Optimal
  AdaBoost within an arbitrarily small error $\tau$'' and ``for any function $f \in C_\Mcal$''
\item ``the original/exact Optimal
  AdaBoost $\A$, if, in addition, Condition~\ref{cond:containment} (Non-Expansive)
    or, Conditions~\ref{C:NoTies2} (No Ties
  Eventually) and~\ref{cond:containment_tailored} (Tailored Non-Expansive)
  holds,'' and ``for any function $f \in C_\Mcal$''
\end{enumerate}
The theorem also holds, under the respective additional
conditions, for ``the \emph{original/exact} Optimal
  AdaBoost $\A$'' if we replace ``for any function $f$'' and
``for every $w_1 \in \Delta_m^\circ$'' with any of the following:
\begin{enumerate}
\item ``for any function $f \in L_1(\mu)$'' and, under Condition~\ref{C:NoTies} (No Ties),
\begin{enumerate}
\item   ``for $\mu$-almost every
  $w_1 \in \Omega$", where $\Omega$
  and $\mu$ are as in Theorem~\ref{the:avgcvg1}, or
\item ``for $\nu_0$-almost every
  $w_1 \in \Delta_m^+$", where $\nu_0$ are as in Corollary~\ref{cor:avgcvg_all_finite_T};
\end{enumerate}
\item ``for any function $f \in C(\Delta_m)$'' and ``for $\nu$-almost every
  $w_1 \in \Delta_m^+$", under
  Conditions~\ref{C:NoTies},~\ref{C:NonEmptyOmega}
  ($\Omega_\infty^+ \neq \emptyset$), and~\ref{C:NonExpansiveGlobal}
  (Globally Non-Expansive), where $C(\Delta_m)$ and $\nu$
  is as in Corollary~\ref{cor:avgcvg_all}
\item ``for any function $f \in C_\Mcal^+$'' and ``for $\nu$-almost every
  $w_1 \in \Delta_m^+$", under  Conditions~\ref{C:NoTies2} (No Ties
  Eventually) and~\ref{C:NonExpansive} (Locally Non-Expansive), where
  $C_\Mcal^+$ and $\nu$
  is as in Corollary~\ref{cor:avgcvg_all_ft}.  
\end{enumerate}
\begin{theorem}{{\bf (Averages over an AdaBoost Sequence of Example Weights Converge.)}} 
\label{the:avgcvg}
Suppose Conditions~\ref{cond:nathypo} (Natural Weak-Hypothesis
Class) and ~\ref{assume:WeakLearn} (Weak Learning) hold. 
\emph{For any function $f$},
\emph{the Optimal-AdaBoost update $\A$} has the property that
$\frac{1}{T} \sum_{t=0}^{T-1} f(\A^{(t)}(w_1))$ converges for 
every $w_1 \in \Delta_m^\circ$.  
\end{theorem}

{\bf {\em Before continuing, we would like to point out that we state the results of the next
section, and those in their associated appendices, in a template-like matter, just like we did for the last
theorem. This is because, just like for the last theorem, they hold if we
replace the phrases regarding the Optimal AdaBoost algorithm version, the class from which
the function being
averaged is taken, and the set from which the initial weight is taken with the corresponding phrase for any of the versions of the
algorithm as listed above. In doing so, our intention is not to
offuscate our results, but to avoid unnecessary clutter in the
statements of those results.}}

\subsection{Preliminaries to the study of the convergence of the
  Optimal-AdaBoost classifier}
\label{sec:adaboost_conv_prelim}

The study of the convergence of the AdaBoost classifier and its
implications is the main goal of the next subsection (Section~\ref{sec:convres}).  Here we provide some preliminary definitions, and introduce some useful concepts and mathematical results.

\begin{definition}\label{D:H_and_F}
Starting from some initial example weights $w_1 \in
\Delta_m$, the \emph{final classifier that AdaBoost outputs after $T$
  rounds}, which we denote by $H_T : \X \to \{-1,+1\}$, labels input examples $x \in \X$ by computing
\(
  H_T(x) \equiv H^{w_1}_T(x) \equiv \sign{F_T(x)} = \sign{\sum_{t=1}^{T} \alpha_t h_t(x)} \; ,
\)
where we define the \emph{final real-valued function $F_T : \X \to \R$
  that AdaBoost built to use for classification} as $F_T(x)\equiv F_T^{w_1}(x)\equiv
\sum_{t=1}^{T} \alpha_t h_t(x)$.
\end{definition}
It is very important to keep in mind that the sequence of
$\alpha_t$'s and $h_t$'s are really functions of the initial example
weights $w_1$. Thus, the functions $H_T$ and $F_T$ are functions of
$w_1$ too.

Carefully note that in the definition of $H_T$ above (Definition~\ref{D:H_and_F}), the weak-hypothesis $h_t$ corresponds to an effective representative hypothesis in $\Hrep \subset \Hypo$, \emph{not} a label dichotomy in 
$\Dich(\Hypo,S)$. Certainly for $x\in S$, $F_T(x)$ does not converge
as the total number of rounds $T$ of AdaBoost approaches infinity. In
fact, if Condition~\ref{assume:WeakLearn} (Weak Learning) holds, that value must be growing
at least \emph{linearly} in
$T$ (see Part 1 of Proposition~\ref{P:secondary}).
So, what do we mean by ``convergence of the AdaBoost classifier'' exactly then?  We can replace $\sign{F_T(x)}$ with $\sign{\frac 1T F_T(x)}$ without changing the classification label output by $H_T(x)$.
Another alternative 
is to use the concept of margins, which we now formally define.
\begin{definition}
Denote the \emph{normalized weak-hypothesis weights after $T$ rounds of
  Optimal AdaBoost} by $\widetilde{\alpha}_t \equiv
\frac{\alpha_t}{\sum_{t=1}^T \alpha_t}$ for all $t=1,\ldots,T$.  Note
that by Condition~\ref{assume:WeakLearn} (Weak Learning), we have $\alpha_t$ is lower
bounded by a positive constant for all finite $T$, which implies that
$\widetilde{\alpha}_t < 1$.
Hence, the normalized weak-hypothesis weights are
well-defined for any finite $T$ and
we can think of them as a probability distribution over the set of
indexes to the
rounds $\{1,\ldots,T\}$. 
For all initial example weights $w_1 \in
\Delta_m$, define the \emph{margin function after $T$ rounds of
  Optimal AdaBoost} $\margin_T : \X \to [-1,1]$ as, for all $x \in
\X$, $\margin_T(x) \equiv \margin_T^{w_1}(x) \equiv \sum_{t=1}^T
\widetilde{\alpha}_t h_t(x)$, the \emph{margin} of input $x$ with
respect to $H_T$.  The range of $\margin_T$ is $[-1,1]$ because the
range of each $h_t$ is $\{-1,1\}$, and we can think of the defining
expression as an expectation over the weak-hypotheses selected at each
round with respect to the normalized selected-weak-hypothesis weights.
Similarly, for all initial example weights $w_1 \in
\Delta_m^\circ$, define the \emph{empirical margin function after $T$ rounds
  of Optimal AdaBoost} $\widehat{\margin}_T : U \to [-1,1]$ as, for
all $x \in U$, $\widehat{\margin}_T(x) \equiv \widehat{\margin}_T^{w_1}(x)
\equiv \margin_T(x)$, the \emph{empirical margin} of input example $x$ with respect to $H_T$.
\end{definition}
Hence, we can equivalently use $\sign{\margin_T(x)}$ instead of $H_T(x)$ for classification.
Then, as we will shortly prove, under certain conditions, if $\frac 1T
F_T(x)$ or $\margin_T(x)$ converges for all $x \in \X$, so does
$\sign{F_T(x)}$ and $\sign{\margin_T(x)}$, respectively. 

\subsection{Technical results on the convergence of Optimal AdaBoost
  and related objects}
\label{sec:convres}

We can use Theorem~\ref{the:avgcvg} 
to obtain
a convergence result for $\frac 1T F_T(x^{(i)})$ for $x^{(i)} \in U$
in the training examples, as stated in the next theorem.  Note that in
the next theorem we depart from the standard notations of
$F_T(x^{(i)}) = \sum_{t=1}^{T}\alpha_t h_t(x^{(i)})$.  The new
notation defines $F_T(x^{(i)})$ in terms of the effective
\emph{mistake} dichotomies in $\Mcal$
constructed from the \emph{label} dichotomies in $\Dich(\Hypo,S)$, not
directly on the effective representative hypothesis $h_t \in \EHypo$
output by the weak learner in Optimal AdaBoost.  The elements of these
mistake dichotomies are defined over $\{0,1\}$, unlike the hypotheses
whose output is in $\{-1,1\}$. Thus, we need to scale and translate
them appropriately.  The new notation for $F_T(x^{(i)})$ results in
the exact same values as the one defined over selected effective
representative hypotheses. To avoid confusion, we denote the
corresponding function over the \emph{set} of input training examples
$U$, generated by Optimal AdaBoost starting from initial example
weights $w_1 \in \Delta_m$, by $\widehat{F}_T \equiv \widehat{F}_T^{w_1}$. Using this notation, we have that for all $x
\in U$, $F_T(x) = \widehat{F}_T(x)$.

Before continuing, we remind the reader that the sequences of $\epsilon_t$'s,
$\alpha_t$'s, and $h_t$'s, as well as the functions $H_T$,
$F_T$, $\widehat{F}_T$, $\margin_T$, and $\widehat{\margin}_T$, referred to in the discussion
and statements that follow, all depend on the initial example
weights $w_1 \in \Delta_m$.

\begin{theorem}\label{T:classifier-convergence}
{\bf (The Average of the Real-Valued Function AdaBoost Builds By Com\-bin\-ing Weak Clas\-si\-fiers Con\-verges for All Training Examples.)}
Sup\-pose Con\-di\-tions~\ref{cond:nathypo} (Natural Weak-Hy\-poth\-e\-sis
Class) and~\ref{assume:WeakLearn} (Weak Learning) 
hold.
Let $\widehat{F}_T : U \to \R$, defined as $\widehat{F}_T(x^{(i)})
\equiv \widehat{F}_T^{w_1}(x^{(i)})
\equiv \sum_{t=1}^{T} \alpha_t (2\eta_t(i)-1)$, be the
Optimal-AdaBoost classifier function at round $T$, starting from
initial example weights $w_1 \in \Delta_m^\circ$, defined only over
the set of (unique) input training examples $x^{(i)} \in U$, and in
terms of $\eta_t \in \{0,1\}^m$, for $t=1,\dots,T$, corresponding to
the mistake dichotomy of the representative hypothesis $h_t =
h^{\eta_t}$ output by the weak learner, such that, for all
$l=1,\ldots,m$, $\eta_t(l) = \indicator{h_t(x^{(l)}) \neq y^{(l)}}$.
For all $x^{(i)}\in U$, the limit $\lim_{T\to\infty}\frac 1T
\widehat{F}_T(x^{(i)})$ exists for every initial example
weights $w_1 \in \Delta_m^\circ$ used for the Optimal AdaBoost algorithm.
\end{theorem}
\begin{proof}
Let $\epsilon_* > 0$ be lower-bound on the weighted error of Optimal
AdaBoost guaranteed by Theorem~\ref{T:lower-bound+} from \emph{any} initial
example weight in $\Delta_m^\circ$ (i.e., independent of $w_1$)
after $T > n+1$ rounds under Conditions~\ref{cond:nathypo} (Natural Weak-Hypothesis Class)
and~\ref{assume:WeakLearn} (Weak Learning).
Let $\alpha^* \equiv \frac{1}{2} \ln\frac{1-\epsilon_*}{\epsilon_*} <
+\infty$. Define the truncation $\alpha_{\alpha^*} : \Delta_m \to \R^+$ as
$\alpha_{\alpha^*}(w) \equiv \max(\alpha(w),\alpha_*)$, and note
that Note that
$\alpha_{\alpha^*}(w)$ is continuous on $\Delta_m$ and that for any initial weight $w_1 \in\Delta_m^\circ$ the sequence
$(w_t)$ of
example weights that Optimal AdaBoost generates satisfies $\alpha(w_t)
= \alpha_{\alpha^*}(w_t)$ for all $t > n+1$. 
Let 
\(
A_\eta(T) \equiv \sum_{t=1}^{T} \alpha(w_t)\chi_{\pi^*(\eta)}(w_t).
\)
Because for each $\eta \in \Mcal$ the
function $\alpha_{\alpha^*} \circ \chi_{\pi^*(\eta)}$ is continuous when
viewed as a map from $\pi^*(\eta)$ to $\R$ only, we have $\alpha_{\alpha^*} \circ \chi_{\pi^*(\eta)} \in C_\Mcal$
(Definition~\ref{def:C_Mcal}) and, by Theorem~\ref{the:avgcvg}, we
have that 
\begin{align*}
  A_\eta^* \equiv \lim_{T\to\infty}\frac 1T A_\eta(T) & =
\lim_{T\to\infty}\frac 1T \sum_{t=n+2}^{T}
                                                        \alpha(w_t)\chi_{\pi^*(\eta)}(w_t) \\
  &= \lim_{T\to\infty}\frac 1T \sum_{t=n+2}^{T}
\alpha_{\alpha^*}(w_t)\chi_{\pi^*(\eta)}(w_t) = \lim_{T\to\infty}\frac 1T \sum_{t=1}^{T}
\alpha_{\alpha^*}(w_t)\chi_{\pi^*(\eta)}(w_t)
\end{align*}
exists for all $\eta \in
\Mcal$ 
and 
every $w_1 \in \Delta_m^\circ$ used to initialize the example
weights in the Optimal
AdaBoost algorithm. 

We are restricting ourselves to the set $U \subset \X$, hence we can write, for all $x^{(i)} \in U$,
\(
\frac 1T \widehat{F}_T(x^{(i)}) = \sum_{\eta\in \Mcal} \frac 1T A_\eta(T)(2\eta(i)-1).
\)
Finally, by taking limits we see that
\[
\lim_{T\to\infty} \frac 1T \widehat{F}_T(x^{(i)} ) = \sum_{\eta\in \Mcal} \lim_{T\to\infty}\left[\frac 1T A_\eta(T)\right](2\eta(i)-1)
= \sum_{\eta\in \Mcal} A_\eta^*(2\eta(i)-1)  ,
\]
exists for every $w_1 \in \Delta_m^\circ$. \qed
 \end{proof}
As a corollary to this theorem, we can show convergence for the margin
on any example in the training set.~\footnote{Note that this corollary
  does not say anything about whether Optimal AdaBoost maximizes the
  margins.
}
\begin{corollary}{\bf (The Margin of Every Training Example Converges.)}
\label{C:example_margins_convergence}
Suppose Conditions~\ref{cond:nathypo} (Natural Weak-Hypothesis
Class) and~\ref{assume:WeakLearn} (Weak Learning) 
hold.
For all $x \in U$, the limit $\lim_{T\to\infty}
\widehat{\margin}_T(x)$ exists for every initial example
weights $w_1 \in \Delta_m^\circ$.
\end{corollary}
\begin{proof}
Let the truncation $\alpha_{\alpha^*} : \Delta_m \to \R$ be as defined in the proof of
Theorem~\ref{T:classifier-convergence}. 
Because
$\alpha_{\alpha^*}$ is continuous on $\Delta_m$, it is in $C_{\Mcal}$.  By
Theorem~\ref{the:avgcvg}, we have that
\begin{align*}
  \Theta \equiv \lim_{T\to\infty}\frac 1T \sum_{t=1}^{T} \alpha(w_t) =
  \lim_{T\to\infty}\frac 1T \sum_{t=n+2}^{T} \alpha(w_t) =
  \lim_{T\to\infty}\frac 1T \sum_{t=n+2}^{T} \alpha_{\alpha^*}(w_t) =  \lim_{T\to\infty}\frac 1T \sum_{t=1}^{T} \alpha_{\alpha^*}(w_t)
\end{align*}
exists 
for every $w_1 \in \Delta_m^\circ$.
From the Condition~\ref{assume:WeakLearn} (Weak Learning), we know that $\epsilon_t = \epsilon(w_t) < \frac 12 - \gamma$ for some $\gamma > 0$.  This gives us a lower bound $\alpha_*>0$ on $\alpha(w_t)$.  Using this lower bound, we see that
\[
\Theta = \lim_{T\to\infty}\frac 1T \sum_{t=1}^{T} \alpha(w_t) 
\geq \lim_{T\to\infty}\frac 1T (\alpha_*T) 
\geq \alpha_* 
> 0
\]
for every $w_1 \in \Delta_m^\circ$.

Now, we can say that
$\lim_{T\to\infty}T\left(\sum_{t=1}^{T}\alpha_t\right)^{-1} = \frac
1\Theta$,
for every $w_1 \in \Delta_m^\circ$.
Combining this with Theorem~\ref{T:classifier-convergence}, we have for all $x^{(i)}\in U$,
\begin{align*}
\lim_{T\to\infty} \widehat{\margin}_T(x^{(i)}) = & \lim_{T\to\infty} \widehat{F}_T(x^{(i)})\left(\sum_{t=1}^{T}\alpha_t\right)^{-1} 
=
\lim_{T\to\infty}\frac{\widehat{F}_T(x^{(i)})}{T}\left(\sum_{t=1}^{T}\alpha_t\right)^{-1}T \\
= & \frac 1\Theta \sum_{\eta\in \Mcal} A_\eta^*(2\eta(i)-1) 
= \sum_{\eta\in \Mcal} \widetilde{A}_\eta(2\eta(i)-1) \; ,
\end{align*}
for every $w_1 \in \Delta_m^\circ$,
where $\widetilde{C}_\eta \equiv A_\eta^*/ \Theta$ is a probability distribution over the
mistake dichotomies in $\Mcal$.
\qed
\end{proof}

Using convergence results about $\frac 1T F_T(x)$ and $\margin_T(x)$
for every $w_1 \in \Delta_m^\circ$, and $P$-almost surely for all $x$ in the training dataset (i.e., for all $x \in U$), we can establish
the convergence of the same functions almost surely on any $x$ \emph{outside} of the training dataset (i.e., for all $x \in \X$).  The upshot is that given such convergence results, we can say something strong about how the generalization error of the Optimal-AdaBoost classifier behaves in the limit.  Intuitively, if the Optimal-AdaBoost classifier is effectively converging, so should its generalization error.  \emph{This outlines one of the main contributions of this paper.}

But the extension of the convergence results to the whole input space $\X$ instead of just the unique input instances $U$ of the training set $S$ does not come without some difficulty.
On $S$, we know that $(2\eta(i) - 1)$ will correspond directly to some $h^{\eta}(x^{(i)})$.  However, outside of $U$ our mistake dichotomies $\eta$'s are no longer defined, because they are simply 0-1 vectors over the examples in $S$.  To evaluate $F_T(x)$ for an arbitrary $x \in \X$, we must appeal to the hypotheses 
selected from the hypothesis space, not just the mistake dichotomies they produced.

Let $\Hrep(\eta) = \left\{ h \in \Hrep | \, (2\eta(i)-1) = h(x^{(i)}) \text{ for all } x^{(i)}\in S\right\}$.  A key observation is that $\Hrep(\eta)$ induces an \emph{equivalence relation} or \emph{partition} on $\Hrep$: $\Hrep = \bigcup_{\eta \in \Mcal} \Hrep(\eta)$ and $\Hrep(\eta)\cap \Hrep(\eta') = \emptyset$ for any pair $\eta,\eta' \in \Mcal$.  All hypotheses in each equivalence class, from the perspective of Optimal AdaBoost, are indistinguishable in the sense that picking any of them will result in no change in the trajectory of $w_t$.  However, the weak learner might have a bias towards certain hypotheses in these classes.  For example, perhaps the weak learner will always pick the ``simplest'' hypothesis in $\Hrep(\eta)$, based on some simplicity measure (e.g., depth of a decision tree or its number of leaves).

Recall that even though our analysis, and sometimes our implementation
as well, uses the set of mistake dichotomies to implement/instantiate
the weak learner, we must have the corresponding representative
hypothesis from $\Hrep$ output by the weak learner and used to put
together the final AdaBoost classifier in order to classify new input samples.

\begin{theorem}{\bf (The Average of the Real-Valued Function Optimal AdaBoost Builds By Combining Weak Classifiers Converges.)}
\label{thm:classifier_fcn_convergence}
Suppose Conditions~\ref{cond:nathypo} (Natural Weak-Hypothesis
Class) and~\ref{assume:WeakLearn} (Weak Learning)
hold.
For every initial example weights $w_1 \in \Delta_m^\circ$, the limit
$\lim_{T \to \infty} \frac 1T F_T(x)$ exists for all $x\in \X$.
\end{theorem}
\begin{proof}
Let $\eta_t \equiv \eta^{w_t}$. Then the representative hypothesis selected at iteration $t$ is $h_t = h^{\eta_t}$.  This yields
\(
\frac 1T F_T(x) = \frac 1T \sum_{t=1}^{T} \alpha_t h_t(x) 
= \frac 1T \sum_{t=1}^{T} \alpha_t h^{\eta_t}(x) 
= \sum_{\eta\in \Mcal} \frac 1T A_{\eta}(T) h^{\eta}(x) \; ,
\)
where $A_{\eta}$ is defined in the same way as in the proof of
Theorem~\ref{T:classifier-convergence}.
As in that same proof, we
have $\lim_{T \to \infty} \frac 1T A_{\eta}(T) = A_{\eta}^*$ for every $w_1 \in \Delta_m^\circ$.  Hence,
\(
\lim_{T \to \infty} \frac 1T F_T(x) = \sum_{\eta \in \Mcal} A_{\eta}^* h^{\eta}(x) 
\)
for every $w_1 \in \Delta_m^\circ$. \qed
\end{proof}
Similarly, we can extend the convergence of the margin distribution to the whole space $\X$.
\begin{corollary}{\bf (The Margin of Any Input in the Feature Space
    Converges.)}
\label{cor:margin_convergence}
Suppose Conditions~\ref{cond:nathypo} (Natural Weak-Hypothesis
Class) and~\ref{assume:WeakLearn} (Weak Learning)
hold.
For every initial example weights $w_1 \in \Delta_m^\circ$, the limit $\lim_{T\to \infty} \margin_T(x)$ exists for all $x\in \X$.
\end{corollary}
\begin{proof}
We arrive at the convergence of $\lim_{T \to \infty}
A_{\eta}(T)\left(\sum_{t=1}^{T} \alpha_t\right)^{-1} =
\widetilde{A}_{\eta}$, for every $w_1 \in \Delta_m^\circ$, the same
way as in the proof Corollary~\ref{C:example_margins_convergence}.  
Then, closely following the
proof of Theorem~\ref{thm:classifier_fcn_convergence}, 
we get
\(
\lim_{T \to \infty} \margin_T(x) = \lim_{T \to \infty} F_T(x)\left(\sum_{t=1}^{T} \alpha_t \right)^{-1}
= \sum_{\eta \in \Mcal} \widetilde{C}_{\eta} \; h^{\eta}(x) , 
\)
for every $w_1 \in \Delta_m^\circ$. \qed
\end{proof}
Recall that the full Optimal-AdaBoost classifier is
$H_T(x) = \sign{F_T(x)}.$ In that equation, we can easily replace $\sign{F_T(x)}$ with $\sign{\frac 1T F_T(x)}$.  From the convergence of $\frac 1T F_T(x)$, we would like to say that $H_T(x)$ converges as well.  However, we have a discontinuity in the $\mathrm{sign}$ function at $0$.  It may be the case that $\lim_{T \to \infty} \frac 1T F_T(x) = 0$ for some $x\in \X$, possibly yielding a non-existent limit for $\sign{\frac 1T F_T(x)}$.  In that case, $\lim_{T \to \infty} H_T(x)$ simply does not exist.

To overcome this obstacle, we consider the following condition. If one
let $F^*(x) \equiv  F^{*,w_1}(x) \equiv \lim_{T \to \infty} \frac 1T
F_T(x)$, intuitively we are saying that the decision boundary of the
function $F^*$ has measure $0$ with respect to the probability space
$(\D,\Sigma,P)$. But first we establish the following so that the statement on the condition
makes sense.
\begin{proposition}
  \label{pro:class_meas}
  Suppose Conditions~\ref{cond:nathypo}
(Natural Weak-Hy\-poth\-e\-sis Class) and~\ref{assume:WeakLearn}
(Weak Learning) hold. 
For every initial example weight $w_1 \in \Delta_m^\circ$, the
functions $F_T$ and $H_T$, for all $T$, and $F^*$, as real-valued functions with
domain $\X$, are $\Sigma_\X$-measurable.
\end{proposition}
\begin{proof}
First note that, for every $w_1 \in \Delta_m^\circ$, each function in
the sequence $(h_t)$ is $\Sigma_\X$-measurable, by
Condition~\ref{cond:nathypo} (Natural Weak-Hypothesis Class). Now, each $F_T$ is
$\Sigma_\X$-measurable because it is a linear combination of a finite
number of $\Sigma_\X$-measurable functions
$(h_t)$~\citep[][Lemma 2.6, pp. 9]{BartleMT}. The definition of
measurability immediately imply that $H_T = \mathrm{sign} \circ F_T =
\mathrm{sign} \circ (\frac1T F_T)$ is
$\Sigma_\X$-measurable~\citep[][Definition 2.3, pp. 8]{BartleMT}.
Because, by Theorem~\ref{thm:classifier_fcn_convergence}, the sequence of
functions $(\frac1T F_T)$ converges to $F^*$, it follows
that $F^*$ is also $\Sigma_\X$-measurable~\citep[][Corollary 2.10,
pp. 12]{BartleMT}. \qed
\end{proof}
\begin{condition}{\bf (The Decision Boundary has $P$-Measure Zero.)}
\label{C:decbnd}
For every initial example weights $w_1 \in \Delta_m^\circ$, we have
$F^* \neq 0$ on $\X$, $P$-almost surely. 
\end{condition}
Under this condition (Condition~\ref{C:decbnd}), the limit of the
classifier behaves nicely.  In fact, under the condition, 
the AdaBoost classifier itself, $H_T$, is converging in classification for \emph{almost all} elements in the instance space $\X$, i.e., except for a subset of $\X$ of measure $0$ with respect to $(\D,\Sigma,P)$.
\begin{theorem}{\bf (The AdaBoost Classifier Converges.)}
\label{thm:classifier_convergence}
Suppose Conditions~\ref{cond:nathypo}
(Natural Weak-Hypothesis Class),~\ref{assume:WeakLearn}
(Weak Learning), 
and~\ref{C:decbnd}
(Measure-Zero Decision Boundary) hold.
For every initial example weights $w_1 \in \Delta_m^\circ$, 
we have that $H^* \equiv H^{*,w_1} \equiv \lim_{T \to \infty} H_T$ exists on $\X$,
$P$-almost surely; 
or equivalently, for each $w_1 \in 
\Delta_m^\circ$, the sequence of
$\Sigma_\X$-measurable functions $(H_T)$, as functions with domain $\X$, converges to
the $\Sigma_\X$-measurable function $H^*$ on $\X$, $P$-almost surely.
\end{theorem}
\begin{proof}
Pick any $w_1 \in \Delta_m^\circ$. By Conditions~\ref{cond:nathypo}
(Natural Weak-Hypothesis Class) and~\ref{assume:WeakLearn}
(Weak Learning), Theorem~\ref{thm:classifier_fcn_convergence} implies that $(\frac1T F_T(x))$
converges to $F^*(x)$ for every $x \in \X$. Let $\X_0 \equiv \{ x \in \X |
F^*(x) \neq 0 \}$. For every $x \in \X - \X_0$, we have \(
  H^*(x) = \lim_{T \to \infty} \sign{\frac 1T F_T(x)}
  = \sign{F^*(x)}, 
  \)
  because the $\mathrm{sign}$ function is continuous, except at
  $0$~\citep[][Theorem 20.2, pp. 137]{BartleRA}. Because $(H_T)$ is a
  sequence of $\Sigma_\X$-measurable functions which converges to
  $H^*$ on $\X - \X_0$, it follows that $H^*$ is also
  $\Sigma_\X$-measurable on $\X-\X_0$~\citep[][Corollary 2.10,
pp. 12]{BartleMT}. Noting that $P(\X_0) = 0$, by Condition~\ref{C:decbnd}
(Measure-Zero Decision Boundary), the theorem follows from the
  definition of almost everywhere convergence~\citep[][Chapter 3,
  pp. 9, Chapter 7, pp. 65]{BartleMT}, or equivalently,
  convergence almost surely~\citep[][Section 1.2, pp. 12]{DurrettProb}.
\qed
\end{proof}

If the Optimal-AdaBoost classifier is converging in the limit, certainly its generalization error should as well.  
\begin{definition}
We can express the $0/1$-loss function $\loss_H : \D \to \{0,1\}$ of a binary classifier $H$ with output labels in $\{-1,+1\}$ as
\(
\loss_H(x,y) \equiv \frac{1-y \;  H(x)}{2} \; . 
\)
We show in the proof of the next theorem (Theorem~\ref{thm:gen_error_convergence}) that for both the Optimal
AdaBoost classifier $H_T$ after $T$ rounds and its converging limit
$H^*$, as established in Theorem~\ref{thm:classifier_convergence}, the
function $\loss_H$ is $\Sigma$-measurable and integrable, so that its
generalization error exists, with respect to $(\D,\Sigma,P)$. 
We can express the \emph{generalization error} of a
($\Sigma$-measurable and integrable) binary classifier $H$ with output labels $\{-1,+1\}$ as the \emph{expected misclassification error}:
\(
  \Err(H) \equiv \E{\loss_H(X,Y)} \equiv \int \; \loss_H \; dP = \int_{\D} \left[\frac{1-yH(x)}{2}\right] \; dP(x,y) \; .
\)
\end{definition}

It follows from the \emph{Lebesgue Dominated Convergence Theorem} that \emph{the generalization error converges.}
\begin{theorem}{\bf(The Generalization Error Converges.)}
\label{thm:gen_error_convergence}
Suppose Conditions~\ref{cond:nathypo}
(Natural Weak-Hy\-poth\-e\-sis Class),~\ref{assume:WeakLearn}
(Weak Learning), 
and~\ref{C:decbnd}
(Measure-Zero Decision Boundary) hold.
For every initial example weights $w_1 \in \Delta_m^\circ$, the
\emph{limit} of the \emph{generalization error}, $\lim_{T \to \infty}
\Err(H_T)$, exists and equals the generalization error $\Err(H^*)$ of the Optimal
AdaBoost classifier $H^*$.
\end{theorem}
\begin{proof}
Pick an arbitrary $w_1 \in \Delta_m^\circ$. 
Under Conditions~\ref{cond:nathypo}
(Natural Weak-Hy\-poth\-e\-sis Class) and~\ref{assume:WeakLearn}
(Weak Learning), $H_T$ is $\Sigma_\X$-measurable as a
function with domain $\X$, by Proposition~\ref{pro:class_meas}, and so is
$H^*$ $P$-almost surely, by
Theorem~\ref{thm:classifier_convergence}.
 By the properties of the probability space over the examples, 
 $H_T$ is immediately also $\Sigma$-measurable, as is $H^*$ $P$-almost
 surely.
Furthermore, the loss function $\loss_{H_T}(x,y)$ is
$\Sigma$-measurable 
because it is a linear function of $y \; H_T(x)$, and $y$ is, straightforwardly,
$2^{\{-1,+1\}}$-measurable and thus immediately also
$\Sigma$-measurable by the properties of the probability space over
the examples~\citep[][Lemma 2.6, pp. 9]{BartleMT}, and thus also
integrable by definition because it is
non-negative~\citep[][Definition 5.1, pp. 41]{BartleMT}. 
It is also dominated by the constant (integrable) function $f(x,y) = 1$ for all $(x,y)\in \D$ and all $T$:
\(
0 \leq \loss_{H_T}(x,y)
=
\frac{1-y \; H_T(x)}{2}
\leq 1.
\)
Therefore, the conditions of the Lebesgue Dominated Convergence Theorem~\citep[][Theorem 5.6, pp. 44]{BartleMT}
are satisfied, implying that we can
``distribute'' the limit over the integral and that $\loss_{H^*}$ is integrable.  
We then have that 
for every $w_1 \in \Delta_m^\circ$,
\begin{align*}
\lim_{T \to \infty} \Err(H_T) = & \lim_{T \to \infty} \E{\loss_{H_T}(X,Y)}
= \lim_{T \to \infty} \int \loss_{H_T} \; dP
= \int \lim_{T \to \infty} \loss_{H_T} \;  dP \\
= & \int_{\D} \lim_{T \to \infty} \left(\frac{1-y \; H_T(x)}{2}\right) dP(x,y)
= \int_{\D} \left(\frac{1-y \; H^*(x)}{2}\right) dP(x,y) \\
= & \E{\loss_{H^*}(X,Y)} = \Err(H^*) \; . 
\end{align*}
\qed
\end{proof}

\subsection{Some remarks about ties 
  and Condition~\ref{C:decbnd} (Measure-zero Decision Boundary)}
\label{sec:condrem}

A couple of remarks about Condition~\ref{C:NoTies2} (No Ties Eventually), or also its
related Condition~\ref{C:NoTies} (No Ties in the Limit), 
and Condition~\ref{C:decbnd} (Typical Decision Boundary) are in order before continuing.

The concepts of a ``support-vector'' example, and more specifically
``non-sup\-port-vec\-tor'' examples within the context of AdaBoost, will prove useful to some of the
discussion.
\begin{definition}
\label{def:sv}
We say an example indexed by $i$ is a \emph{non-support-vector
  example} with respect to initial example weights $w_1 \in \Delta_m$
in AdaBoost if $\lim_{t
  \to \infty} w_t(i) = 0$. Otherwise, we call the example a
\emph{support-vector example} with respect to $w_1$.
\end{definition}

\begin{remark}
\label{R:Rerun}
Roughly speaking, Condition~\ref{C:NoTies2} (No Key Ties) (see also Condition~\ref{C:NoTies}) states that any two effective mistake
di\-chot\-o\-mies are either never tied for best within the set $G$ (Part
1), or if they were tied, then it must be the case that $G$ is a
lower-dimensional subspace of $\Delta_m$ in which $w(i) = 0$ for all
$i$ such that $\eta(i) \neq \eta'(i)$ (Part 2). 

The implication of
Part 2 of the condition follows because for all $\eta, \eta' \in
\Mcal$, we have $\eta
\neq \eta'$. Thus, there exists at least one $i$ for which $\eta(i)
\neq \eta'(i)$, and at least $w(i) = 0$ for such $i$. Because all the
elements of $w_1$ are positive, with probability 1, and the
Optimal-AdaBoost example-weights update maps to a positive $w$, it
must have been the case that the at least one of the $w_t(i)$'s
converged to $0$. 

Indeed, any corresponding example $i$ referred to in the previous paragraph cannot
be a ``support vector'' example.
In such a case, both mistake
dichotomies 
$\eta$ and $\eta'$ referred to in the last paragraph would behave the same starting from any 
$w$ leading to non-support-vector examples. Hence, at least
computationally, we can equivalently ``reset'' the learning problem by
removing any tying $\eta''$ that is not the one selected by $\AdaSel$
from $\Mcal$ leading to a new set of mistake dichotomies
$\Mcal' \equiv \Mcal - \left( \argmin_{\eta'' \in \Mcal - \{\eta^w\}}
  \eta'' \cdot w \right) \subset \Mcal$. 

By the same reasoning, we can
remove any training dataset example for which $w(i)=0$: that is,
create $D' \equiv D'(w) \equiv \{ (x^{(i)},y^{(i)}) \in D \mid w(i) > 0 \}$. We
would now have a new dynamical system with 
$\Delta_{m'}$ as the state space, where $m' \equiv |D'|$,
corresponding to a lower-dimensional subspace of $\Delta_m$. 

Note that
this process of removing dichotomies from $\Mcal$ and examples from
$D$ may reveal what we call ``dominated hypotheses'' (see
Section~\ref{S:Exp} and Appendix~\ref{app:pac_bnds}) in the new set of mistake
hypotheses of lower-dimensional space (i.e., each mistake dichotomy is
now an $m'$-dimensional vector, and $m' < m$); of course, we must also
remove those ``dominated mistake dichotomies'' before
continuing/re-starting the process on the lower-dimensional
space. 

Note also that this process of removal cannot continue forever,
nor can the resulting sets become empty.
Under Condition~\ref{assume:WeakLearn} (Weak Learning), there must exist at least one positive
      and one negative example with positive weight at every round.
Also, by the properties of the AdaBoost example-weight update (see Appendix~\ref{app:abup}), consecutive mistake dichotomies selected
at consecutive rounds must be different: that is, for every round $t$,
we have $\eta_t \neq \eta_{t+1}$; thus, any set of mistake dichotomies
composed of exactly two dichotomies would break Condition~\ref{assume:WeakLearn}.
\end{remark}

\begin{remark}
Let us address the reasonableness of 
Condition~\ref{C:decbnd} (Measure-Zero Decision
Boundary) in our setting.
To do this, consider the following condition.
\begin{condition}{\bf (Sufficiently Rich Weak-Hypothesis Class.)}
\label{C:WL_decbnd}
For every $h \in \Hypo$, denote by $g : \X \to \R$ its ``proxy'' classifier
function: i.e., $h = \mathrm{sign} \circ g$. Let
$(\Hypo,\Sigma_\Hypo,\mu_{\Hypo})$ be an appropriate 
measure space over the weak-hypothesis class. For $P$-almost every  
dataset $S$ of input examples in a training dataset $D$ of $m$ examples 
drawn according to $(\D,\Sigma,P)$, and for every label dichotomy $o \in
\Dich(\Hypo,S)$, we have \[ \mu_{\Hypo}(\{ h \in \Hypo |
h(x^{(l)}) =o_l, \text{ for all } l = 1,2,\ldots, m, \text{ and }
P(g(X) \neq 0) \} > 0 \; . \] 
\end{condition}
This seemingly esoteric condition essentially says that we are almost always
able to find representative hypothesis with ``nice,'' ``typical,'' or
``non-degenerate'' decision boundaries. Paraphrasing the condition, it
says that for $P$-almost every possible datasets $D$,
and every output label dichotomy $o \in \Dich(\Hypo,S)$ that $\Hypo$
can produce on any $x$ in the set of input examples $S$ in $D$, we
can $\mu_\Hypo$-almost surely select (or draw) a representative hypothesis
$h^o \in \Hypo$ for $o$, to include in $\Hrep$,
whose decision boundary
has $P$-measure zero. If the decision boundary of every representative weak-hypothesis in
$\Hrep$ has $P$-measure zero, any classifier built using a linear
combination of them will also have $P$-measure zero for Borel-almost every
assignment to the coefficients of the linear combination.  For instance, returning to our running example
of earlier technical sections, axis-parallel decision-stumps on
feature spaces $\X$ that are subsets of Euclidean space satisfy this
condition. In that case, each $h \in \Hypo$ is of the form $h(x) = \sign{x_j - v}$
or $h(x) = \sign{-x_j - v}$ for some dimension/axis $j$ and threshold
value $v \in \R$ (see Equation~\ref{eqn:exhypo}). Hence, we can relate each threshold value in each
dimension to a hypothesis. So, effectively, we can think of $\Hypo$ as a
subset of $\times_j \R$, the Cartesian product of real-valued
spaces. We can define a Borel measure space over each input dimension
and then define $(\Hypo,\Sigma,P)$ as the Cartesian product of the
measure spaces with the Borel $\sigma$-algebra for each input feature
dimension $j$. Often, $P$ induces a probability density function over
$\X$, for which sets of Borel-measure zero would also have $P$-measure
zero. In that sense, \emph{every} $h \in \Hypo$ satisfies
Condition~\ref{C:decbnd}.  Other examples in the context of binary
classification include most 
typical implementations of decision trees,
nearest-neighbors,
linear and generalized linear classifiers, neural networks, and SVMs,
among others.

We can also prove the following proposition,
somewhat related to Condition~\ref{C:decbnd}.
\begin{proposition}
Suppose Conditions~\ref{cond:nathypo}
(Natural Weak-Hy\-poth\-e\-sis Class) and~\ref{assume:WeakLearn}
(Weak Learning). Let $\Delta_n$ be the probability
$n$-simplex and $\Hypo_{\mathrm{AdaBoost}} \equiv \{
\, 
\mathrm{sign} \circ \widetilde{F} \, \mid \,  \widetilde{F} : \X \to [-1,+1], \widetilde{F} = \sum_{\eta \in \Mcal}
\widetilde{\alpha}^\eta h^\eta, (\widetilde{\alpha}^\eta)_{\eta \in
  \Mcal} \in \Delta_n \}$, the set
the hypothesis class of all possible Optimal AdaBoost classifiers that
Optimal AdaBoost could output from
a given set of mistake dichotomies $\Mcal$. Consider the (finite) Borel measure on $\Delta_n$,
i.e., equivalent to the uniform distribution over $\Delta_n$; or said differently,
the Dirichlet distribution with all concentration parameters equal to $1$. The
decision boundary of every
classifier $H \in \Hypo_{\mathrm{AdaBoost}}$ has $P$-measure zero
 for Borel-almost every $(\widetilde{\alpha}^\eta)_{\eta \in
  \Mcal} \in \Delta_n$. 
\end{proposition}
But the current discussion is missing an important point, and one that
the Action Editor brought to our attention: that we are \emph{not drawing} 
classifiers from a pile \emph{at random}, but \emph{selecting} them
deliberately on the basis of some optimization process on the training
data, ultimately seeking to
reduce generalization error. We must admit that we did
not think of this until the Action Editor brought it to our
attention. We respectfully argue, however, that the selection
mechanism obtained via optimization is intrinsically tied to the data
generating process. Our selection is essentially ``drawing
inferences'' from the available data, which we assume came from some
underlying random mechanism governed by $(\D,\Sigma,P)$.

For instance, returning
to the particular case of Optimal AdaBoost, based on our intuition, we
argue that even if $\Hrep$ contains hypotheses with non-zero
$P$-measure, the likelihood that Optimal AdaBoost would select such
hypotheses with sufficient frequency to play any significant role in
the final classifier seems low to us, because we would expect the weighted
error of hypotheses with a non-negligible number of indifferences to
be higher than those that with negligible numbers of the
same. Classification is inherently about discrimination after all. In
our opinion, the objective of supervised learning algorithms is, or at least should
be, to reduce
indifferences, not create more, or at least not more than the data
reflects. But we admit that this is simply our educated opinion, and
that a more formal treatment of this topic is required.

Yet, there are
some statements we can make on this topic with certain degree of
certainty. For instance, the underlying assumption that allows the
training error to vanish exponentially fast, in addition to some
version of Condition~\ref{assume:WeakLearn} (Weak Learning), of
course, is that no two labeled
examples in $(x,y)$ and $(x',y')$ exists in $D$ for which the input is
the same, $x=x'$, but their label is not, $y\neq y'$. That would
violate Condition~\ref{assume:WeakLearn} (Weak Learning), and Optimal
AdaBoost will detect that because the sequence of weighted error
$(\epsilon_t)$ will converge to $\frac{1}{2}$, so that the sequence of
example weights $(w_t)$ also converge. So that assumption is
essentially saying that $P(Y = +1) \in \{0,1\}$; or said differently,
that the set of examples with non-deterministic output labels has
$P$-measure zero. Hence, the decision boundary of the ``ground-truth''
classifier $\sign{P(Y=+1|X) - \frac12}$ in this case has $P$-measure
zero. Because the training error of Optimal AdaBoost in this case is
guaranteed to go to zero, the margin of the training examples will be
different than zero after $O(\log T)$ rounds. We can use this to
derive a PAC-Learning type statement: that with arbitrarily high
accuracy, the probability that any future example will fall
arbitrarily close to the boundary is $\tilde{O}(\sqrt{n/m})$. While
$n$ can grow with $m$, it does so only up to a point, as $n$ is
bounded by a function of the VC-dimension of the weak-hypothesis
class.
This implies that we can obtain an upper bound on the $P$-measure of
the decision boundary of an Optimal AdaBoost classifier. We can use this to relax Condition~\ref{C:decbnd}.

\paragraph{Relaxing the $P$-measure-zero condition on the decision boundary.}
For any starting $w_1 \in \Delta_m^\circ$, let
\(
E \equiv E_{w_1} \equiv  \left\{ (x,y) \in \D \left| \, \lim_{T \to \infty} \frac 1T F_T(x) \neq 0 \right. \right\}.
\)
By Theorem~\ref{thm:classifier_fcn_convergence}, we have that $\lim_{T
  \to \infty} \frac 1T F_T$ converges to a $\Sigma_\X$-measurable
bounded real-valued function $F^*$. By the properties of the probability space over the examples, $F^*$ is also
$\Sigma$-measurable. The set 
\[
E = \left\{ (x,y) \in \D \left| \, F^*(x) > 0 \right. \right\} \bigcup
\left\{ (x,y) \in \D \left| \, F^*(x) < 0 \right. \right\} 
\]
is $\Sigma$-measurable, by the
definiton of $\sigma$-algebra, because it is
the union of two (disjoint) sets, each $\Sigma$-measurable by definition
because $F^*$ is a
$\Sigma$-measurable real-valued function~\citep[][Definition 2.3 and Lemma
2.4, pp. 8]{BartleMT}. 
We allow the complement of this set with respect to $\D$ to take
non-zero measure.
Because the set $E$ is $\Sigma$-measurable,
we can say that the stability of the generalization error depends on $P(\D - E)$.
\begin{proposition}
The set $E$ is $\Sigma$-measurable and the following holds.
\begin{enumerate}
\item $\limsup_{T \to \infty} \Err(H_T) \leq \int_{E} \; \loss_{H^*} \; dP + P(\D - E)$
\item $\liminf_{T \to \infty} \Err(H_T) \geq \int_{E}  \; \loss_{H^*} \; dP - P(\D - E)$
\end{enumerate}
So that $\limsup_{T \to \infty} \Err(H_T) - \liminf_{T \to \infty}
\Err(H_T) \leq 2 P(\D - E)$.\\
Additionally, if  \( \lim_{T \to \infty} \Err(H_T) \)  exists, then
\(
\left|\lim_{T \to \infty} \Err(H_T) - \int_{E}  \; \loss_{H^*} \; dP \right| \leq P(\D - E).
\)
\end{proposition}
\begin{proof}
We can bound $\Err(H_T)$ as follows.
\begin{align} \label{E:limsup}
\Err(H_T) =  \int  \; \loss_{H_T} \; dP 
=  \int_{E} \; \loss_{H_T} \; dP  + \int_{\D - E} \; \loss_{H_T} \; dP 
\leq  \int_{E} \; \loss_{H_T} \; dP  + P(\D - E).
\end{align}
Symmetrically, we also have
\begin{align} \label{E:liminf}
\Err(H_T) \geq \int_{E} \; \loss_{H_T} \; dP - P(\D - E).
\end{align}
We will consider only Equation~\ref{E:limsup}, and results for Equation~\ref{E:liminf} will follow symmetrically.  By taking $\limsup$ on both sides of Equation~\ref{E:limsup}, we see that 
\begin{align*}
\limsup_{T \to \infty} \Err(H_T) &\leq \limsup_{T \to \infty} \int_{E} \; \loss_{H_T} \; dP+ P(\D-E) \\
 & =  \lim_{T \to \infty} \int_{E} \; \loss_{H_T} \; dP + P(\D-E)
=  \int_{E} \; \loss_{H^*} \; dP + P(\D-E),
\end{align*}
where the exchange of the limit and the integral follows from the same
argument about the loss function used in the proof of
Theorem~\ref{thm:gen_error_convergence}.
Symmetrically, we find
\(
\liminf_{T \to \infty} \Err(H_T) \geq \int_{E} \; \loss_{H^*} \; dP -
P(\D - E).
\)
Finally, if $\lim_{T \to \infty} \Err(H_T)$ exists, then
\(
\liminf_{T \to \infty} \Err(H_T) = \lim_{T \to \infty} \Err(H_T) = \limsup_{T \to \infty} \Err(H_T). 
\) \qed
\end{proof}

\end{remark}

\section{Preliminary experimental results on high-dimensional real-world datasets}\label{S:Exp}

This section provides empirical evidence 
that Optimal AdaBoost moves away form ties and illustrates
the
difficulty of finding evidence of cycling behavior in practice, despite
our theoretical results.
The empirical
results are in the context of \emph{decision stumps}. Decision stumps are simple decision tests based on a single
attribute of the input; i.e., a decision tree with a single node: the
root corresponding to the attribute test. Studying decision stumps has very practical implications. Because of their
simplicity, and effectiveness, decision stumps are arguably the most
commonly used weak-hypothesis class $\Hypo$ with AdaBoost in practice,
as the first slide from Breiman's Wald Lecture we quote at the beginning of this paper suggests.

We have observed that the ``effective'' number of decision stumps is relatively
  smaller than expected. By ``effective'' here we mean decision
  stumps that may have a chance of being selected by Optimal AdaBoost because
  they are not strictly ``dominated.''~\footnote{A ``dominated''
    decision stump with respect to Optimal AdaBoost is one whose set of mistakes on the training dataset is a strict superset
    of another decision stump.} Another
observation is that the number of \emph{uniquely-selected} decision stumps
consistently grows \emph{logarithmically} with the number of rounds of
boosting, at least
within $100K$ rounds in several high-dimensional real-world datasets
for binary classification publicly available from the UCI ML
Repository. The set of weak-hypotheses \emph{uniquely
    selected} by Optimal AdaBoost with respect to initial
  example-weights $w_1$ is $\cup_{t=1}^T \{ h_t \}$. In passing, it is important to point out that although the effective number of decision stumps is relatively small,
the number of those selected by Optimal AdaBoost is even smaller, and
the empirically observed logarithmic
growth suggests that it would take a very long time before
AdaBoost would have selected all effective weak classifiers, if ever;
of course, it will eventually plateau by the Pigeonhole Principle. In order to keep the presentation here brief, we discuss
the results of those experiments in 
Appendix~\ref{app:pac_bnds}.
In that
appendix,
we also show and discuss technical implications of our empirical
observations, including new, potentially tighter uniform-convergence data-dependent PAC
bounds on the generalization error. Our data-dependent bounds may be tighter than
those previously derived because their direct dependence on $T$ is
expected to be a
very low-degree polynomial of $\ln{T}$, as opposed to $\sqrt{T
  \ln{T}}$ which is the traditional bound for AdaBoost based on $T$.

It is important to keep in mind that the fact that the empirical
observation on the logarithmic growth suggests that Optimal AdaBoost is selecting the same decision stump
many times does not in itself immediately imply that Optimal AdaBoost is
cycling over the example weights. In
fact, we did not find any empirical evidence in our experiments of AdaBoost cycling, or
being even ``near'' any detectable cycle, despite our theoretical
results.

What the empirically observed logarithmic
growth in the number of unique decision stumps does suggest is that
Optimal AdaBoost may take a very long time to even complete a cycle in high-dimensional, real-world data
sets.
Said differently, the
cycling behavior seems to take an extremely long time in practice. At the same time, the
empirical evidence suggests that Optimal AdaBoost reaches
stability of the averages of many quantities, as well as the
convergence of its generalization error, relatively quickly. 
 We delay further discussion on this topic to Section~\ref{sec:close}
(Closing remarks); we also refer the reader to
Appendix~\ref{app:pac_bnds}
for a detailed
discussion of those empirical results, including plots.

\subsection{Experimental results show empirical evidence that Optimal
  AdaBoost moves away from ties eventually}

This section discusses \emph{preliminary} experimental evidence
consistent with our condition that Optimal AdaBoost moves
away from ties eventually, and with our theoretical result that its time averages converge quickly, as does its
classifier, its generalization, as well other typically studied quantities/objects.
We provide empirical evidence on commonly used data sets in practice
suggesting that the following two conditions are satisfied: for any pair $\eta,\eta'\in 
\Mcal$, either (1) there are no ties between $\eta$ and $\eta'$ in the
limit, or (2) if they are tied, they are effectively the same with
respect to the weights in the limit.  We provide empirical evidence on
commonly used data sets in practice demonstrating these two conditions. In fact, 
we have 
empirical evidence 
that demonstrate that those conditions
hold in \emph{all} other
UCI datasets that are applicable to our setting and have been used in
the literature.
We only report our evidence in some real-world datasets here for
brevity. The results presented here are representative of those we
observed on the other datasets. 

\begin{figure}
  \begin{center}
\includegraphics[scale=.5]{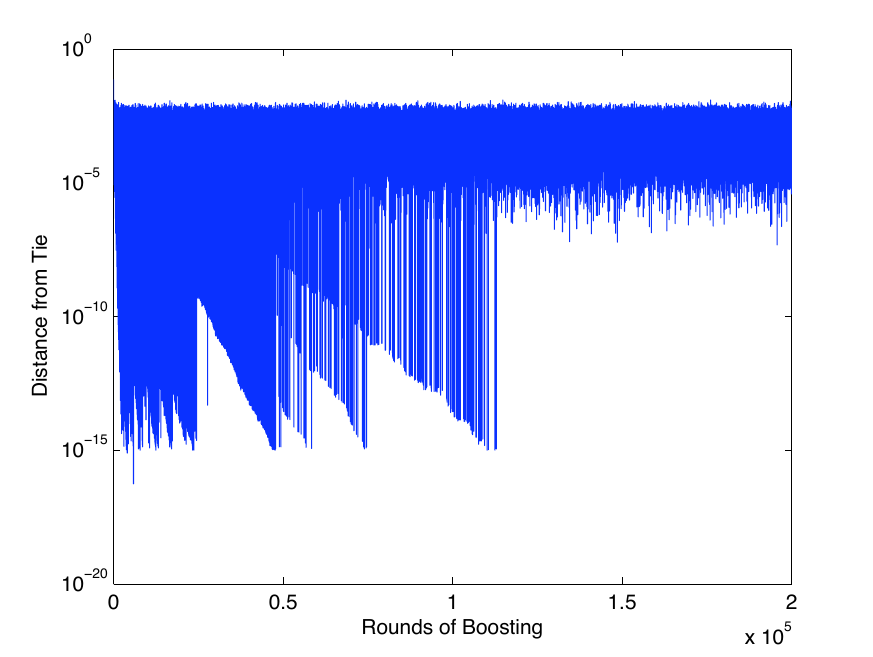} \\
\includegraphics[scale=.5]{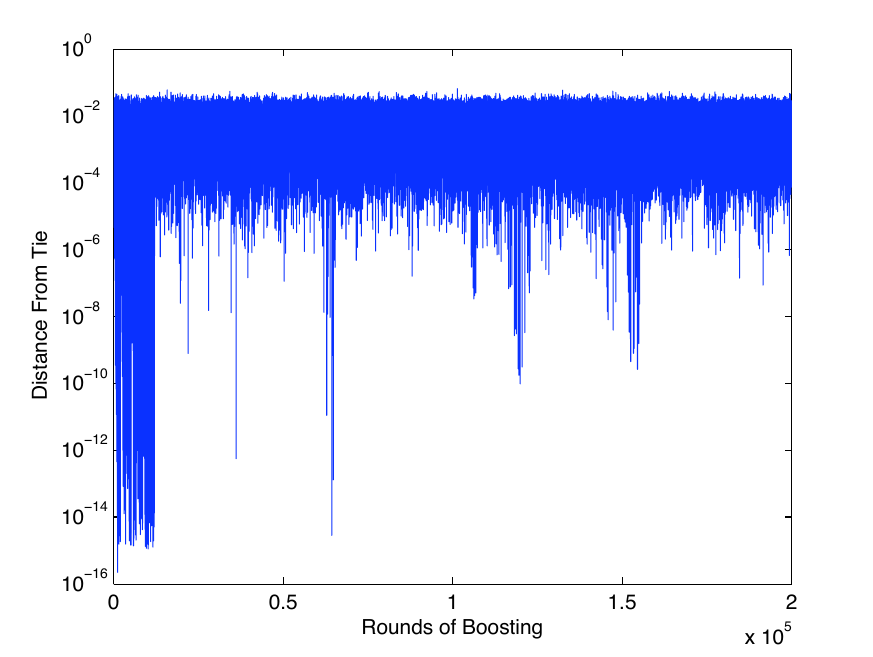} \\
\includegraphics[scale=.5]{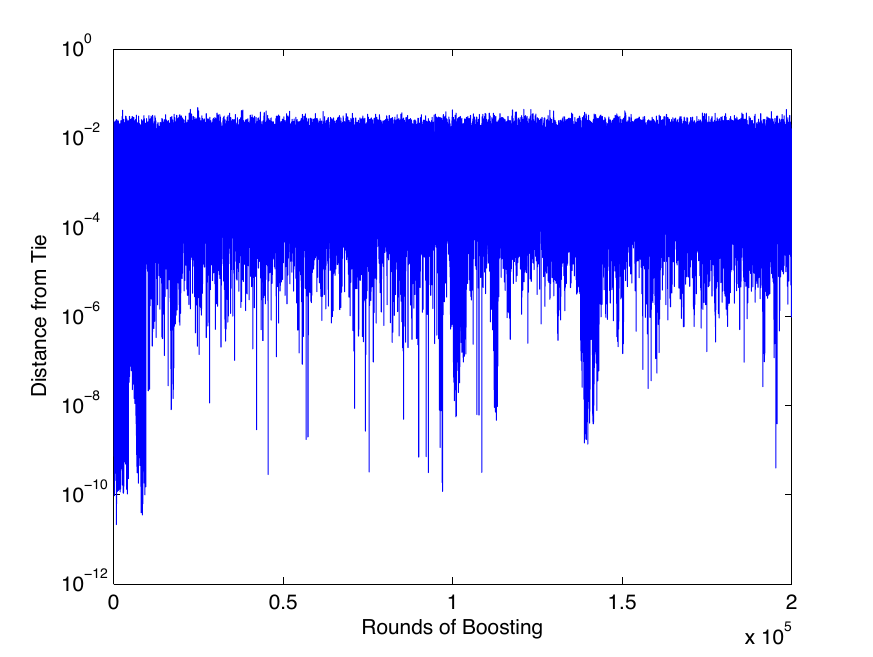}
\end{center}
\caption{{\bf Empirical Evidence is Consistent with 
    Condition~\ref{C:NoTies} (No Key Ties) in High-Dimensional,
    Real-World Datasets.} These plots depict the difference
  between the errors for the best and second best
  mistake-dichotomies/representative-hypotheses (in log scale) as a
  function of the number of rounds $T$ of boosting decision stumps on
  the Heart-Disease {\bf (top)}, Sonar {\bf (center)}, and Cancer
  datasets {\bf (bottom)}. The behavior depicted in these plots
  the two conditions
  described in the main body of the text, as our
  theoretical results predict.  Recall that, as
  described in the body of the text, when looking for the second best
  mistake-dichotomy/representative-hypothesis at time $t$, we ignore
  mistake-dichotomies $\eta'$ such that $\sum_{i: \eta_t(i) \neq
    \eta'(i)} w_t(i) < \kappa$, where we set $\kappa = 10^{-15}$,
  and that, for these experiments, we start AdaBoost from a weight over the training examples
  drawn uniformly at random from $m$-simplex.}
\label{fig:boundedties}
\end{figure}

In Fig.~\ref{fig:boundedties}, we present the results of running
Optimal AdaBoost using decision stumps on the Heart-Disease, Sonar,
and Breast-Cancer datasets, while tracking the difference between the
error of the best and the second best mistake-dichotomies of $\Mcal$
at each round $t$. Let $\eta$ be the optimal mistake dichotomy in
$\Mcal$ at round $t$.  When looking for the second best mistake
dichotomy at round $t$, we ignore mistake dichotomies $\eta'$ such
that $\sum_{i: \eta(i) \neq \eta'(i)} w_t(i) < \kappa$, where we set
$\kappa = 10^{-15}$.  Recall that we start Optimal AdaBoost from an
arbitrary initial weight in $\Delta_m^\circ$. In this experiment, we
draw $w_1$ uniformly at random from the $m$-simplex, as opposed to the
traditionally used weight corresponding to a uniform distribution over
the training examples. Our theoretical results hold for that case too,
of course, but we want to illustrate the robustness of the algorithm to initial
conditions, at least as it relates to the relatively quick convergence
of time averages and other related objects, including the Optimal
AdaBoost classifier itself.

The difference between the best and second best
mistake-dichotomy/representative-hypothesis tends to decrease to
$\kappa$ early on.  This happens because some weights for
non-minimal-margin examples go to zero. 
This set of
minimal-margin examples is precisely the ``support-vectors'' examples (Definition~\ref{def:sv}), a
term~\citet{RudinDynamics} also use because of the similar
interpretation to those examples in SVMs.~\footnote{\label{foot:sv}
  For all training examples indexed by $i=1,\ldots,m$, denote by
  $\beta_T(i) \equiv y^{(i)} \, \margin_T(x^{(i)})$ the ``signed''
  margin of example indexed by $i$. From our convergence results we
  can show that $\beta^{\min} \equiv \lim_{T \to \infty} \min_i
  \beta_T(i)$ exists.  We can also show that $\beta^{\min} = \lim_{T
    \to \infty} \sum_i w_{T+1}(i) \, \beta_T(i)$. This implies that,
  for all training examples, indexed by $i$,  $ \lim_{T \to \infty}
  \beta_T(i) > \beta^{\min}$ implies $\lim_{T \to \infty} w_{T+1}(i) =
  0$; and that $\lim_{T \to \infty} w_{T+1}(i) > 0$ implies $\lim_{T
    \to \infty} \beta_T(i) = \beta^{\min}$. Also, assuming training
  examples with different outputs, there always exists a pair of
  different-label examples, indexed by $(i^+,i^-)$, with $y^{(i^+)} = 1$
  (positive example) and $y^{(i^-)} = -1$ (negative example), such that
  $\lim_{T \to \infty} w_{T+1}(i^+) > 0$ and $\lim_{T \to \infty}
  w_{T+1}(i^-) > 0$ (because the error $\eta_T \cdot w_{T+1} =
  \frac{1}{2}$, where $\eta_T$ is the mistake dichotomy in $\Ecal$
  corresponding to the label-dichotomy/representative-hypothesis
  selected at round $T$). This in turn implies $ \lim_{T \to \infty}
  \beta_T(i^+) = \lim_{T \to \infty} \beta_T(i^-) = \beta^{\min}$, leading
  to our interpretation of the set $\{ i  \mid \lim_{T \to \infty}
  \beta_T(i) = \beta^{\min}\}$ as the set of indexes to
  support-vectors examples.}
Such zero-weight examples could cause certain rows of the mistake
matrix to become
essentially equal with respect to the weights.  Once such weights go
below $\kappa$, a condition which we equate to essentially
satisfying the second condition stated earlier about the ``equivalence
of mistake dichotomies in lower-dimensional subspaces,''
we ignore these ``equivalent''
mistake-dichotomies/hypotheses. In turn, this causes the trajectory of
the differences between best and second best to jump upwards.  After a
sufficient number of rounds, the set of support-vector examples with
respect to $w_1$ manifests, and this jumping behavior stops.  At this
point, as the data shows, the distance from ties is bounded away from
zero, as predicted by the theory.

\begin{figure}
  \begin{center}
\begin{tabular}{c}
\hspace*{-0.25in}\includegraphics[scale=.7]{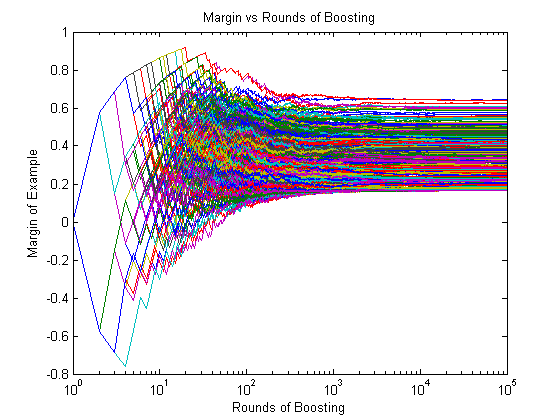}
\end{tabular}
  \end{center}
  \caption{
{\bf Evidence for the Convergence of the (Signed) Margins of the Optimal AdaBoost Classifier when
  Boosting Decision Stumps on the Cancer Dataset.}  This
plot 
shows the behavior of the ``signed'' margin $y^{(i)} \,
\margin_T(x^{(i)})$ of every example $i=1,\ldots,m$ as a function of
the number of rounds $T$ of boosting (in log scale).}
\label{fig:classconv}
\end{figure}

\begin{figure}
  \begin{center}
\begin{tabular}{c}
\hspace*{-0.75in} \includegraphics[scale=.48]{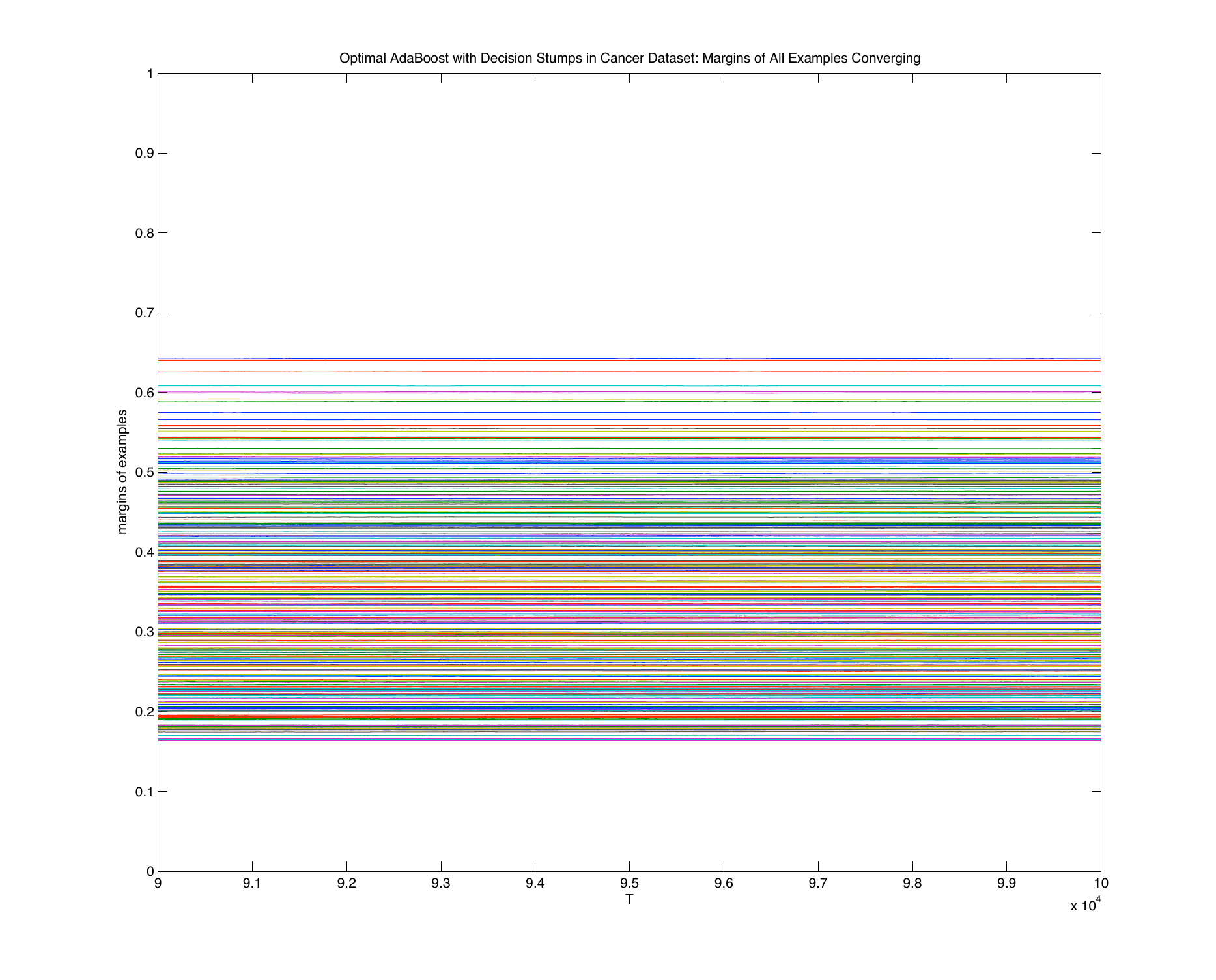}
\end{tabular}
  \end{center}
  \caption{
{\bf Evidence for the Convergence of the (Signed) Margins of the Optimal AdaBoost Classifier when
  Boosting Decision Stumps on the Cancer Dataset (Zoom).}  
This plot 
is a closer look at the asymptotic behavior of the signed margins in
the plot in Fig.~\ref{fig:classconv} from
rounds $T=90K$ to $100K$.  Evidence for the convergence of the signed
margins is more evident at this resolution.}
\label{fig:classconvzoom}
\end{figure}

Fig.~\ref{fig:classconv} provides reasonably clear empirical evidence for the
convergence of the (signed) margins of the Optimal AdaBoost classifier when boosting decision
stumps on the Cancer dataset.  In this figure, the signed margin for every
example appears to be converging: From rounds 90k to 100k there is
very little change, as seen most clearly in Fig.~\ref{fig:classconvzoom}.

\begin{figure}
  \begin{center}
\begin{tabular}{c}
\hspace*{-0.75in} \includegraphics[scale=.5]{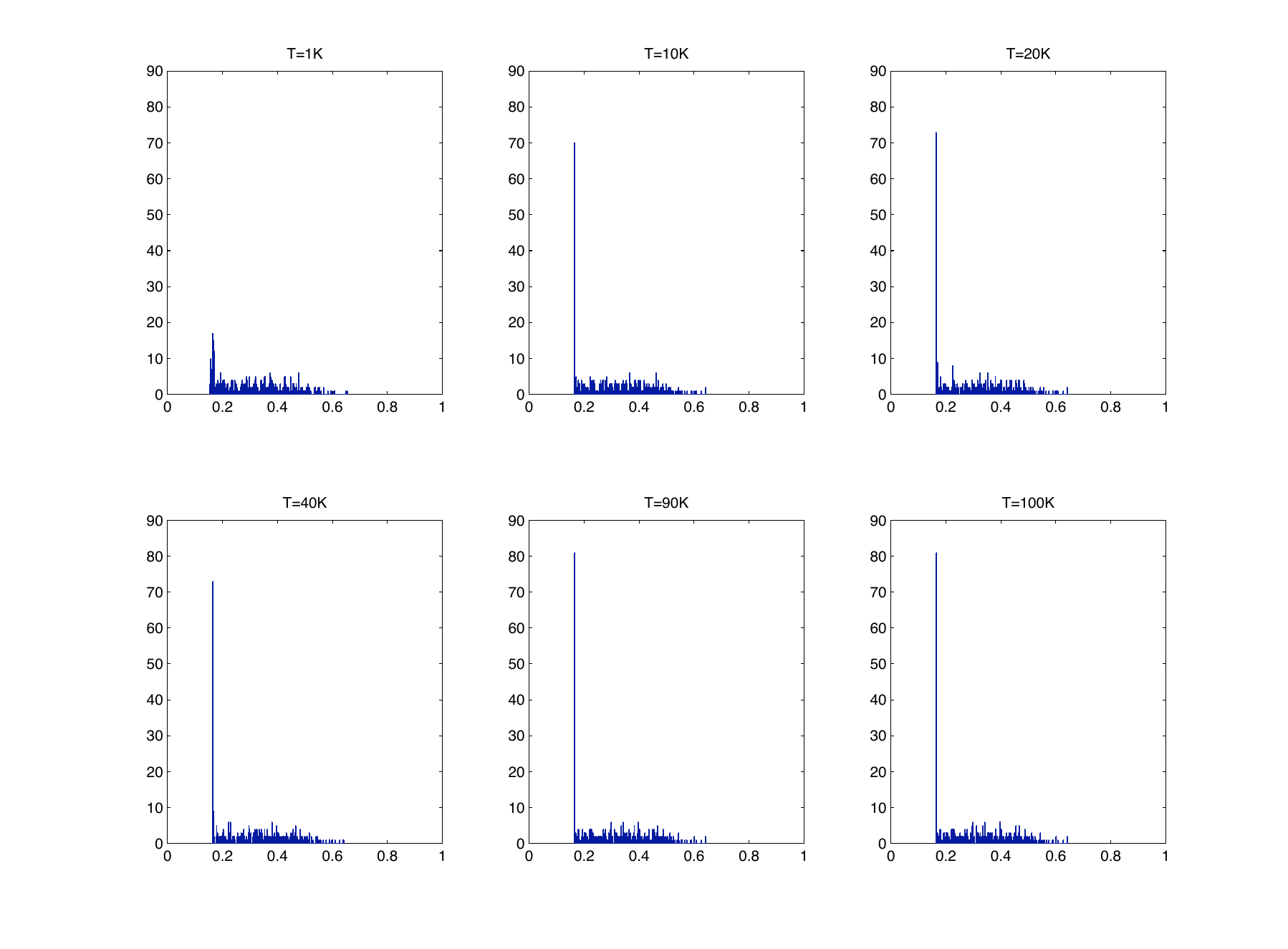}
\end{tabular}
  \end{center}
  \caption{{\bf Evidence for the Convergence of the Minimum Margin.} This
    plot depicts the minimum margin as a function of the number of
    rounds of boosting (log scale) on the Cancer dataset, using
    decision stumps.  This is an isolation of the minimum margin from
    Figure~\ref{fig:marginhist}. (see main text for further discussion)}
  \label{fig:cancermargin}
\end{figure}

\begin{figure}
  \begin{center}
\begin{tabular}{c}
\includegraphics[scale=.7]{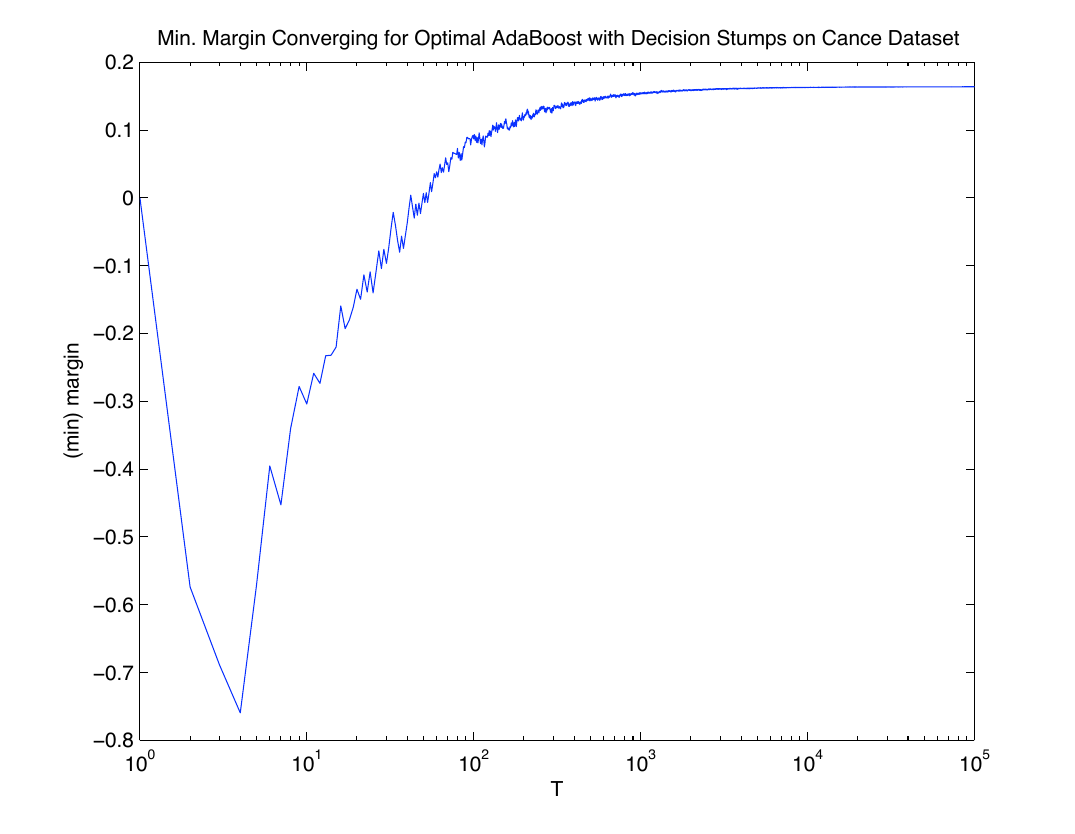}
\end{tabular}
  \end{center}
  \caption{{\bf Evidence for the Convergence of 
      the Signed-Margins' Distribution.} This plot shows the
histogram of signed margins at rounds $T=1K, 10K, 20K, 40K, 90K,
100K$. The histograms contain $200$ bins. Note that they are all
positive because, from the theory of AdaBoost, supposing
Condition~\ref{assume:WeakLearn} (Weak Learning) holds, all the
training examples are correctly classified eventually after some
finite number of rounds (logarithmic in $m$), so that the signed
margin will always be positive. Note also that the examples in the
histogram whose signed margin is closest to zero correspond to the
``support vectors'' (see main text for further discussion).}
  \label{fig:marginhist}
\end{figure}

Fig.~\ref{fig:cancermargin} shows convergence of the minimum margin;
this is essentially a more complete view of the convergence of the
minimum margin clearly seen in the histograms in
Fig.~\ref{fig:marginhist}. 

\begin{figure}
  \begin{center}
\includegraphics[width=\textwidth]{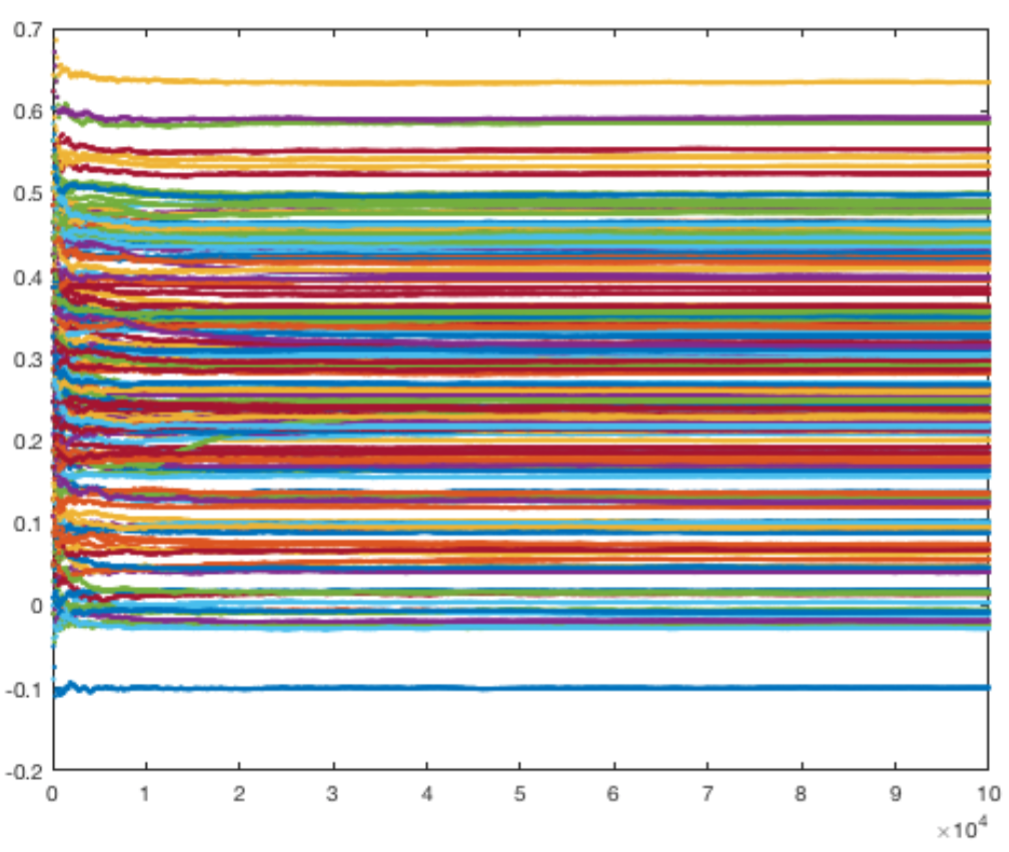}
  \end{center}
  \caption{
{\bf Evidence for the Convergence of Optimal AdaBoost
      Classifier: Signed Margins on a Test Dataset from Full Cancer Dataset.}  This
plot 
shows the behavior of the ``signed'' margin $y \,
\margin_T(x)$ (given in $y$-axis), as a function of
the number of rounds $T$ of Optimal AdaBoost ($x$-axis), of every $(x,y)$ input-output pair on a randomly drawn (without
replacement) test
dataset of size $169$ from the full "Cancer Dataset" of size
$569$. In order for the plotting tool to draw the plots quickly, the
plot is the result of the computation of the signed margin after every
$100$ rounds that the AdaBoost classifier $H_T$, would produce after
running for $T$ rounds. }
\label{fig:classconv_test}
\end{figure}

Fig.~\ref{fig:classconv_test} shows the signed margins on the \emph{test}
dataset from the whole Cancer dataset. This test dataset is the result
of a random permutation of the $569$ examples in the Cancer dataset,
which we partition uniformly at random into two partition of $400$ and
$169$ examples to form the training and test datasets,
respectively. We ran the \emph{same implementation} of
AdaBoost described for the experiments on the signed margin of the
\emph{training} dataset used in this section. However, this experiment is independent from that presented in the
results for the traditional signed-margins on the training set, given
in Figs.~\ref{fig:classconv} and~\ref{fig:classconvzoom}; that is, the random partition between train and test
dataset was different, as was the initial $w_1$. We are only
presenting a single figure, but the converging behavior was consistent
across all runs we tried. Also, a ``zoom'' into the period of $90K$ to
$100K$ would look similar to that in Fig.~\ref{fig:classconvzoom}. We should also note that
empirical evidence for the behavior of moving away from ties was
strongly present 
during this run, and over all other runs
we tried for that matter.

To provide further evidence of the convergence of the Optimal-AdaBoost
\emph{final classifier} over almost all input values in the whole
feature space $\X$, \emph{outside} the training and test datasets, we
generated $200$ i.i.d. input $30$-dimensional examples as follows. In this case, we
have $\X \subset \R^{30}$. First, we computed the largest $v_j^{\mathrm{max}}$ and lowest $v_j^{\mathrm{min}}$ values for
each of the $30$ feature-dimensions indexed by $j$, individually, as
given in the complete
Cancer dataset consisting of $569$ examples. To generate each of the
$200$ examples, we then independently sampled a value for each that feature $x_j
\sim \mathrm{Uniform}([v_j^{\mathrm{min}},
v_j^{\mathrm{max}}])$.  Fig.~\ref{fig:cancer_avg_F_T_outside} provides a plot of the value of
$\frac1T F_T(x)$ for each of the $200$ examples $x$ in the newly
generated, uniformly at random, unseen examples that are not in the
original Cancer dataset. The convergence of the average of classifier
function $F_T$, over $T$, for each of those outside the examples in the original dataset
is clear from the plot. A ``zoom'' into the period of $90K$ to
$100K$ would look similar to that in Fig.~\ref{fig:classconvzoom}. 

The
plots of the weighted-error difference at each round between the best
and second-best weak-classifier for the runs leading to Figs.~\ref{fig:classconv_test}
and~\ref{fig:cancer_avg_F_T_outside} are similar to that in
Fig.~\ref{fig:boundedties} {\bf (bottom)}, so we do not present them here.

All of the converging behavior just described is as predicted by
the theoretical work in Section~\ref{S:classifier-convergence}.

\begin{figure}
  \begin{center}
\includegraphics[width=\textwidth]{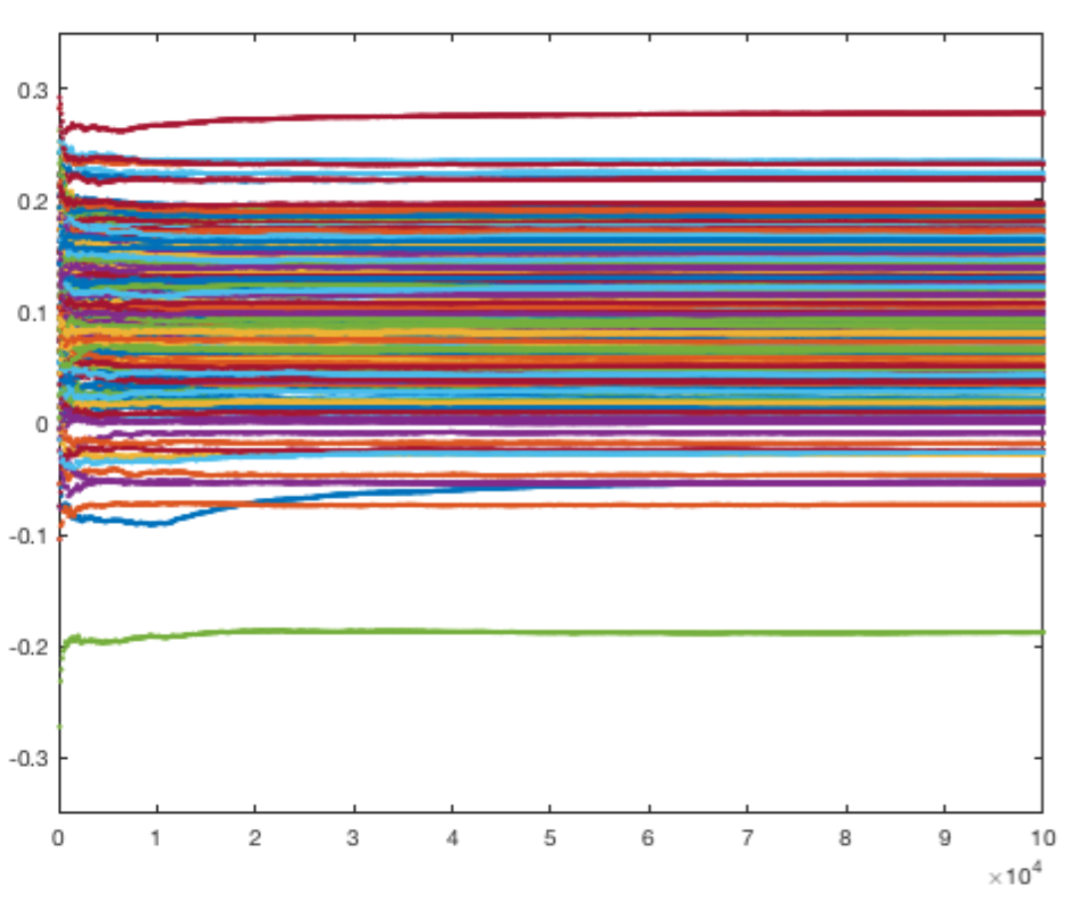}
  \end{center}
  \caption{{\bf Evidence for the Convergence of Optimal AdaBoost
      Classifier: Average $F_T$ Value on Random  Set.} This
    plot depicts the average $F_T$ value ($y$-axis) as a function of the number of
    rounds $T$ of boosting ($x$-axis) on $200$ i.i.d. samples drawn by
    independently sampling from the uniform distribution over each
    range of each of the $30$ feature dimensions as calculated
    from all the $569$ examples in the Cancer dataset. Like all
    previous experiments, here we use
    decision stumps as the weak-hypothesis class. The plot shows that
    the quantity $\frac1T F_T(x)$ is converging for each of the $200$
    randomly drawn examples not contained within the original Cancer dataset. We refer the reader
    to the main text for further details and discussion.}
  \label{fig:cancer_avg_F_T_outside}
\end{figure}

\section{Closing remarks}
\label{sec:close}

We end this paper with a discussion of our results, statements of future work, and a summary of our contributions.

\subsection{Discussion}
\label{sec:disc}

From a practical perspective, it is easy to find reasonable
  evidence of the absence of
  key ties within a large, but finite number of rounds. 
We can say the same for the appearance of converging behavior of time averages of
  functions of the example weights, such as the (signed) margins of
  every training example of
  the final Optimal AdaBoost
  classifier, or its
  test error, albeit within a finite number of rounds, even in
  high-dimensional real-world datasets (see 
  Section~\ref{S:Exp}).~\footnote{We also refer the reader to 
Fig.~\ref{fig:all}, on pg.~\pageref{fig:all}
in Appendix~\ref{app:pac_bnds}.}
Yet, despite our theoretical results, we found very hard to detect
  in practice
whether Optimal AdaBoost has reached or will reach a cycle within a reasonable number
  of rounds in \emph{high-dimensional real-world datasets}.

We would like to make a related note about
some of the empirical observations we made at the beginning of
Section~\ref{S:Exp}, and in particular those about a relatively small number of ``effective''
decision stumps and their consistently apparent logarithmic growth
(see also Appendix~\ref{app:pac_bnds}).
Because of the particular
implementation of the weak learner used here, the size of the
``effective hypothesis class'' is the number of representative
classifiers for the effective mistake dichotomies, thus
finite. Hence, by the Pigeonhole Principle, the number of unique, effective representative base-classifiers
selected will converge: that set of unique base-clasifiers will either
``saturate'' the whole set of effective representative base-classifiers, or
``plateau'' at a smaller subset. If the logarithmic pace of growth of
the number of unique, effective representative base-classifiers 
we have observed were to continue for longer runs of
AdaBoost,~\footnote{We refer the reader to the plots in the center
  column of 
Fig.~\ref{fig:all} in Appendix~\ref{app:pac_bnds}.} 
a simple
``back of the envelope'' analysis suggests that saturation
would happen after approximately $10^{58}$, $10^{22}$, and $10^{52}$
rounds for the Breast Cancer, Parkinson, and Spambase datasets,
respectively. Needless to say, those are quite large numbers to
reach or test convergence of the example weights to a cycle. We have never seen AdaBoost plateau or
saturate in our experiments with real-world high-dimensional
datasets within the large but finite maximum number of rounds of the
execution of the algorithm.

\subsection{Future work}
\label{sec:future_work}

We believe that the technical statements that result from our
construction about convergence may be streghtened, but only up to a
point. We know that Optimal AdaBoost cannot be a strongly ergodic
dynamical system on $\Delta_m$, but whether it is always a \emph{uniquely} ergodic dynamical
system \emph{modulo permutation of the examples} is an interesting
open problem.

In the so-called ``Non-optimal AdaBoost,'' a term also coined
by~\citet{RudinDynamics}, the weak hypothesis $h_t$ that the function
{\bf WeakLearn} outputs may not be that which achieves the minimum
weighted error among all $h \in \Hypo$ with respect to $w_t$ over $D$,
at any round $t$. The convergence results may also extend to this
version of AdaBoost, but a careful study is certainly needed.
Adapting the
analysis to the non-optimal setting may require careful derivation and significant effort.

\hspace{1in} 
\scalebox{0.82}{\begin{minipage}{0.8\textwidth}
\framebox[1.1\textwidth]{
\begin{minipage}{\textwidth}
\begin{center}
{\em WHAT ADABOOST DOES}
\end{center}
In its first stage, Adaboost tries to emulate the population version. This continues for thousands of trees. Then it gives up and moves into a second phase of increasing error .

\vspace{\baselineskip}

Both Jiang and Bickel-Ritov have proofs that for each sample size $m$, there is a stopping time $\tau(m)$ such that if Adaboost is stopped at $\tau(m)$, the resulting sequence of ensembles is consistent.

\vspace{\baselineskip}

There are still questions--what is happening in the second phase? But this will come in the future.

\vspace{\baselineskip}

For years I have been telling everyone in earshot that the behavior of Adaboost, particularly consistency, is a problem that plagues Machine learning.

\vspace{\baselineskip}

Its solution is at the fascinating interface between algorithmic behavior and statistical theory.

\vspace{\baselineskip}

\end{minipage}
}

\vspace{0.5\baselineskip}

\hfill \begin{minipage}{0.9\textwidth}
{\em Leo Breiman, Machine Learning. 2002 Wald Lecture. Slide 38}
\end{minipage}
\end{minipage}
}

\vspace{\baselineskip}

From a statistical perspective,
one question that follows from our
work is, can Optimal AdaBoost converge to the Bayes-risk if we
introduce just the ``right'' bias in the deterministic selection of
base classifiers via just the ``right'' implementation of the
{\bf WeakLearn} function? From an ML perspective, which is the
intended focus of this paper, one open question is, can Optimal AdaBoost converge to the ``minimum risk/loss'' for the given amount of data, under the same kind of implementation conditions?

We wish we could say something about the
\emph{quality} of the generalization error, beyond that it
converges.  In all of our experiments involving decision stumps, we
have observed a logarithmic growth of the number of \emph{unique}
hypothesis contained in the combined AdaBoost classifier as a function
of time.  Such a logarithmic growth yields potentially tighter data-dependent
bounds on the generalization of the AdaBoost classifier.~\footnote{We
  refer the reader to
Appendix~\ref{app:pac_bnds} for additional details and discussion.}
We believe
that the distribution of the invariant measure over the regions
$\pi(\eta)$ (see Definition~\ref{def:pi}) is an important factor for
this behavior. Empirically, the relative frequency of selecting each
hypothesis $h^{\eta}$ seems Gamma distributed. While the empirical behavior of Optimal AdaBoost of
repeating classifiers may suggest why the algorithm tends to resist over-fitting (i.e., the final, global classifier's
complexity remains low), the observed logarithmic-growth behavior itself is
still a mystery.  While our results indicate that the growth will stop
\emph{asymptotically}, it is still interesting to investigate this
behavior in the \emph{non-asymptotic regime} and its potential
connection to convergence rates and resistance to over-fitting.  We attempted to provide precise mathematical
statements of some related open problems in a
separate short manuscript~\citep{DBLP:journals/corr/BelanichO15}.

We have
provided a proof of convergence of Optimal AdaBoost, under some
conditions.  But we do not provide \emph{convergence rates}.
We believe that formally determining convergence rates for the simple, classical
version of Optimal AdaBoost is an important problem in ML.  We suspect
that convergence rates in this case vary significantly depending on
the idiosyncrasies of the datasets and the choice of weak
learner.  
We believe that it is reasonable to begin to move
on to the study of how to establish non-asymptotic convergence rates for Optimal
AdaBoost. 
The constructions described here may be a good place to start that line
of research.

\subsection{Summary of contributions}

We formally establish convergence results for various objects of
interest related to 
Optimal AdaBoost. For instance, we showed that the margin of all examples in the
training set converge.  Using a particular function implementation for
the weak learner, we extended the convergence results to the whole instance space $\X$.  Finally, under the condition that
the decision boundary of $F^*$ (i.e., the limiting function that Optimal
AdaBoost would use for classification) has probability $0$, we
proved that the Optimal AdaBoost classifier $H_T$ and its
generalization error converge.  If the decision boundary has
non-zero probability measure, we can say that the stability of the
generalization error depends on the probability of drawing an example
on the decision boundary of the converged classifier.
We believe our results provide largely positive answers to two important open problems about the behavior of AdaBoost in the machine-learning
community, at least within a computational perspective: Somewhat to our surprise, we can state with reasonable
strength that Optimal AdaBoost always exhibit cycling behavior and is
an ergodic dynamical system.

\begin{acknowledgements}
The work presented in this
    manuscript has not been published in any conference
    proceedings or any other venue, except for the following: (1) the
    undergraduate Honor's thesis in Computer Science at Stony Brook
    University of the first author include parts of the work presented
    in this manuscript; and (2) previous versions of this arXiv
    technical report also appear at
    \url{http://arxiv.org/abs/1212.1108}. The work was supported in
    part by National Science Foundation's Faculty Early Career
    Development Program (CAREER) Award IIS-1643006 (transferred from IIS-1054541).
\end{acknowledgements}

\bibliographystyle{spbasic} 
\bibliography{citation}

\appendix

\section{Our work in context}

In this section of the appendix we place our work in context of the closest related work
and other previous work on other forms of convergence of the AdaBoost
algorithm, or AdaBoost's variants.

\subsection{Related work:~\citet{RudinDynamics}}
\label{sec:rw}

We refer the reader to~\citet{boosting_book} for a textbook
introduction to AdaBoost.
We also refer the reader to
Appendix~\ref{sec:mlcontext} for a discussion of previous work on other forms of convergence of AdaBoost not considered in this article.

As mentioned in the Introduction, others have also approached the study of AdaBoost from a
dynamical-system perspective.  \citet{RudinDynamics} pioneered this
approach, demonstrating that the example weights that AdaBoost
generates enter cycles in many low
dimensional cases. To the best of our understanding, they proved that if
the AdaBoost ``is cycling,'' and several other conditions on the
cycling itself, the 
mistake matrix, and the selection of weak hypotheses at every round,
hold, then AdaBoost produces the maximum margin
solution~\citep[][Theorem 5]{RudinDynamics}. They also provably show several
other results regarding margin maximization, or lack thereof, for 
Optimal AdaBoost, as well as for the so-called ``Non-optimal
AdaBoost''~\citep[][Theorems 4, 6, and 7]{RudinDynamics}. 
We do not consider Non-optimal AdaBoost in
this paper; we discuss it briefly in Section~\ref{sec:future_work}
(Future Work).  

Their results also establish convergence to a
cycle for other special mistake matrices; e.g., those isomorphic to
the identity matrix. We discuss those cases in details as a way to
illustrate our approach in Appendix~\ref{app:ident}, where we also present
alternative derivations of some of their \emph{convergence}
results,
\emph{exlcuding those about margin maximization}~\citep[][Part 1 of
Theorems 1 and 2]{RudinDynamics}. They also show, by means
of an example, that an
infinite number of cycles may exist if there are examples that are
``identically classified''~\citep[][Theorem 3]{RudinDynamics}. 

We must
admit that once we
established that Optimal AdaBoost always exhibits cycling behavior, we
may have been able to obtain our results on the convergence of the
clasifier and its generalization error directly from their results.  In
that sense, the work here may provide an alternative approach and
proofs 
leading to those
results. A more careful study is needed to say that definitively.

Yet, it is fair to say that little was understood on what appeared to be the ``non-cyclic case''
in practice.  
Despite our theoretical results in this paper, it is common to observe seemingly chaotic
non-cyclic behavior on most higher dimensional cases; and
that behavior is typical
of AdaBoost on large real world
datasets. \citet{RudinDynamics} themselves point this out in their
Section 10, entitled, ``Indications of Chaos,'' where they show
empirical evidence of chaotic behavior when considering matrices other
than random
low-dimensional matrices. They attribute the chaotic behavior to
``sensitivity to initial conditons'' and ``movement into and out of cycles.'' 

In closing the discussion of the seminal work of~\citet{RudinDynamics}, 
we remind the reader that the main
interest in that work is the study of maximization of margins, a very important
problem that we do not consider in this paper.

\subsection{Previous work on convergence of other variants of AdaBoost or
  other types of convergence}
\label{sec:mlcontext}

In this appendix we discuss work on other forms of convergence of AdaBoost not considered in this article.

A bulk of the 
asymptotic analysis on AdaBoost
has been focused on how it minimizes different types of loss
functions, with most emphasis on the exponential loss.  Breiman and
others demonstrated how one can view AdaBoost as a coordinate-descent
algorithm that iteratively minimizes the exponential
loss~\citep{Breiman:1999:PGA:334369.334370,Mason00boostingalgorithms,FriedmanAdditiveLogistic}.
Under the so called ``Weak-learning Assumption,'' which we formally
state in our context as Condition~\ref{assume:WeakLearn} in
Section~\ref{sec:ds}, this minimization procedure is well understood,
and has a fast convergence rate: the exponential loss is an
upper-bound of the misclassification error rate on the training
dataset and goes to zero exponentially fast.  Later,
\citet{Collins2002} and~\citet{ZhangBoostingEarlyStopping} showed that AdaBoost minimizes the exponential loss 
even without the Weak-learning Assumption, in the ``unrealizable'' or ``non-separable'' case (i.e., the training error cannot achieve zero);
they do not provide convergence rates.  Finally,
\citet{MukherjeeRateOfConv} proved that AdaBoost enjoys a rate
polynomial in $1 / \epsilon$.  \citet{Telgarsky} achieves a similar
result by exploring the primal-dual relationship implicit in AdaBoost.
~\citet{Telgarsky_COLT13} deals with the convergence in terms of a variety of 
loss functions, not the classifier itself or its generalization error.
These results all concern the convergence of several types of loss functions, with the \emph{exponential loss}, and some of its variants, perhaps receiving most of the attention.  Meanwhile, in this paper our interest is the convergence of the basic, ``vanilla'' Optimal AdaBoost classifier itself, along with its generalization error and the data examples' margins, and time or per-round average of functions of its example weights.

There is another line of research related to the work just mentioned, mostly within the statistics community, that considers what happens in the limit of the number of rounds $T$ of AdaBoost, while  simultaneously letting the number of training examples $m$ go to infinity. From an ML perspective, we end up with a dataset of training examples of \emph{infinite size}, so that we have an infinite number of training examples at our disposal. Statistician often called this ``version'' of AdaBoost the \emph{population version}, while calling the version that considers a set of finite training examples the \emph{sample version}. Being slightly more technical, often in addition to letting $T \to \infty$, that work concerns the 
\emph{consistency}, and more specifically, 
statistical consistency in various forms,
of the 
asymptotic behavior of AdaBoost.
There are a number of papers that show that variants of AdaBoost are consistent 
(see, e.g., \citet{ZhangBoostingEarlyStopping}  and \citet{Bickel_BayesConsistent}).  
\citet{BartlettConsistent} showed that AdaBoost is
consistent if stopped at time $m^{1-\epsilon}$ for $\epsilon \in
(0,1)$, where $m$ is the number of examples in the training set. But
consistency, an inherently \emph{statistical} concept, is distinct
from the notion of convergence in this paper. There are also various
notions of consistency. In the context of the AdaBoost literature, the
Bayes-consistency of a predictor is of particular interest. In
statistics, an algorithm, predictor or estimator is
\emph{Bayes-consistent} if it produces a hypothesis whose
generalization error approaches the \emph{Bayes risk} in the limit of
the number of \emph{examples} $m$ in the training dataset; said
differently, the study of statistical consistency is by its very
nature under the condition of \emph{infinite-size} training  datasets.
Here our concern is the convergence of the generalization error of the
produced hypothesis in the limit of the number of \emph{iterations}
$T$ of the algorithm on a \emph{fixed-size} training dataset.

\section{Mathematical terminology and definitions}\label{A:BasicMathDefn}
\label{app:math}

Here we present a brief description and formal definitions of some fundamental concepts in real analysis, measure theory, and probability theory that we use during the technical sections. For more in depth information, we refer the reader to standard textbooks in those subjects such as~\citet{BartleRA,KFReal,BartleMT,WZMeasure,BTIntroProb,BreimanProb,DurrettProb}.

The notation within this section is self-contained, and generally
considered standard. The
reader should try to avoid confusing the notation used in this section
of the appendix with that used within other parts of the article.

\subsection{Concepts from real analysis and topology}
\label{app:real}

\begin{definition}{\bf (Topological Spaces, Interior Points, and Interior Sets)}
Let $X$ be a set. A function $\Nbf : X \to (2^X - \emptyset)$ is called a \emph{neighborhood topology} (on $X$) if it satisfies the following axioms for all $x \in X$:
\begin{enumerate}
\item
if $S \in \Nbf(x)$ then $x \in S$;
\item 
If $S \subset X$ and $Y \subset X$ for some $Y \in \Nbf(x)$, then $S \in \Nbf(y)$; 
\item 
for all $S,Y \in \Nbf(x)$, we have $S \bigcap Y \in \Nbf(x)$; 
and
\item 
for all $S \in \Nbf(x)$, there exists $Y \in \Nbf(x)$ such that, for all $y \in Y$, we have $S \in \Nbf(y)$.
\end{enumerate}
We call the ordered pair $(X,\Nbf)$ a \emph{topological space}. Whenever the neighborhood topology $\Nbf$ is implicit, we call the set $X$ a \emph{topological space}. If $S \subset X$, then $x$ is an \emph{interior point} of $S$ if there exists a neighborhood $Y \in \Nbf(x)$, such that $Y \subset S$. The \emph{interior of a set} $S \subset X$, denoted by $\Int{S}$, is the subset of $S$ that contains exactly all its interior points: $\Int{S} \equiv \{ x \in S \mid x \text{ is an interior point of } S \}$.
\end{definition}

\begin{definition}{\bf (Metric Spaces and Metric/Distance Functions)}
Let $X$ be some set and a function $d : X \times X \to \R$. We call the ordered pair $(X,d)$ a \emph{metric space} if $d$ is a \emph{metric} or \emph{distance function} on $X$; that is, if $d$ satisfies the following conditions:
\begin{enumerate}
\item ({\bf non-negative}) for all $x,y \in X$, $d(x,y) \geq 0$;
\item ({\bf identity}) for all $x,y \in X$, $d(x,y) = 0$ if and only if $x = y$;
\item ({\bf symmetric}) for all $x,y \in X$, $d(x,y) = d(y,x)$; and
\item ({\bf triangle inequality}) for all $x,y,z \in X$, $d(x,y) \leq d(x,z) + d(z,y)$.
\end{enumerate}
Whenever the metric $d$ is implicit, we call the set $X$ a \emph{metric space}.
\end{definition}

\begin{definition}{\bf (Metrizable Topological Spaces)}
Let $(X,d)$ be a metric space and $(X,\Nbf)$ be a topological space. We say the metric $d$ \emph{induces} the neighborhood topology $\Nbf$ if we define the function $\Nbf$ as follows.
\begin{enumerate}
\item Denote by $B(x,r) \equiv \{ y \in X \mid d(x,y) < r \}$ the ``open ball centered at $x \in X$ of radius $r > 0$'' with respect to metric $d$ and metric space $X$.
\item Set $\Nbf(x) \equiv \{ S \subset X \mid x \in S, B(x,r) \subset S, \text{ for some } r>0 \}$ .
\end{enumerate}
We call $(X,\Nbf)$ a \emph{metrizable topological space}. Hence, every metric space induces a metrizable topological space; and every metrizable topological space is inherently a metric space. Thus, viewed from this perspective, every metric space is a topological space.
\end{definition}

At this point, we could state the definition of the notions defined below (e.g., sequences, limits, and open sets), using the more general mathematical object of topological spaces. Instead, we find it more convenient to define them in terms of the more special concept of metric spaces.

\begin{definition}{\bf (Sequences and their Limits)} Let $(X,d)$ be a metric space. 
If, for all $t = 1, 2,\ldots$, $x_t \in X$, then we denote by $\{x_t\}$ the corresponding \emph{sequence} in $X$. We say the sequence $\{x_t\}$ of elements in $X$, denoted by $\{x_t\} \subset X$ for simplicity, has a \emph{limit} with respect to the metric space $(X,d)$, denoted by $\lim_{t \to \infty} x_t \equiv x^*$, if for all $\epsilon > 0$, there exists $T$, such that for all $t >T$, we have $d(x_t,x^*) < \epsilon$. 
\end{definition}

\begin{definition}{\bf (Open and Closed Sets, Bounded Sets, and Compact Sets)}
Let $(X,d)$ be a metric space. We say a set $S \subset X$ is a
\emph{closed set} if for every sequence $\{x_t\}$ of elements in $S$,
$\lim_t x_t \in S$ (i.e., if the set $S$ contains all of its limit
points). We say $S$ is an \emph{open set} if its \emph{complement}
$S^c \equiv X - S$ is closed. We say $S$ is a \emph{bounded set} if
there exists $r > 0$, such that for all $x,y \in S$, we have $d(x,y) <
r$. We say $S$ is a \emph{compact set} if $S$ is closed and bounded
(by the Heine-Borel Theorem).
\end{definition}

\subsection{Concepts from measure theory and probability theory}
\label{app:measure}

\begin{definition}{\bf ($\sigma$-algebra)} Let $X$ be some set. The
  set $\Sigma \equiv \Sigma_X$, composed of subsets of $X$, is called a \emph{$\sigma$-algebra} over $X$ if it satisfies the
following properties:
\begin{enumerate}
\item ({\bf non-empty}) $\Sigma \neq \emptyset$;
\item ({\bf closed under complementation, with respect to $X$}) if $A
 \in \Sigma$ then $A^c \equiv X-A \in \Sigma$; and
\item ({\bf closed under countable unions}) if $A_1,A_2,A_3,\ldots \in
   \Sigma$ then $A = A_1 \cup A_2 \cup A_3 \cup \cdots \in \Sigma$.
\end{enumerate}
\end{definition}
\begin{definition}{\bf (Measure, Measurable Space, Measurable Set,
 Measurable Function, and
 Meas\-ure Space)} Let $X$ be
some set and $\Sigma$
a $\sigma$-algebra over $X$. A function $\mu : \Sigma \to [-\infty,+\infty]$ is
called a \emph{measure} if it satisfies the
following properties:
\begin{enumerate}
\item ({\bf non-negative}) $\mu(A) \geq 0$ for all $A \in \Sigma$;
\item ({\bf null empty set}) $\mu(\emptyset) = 0$; and
\item ({\bf countable additivity}) if for any countable collection
 $\{A_i\}_{i \in I}$ of pairwise disjoint sets in $\Sigma$, we have
 $\mu\left( \cup_{i \in I} A_i \right) = \sum_{i \in I} \mu\left( A_i \right)$.
\end{enumerate}
The ordered pair $(X,\Sigma_X)$ is called a \emph{measurable space} and the
members of $\Sigma_X$ are called \emph{measurable sets}. If the ordered pair
$(Y,\Sigma_Y)$ is another measurable space, then a function $f : X \to
Y$ is called \emph{a measurable function} if for all measurable sets
$B \in \Sigma_Y$, the pre-image (i.e., inverse image) is $X$-measurable:
i.e., $f^{-1}(B) \in \Sigma_X$. An ordered triple $(X,\Sigma,\mu)$ is called a
\emph{measure space}. 
\end{definition}
\begin{definition}{\bf (Probability Measure, Probability Space, Probabilistic Model, Outcome, Samples Space, Event, Probability Law, and the Axioms of Probability)} 
Let $(X,\Sigma,\mu)$ be a measure space.
A measure $\mu$ such that $\mu(X) = 1$ is called
a \emph{probability measure}. A measure space with a probability
measure is called a \emph{probability space}, and we typically refer to $X$ as the \emph{set of outcomes} and $\Sigma$ the \emph{set of events}, i.e., set of \emph{(measurable) subsets} of $X$. Note that the probability measure $\mu$, when viewed as the \emph{probability law}, over \emph{events} in $\Sigma$, of a \emph{probabilistic model} with \emph{sample space} $X$ (i.e., the set of \emph{outcomes}), satisfies all the \emph{(Kolmogorov's) axioms of probability}.
\end{definition}
\begin{definition}{\bf (Borel $\sigma$-algebra and Borel measures)}
Let $(X,d)$ be a metric space. A $\sigma$-algebra $\Sigma_X$ over $X$
is called a
\emph{Borel $\sigma$-algebra}, with respect to $(X,d)$, if $\Sigma_X$
is generated using the
  the set of all open subsets of $X$, with respect to $(X,d)$ (i.e., we start defining
  $\Sigma_X$ by first including all open subsets of $X$ and then recursively
  applying the result of the complement and countable union operations
  over all the resulting sets). Given a measure space
  $(X,\Sigma_X,\mu)$, the measure $\mu$ is called a \emph{Borel
    measure} if $\Sigma_X$ is a Borel $\sigma$-algebra; if $\mu$ is a
  probability measure, then it is called a
  \emph{Borel probability measure}.
\end{definition}
\begin{definition}{\bf (Measurable and Measure-Preserving
    Transformations)}
\label{def:invariant}
Let $(X,\Sigma_X)$ be a meas\-ur\-a\-ble space. A \emph{transformation} $M : X \to
X$ is \emph{measurable}, with respect to $(X,\Sigma_X)$, if $M$ is a
measurable function, with respect to $(X,\Sigma_X)$.
Let $(X,\Sigma_X,\mu)$ be a measure space. We say transformation $M$
is \emph{measure preserving}, with respect to $(X,\Sigma_X,\mu)$, if $M$ is
measurable, with respect to $(X,\Sigma_X)$, and for any measurable set $A
\in \Sigma_X$, we have $\mu(A) = \mu(M^{-1}(A))$. We also call such
$\mu$ an \emph{invariant measure} of $M$ on $(X,\Sigma_X)$, or simply
on $X$ when $\Sigma_X$ is obvious from context.
\end{definition}

\section{Classroom Example: Addendum}
\label{app:clex}
The full set of label dichotomies for the example is
\begin{align*}
\Dich(\Hypo,S) = \{ 
& (+1,+1,+1,+1,+1,+1), (-1,-1,-1,-1,-1,-1),\\
& (-1,+1,+1,+1,+1,+1), (+1,-1,-1,-1,-1,-1),\\
& (-1,-1,+1,+1,+1,+1), (+1,+1,-1,-1,-1,-1),\\
& (-1,-1,-1,+1,+1,+1), (+1,+1,+1,-1,-1,-1),\\
& (-1,-1,-1,-1,+1,+1), (+1,+1,+1,+1,-1,-1),\\
& (-1,-1,-1,-1,-1,+1), (+1,+1,+1,+1,+1,-1),\\
& (+1,+1,-1,+1,+1,+1), (-1,-1,+1,-1,-1,-1),\\
& (+1,+1,-1,-1,+1,+1), (-1,-1,+1,+1,-1,-1),\\
& (-1,+1,-1,-1,+1,+1), (+1,-1,+1,+1,-1,-1),\\
& (-1,+1,-1,-1,-1,+1), (+1,-1,+1,+1,+1,-1),\\
& (-1,+1,-1,-1,-1,-1), (+1,-1,+1,+1,+1,+1) \} \; .
\end{align*}
The full set of mistake dichotomies for the example is
\begin{align*}
\Mcal = \{ 
& (0,1,0,1,0,1), (1,0,1,0,1,0),\\
& (1,1,0,1,0,1), (0,0,1,0,1,0),\\
& (1,0,0,1,0,1), (0,1,1,0,1,0),\\
& (1,0,1,1,0,1), (0,1,0,0,1,0),\\
& (1,0,1,0,0,1), (0,1,0,1,1,0),\\
& (1,0,1,0,1,1), (0,1,0,1,0,0),\\
& (0,1,1,1,0,1), (1,0,0,0,1,0),\\
& (0,1,1,0,0,1), (1,0,0,1,1,0),\\
& (1,1,1,0,0,1), (0,0,0,1,1,0),\\
& (1,1,1,0,1,1), (0,0,0,1,0,0),\\
& (1,1,1,0,1,0), (0,0,0,1,0,1) \} \; .
\end{align*}

\section{Properties related to Optimal AdaBoost update}
\label{app:abup}

The following properties related to the Optimal AdaBoost update will
be useful in our technical proofs. Some may be of independent
interest. They follow directly from the respective definitions.
\begin{proposition}\label{P:update}
The following statements about the AdaBoost update hold.
\begin{enumerate}
\item Suppose Condition~\ref{cond:nathypo} (Natural Weak-Hypothesis Class)
holds. Then the following also holds.
\begin{enumerate}
\item For all $w \in \Delta_m^+$, $\T_\eta(w) \in \Delta_m^+ \cap \pi_{\frac12}(\eta)$ for all $\eta \in
\Mcal$, and thus $\A(w) \in \Delta_m^+ \cap \pi_{\frac12}(\eta^w)$. 
\item For all $w \not\in \Delta_m^+$, for all $\eta \in \Mcal$,
\begin{enumerate}
\item if $\eta \cdot w
> 0$, then $\T_\eta(w) \in \pi_{\frac12}(\eta)$, and thus $\A(w) \in
\pi_{\frac12}(\eta^w)$; 
\item otherwise, if $\eta \cdot w
= 0$, then $\eta \cdot \T_\eta(w) = 0$, and thus, if $\eta^w \cdot w =
0$, then $\eta^w \cdot \A(w) = 0$.
\end{enumerate}
\end{enumerate}
\item For all $\eta \in \Mcal$, for all $w \in \pi_{\frac12}(\eta)$,
  $\T_\eta(w) = w$, and thus $\A(w) = w$ if $w \in \pi_{\frac12}(\eta^w)$. Hence, for all $\eta \in
  \Mcal$, $\T_\eta(\pi_{\frac12}(\eta))
=
\pi_{\frac12}(\eta)$, and thus $\A(\pi_{\frac12}(\eta^w))
=
\pi_{\frac12}(\eta^w)$. 
\item Suppose Condition~\ref{cond:nathypo} (Natural Weak-Hypothesis Class)
holds. Then, for almost every $w \in
\Delta_m$, we have that $\T_\eta(w) \in \pi_{\frac12}(\eta)$ for all $\eta \in
\Mcal$, and thus $\A(w) \in \pi_{\frac12}(\eta^w)$. If, in addition,
Condition~\ref{assume:WeakLearn} (Weak Learning) holds,
then, for almost every $w \in \Delta_m$, we have $\A(w) \neq w$.
\end{enumerate}
\end{proposition}

\section{When a dynamical system converges to a cycle}
\label{app:cycles}

In this appendix we explore the consequences of convergence to a cycle
in terms of the convergence of time averages and the construction of
invariant measures.

\subsection{Convergence of time averages}
\label{app:cyclesra}

If the evolutions of a dynamical system 
converges to a cycle, then the convergence of the time/per-round average of
functions of its state evolution follow easily.

\begin{proposition}
\label{pro:cycles}
Let $M : W \to W$ be a transformation and $f : W \to \R$ be a
function. Let $\omega_1 \in W$ be an initial point in a sequence and
$\omega_{t+1} \equiv M^{(t)}(\omega_1)$. Suppose
that the sequence $(\omega_t)$ converges to a cycle $(\omega^{(s)})_{s=0,1,\ldots,p}$ of periodicity $p \equiv
p(\omega_1)$ in finite time and let \(
\widehat{f}_{\omega_1} \equiv \frac{1}{p} \sum_{s=0}^{p-1} \omega^{(s)} \;
.\) 
Then \( \lim_{T \to \infty} \frac{1}{T} \sum_{t=1}^T
  f(\omega_t) = \widehat{f}_{\omega_1} \; . \)
The same holds for convergence in the limit if $f$ is uniformly
continuous on $W$; e.g., if $f$
is continuous on $W$ and $W$ is compact, by the Uniform Continuity Theorem~\citep[][Theorem 23.3, pp.160]{BartleRA}.
\end{proposition}
\begin{proof}
Suppose the sequence $(\omega_t)$ enters a cycle in finite time and
let $T_1 \equiv T_1(\omega)$ be the first time it does. Consider the average, assuming, without loss of generality, that
$T > T_1+ p - 1$. Let $L \equiv L(\omega) \equiv \lfloor
\frac{T-p+1}{p} \rfloor$ and $r \equiv r(\omega) \equiv T - T_1 - p
L$. We have that 
\begin{align*}
\frac{1}{T} \sum_{t=1}^T f(\omega_t) = & \frac{1}{T} \sum_{t=1}^{T_1-1} f(\omega_t) + \frac{1}{T} \sum_{t=T_1}^T f(\omega_t) 
=  \frac{1}{T} \sum_{t=1}^{T_1-1} f(\omega_t) + \frac{1}{T}
    \sum_{t=T_1}^{T_1 + p L - 1} f(\omega_t)) + \frac{1}{T} \sum_{t=T_1 +
    p L}^T f(\omega_t) \\
= & \frac{1}{T} \sum_{t=1}^{T_1-1} f(\omega_t) + \frac{1}{T} L 
    \widehat{f}_\omega + \frac{1}{T}
    \sum_{t=T_1}^{T_1+r} f(\omega_t) \; .
\end{align*}
Now taking the $\liminf_{T \to \infty}$, which always exists, we obtain
\begin{align*}
\liminf_{T \to \infty} \frac{1}{T} \sum_{t=1}^T f(\omega_t) =
  \liminf_{T \to \infty} \frac{1}{T} L \widehat{f}_w + \frac{1}{T}
    \sum_{t=T_1}^{T_1+r} f(\omega_t) 
\geq \liminf_{T \to \infty} \frac{1}{T} L \widehat{f}_w +
       \frac{1}{T} (r+1) \min_{t=T_1,\ldots,T_1+r} f(\omega_t) 
= \widehat{f}_w \; .
\end{align*}
A similarly derivation yields
\(
\limsup_{T \to \infty} \frac{1}{T} \sum_{t=1}^T f(\omega_t)
\leq  \widehat{f}_w \; .
\)
The result for the case of convergence in finite time 
follows because
\begin{align*}
\limsup_{T \to \infty} \frac{1}{T} \sum_{t=1}^T f(\omega_t)
\leq \widehat{f}_w \leq \liminf_{T \to \infty} \frac{1}{T} \sum_{t=1}^{T} f(\omega_t) \leq \limsup_{T \to \infty} \frac{1}{T} \sum_{t=1}^T f(\omega_t)
\end{align*}
implies 
\(
\limsup_{T \to \infty} \frac{1}{T} \sum_{t=1}^T f(\omega_t) =\liminf_{T \to \infty} \frac{1}{T} \sum_{t=1}^{T} f(\omega_t)  = \lim_{T \to \infty} \frac{1}{T} \sum_{t=1}^T f(\omega_t) = \widehat{f}_w \; .
\)

The proof for the case of convergence in the limit is almost
identical, \emph{except} that $T_1$ may no longer by finite. However, by the uniform continuity of $f$ on $W$, we can find a $\tau \equiv \tau(\tau') > 0$ such that if
$\omega,\omega' \in W$ and $d(\omega,\omega') < \tau$ then
$|f(\omega) - f(\omega')| < \tau'$. Let $s_t \in \argmin_{s=0,1,\ldots,p-1}
d(\omega_t,\omega^{(s)})$. By the convergence of $(\omega_t)$
to $(\omega^{(s)})$, 
we have that $d(\omega_t, \omega^{(s_t)}) < \tau$ for all $t > T_1 -
1$, for some corresponding 
finite time $T_1 \equiv T_1(\omega_1,\tau')$. Following the same
argument used for the case of convergence in finite time, we can derive that
\(
\liminf_{T \to \infty} \frac{1}{T} \sum_{t=1}^T f(\omega_t) 
\geq \widehat{f}_{\omega_1} - \tau'
\)
and that
\(
\limsup_{T \to \infty} \frac{1}{T} \sum_{t=1}^T f(\omega_t) 
\leq \widehat{f}_{\omega_1} + \tau' \; ,
\)
so that 
\[ \limsup_{T \to \infty} \frac{1}{T} \sum_{t=1}^T f(\omega_t) -
  \liminf_{T \to \infty} \frac{1}{T} \sum_{t=1}^T f(\omega_t)  
\leq 2 \tau' \; , 
\]
 from which the same result follows immediately because $\tau'$ is
 arbitrary.
\qed
\end{proof}

Note that this proposition immediately implies that we can easily obtain all of the same results about
the convergence of the classifier itself and its generalization error
that we prove in this paper when Optimal AdaBoost converges to cycle,
not just a cycle of sets, and even if it only does so in the limit. We leave out the formal statements in the interest of terseness.

\subsection{Constructive proof of existence of invariant
  measures}
\label{app:cyclesBirkhoff}

Here we show how to construct an invariant measure when a dynamical
system is guarantee to converge to a cycle. In particular, we show that the resulting
empirical measure
is invariant and ergodic. The proposition also serves to illustrate
some of the concepts sorrounding ergodicity in
dynamical system.

\begin{proposition}
\label{pro:inv}
Let $M : W \to W$ be a transformation, $\omega \in W$ be an
initial point in a sequence.
Suppose
that there exists $\omega_1(\omega) \in M$ such that the sequence
$(M^{(t-1)}(\omega))$ converges in finite time to a cycle \[
(M^{(s)}(\omega_1(\omega)))_{s=0,1,\ldots,p(\omega)-1} \] of periodicity $
p(\omega)$ anchored by $\omega_1(\omega)$. Then the following holds.
\begin{enumerate}
\item The sequence of empirical measures \(
(\frac1T \sum_{t=1}^T \delta_{M^{(t-1)}(\omega)}) \)
converges 
to the (discrete) probability
measure \( \widehat{\mu}_\omega \equiv \frac{1}{p(\omega)} \sum_{s=0}^{p(\omega)-1}
\delta_{M^{(s)}(\omega_1(\omega))} \). In addition, $\widehat{\mu}_\omega$ is $M$-invariant and ergodic.
\item Consider another sequence $(M^{(t-1)}(\omega')$ that
  converges to a cycle
  $(M^{(s')}(\omega_1(\omega')))_{s'=0,1,\ldots,p(\omega')-1}$ of periodicity $
p(\omega')$ anchored by $\omega_1(\omega')$.
Then the
  following also holds.
\begin{enumerate}
\item The dynamical system is not uniquely ergodic if \[ (M^{(s')}(\omega_1(\omega')))_{s'=0,1,\ldots,p(\omega')-1} \neq
  (M^{(s)}(\omega_1(\omega)))_{s=0,1,\ldots,p(\omega)-1} \; , \] i.e., if
  the empirical measures are not unique.
\item Every (strict) mixture of two distinct empirical measures is
  invariant but \emph{not}
  ergodic: i.e., conversely, for every $\rho \in (0,1)$, the mixture of empirical
  measures $\rho \widehat{\mu}_{\omega} + (1-\rho)
  \widehat{\mu}_{\omega'}$ is ergodic if and only if
  $\widehat{\mu}_{\omega} = \widehat{\mu}_{\omega'}$.
\end{enumerate}
\end{enumerate}
The same holds for convergence in the limit if 
$f$ is uniformly continuous on $W$; e.g., if $f$
is continuous on $W$ and $W$ is compact, by the Uniform Continuity Theorem~\citep[][Theorem 23.3, pp.160]{BartleRA}.
\end{proposition}
\begin{proof}
For Part 1, the proof for convergence in finite time follows from
Proposition~\ref{pro:cycles} by applying it several times with $f = \delta_{\omega^{(s)}}$ for
each $s = 0,1,\ldots,p(\omega)-1$. The proof for the case of
convergence in the limit follows the
same argument used for Part 2.c of Theorem~\ref{thm:CFAB}: that a
sufficient condition for the convergence of the empirical measure is
that the time average of any continuous function $f : W \to \R$
exists. But then, by Proposition~\ref{pro:cycles} again, the time average of any
such function converges in this case too. Hence, the empirical measure exists and
equals $\widehat{\mu}_\omega$. Now, note that 
for any measurable $M$-invariant set $V$, i.e., $V =
M^{-1}(V)$, we have $\widehat{\mu}_\omega(V) = \widehat{\mu}_\omega(M^{-1}(V))
= 1$ if and only if $V = \cup_{s=0}^{p(\omega)-1} \{
\omega^{(s)} \}$. Hence, $\widehat{\mu}_\omega$ is both invariant
(Definition~\ref{def:invariant}) and
ergodic (Definition~\ref{def:ergodic}) by their respective
definitions.

Part 2.a follows by the definition of unique ergodicity
(Definition~\ref{def:ergodic}) because the condition there implies that $\widehat{\mu}_\omega$ and
$\widehat{\mu}_{\omega'}$ have different supports. Part 2.b follows for
similar reasons: the support of any (strict) mixture of empirical measure is
the union of their support. But, the mixture will have two invariant
sets of positive measure if and only if the empirical measures are
different, from which the result follows by the definition of an
ergodic measure. Also, note that invariance is not affected by convex combinations.
\qed
\end{proof}

\section{Characterizing the inverse of the Optimal-AdaBoost update}
\label{app:Ainv}

When studying the dynamics of the AdaBoost update $\A$, it is natural
to ask, when given $w \in \Delta_m$, what is $\A^{-1}(w)$?  Or
similarly, when given $E \subset \Delta_m$, what is $\A^{-1}(E)$?  An
analysis of the inverse is essential in establishing an invariant
measure, and is useful for some of the technical proofs. In particular, it will allow us to establish the existence of a measure over the set of interest, its ``trapping/attracting set,'' on which $\A$ is measure-preserving. To approach this problem, we decompose the inverse into a union of line segments.
\begin{proposition}\label{P:Inv}
Suppose we have $\eta \in \Mcal$ and 
$w \in \A(\Delta_m)$.
Define $w_\eta^-$ and $w_\eta^+$ such that $w_\eta^-(i) \equiv w(i)
\eta(i)$ and $w_\eta^+(i) \equiv w(i)
(1-\eta(i))$. Under Condition~\ref{cond:nathypo} (Natural
Weak-Hypothesis Class),
if $\eta \cdot w > 0$, then $\T_\eta^{-1}(w) = \{ 2\rho w_\eta^- + 2(1-\rho)w_\eta^+ |\, \rho \in (0,1) \}$.
\end{proposition}
\begin{proof}
Let $L(\eta,w) = \{ 2\rho w_\eta^- + 2(1-\rho)w_\eta^+ |\, \rho \in (0, 1) \}$.  Consider an element $w' \in L(\eta,w)$.  Clearly, $w' = 2\rho' w_\eta^- + 2(1-\rho')w_\eta^+$ for some $\rho' \in (0,1)$.  Then, we have
\(
\eta \cdot w' = \eta \cdot (2\rho' w_\eta^-) + \eta \cdot (2(1-\rho')w_\eta^+)
= 2\rho' (\eta \cdot  w_\eta^-) 
= 2\rho' (\eta \cdot w) 
= (2\rho') \frac 12 
= \rho',
\)
where $\eta \cdot w = \frac12$ follows from Proposition~\ref{P:update}.
Using this fact, we see for $i$ such that $\eta(i)=1$, we have
\(
[\T_\eta(w')](i) = w'(i) \times \frac{1}{2(\eta\cdot w')} 
 =  2\rho' w(i) \times \frac{1}{2\rho'} 
 =  w(i).
\)
And similarly, for $i$ such that $\eta(i) = 0$, we have
\(
[\T_\eta(w')](i) = w'(i) \times \frac{1}{2(1-(\eta \cdot w'))} 
 =  2(1-\rho') w(i) \times \frac{1}{2(1-\rho')} 
 =  w(i).
\)
Pulling the cases together, we conclude that $\T_\eta(w') = w$ and $w' \in \T_\eta^{-1}(w)$.

Instead, suppose that $w' \in \T_\eta^{-1}(w)$.  From Definition~\ref{D:1}, we see that
\(
w'(i) = 2 w(i) ( \eta(i) (\eta\cdot w') + (1-\eta(i)) (1-(\eta\cdot w')) ).
\)
Setting $\rho' = \eta\cdot w'$, we see that $w' \in L(\eta,w)$.
\qed
\end{proof}

\begin{figure}[t]
\begin{center}
\begin{tabular}{ll}
\includegraphics[width=0.45\textwidth]{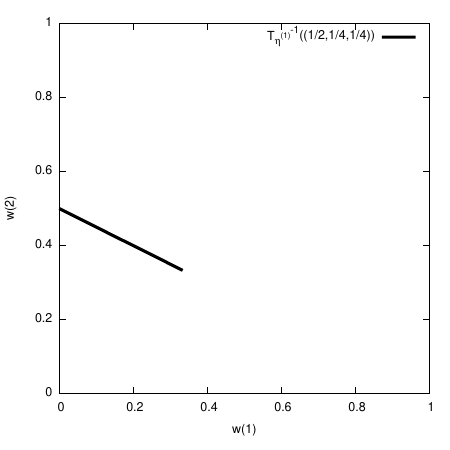} & \includegraphics[width=0.45\textwidth]{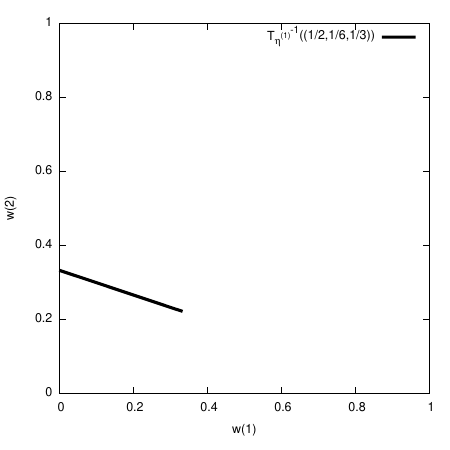}\\
\includegraphics[width=0.45\textwidth]{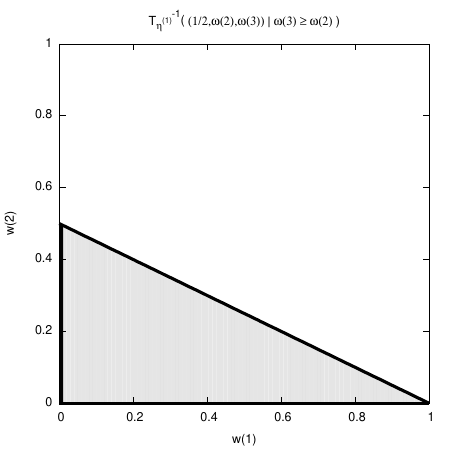} & \includegraphics[width=0.45\textwidth]{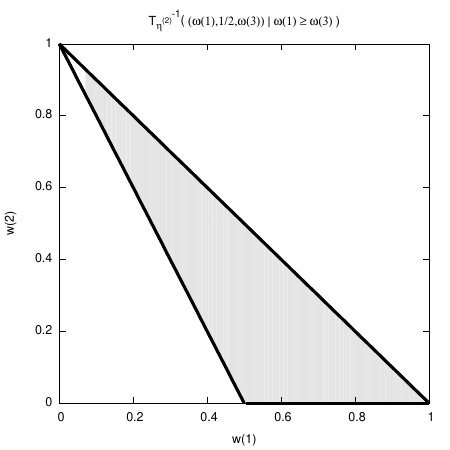}\\
\end{tabular}
\end{center}
\caption{{\bf An Illustration of $\T_\eta^{-1}$:} (\emph{top left})
Eqn.~\ref{eqn:Tinv_ex_1} is a (closed) line segment; (\emph{top right})
Eqn.~\ref{eqn:Tinv_ex_2} is another (closed) line segment; (\emph{bottom left})
Eqn.~\ref{eqn:Tinv_ex_3} is a closed, filled triangle (i.e., a compact set); (\emph{bottom right})
Eqn.~\ref{eqn:Tinv_ex_4} is also a closed, filled triangle (i.e., a
compact set).}
\label{fig:Tinv}
\end{figure}

Fig.~\ref{fig:Tinv} provides an illustration of $\T^{-1}$ in the context of the simple $\Mcal$ isomorphic to the $(3
\times 3)$ identity matrix used in Appendix~\ref{app:ident}. 
For instance, in the context of that example, 
for all $k \in \{1,2,3\}$, and all $w \in \Delta_3^\circ$, we have
\begin{align*}
\T_{\eta^{(k)}}^{-1}(w) = 
\begin{cases}
\emptyset, & \text{if $w(k) \neq\frac12$,}\\ 
\Delta_3^\circ, & \text{if $w(k) = \frac12$.}
\end{cases}
\end{align*}
The following are other examples, presented in Fig.~\ref{fig:Tinv}, which we can compute using the
characterization provided in the last proposition above (Proposition~\ref{P:Inv}).
\begin{align}
\label{eqn:Tinv_ex_1}
\T_{\eta^{(1)}}^{-1}\left( \left(\frac12, \frac14, \frac14 \right)
  \right) = \left\{ w \in \Delta_3 \, \mid \, w(1) \leq \frac13, w(2) =
  \frac12 - \frac12 w(1) \right\}
\end{align}
\begin{align}
\label{eqn:Tinv_ex_2}
\T_{\eta^{(1)}}^{-1}\left( \left(\frac12, \frac16, \frac13 \right)
  \right) = \left\{ w \in \Delta_3 \, \mid \, w(1) \leq \frac13, w(2) =
  \frac13 - \frac13 w(1) \right\}
\end{align}
\begin{align}
\label{eqn:Tinv_ex_3}
\T_{\eta^{(1)}}^{-1}\left( \left\{ \left( \frac12, w(2), w(3) \right)
  \in \Delta_3 \,
  \mid \, w(2) \leq w(3) \right\}\right) = \left\{ w \in \Delta_3 \, \mid \, w(2)
  \leq \frac12 - \frac12 w(1) \right\} 
\end{align}
\begin{align}
\label{eqn:Tinv_ex_4}
\T_{\eta^{(2)}}^{-1}\left( \left\{ \left( w(1), \frac12, w(3) \right)
  \in \Delta_3 \,
  \mid \, w(3) \leq w(1) \right\}\right) = \left\{ w \in \Delta_3 \,
  \mid \, 1-2w(1) \leq w(2) \right\} \; .
\end{align}

So, $\T_\eta(w)$ has a very clean inverse, being simply a line through simplex space.  But, it is important to note that $\T_\eta(w)$ is hypothetical, asking ``where would $w$ go if $\eta=\eta^w$?'' and is not the true AdaBoost weight update, $\A(w)$.  Regardless, the inverse $\A^{-1}(w)$ does decompose into a union of these line segments.
\begin{proposition}
\label{pro:Ainv}
Let 
$w \in \A(\Delta_m)$.  
Then $\A^{-1}(w) = \bigcup_{\eta\in \Mcal} (\T_\eta^{-1}(w) \cap \pi^*(\eta))$.
\end{proposition}
\begin{proof}
Take $w' \in \A^{-1}(w)$.  First, by Definitions~\ref{D:2}
and~\ref{def:pistar} about the tie-breaking of mistake dichotomies and
the respective partition tie-breaking induces in $\Delta_m$,
respectively, for $\eta^{w'} \in \Mcal$ we have $w'\in \pi^*(\eta^{w'})$. By the definition of the inverse of a function, we have
$\A(w') = w$. By the definition of $\A$
(Definition~\ref{D:AdaBoost_update}), we have $\A(w') =
\T_{\eta^{w'}}(w') = w$, which implies $w' \in
\T_{\eta^{w'}}^{-1}(w)$.  Therefore, we have $w' \in (\T_{\eta^{w'}}^{-1}(w)
\cap \pi^*(\eta^{w'})) \subset \bigcup_{\eta\in \Mcal} (\T_\eta^{-1}(w) \cap \pi^*(\eta))$.

Instead, take $w' \in \bigcup_{\eta\in \Mcal} (\T_\eta^{-1}(w) \cap \pi^*(\eta))$.  It must be the case that $w' \in \T_{\eta^w}^{-1}(w) \cap \pi^*(\eta^w)$, because $w'$ can only be in $\pi^*(\eta')$ for one possible $\eta'\in \Mcal$, namely $\eta'=\eta^{w'}$.  But by Definition~\ref{D:2}, we see that implies $\A(w') = w$.  Therefore, $w' \in \A^{-1}(w)$.
\qed
\end{proof}

\begin{figure}
\begin{center}
\begin{tabular}{ll}
\includegraphics[width=0.45\textwidth]{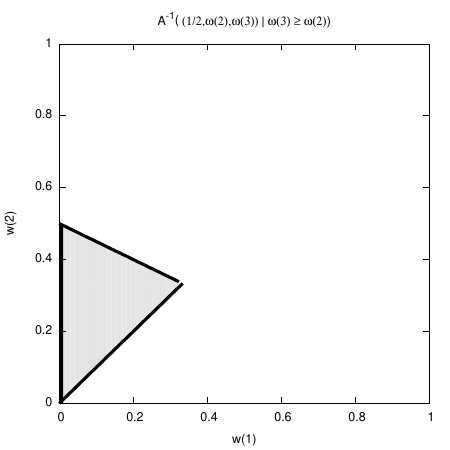} & \includegraphics[width=0.45\textwidth]{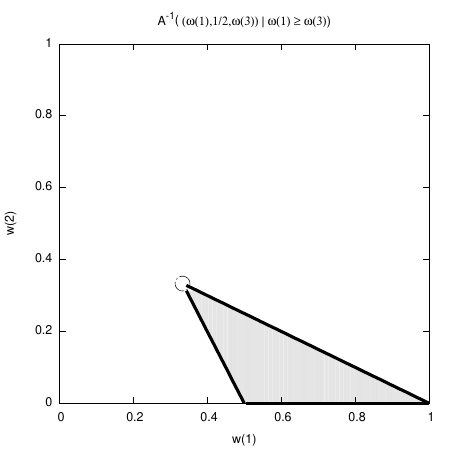}\\
\end{tabular}
\end{center}
\caption{{\bf An Illustration of $\A^{-1}$:} (\emph{left})
Eqn.~\ref{eqn:Ainv_ex_3} is a closed, filled triangle (i.e., a compact set); (\emph{right})
Eqn.~\ref{eqn:Ainv_ex_4} is also a mostly closed, filled triangle,
\emph{except} for the point $(\frac13,\frac13,\frac13)$.}
\label{fig:Ainv}
\end{figure}

Fig.~\ref{fig:Ainv} provides an illustration of 
$\A^{-1}$ in the context of the same example.
used in Appendix~\ref{app:ident}. 
For instance, in the context of that example, 
for all $w \in \Delta_m$, we have
\begin{align*}
\A^{-1}(w) = 
\begin{cases}
\emptyset, & \text{if $w(k) \neq\frac12$ for all $k$,}\\ 
\Delta_3, & \text{if $w(k) = \frac12$ and $w \in \pi^*(\eta^{(k)})$ for some $k$.}
\end{cases}
\end{align*}
The following are other examples, presented in Fig~\ref{fig:Ainv}, which we can compute using the
characterization provided in the last proposition above (Proposition~\ref{pro:Ainv}).
\begin{align*}
\A^{-1}\left( \left(\frac12, \frac14, \frac14 \right)
  \right) = \left\{ w \in \Delta_3 \, \mid \, w(1) \leq \frac13, w(2) =
  \frac12 - \frac12 w(1) \right\}
\end{align*}
\begin{align*}
\A^{-1}\left( \left(\frac12, \frac16, \frac13 \right)
  \right) = \left\{ w \in \Delta_3 \, \mid \, w(1) \leq \frac13, w(2) =
  \frac13 - \frac13 w(1) \right\}
\end{align*}
\begin{align}
\label{eqn:Ainv_ex_3}
\nonumber \A^{-1}\left( \left\{ \left( \frac12, w(2), w(3) \right) \,
  \mid \, w(2) \leq w(3) \right\}\right) = & \left\{ w \in \Delta_3 \,
  \mid \, w(1) \leq w(2) \leq w(3) \right\}\\
= &  \left\{ w \in \Delta_3 \,
  \mid \, w(1) \leq w(2) \leq \frac12 - \frac12 w(1) \right\} 
\end{align}
\begin{align}
\label{eqn:Ainv_ex_4}
\nonumber \A^{-1}\left( \left\{ \left( w(1), \frac12, w(3) \right)
  \in \Delta_3 \,
  \mid \, w(3) \leq w(1) \right\}\right) = & \left\{ w \in \Delta_3 \,
  \mid \, w(2) < w(1), w(3) < w(1) \right\} \\
= & \left\{ w \in \Delta_3 \,
  \mid \, w(1) > \frac13, 1-2w(1) \leq w(2) \leq \frac12 - \frac12 w(1)
    \right\} \; .
\end{align}

\section{Mistake dichotomies isomorphic to $(m \times
  m)$ identity matrix}
\label{app:ident}

In order to illustrate the notation introduced in Sections~\ref{sec:details} and~\ref{sec:prelimds}, let us use a simple set of mistake dichotomies, equivalent to the $(3
\times 3)$ identity matrix: i.e.,
$\Mcal = \{(1,0,0),(0,1,0),(0,0,1)\} =
\{\eta^{(1)},\eta^{(2)},\eta^{(3)}\}$. (We note
that~\citet{RudinDynamics} studied this example, and its
generalization, and established convergence
properties. We discuss this further at the end of this section of the
appendix.) For this illustration, we define $\AdaSel$ such that it
                              encodes the strict
                                               preference relation
                                               $\eta^{(1)} \succ
                                               \eta^{(2)} \succ \eta^{(3)}$ (e.g.,
                                               $\AdaSel(\{\eta^{(2)},\eta^{(3)}\})
                                               = \eta^{(2)}$).

\begin{figure}
\begin{center}
\begin{tabular}{ll}
\includegraphics[width=0.45\textwidth]{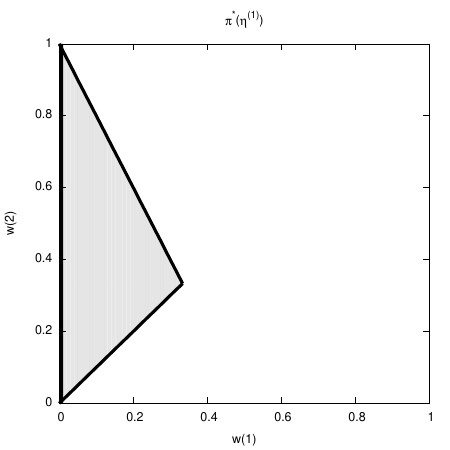} & \includegraphics[width=0.45\textwidth]{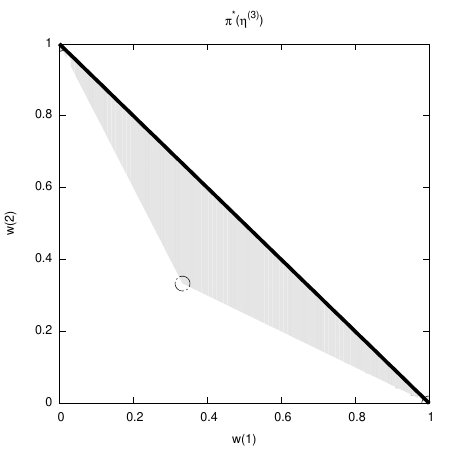}\\
\includegraphics[width=0.45\textwidth]{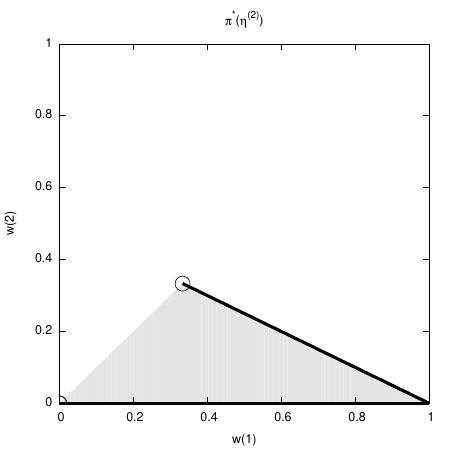} & \includegraphics[width=0.45\textwidth]{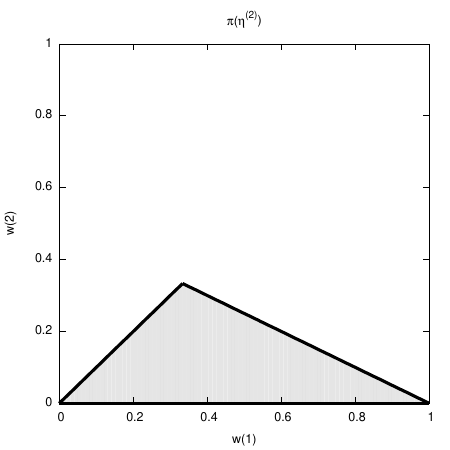}\\
\end{tabular}
\end{center}
\caption{{\bf An Illustration of $\pi^*$ and $\pi$:} (\emph{top left})
$\pi^*(\eta^{(1)}) = \pi(\eta^{(1)})$ is a closed set; (\emph{bottom
  left}) $\pi^*(\eta^{(2)})$ is an open set at the line segment $w(2)=w(1),
w(1) \in
[0,\frac13]$. 
(\emph{top
  right}) $\pi^*(\eta^{(3)})$ is an open set at the line segments $w(2)=1-2w(1),
w(1) \in
[0,\frac13]$ and $w(2)=\frac12 -\frac12 w(1), w(1) \in (\frac13,1]$,
but contains the line segment $w(2)=1-w(1), w(1) \in (0,1)$; (\emph{bottom
  right}) $\pi(\eta^{(2)})$ is the closure of $\pi^*(\eta^{(2)})$ and
thus closed (i.e., contains the line segment $w(2)=w(1),
w(1) \in
[0,\frac13]$)}
\label{fig:pi}
\end{figure}

Fig.~\ref{fig:pi} provides an illustration of $\pi$ and $\pi^*$ using
this example.
For simplicity, we use a
$2$-dimensional projection of the $3$-dimensional simplex. 
For this example, from Definition~\ref{def:pistar}, we have
\begin{align*}
\pi^*(\eta^{(1)}) = & \{ w \in \Delta_3 \, \mid \, w(1) \leq
\min(w(2),w(3)) \},\\
 \pi^*(\eta^{(2)}) = &\{ w \in \Delta_3 \, \mid \,
w(2) < w(1), w(2) \leq w(3) \}, \text{ and}\\
\pi^*(\eta^{(3)}) = &\{ w \in \Delta_3 \, \mid \,
w(3) < \min(w(1), w(2)) \} \; ;
\end{align*}
and from Definition~\ref{def:pi}, we have
\begin{align*}
\pi(\eta^{(1)}) = & \pi^*(\eta^{(1)}),\\
 \pi(\eta^{(2)}) = &\{ w \in \Delta_3 \, \mid \,
w(2) \leq \min(w(1),w(3)) \}, \text{ and}\\
\pi(\eta^{(3)}) = &\{ w \in \Delta_3 \, \mid \,
w(3) \leq \min(w(1), w(2)) \} \; .
\end{align*}
From Definition~\ref{D:1}, we have, for $k \in \{1,2,3\}$,
\begin{align*}
\T_{\eta^{(k)}}\left(\left(\frac13,\frac13,\frac13\right)\right) = 
\begin{cases}
\left(\frac12,\frac14,\frac14\right), & \text{if $k=1$,}\\
\left(\frac14,\frac12,\frac14\right), & \text{if $k=2$,}\\
\left(\frac14,\frac14,\frac12\right), & \text{if $k=3$,}\\
\end{cases}
\end{align*}
so that, due to our tie-breaking scheme implemented via $\AdaSel$,
and from Definition~\ref{D:AdaBoost_update}, we have
\begin{align*}
\A\left(\left(\frac13,\frac13,\frac13\right)\right) =
  \T_{\eta^{(1)}}\left(\left(\frac13,\frac13,\frac13\right)\right) = \left(\frac12, \frac14,
                                               \frac14\right) \; .
\end{align*}
More generally, for any $w_1 \in \Delta_m^\circ$, we have, for any $k
\in \{1,2,3\}$,
\begin{align*}
\T_{\eta^{(k)}}(w_1) = 
\begin{cases}
\left(\frac12, \frac{w_1(2)}{2 (1 - w_1(1))},\frac{w_1(3)}{2 (1-w_1(1))}\right), & \text{if $k=1$,}\\
\left(\frac{w_1(1)}{2 (1 - w_1(2))}, \frac12,\frac{w_1(3)}{2 (1-w_1(2))}\right), & \text{if $k=2$,}\\ 
\left(\frac{w_1(1)}{2 (1 - w_1(3))}, \frac{w_1(2)}{2 (1-w_1(3))}, \frac12\right),
& \text{if $k=3$,}
\end{cases}
\end{align*}
so that, using our tie-breaking scheme,
\begin{align*}
\A(w_1) = 
\begin{cases}
\left(\frac12, \frac{w_1(2)}{2 (1 - w_1(1))},\frac{w_1(3)}{2 (1-w_1(1))}\right), &
\text{if $w_1 \in \pi(\eta^{(1)})$,}\\
 \left(\frac{w_1(1)}{2 (1 - w_1(2))}, \frac12,\frac{w_1(3)}{2 (1-w_1(2))}\right), & \text{if $w_1 \in \pi(\eta^{(2)})$,}\\ 
\left(\frac{w_1(1)}{2 (1 - w_1(3))}, \frac{w_1(2)}{2 (1-w_1(3))}, \frac12, \right),
& \text{if $w_1 \in \pi(\eta^{(3)})$.}
\end{cases}
\end{align*}

\begin{figure}
\begin{center}
\begin{tabular}{ll}
$1$st Iteration & $2$nd Iteration\\
\includegraphics[width=0.45\textwidth]{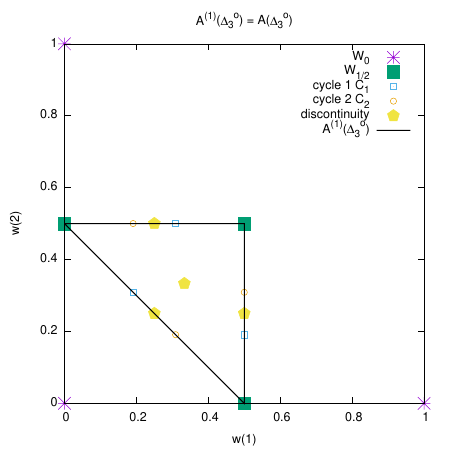}
              & \includegraphics[width=0.45\textwidth]{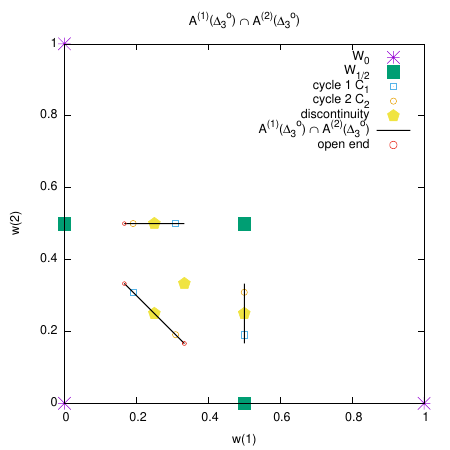} \\
$3$th Iteration & $\infty$ Iteration\\
\includegraphics[width=0.45\textwidth]{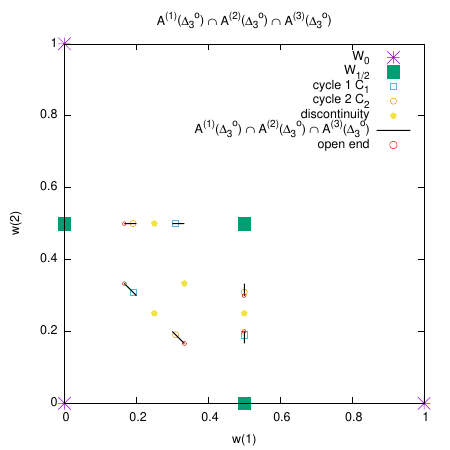} & \includegraphics[width=0.45\textwidth]{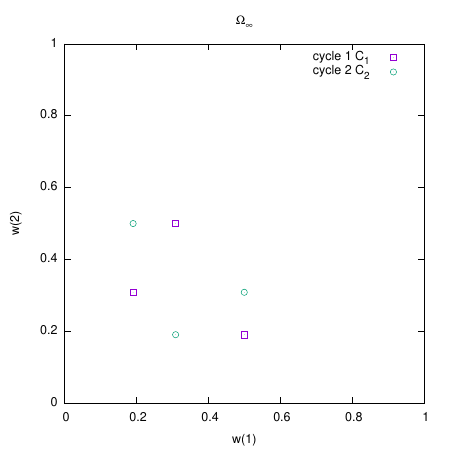}\\
\end{tabular}
\end{center}
\caption{{\bf An Illustration of Recursively Applying $\A$ to Reach
    $\Omega_{\infty}$:} We refer the reader to the main body of the
  paper for a discussion of the plots in this figure.}
\label{fig:mapex}
\end{figure}

Fig.~\ref{fig:mapex} provides an illustration of the application of
$\A$ to calculate $\Omega_\infty$ (Definition~\ref{def:OmegaInf}) in the same simple example. Before
we move to describe the figure in detail, the following sets will be useful in the discussion:
\begin{align}
\label{eqn:W_0}
W_0 = \left\{ (1,0,0),(0,1,0),(0,0,1) \right\}
\end{align}
and
\begin{align}
\label{eqn:W_half}
W_{\frac12} = \left\{ \left(\frac12,\frac12,0\right),
  \left(\frac12,\frac12,0\right), \left(\frac12,\frac12,0\right)
  \right\} \; .
\end{align}
For this example, the following is the set of type-$1$ discontinuities of
$\A$ within $\Delta_3^\circ$:
\begin{align}
\label{eqn:W_disc}
W_{\mathrm{disc}} = \left\{ \left(\frac13,\frac13,\frac13\right),
  \left(\frac14,\frac12,\frac14\right),
  \left(\frac12,\frac14,\frac14\right),
 \left(\frac14,\frac14,\frac12\right) \right\} \; .
\end{align}
Note that there are no type $2$-discontinuities within
$\Delta_3^\circ$ because they are all in the boundary of $\Delta_m$. 
As we can see from the figure, we have $W_{\mathrm{disc}} - \left\{
  \left(\frac13,\frac13,\frac13 \right)
\right\} \subset
\A^{(t)}(\Delta_3^\circ)$ for $t = 1$ and $2$, but $W_{\mathrm{disc}}
\cap \bigcap_{t=3}^T A^{(t)}(\Delta_3^\circ) = \emptyset$ for every $T >
2$. Hence, we have $W_{\mathrm{disc}} \cap \Omega_\infty =
\emptyset$, so that
there are no ties after
the second round of Optimal AdaBoost for every $w_1 \in \Delta_3^+$,
as established by our theoretical results. 
In fact, for this example, we have that the following two sets
correspond to the only two possible cycles one could reach starting
from any $w_1 \in \Delta_m^\circ$: we will call these sets \emph{cycle $1$},
\begin{align}
\label{eqn:C_1_ex}
C_1 = \left\{\left(\frac{1}{2}, \frac{1}{2\varphi^2}, \frac{1}{2\varphi}\right),
\left(\frac{1}{2\varphi}, \frac{1}{2}, \frac{1}{2\varphi^2}\right),
\left(\frac{1}{2\varphi^2}, \frac{1}{2\varphi}, \frac{1}{2} \right)
  \right\} \; ,
\end{align}
and \emph{cycle $2$},
\begin{align}
\label{eqn:C_2_ex}
C_2 = \left\{\left(\frac{1}{2}, \frac{1}{2\varphi}, \frac{1}{2\varphi^2}\right),
\left(\frac{1}{2\varphi^2}, \frac{1}{2}, \frac{1}{2\varphi}\right),
\left(\frac{1}{2\varphi}, \frac{1}{2\varphi^2}, \frac{1}{2} \right)
  \right\} \; ,
\end{align}
where $\varphi \equiv (1 + \sqrt{5})/2$ is the well-known \emph{golden
  ratio}.
In what follows, we describe the figure in detail.
\begin{itemize}
\item {\em 1st Iteration:}
The inner triangle represents a subspace of $\A(\Delta_m)$. In
particular, it is a subspace of the set 
$\bigcup_{\eta \in \Mcal} \pi_{\frac12}(\eta)$. For instance, the
vertical line segment corresponds to $A(\Delta_3^\circ \cap \pi^*(\eta^{(1)}))$, which is open; that is, it does not contain the end
points $(1/2,0,1/2)$ or $(1/2.1/2,0)$, which are both in $W_{\frac12}$
(Eqn.~\ref{eqn:W_half}). More generally, because
$\Omega_\infty$ is defined in terms of $\Delta_3^\circ$, \emph{not}
$\Delta_m$, the points in $W_0$ (Eqn.~\ref{eqn:W_0}) and
$W_{\frac12}$
are not in
$\A^{(1)}(\Delta_m^\circ) = \A(\Delta_m^\circ)$.
\item {\em 2nd
  Iteration:} 
\begin{itemize}
\item {\bf (Begin)}
For all points from the first iteration, we apply $\A$. Note that, in
this example,
$\A^{(2)}(\Delta_m)$ will remain in $\A^{(1)}(\Delta_m)$. And in
particular, we have
$\A^{(2)}(\Delta_m^\circ)=\A^{(1)}(\Delta_m^\circ)\cap \A^{(1)}(\Delta_m^\circ)$.
To find $\A^{(2)}(\Delta_m)$, we find all
points that result from applying $\A$ to every point $w \in
\A^{(1)}(\Delta_m)$. We do so in stages.
For instance, the left-most parts of the horizontal and diagonal line segments, to
the left of their respective discontinuities along those line segmenets (i.e., $(1/4,1/2,1/4)$ and $(1/4,1/4,1/2)$,
respectively), are the result of $\T_{\eta^{(2)}}(\A^{(1)}(\Delta_m^\circ))
\cap \pi^*(\eta^{(1)})$ and $\T_{\eta^{(3)}}(\A^{(1)}(\Delta_m^\circ)) \cap
\pi^*(\eta^{(1)})$, respectively.  
The circles at the left-most end points of those line segments are those
points \emph{not} in $\A^{(2)}(\Delta_m^\circ)$. Line segments \emph{without}
circles at their end-points \emph{are in} $\A^{(2)}(\Delta_m^\circ)$.
\item {\bf (End)}
The union of the corresponding line-segments are
a subspace of $\A^{(2)}(\Delta_3)$ and form $\A^{(1)}(\Delta_3^\circ)
\cap \A^{(2)}(\Delta_3^\circ)$, which in this examples equals $\A^{(2)}(\Delta_3^\circ)$. 
Recall that line segments with the circles at their end points correspond to open sets, while the
lack of circles at both ends means the line segment is closed.
Some line segments are open because of tie-breaking. 
\end{itemize}
\item {\em 3rd Iteration:}
We repeat the previous step on our new set to get
$\A^{(3)}(\Delta_m^\circ)$. Note that once again
$\A^{(3)}(\Delta_m)$ will remain in $\A^{(2)}(\Delta_m)$. And in
particular we have
$\A^{(3)}(\Delta_m^\circ)=\bigcap_{t=1}^3 \A^{(t)}(\Delta_m^\circ)$.
Note that this time, the set $\A^{(3)}(\Delta_3^\circ)$ is bounded away from ties.
Indeed, after the $3$th iteration, the resulting recursive
intersection yields a set that is \emph{forever bounded away from
  ties.} 
Define the following (six) sets, all subsets of $\Delta_3^\circ$,
forming a partition of $\A^{(3)}(\Delta_3^\circ)$: for
all $\eta,\eta' \in \Mcal$, such that $\eta \neq \eta'$,
\begin{align}\label{eqn:W_3_etas}
W_3^{\eta,\eta'} \equiv &
  \T_\eta(\A^{(3)}(\Delta_3^\circ)) \cap \pi^*(\eta') \\
\nonumber = & \begin{cases}
\left\{ \left( \frac12, w(2), w(3) \right) \in \Delta_3^\circ \, \mid
  \, \frac16 \leq w(2) < \frac{2}{10}, w(2) + w(3) = \frac12 \right\},
  & \text{if $\eta = \eta^{(1)}$ and $\eta' = \eta^{(2)}$,}\\
\left\{ \left( \frac12, w(2), w(3) \right) \in \Delta_3^\circ \, \mid
  \, \frac{3}{10} < w(2) \leq \frac13, w(2) + w(3) = \frac12 \right\},
  & \text{if $\eta = \eta^{(1)}$ and $\eta' = \eta^{(3)}$,}\\
\left\{ \left( w(1), \frac12, w(3) \right) \in \Delta_3^\circ \, \mid
  \, \frac16 < w(1) \leq \frac{2}{10}, w(1) + w(3) = \frac12 \right\},
  & \text{if $\eta = \eta^{(2)}$ and $\eta' = \eta^{(1)}$,}\\
\left\{ \left( w(1), \frac12, w(3) \right) \in \Delta_3^\circ \, \mid
  \, \frac{3}{10} \leq w(1) \leq \frac13, w(1) + w(3) = \frac12
  \right\}, & \text{if $\eta = \eta^{(2)}$ and $\eta' = \eta^{(3)}$,}\\
\left\{ \left( w(1), w(2), \frac12 \right) \in \Delta_3^\circ \, \mid
  \, \frac16 < w(1) \leq \frac{2}{10}, w(1) + w(2) = \frac12 \right\},
  & \text{if $\eta = \eta^{(3)}$ and $\eta' = \eta^{(1)}$,}\\
\left\{ \left( w(1), w(2), \frac12 \right) \in \Delta_3^\circ \, \mid
  \, \frac{3}{10} \leq w(1) < \frac13, w(1) + w(2) = \frac12 \right\}, & \text{if $\eta = \eta^{(3)}$ and $\eta' = \eta^{(2)}$.}
\end{cases}
\end{align}
In general, for higher-dimensions, we expect the hyperplanes, which
are $(m-2)$-dimensional manifolds used to build $\bigcap_{t=1}^T
A^{(t)}(\Delta_m^\circ)$ get chopped up into 
hyperplanes of smaller size. For this example, the line segments are just
reduced in size, and $\bigcap_{t=1}^T
A^{(t)}(\Delta_m^\circ)  = A^{(T)}(\Delta_m^\circ) =
\bigcup_{\eta,\eta' \in \Mcal, \eta \neq \eta'} W_t^{\eta,\eta'}$,
where $W_t^{\eta,\eta'}$ for $t>3$ is defined similarly to that case
for $t=3$ in Eqn.~\ref{eqn:W_3_etas}. This pattern continues
indefinitely. Our intuition
for the general case of higher dimensions is that the number of
hyperplanes in $\A^{(t)}(\Delta_m^\circ)$ will be no more than $m-1$ times that in $\A^{(t-1)}(\Delta_m^\circ)$.
For this example, the resulting smaller line segments will be
bounded on the ``left'' and ``right'' the same way indefinitely; i.e.,
throughout the run of the algorithm, each smaller line segment will
keep being open or closed as those ends appeared in the $3$th
iteration. 
We believe that in general, the result of this process of
countably infinite intersections is a \emph{Cantor-like pattern}. But recall that, under our conditions,
$\Omega_{\infty}$ will have a Borel \emph{probability} measure. Hence,
unlike the Cantor set with respect to the standard Borel measure in
$\R$, in our case, the
measure of $\Omega_\infty$ is non-zero with respect to an existing Borel
probability measure space (Proposition~\ref{A:1}). We construct a
specific measure after the description of this figure.
\item {\em $\infty$ Iteration:}
Our intuition for general high-dimensional problems is that each
resulting $W_t^{\eta,\eta'}$ is a Cantor-like set. Our theory states that under the
respective conditions we should expect those sets to be non-empty,
uncountably-infinite, and have no isolated points. In this particular
example, a simple low-dimensional set of mistake dichotomies, however,
the union of all those sets consists of the elements of only two sets,
$C_1$ (Eqn.~\ref{eqn:C_1_ex}) and $C_2$ (Eqn.~\ref{eqn:C_2_ex}), and they are of size $3$,
corresponding to each of the two possible $3$-cycles that the dynamics of
$\A$ will lead us to starting from any $w_1 \in
\Delta_3^\circ$~\citep{RudinDynamics}. Finally, in this case, we have
$\bigcup_{\eta,\eta' \in \Mcal, \eta \neq \eta'} W_\infty^{\eta,\eta'} =
C_1 \cup C_2 =
\Omega_{\infty}$, where $W_\infty^{\eta,\eta'} \equiv \lim_{t \to
  \infty} W_t^{\eta,\eta'}$. 
Our intuition for high-dimensional problems in general is that the $W_t^{\eta,\eta'}$'s are not obviously compact, but that is ok in our case:
We note that $\A^{(4)}(\Delta_m^+)  \supset \Omega_{\infty}$
bounded away from ties. Thus, as we establish in Theorem~\ref{T:compactness}, we obtain that
$\Omega_{\infty}$ \emph{is} compact. For this particular example,
because $\Omega_\infty$ is finite, it is ``trivially'' compact with
respect to the respective topological space. 
Note that $W_0 \cap \Omega_\infty = \emptyset$.
In general, it may be that $W_{\frac12} \cap \Omega_\infty \neq \emptyset$, because
some examples may be non-support-vector examples with respect to $w_1$
(see Definition~\ref{def:sv}). But, for this example, we do
have $W_{\frac12} \cap \Omega_\infty = \emptyset$. Indeed, for this
example, we have $\Omega_\infty \cap (W_0 \cup W_{\frac12}) =
\emptyset$. As an aside, we note that, had we defined $\Omega_\infty$ in
terms of countably infinite intersections of recursive applications of
$\A$ over $\Delta_3$, instead of $\Delta_3^\circ$, the resulting set
would also include sets such as $W_0$ and $W_{\frac12}$, and more
generally, the \emph{boundary} of $\Delta_3$.
\end{itemize}
Thus, for this example, we have
\(
\Omega_\infty = C_1 \cup C_2 \; ,
\)
which is a compact set and $\A$ is continuous on it.
In this case, as 
a result of Proposition~\ref{pro:inv} in Appendix~\ref{app:cycles}, we do not need to use the Krylov-Bogolyubov Theorem
(Theorem~\ref{Krylov-Bogolyubov}), as further discussed in
Section~\ref{app:cyclesBirkhoff}.
We can actually construct 
an uncountably infinite number of invariant measures
$\mu_\infty$ as follows. Let $\Sigma_\infty \equiv
2^{\Omega_\infty}$, so that $(\Omega_\infty,\Sigma_\infty)$ is a Borel
$\sigma$-algebra of $\Omega_\infty$.  By Proposition~\ref{pro:inv}, there are exactly two empirical
measures in this case, one for each set $C_1$ and $C_2$, both of which are
invariant, and which we can
use to define an uncountably infinite number of other invariants measures: Define
$\mu_\infty^{(0)}(W) \equiv \frac{| W \cap C_1|}{3}$ and
$\mu_\infty^{(1)}(W) \equiv \frac{| W \cap C_2|}{3}$, and for any
$\rho \in (0,1)$, $\mu_\infty^{(\rho)}(W) \equiv (1-\rho)
\mu_\infty^{(0)}(W) + \rho \mu_\infty^{(1)}(W)$. By
Proposition~\ref{pro:inv}, only $\mu_\infty^{(0)}$ and
$\mu_\infty^{(1)}$ are ergodic. Note that by the same propositon, the
resulting dynamical system is not uniquely ergodic, even when the
state space is resctricted to $\Omega_\infty$ instead of the full
$\Delta_m$. But one may argue that it \emph{is} uniquely ergodic, even
on $\Delta_m$, 
\emph{modulo permutaions of the examples}. This qualification is
reasonable given that permutations do not affect the inherent
behavior of the algorithm. This is because we can always run the algorithm
using a fixed ordering and, should the example be permuted, we can
simply permute the weights in the evolution accordingly to retrace the
dynamics under the new permutaion. So, in this sense, we
technically do not have to re-run the algorithm. (The observations in
this  paragraph also hold for the more general case of mistake matrices isomorphic to
the $(m\times m)$-identity matrix considered later in this appendix.)

As for the properties of the secondary quantities, 
In this example, we have $\lim_{t\to \infty} \epsilon_t =
\frac{1}{2\varphi^2}$~\citep{RudinDynamics}. Also, if we initialize
$w_1 \in \Delta_m^\circ$, such that $w_1(1) \leq w_1(2) \leq w_1(3)$,
and use the definition of $\AdaSel$ stated at the beginning of this
section (i.e., such that it
                              encodes the strict
                                               preference relation
                                               $\eta^{(1)} \succ
                                               \eta^{(2)} \succ \eta^{(3)}$), then the
                                               sequence of $\eta_t$'s
                                               converges to the $3$-cycle
\(
\eta^{(1)} \to \eta^{(2)} \to \eta^{(3)} \to \eta^{(1)} 
\)
right from the start. (As we discuss in Remark~\ref{rem:init_cond_wlog} later, the
conditions on $w_1$ and $\AdaSel$ are essentially without loss of generality.)
For this example, starting from the uniform initial weight
$w_1$ given above, there will be ties during the first two
rounds (i.e, $|\argmin_{\eta \in
  \Mcal} \eta \cdot w_1| = 3$, and thus
$\eta_1 = \eta^{(1)}$; $\argmin_{\eta \in
  \Mcal} \eta \cdot w_2 =
\{\eta^{(2)},\eta^{(3)}\}$, and thus $\eta_2 = \eta^{(2)}$); but there will never be ties again after the second round (i.e., $\argmin_{\eta \in
  \Mcal} \eta \cdot w_t =
\{\eta^{((t\mod 3) + 3)}\}$, and thus, $\eta_t = \eta^{((t\mod 3) + 3)}$, for all $t \geq 3$). It turns out that one can extend convergence to the same
$3$-cycle in the limit for any initial $w_1 \in \Delta_3^+$.
(see
Definition~\ref{D:Delta_m^+}), and thus for every $w_1 \in
\Delta_m^+$. In fact, the convergence to an $m$-cycle generalizes to any
$m$ when $\Mcal$ is isomorphic to an $(m \times m)$ identity
matrix. This is implicit in the proofs given by~\citet{RudinDynamics}. Here we
provide an alternative proof below (Theorem~\ref{the:conv_ident})
based on the Fibonacci sequence, and
its higher-order generalizations. (Some reader may find this
alternative proof and presentation
simpler, and it may be of independent interest). Although the
statements of the following technical results on the ``global''
convergence of Optimal AdaBoost for such $\Mcal$ considered here are 
under
conditions on $w_1$ and $\AdaSel$, doing so is in some sense ``without loss
of generality.'' We also note that the results hold
slightly more broadly than just the $\Mcal$ isomorphic to $(m \times
m)$ identity matrix. We remark on both of those points after the statements and respective
proofs of the technical results.
\begin{lemma}\label{lem:w_expr_ident}
Let $m \geq 3$. Given the set of mistake dichotomies $\Mcal$
isomorphic to the $(m \times m)$ identity matrix $\mathbf{I}_{m \times
  m}$, the sequence of example weights $w_t$ generated by Optimal
AdaBoost for $\Mcal$ starting from any initial $w_1 \in \Delta_m^{\circ}$ such
that $w_1(1) \leq w_1(2) \leq \cdots \leq w_1(m)$, and using an
implementation of $\AdaSel$ such that $\eta^{(1)} \succ \eta^{(2)}
\succ \cdots \succ \eta^{(m)}$, are given by an expression involving the Fibonacci sequence for $m=3$, or one closely related to its standard higher-order generalization for $m > 3$.  
\end{lemma}
\begin{proof}
Given any initial $w_1 \in \Delta_m^{\circ}$, define the recurrence relation $Z_t = Z_{t-1} + \cdots + Z_{t-m+1}$ for all $t \geq m$, and $Z_t = \sum_{l=t+1}^m w_1(l)$ for all $t=1,\ldots,m-1$.
For the first $t=2,\ldots,m$ rounds we have the following form for the $w_t$'s: for all $t=2,\ldots,m$, 
\[
w_t = \left(\frac{Z_1}{2 Z_{t-1}}, \frac{Z_2}{2 Z_{t-1}}, \ldots, \frac{Z_{t-2}}{2 Z_{t-1}}, \frac12, \frac{w_1(t+1)}{2 Z_{t-1}}, \ldots, \frac{w_1(m)}{2 Z_{t-1}} \right)
\]
and more generally, for all $t > m$, and $s = t \mod m$,
\begin{align*}
w_t = \left(\frac{Z_{t-s}}{2 Z_{t-1}}, \frac{Z_{t-s+1}}{2 Z_{t-1}}, \ldots, \frac{Z_{t-m-2}}{2 Z_{t-1}}, \frac12, \frac{Z_{t-m+1}}{2 Z_{t-1}}, \ldots, \frac{Z_{t-s-1}}{2 Z_{t-1}} \right) \; .
\end{align*}
The result follows by applying the Principle of Mathematical Induction.
\qed
\end{proof}
\begin{theorem}\label{the:conv_ident}
Let $m \geq 3$. Given a set of mistake dichotomies $\Mcal$ isomorphic to the $(m \times m)$ identity matrix $\mathbf{I}_{m \times
  m}$, the sequence of example weights $w_t$ generated by Optimal
AdaBoost for that matrix and any initial $w_1 \in \Delta_m^{\circ}$ such
that $w_1(1) < w_1(2) < \cdots < w_1(m)$, , and using an
implementation of $\AdaSel$ such that $\eta^{(1)} \succ \eta^{(2)}
\succ \cdots \succ \eta^{(m)}$, always converges to the following $m$-cycle in $\Delta_m^{\circ}$:
\begin{align}
\label{eqn:mcycle} & \left( \frac12, \frac{1}{2r^{m-1}}, \frac{1}{2r^{m-2}},\ldots,\frac{1}{2r^2}, \frac{1}{2r} \right)\\
\nonumber \to& \left( \frac{1}{2r}, \frac12, \frac{1}{2r^{m-1}}, \frac{1}{2r^{m-2}},\ldots, \frac{1}{2r^2} \right)\\
\nonumber \to& \cdots\\
\nonumber \to& \left( \frac{1}{2r^{m-1}}, \frac{1}{2r^{m-2}}, \ldots, \frac{1}{2r^2}, \frac{1}{r}, \frac12 \right)\\
\nonumber \to& \left( \frac12, \frac{1}{2r^{m-1}}, \frac{1}{2r^{m-2}},\ldots,\frac{1}{2r^2}, \frac{1}{2r} \right) \; ,
\end{align}
where $r \equiv r(m)$ is the only positive solution to the equation
$z+z^{1-m} = 2$, and thus 
depends on $m$.
\end{theorem}
\begin{proof}
The proof follows from that used for the last lemma (Lemma~\ref{lem:w_expr_ident}). In
particular, note that we can also express the sequence as $Z_{t+1} - Z_t = Z_t -
Z_{t-m}$ for all $t > m$. Hence, the ratio of two consecutive values
in the sequence satisfy the following equation $\frac{Z_{t+1}}{Z_t}
+\frac{Z_{t-m}}{Z_t}  = 2$, which we can solve in the limit by finding
a solution with positive values to the following equation: $z + z^{-(m-1)}
= 2$. If the value $r \equiv r(m)$ is a proper solution to the equation with
respect to our setting, then we can use that value to provide exact
values to the $w_t$'s in the limit. That is, we have that for all
$s=0,\ldots,m-1$, $\lim_{t \to \infty} Z_{t-s}/Z_t = r^{-s}$, so
that the $w_t$'s will converge in the limit to the $m$-cycle
in $\Delta_m^\circ$ stated in the theorem (Equation~\ref{eqn:mcycle}).\qed
\end{proof}
\begin{remark}
Note that for the special case of $m=3$, we have $r = \varphi$, the golden ratio.
It is curious to note that $r$ increases monotonically with $m$, with
$r=\varphi$ for $m=2$ to $r=2$
as $m \to \infty$. Hence, for every $w_1 \in \Delta_m^\circ$,
as we increase 
$m$, the elements of the $w_t$'s
will tend to powers of $1/2$, the ratio of the error of the second
best with respect to the best hypothesis will tend to the value $2$,
and $\epsilon_t \to 0$ as $m \to \infty$. The
reader should keep in mind, however, that in this example, for any
\emph{finite} $m$, we have, for all $t = 1,2,\ldots,T$, $\epsilon_t =
\frac{Z_{t-m-2}}{2 Z_{t-1}} > 0$, where $Z_t$ comes from the
recurrence relation defined in terms of a Fibonacci-like sequence, or
one of its higher-order generalizations, as given in the proof of the
last lemma (Lemma~\ref{lem:w_expr_ident}); and from which we can also obtain the
expressions for $\alpha_t$ and $\widetilde{\alpha}_t$ via simple
substitution. Thus, we have that 
$\lim_{t \to \infty} \epsilon_t = \frac{1}{2 r^{m-1}}$, so that
$\lim_{t \to \infty} \alpha_t = \frac12 \ln\left( 2 r^{m-1} - 1
\right)$, and $\lim_{t \to \infty} \widetilde{\alpha}_t = \frac{1}{m}$.
\end{remark}
\begin{remark}
\label{rem:init_cond_wlog}
As~\citet{RudinDynamics} state, we can always relabel the training
examples via a reordering of their indexes without really affecting
the general nature of the resulting dynamical system induced by the
AdaBoost update. It is in this sense what we mean that the technical
results are really ``without loss of generality.'' In particular, we can
always relabel the indexes to the training examples by sorting any
initial $w_1' \in \Delta_m^\circ$. Let the sequence of indexes to the
training examples $(i_1,i_2,\ldots,i_m)$ be such that $w_1'(i_1) \leq
w_1'(i_2) \leq \cdots \leq w_1'(i_m)$. Set $w_1(s) \equiv w_1'(i_s)$
for all $s=1,2,\ldots,m$, so as to satisfy the condition on $w_1$
imposed on the technical results just presented above (Lemma~\ref{lem:w_expr_ident}
and Theorem~\ref{the:conv_ident}). In addition, redefine $\AdaSel$ to use the
preference order $\eta^{(i_1)} \succ \eta^{(i_2)} \succ \cdots \succ
\eta^{(i_m)}$, so as to satisfy the conditions on the preference order,
stated in the same technical results. From this discussion we can
conclude that convergence will occur to one of the
$(m-1)!$ possible permutations of the cycle given in Theorem~\ref{the:conv_ident}, one for
each permutation resulting from the sorting of the $m-1$ degrees of
freedom defining the $m$-probability simplex. The exact
cycle to which the process will converge will depend on $w_1$ and
$\AdaSel$, as determined by ``inverting'' the corresponding
permutation (i.e., mapping back to the original weights $w_1'$ based
on their sorted order). The final classifier may change because of the
difference in tie-breaking, but the general nature of the dynamics of the training process
will not change in any critical way in terms of convergence. We note that, for this example, the values of the
secondary quantities generated by Optimal AdaBoost will not change
with the re-ordering.
\end{remark}
\begin{remark}
As~\citet{RudinDynamics} also essentially state using different terminology, any
set of mistake dichotomies $\Mcal' \supset \Mcal$ which does not
contain the ``all zeros'' mistake dichotomy, and thus maintain Part 3 of
Condition~\ref{cond:nathypo} (Natural Weak-Hypothesis Class), will
behave essentially equivalent to the $\Mcal$ considered in this
example.
Hence, Optimal
AdaBoost will still exhibit the same global convergence, starting from every $w_1 \in \Delta_m^\circ$, to some
$m$-cycle formed by a permutation of the specific $m$-cycle given in
Theorem~\ref{the:conv_ident}. 
\end{remark}

\section{On Optimal-AdaBoost's weak-hypothesis class: dominated hypotheses
  and data-dependent PAC bounds}
\label{app:pac_bnds}

In this appendix we provide further discussion on our statement about
the empirically observed logarithmic growth on the number of
weak-hypothesis that Optimal AdaBoost selects during its execution
(i.e., in the non-asymptotic regime),
within a large, but albeit finite, number of rounds. Some of the
concepts, terminology, and notation used here are already defined in
the main body.

The inspiration for the work presented in this appendix comes from our empirical observations
on the considerably small number of ``effective'' weak hypothesis, and
the even smaller number of ``unique'' weak hypothesis selected by
Optimal AdaBoost in practice. Those observations led us to the
derivation of potentially tighter data-dependent bounds on the
generalization error, that are closer to the behavior of Optimal AdaBoost in
high-dimensional real-world datasets. We
present those experimental results later in
Appendix~\ref{S:Exp_v2}. We present the
data-dependent PAC bounds 
in Appendix~\ref{A:DataBnd}. We believe the empirical results
presented here also provide some practical
perspective on
our theoretical results on the provably cycling-like behavior of
arbitrarily-accurate approximations of Optimal
AdaBoost, and of the exact version in certain conditions.

\subsection{Dominated, effective, and uniquely-selected weak
  hypotheses}
\label{sec:dom}

We now formally introduce the concept of \emph{dominated} mistake
dichotomies. We say a mistake dichotomy $\eta$ \emph{dominates}
another mistake dichotomy $\eta'$ if for all $i=1,\ldots,m$, we have
$\eta(i) = 1 \implies \eta'(i)=1$, so that the set of mistakes
associated with $\eta$ is a subset of those associated with
$\eta'$. For instance, in the context of the classroom example in
Fig.~\ref{fig:clex}, 
shown on pg.~\pageref{fig:clex} of the main body,
the
dichotomy $\eta = (0,0,1,0,1,0)$ dominates $\eta' = (1,0,1,0,1,0)$.
Because we are studying Optimal AdaBoost, we can eliminate
\emph{dominated} mistake dichotomies from the \emph{set of all possible
mistake or error dichotomies} $\Mcal$. To see why this removal is sound, note that any dominated hypothesis would never be selected by Optimal AdaBoost during its execution. This is because any $w_t$ that Optimal AdaBoost can generate during its execution is strictly positive (i.e., satisfies $w_t(i) > 0$ for all $i$).
Thus, if $\eta$ dominates $\eta'$, for any such $w_t$, the weighted error $\eta \cdot w_t$ of $\eta$ would always be 
\emph{strictly smaller} than 
the weighted error $\eta' \cdot w_t$ of $\eta'$, at every round $t$ of
Optimal AdaBoost. The result after such removals is what we call
\emph{the set of effective mistake dichotomies} $\Ecal \equiv \{ \eta
\in \Mcal \mid \text{ no other } \eta' \in \Mcal \text{ dominates }
\eta \} \subset \Mcal$. Let $r \equiv |\Ecal|$ and $\Ecal \equiv
\{\eta^{(1)},\ldots,\eta^{(r)}\}$.~\footnote{In essence, in the
  context of Optimal AdaBoost as defined in Section~\ref{sec:details} in the main
  body of the paper,
we
could have replaced $\Mcal$ by $\Ecal$ throughout the presentation
there and obtain exactly the same technical results.} (Note that $r \leq n \equiv |\Mcal|$.) 
For instance, in the context of the classroom example in Fig.~\ref{fig:clex}, we have
\begin{align*}
\Ecal = \{ 
& (1,0,1,0,0,1),\\
& (0,1,1,0,0,1),\\
& (0,0,1,0,1,0),\\
& (0,1,0,0,1,0),\\
& (1,0,0,0,1,0),\\
& (0,0,0,1,0,0) \} \; ,
\end{align*}
so that $r=6$. 

We can define \emph{the effective mistake matrix} $\M \in \{0,1\}^{r
  \times m}$ similarly from $\Ecal$.~\footnote{
The notation used in~\citet{RudinDynamics} for the so called ``mistake
matrix'' is syntactically different from the one used here. In their
notation, the $(i,j)$ element of the matrix equals $+1$ or $-1$
depending on whether the hypothesis indexed by $j$ correctly or
incorrectly classifies example $i$, respectively. In our case, we
replace the $-1$ elements by $0$ and transpose their mistake matrix;
that is, the values of the mistake matrix we use are $+1$ or $0$
depending on whether the hypothesis indexed by $i$ correctly or
incorrectly classifies example $j$, respectively. While we regret the
change of notation, there are no semantic differences and we found our
syntactic changes in notation to prove extremely convenient in
simplifying the presentation of our technical results and their
proofs.}
 This matrix has the form $\M(k,i) \equiv \eta^{(k)}(i)$, where each
 $(k,i)$ pair indexes a row $k$ and column $i$ of the matrix $\M$,
 respectively.  Note that by construction, a row $k$ of the matrix is
 a 0-1 (bit) vector corresponding to an effective mistake dichotomy
 $\eta^{(k)} \in \Ecal$, which we call an \emph{effective mistake
   dichotomy}. It indicates where the representative hypothesis
 $h^{\eta^{(k)}}$ is incorrect on the dataset of input examples $S$.
 Said differently, we have that \[ \left(\M(k,i) = \indicator{y^{(i)}
     \neq h^{\eta^{(k)}}(x^{(i)})}\right)_{i=1,\ldots,m}\] is a
 bit-vector encoding of the set of training examples that the
 hypothesis $h^{\eta^{(k)}}$ representing the dichotomy $\eta^{(k)}$
 classifies incorrectly: i.e., $h^{\eta^{(k)}}(x^{(i)}) \neq
 y^{(i)}$. For instance, in the context of the classroom example in Fig.~\ref{fig:clex}, we
 could let 
\begin{align*}
\M = \left[ 
\begin{array}{cccccc} 
1 & 0 & 1 & 0 & 0 & 1\\
0 & 1 & 1 & 0 & 0 & 1\\
0 & 0 & 1 & 0 & 1 & 0\\
0 & 1 & 0 & 0 & 1 & 0\\
1 & 0 & 0 & 0 & 1 & 0\\
0 & 0 & 0 & 1 & 0 & 0 
\end{array}
\right] \; ,
\end{align*}
so that $\M$ is a $(6 \times 6)$ matrix and, for example, $\eta^{(1)} = (1,0,1,0,0,1)$ is the first row of $\M$. 

Hence, we can equivalently define the \emph{set of effective
  representative hypotheses} \[\EHypo \equiv
\left\{h^{\eta^{(1)}},h^{\eta^{(2)}},\ldots,h^{\eta^{(r)}}\right\}.\] 
This
is because for any row $k$ of $\M$, we have $\err(h^{\eta^{(k)}}; D ,
w) = 
\sum_{i=1}^m \M(k,i) w(i) = \sum_{i=1}^m \eta^{(k)}(i) w(i) =
\eta^{(k)} \cdot w$.
For instance, in the context of the classroom
example in Fig.~\ref{fig:clex}, we have
\begin{align*}
&h^{\eta^{(1)}}(x_1,x_2) = \sign{x_1-8}, h^{\eta^{(2)}}(x_1,x_2) =
\sign{x_2-4}, h^{\eta^{(3)}}(x_1,x_2) = -\sign{x_1-2},\\
&h^{\eta^{(4)}}(x_1,x_2) = -\sign{x_1-6}, h^{\eta^{(5)}}(x_1,x_2) =
-\sign{x_2-2}, h^{\eta^{(6)}}(x_1,x_2) =
-\sign{x_2-8}
\end{align*}
and
\begin{align*}
\Ecal = \{ & \sign{x_1-8}, 
\sign{x_2-4}, -\sign{x_1-2},\\
&-\sign{x_1-6}, 
-\sign{x_2-2}, 
-\sign{x_2-8}\} \; .
\end{align*}

Finally note that we can construct a matrix very similar to $\M$ directly
from $\Mcal$.  Then, applying a ``pruning'' procedure on $\M$ removes
repeated and \emph{dominated} mistake dichotomies. This process leads
to a common selection scheme in the context of decision stumps. 

\subsection{Preliminary experimental results about dominated, effective, and uniquely-selected weak
  hypotheses on high-dimensional
  real-world datasets using decision stumps}\label{S:Exp_v2}

The empirical results are in the context of decision stumps, which are
one of the most common instantiations of the weak learner for Optimal
AdaBoost effectively used in practice.~\footnote{Some students in an undergraduate
AI course at the University of Puerto Rico at Mayag\"{u}ez, taught by the second author, as well as Girish Kathalagiri, an MS student in
  Computer Science at Stony Brook University,
  working under the supervision of the second author, performed
  very preliminary work on similar experiments.}

\subsubsection{AdaBoosting decision stumps}
\label{S:ds}

Decision stumps are simple decision tests based on a single
attribute of the input; i.e., a decision tree with a single node: the
root corresponding to the attribute test. For instance, we can define the following $\mathrm{test}$ function as an implementation of a decision stump:
\[
\mathrm{test}(\textrm{\em condition}) = \left\{
\begin{array}{ll}
+1, & \textrm{if {\em condition\/} holds},\\
-1, & \textrm{otherwise.}
\end{array}
\right.
\]

\subsubsection{The effective number of decision stumps is relatively
  smaller than expected}
\label{S:eff}

\begin{figure}
\begin{center}
\resizebox{\textwidth}{!}{
\begin{tabular}{|c|c||c|c|}
\hline
Test $\textrm{\em condition}$ & Mistakes & Test
$\textrm{\em condition}$ & Mistakes\\
 & & (inverse) & \\
\hline
TRUE  & \st{2, 4, 6} & FALSE & \st{1, 3, 5} \\
\hline
$x_1 > 2$ & \st{1, 2, 4, 6} & $x_1 \leq 2$ & 3, 5\\
\hline
$x_1 > 4$ & \st{1, 4, 6} & $x_1 \leq 4$ & \st{2, 3, 5}\\
\hline
$x_1 > 6$ & \st{1, 3, 4, 6} & $x_1 \leq 6$ & 2, 5\\
\hline
$x_1 > 8$ & 1, 3, 6 & $x_1 \leq 8$ & \st{2, 4, 5}\\
\hline
$x_1 > 10$ & \st{1, 3, 5} & $x_1 \leq 10$ & \st{2, 4, 6} \\
\hline
$x_2 > 2$ & \st{2, 3, 4, 6} & $x_2 \leq 2$ & 1, 5 \\
\hline
$x_2 > 4$ & 2, 3, 6 & $x_2 \leq 4$ & \st{1, 4, 5} \\
\hline
$x_2 > 6$ & \st{1, 3, 2, 6} & $x_2 \leq 6$ & \st{4, 5} \\
\hline
$x_2 > 8$ & \st{1, 2, 3, 5, 6} & $x_2 \leq 8$ & 4\\
\hline
\end{tabular}
}
\end{center} 
\caption{{\bf Dominated Hypotheses in Classroom Example.} This figure 
illustrates the
  concept of \emph{dominated hypothesis} and \emph{effective hypothesis
  spaces} with respect to Optimal AdaBoost, and within the context of
the ``classroom example'' in Fig.~\ref{fig:clex} on pg.~\pageref{fig:clex} in the
main body. The table displays
the set of mistakes for each decision stump. Strictly dominated
decision stumps are crossed out; said differently, the crossed-out \emph{mistake sets} means the corresponding
  classifier is dominated and will never be selected by Optimal AdaBoost. Out of a maximum of $20$ decision stumps
suggested by the data via the midpoint split rule, only $6$ could ever
be selected by AdaBoost, a reduction of $70\%$ in the size of the
hypothesis space. We found such reduction levels to be common in both synthetic
and real data sets of larger size and dimensionality. }
\label{fig:clexdom}
\end{figure}

Fig.~\ref{fig:clexdom} builds on the classroom example in
Fig.~\ref{fig:clex} to further illustrate the concept of
dominated and effective hypotheses in the context of decision stumps.
Table~\ref{tab:eff}
contains examples of the number of \emph{effective/non-dominated}
decision stumps and the number of \emph{unique} decision stumps from that set
that Optimal AdaBoost actually uses in the context of high-dimensional real-world datasets publicly available from the UCI ML Repository. Note the significant reduction in both
the \emph{``effective'' size of $\Hypo$} and the \emph{actual number of decision
stumps selected}.

\begin{table}[t]
\begin{center}
\begin{tabular}{|l|c|c|c|c|c|}
\hline 
{\bf Dataset} & \multicolumn{2}{c|}{{\bf Examples}} &
\multicolumn{3}{c|}{{\bf Classifiers}} \\
\hline
 & {\bf train}  & {\bf test} & {\bf total} & {\bf non-dominated} & {\bf used} \\
\hline
Breast Cancer & 400 & 169 & 11067 & 3408 \hspace{5pt} \emph{(69\%)} & 290 \hspace{5pt} \emph{(97\%,91\%)}\\
\hline
Parkinson & 150 & 45 & 3029 & 351 \hspace{5pt} \emph{(88\%)} & 79 \hspace{5pt} \emph{(97\%,77\%)}\\
\hline
Sonar & 104 & 104 & 5806 & 436 \hspace{5pt} \emph{(92\%)} & 154 \hspace{5pt} \emph{(97\%,65\%)}\\
\hline
Spambase & 2500 & 2100 & 11406 & 7495 \hspace{5pt} \emph{(34\%)} & 710 \hspace{5pt} \emph{(94\%,91\%)} \\
\hline
\end{tabular}
\end{center}
\caption{{\bf An Illustration of Effective/Non-Dominated and Actually-Used/Unique Decision Stumps When Running AdaBoost on Several High-Dimensional Real-World Datasets Publicly Available from the UCI ML Repository.} The table shows the number of \emph{effective/non-dominated}
decision stumps and the number of \emph{unique} decision stumps from that set
that Optimal AdaBoost actually uses. The number in parenthesis are {\em percent reductions.} In the case of the {\bf non-dominated} column, the percentages are with respect to the {\bf total} (number of classifiers) column; while for the {\bf used} column, the pair of percentages are with respect to the {\bf total} and {\bf non-dominated} columns, respectively. Note the significant reduction in both
the \emph{``effective'' size of $\Hypo$} and the \emph{actual number of decision
stumps selected}. The resulting numbers provided in the table are robust to random variations on the train-test
validation sets of the same size (generated by combining train and
test samples into a single dataset and then randomly splitting into
train and test several times), as the original train-test set size
above.}
\label{tab:eff}
\end{table}

Table~\ref{tab:eff}, just like the plots in
Fig.~\ref{fig:all}, on pg.~\pageref{fig:all}, clearly suggests that Optimal AdaBoost is selecting the same decision stump
many times. While one may argue that this is empirical evidence of ergodicity,
even if not cycling, the fact that the growth on the number of weak-hypothesis
selected appears to be too substantial may counter the original argument.

\subsubsection{The number of uniquely-selected decision stumps grows logarithmically}
\label{S:Cyc}

\begin{figure}
\begin{center}
\begin{tabular}{lccc}
 & {\bf Train \& Test Error} & $\left| \bigcup_{t=1}^T \{ h_t \} \right|$ & $\epsilon_t$\\
\hline
& {\bf $y$-axis range $\approx [0,0.13]$} & {\bf $y$-axis range $\approx [0,300]$} & {\bf $y$-axis range $\approx [0.075,0.4]$}\\
\rotatebox{90}{\bf Breast Cancer} & \includegraphics[width=0.28\textwidth]{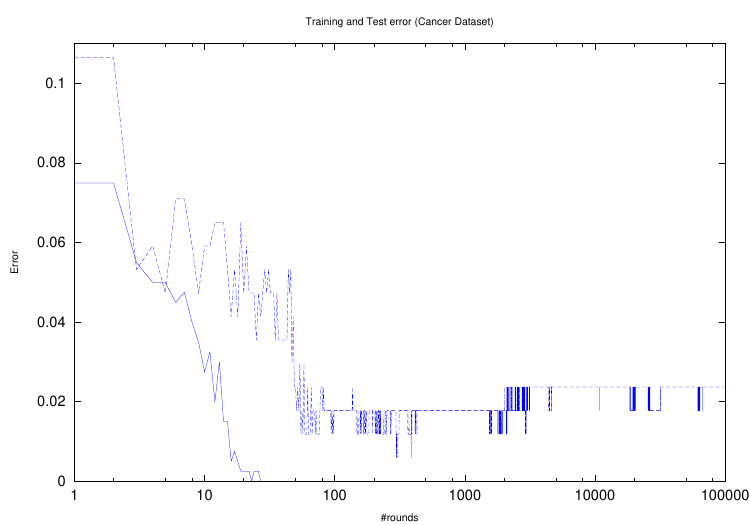} &
\includegraphics[width=0.28\textwidth]{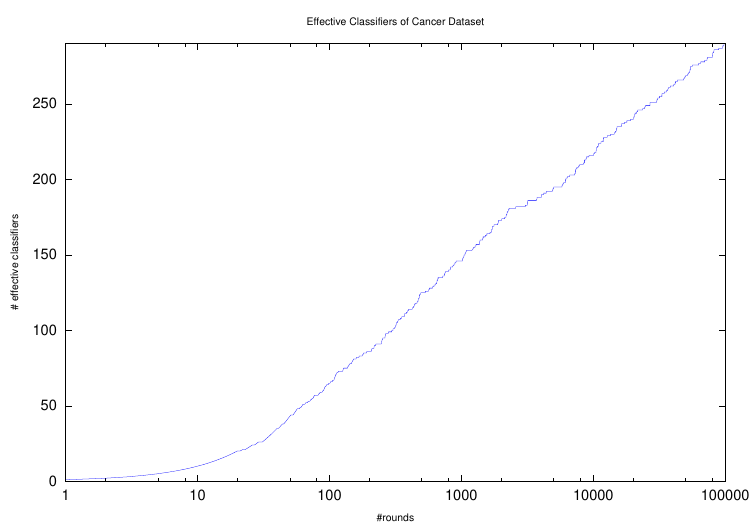} &
\includegraphics[width=0.28\textwidth]{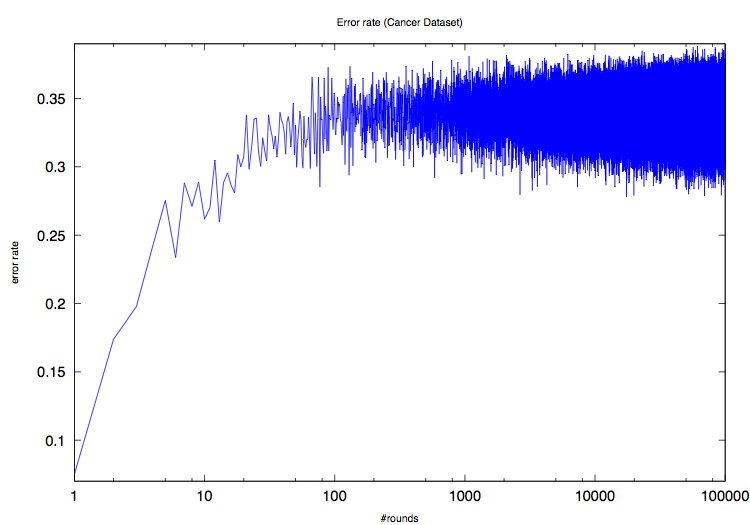} \\
\hline
& {\bf $y$-axis range $\approx [0,0.45]$} & {\bf $y$-axis range $\approx [0,80]$} & {\bf $y$-axis range $\approx [0.04,0.35]$}\\
\rotatebox{90}{\bf Parkinsons} & \includegraphics[width=0.28\textwidth]{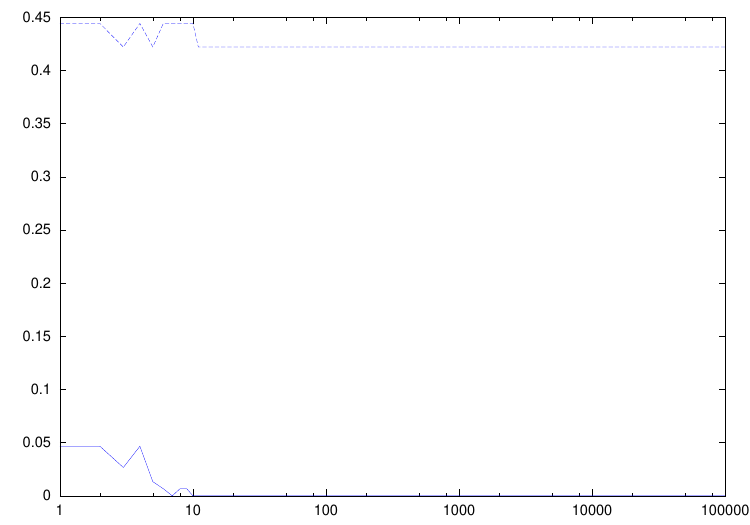} &
\includegraphics[width=0.28\textwidth]{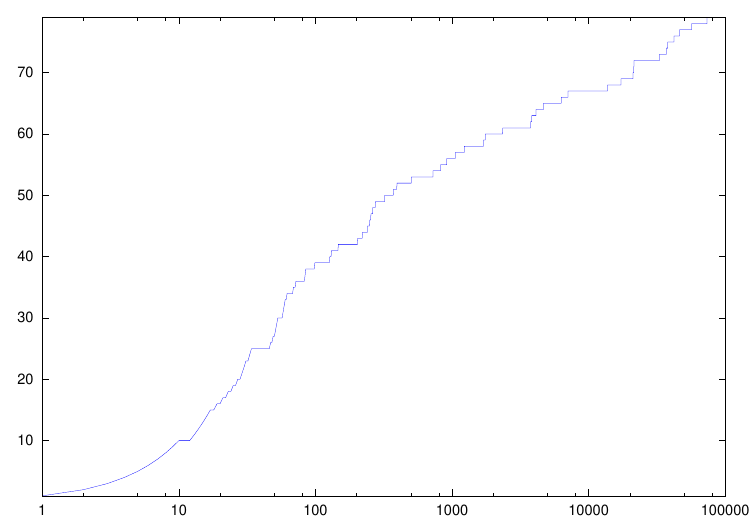} &
\includegraphics[width=0.28\textwidth]{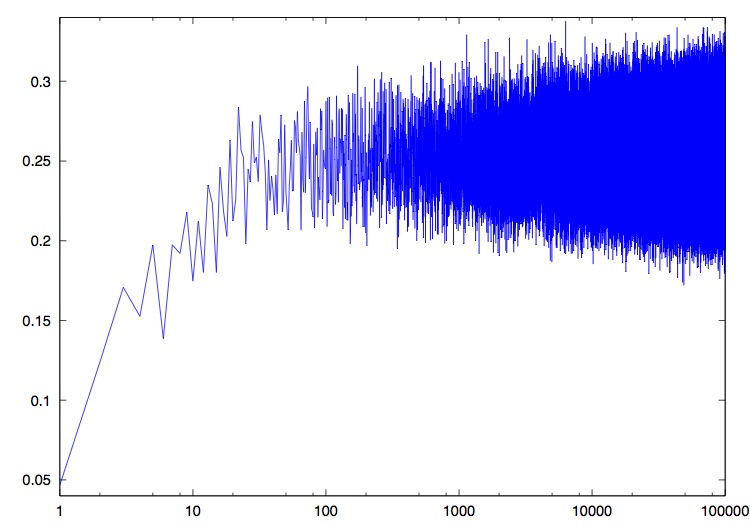} \\
\hline
& {\bf $y$-axis range $\approx [0,0.26]$} & {\bf $y$-axis range $\approx [0,550]$} & {\bf $y$-axis range $\approx [0.25,0.4]$}\\
\rotatebox{90}{\bf Sonar} &\includegraphics[width=0.28\textwidth]{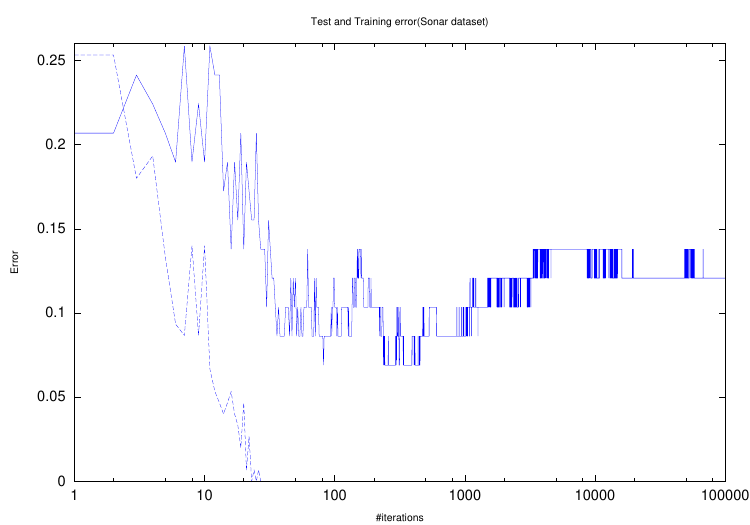} &
\includegraphics[width=0.28\textwidth]{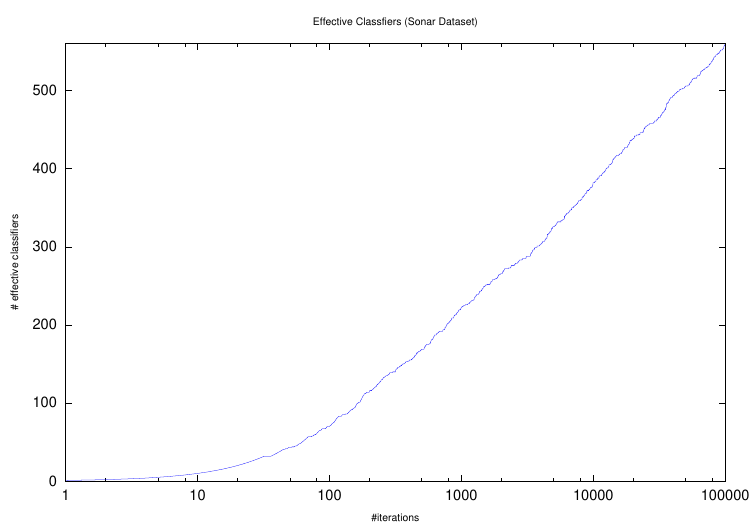} &
\includegraphics[width=0.28\textwidth]{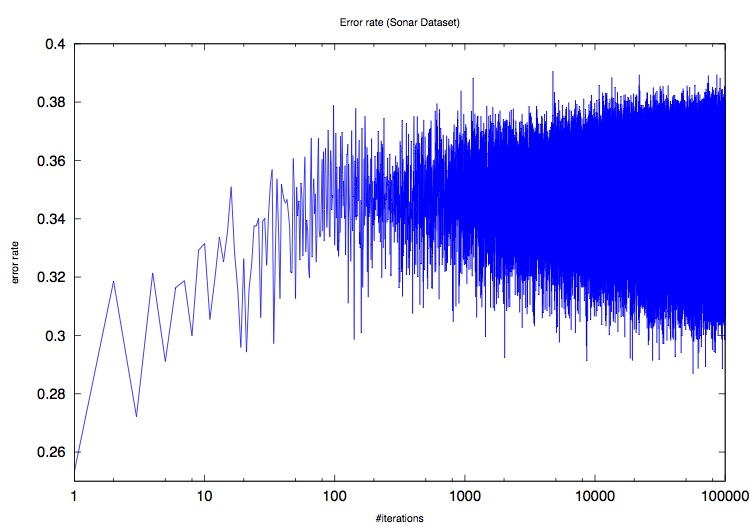}\\
\hline
& {\bf $y$-axis range $\approx [0,0.22]$} & {\bf $y$-axis range $\approx [0,710]$} & {\bf $y$-axis range $\approx [0.2,0.5]$}\\
\rotatebox{90}{\bf Spambase} &\includegraphics[width=0.28\textwidth]{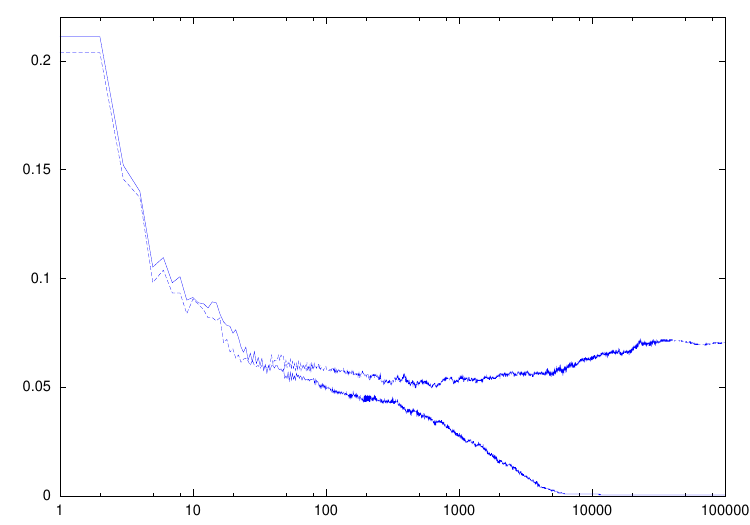} &
\includegraphics[width=0.28\textwidth]{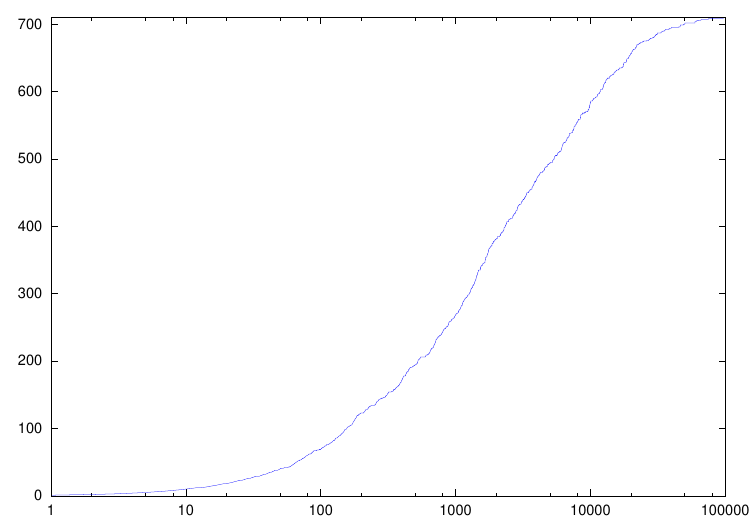} &
\includegraphics[width=0.28\textwidth]{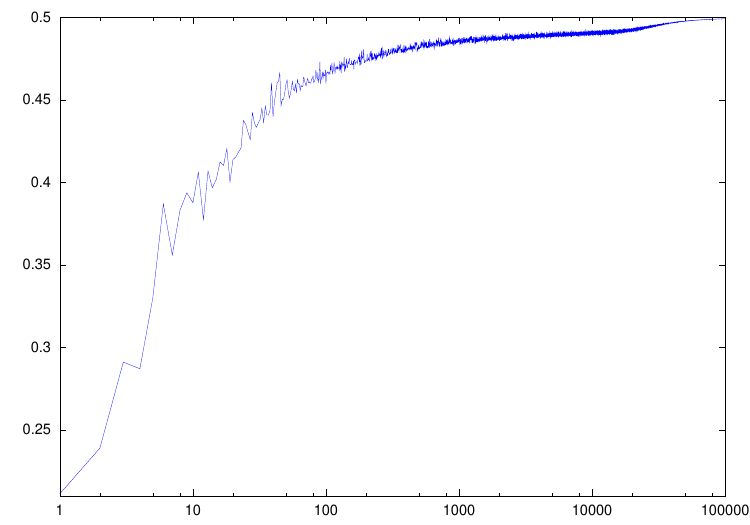}\\
\\
&\multicolumn{3}{c}{\bf $x$-axis, \emph{all} plots: max. number of rounds $T = 1, 2, \ldots, 100K$ (in \emph{log scale})}
\end{tabular}
\end{center}
\caption{{\bf Results for a Single Random Training-Test Data Split of the
  Breast Cancer Wisconsin (Diagnostic), Parkinsons, Sonar, and Spambase
  Datasets (first through fourth labeled rows, respectively).} The figure shows the
  train (dashed-line) and test error (left column), number of unique decision stumps
  selected, with number of rounds shown in log-scale (center column) and the value $\epsilon_t$ of the weighted error of the
  weak/base classifier found in round $t$ (right column) as a function of
  the number of rounds (\emph{in log-scale, up to $100K$ rounds}).
  The reader should place their attention on the general patterns reflected on
  the plots, and not the specific values. We can summarize these empirical results
  as follows. First, if AdaBoost always cycles, then it may take a long time to
  enter or reach a cycle, or the cycle may be very long in
  high-dimensional real-world datasets. Second, the test error always
  becomes stable within the $100K$ rounds. Third, the weighted-error
  $\epsilon_t$ of the weak-hypothesis that our implementation of
  Optimal AdaBoost selects at round $t$, while seemingly chaotic, does
  shows sufficient symmetry. That in turn suggests that its average over the
  number of rounds converges. That is as predicted by the technical results
  in
Section~\ref{sec:convres} in the main body,
under certain conditions which held throughout the execution
  of the algorithm. We
  refer the reader to the main body of this appendix for further discussion.
}
\label{fig:all}
\end{figure}

We found that, in most benchmark datasets used, the number
of {\em unique\/} decision-stumps classifiers that AdaBoost combines grows with the
number of rounds, but only
{\em logarithmically\/}. Figure~\ref{fig:all} illustrates this over
a variety of data sets (see plots in the
center column). It is important to point out that even
though the effective number of decision stumps is relatively small,
the number used/selected by AdaBoost is even smaller, and the logarithmic
growth suggests that it would take a \emph{very} long time before
AdaBoost would have selected all effective classifiers, if ever, or converges
to a cycle or the neighborhood of a cycle, thus stopping any new selections, as out theoretical results show. 

Here are some additional miscellaneous observations. Spambase has a pair of examples with the
  same input but different output; note how $\epsilon_t$ approaches
  $50\%$ error with the number of rounds in that case, thus breaking
  the so-called Weak Learning Assumption. Indeed, one can show that
  if the training data has at least a couple of examples with such
  ``noise,'' the Optimal AdaBoost will always behave this way, for
  almost every $w_1 \in \Delta_m$. Yet, both training and test
  errors appear stable, and there is still logarithmic growth in the number of
  unique stumps selected.

\subsection{Data-dependent bounds on the generalization error of
  Optimal AdaBoost}\label{A:DataBnd}

Using knowledge about the effective number of weak-hypothesis and the log-growth behavior illustrated in
Fig.~\ref{fig:all}, and discussed in Sections~\ref{S:eff} and~\ref{S:Cyc}, we derived two data-dependent
uniform-convergence bounds. 
Both bounds essentially state that, with high probability, the
generalization error grows as
\[
O(\sqrt{(\log{T}) \log\log{T}}),
\]
a
significantly tighter bound than the standard $O(\sqrt{T \log{T}})$ that was previously
known~\citep{Freund97adecision-theoretic}.

We use the following definition of the \emph{VC-dimension} of a
hypothesis class $\Hypo$, which we present within the context of the results
described in this appendix.
\begin{definition}{\bf (VC-dimension)} 
\label{def:vcdim}
We say that hypotheses class $\Hypo$ has \emph{VC-dimension} $\VC{\Hypo} \equiv m^*
\equiv m^*(\Hypo)$ if for every dataset of inputs $S$ of size $m
\leq m^*$ we have $|\Dich(\Hypo,S)| = 2^m$, and 
there exists some dataset $S' \equiv D'(m^*)$ of size $m^*+1$ such that
$|\Dich(\Hypo,S')| < 2^{m^*+1}$. Note that if
$\Hypo$ is finite, and $|\Hypo|$ is its cardinality, then we have $\VC{\Hypo}
\leq \lceil\log_2{|\Hypo|}\rceil$. 
\end{definition}
Note that $|\EHypo| \leq |\Hrep| \leq T^* \equiv T^*(m,\VC{\Hypo}) \equiv 2^{\min(\VC{\Hypo},m)}$ and
$|\DichSets| \leq 2^{T^*}$. 

In what follows, we denote by 
\begin{align}
\label{eqn:hyporep}
\HypoRep \equiv \HypoRep(m,\Hypo) \equiv 
\bigcup_{O \in \DichSets} 2^{\left\{ \bigcup_{o \in O} \{ h^o \} \right\}}
\end{align}
the \emph{set of all subsets of representative hypotheses of 
each set of label dichotomies that $\Hypo$ can induce over any dataset of inputs of size
$m$.} Note that $|\HypoRep| \leq 2^{T^*}$.

\subsubsection{Exploiting the logarithmic rate of the uniquely-selected weak hypotheses}

The following is a data-dependent PAC-style bound.
This uniform-convergence probabilistic bound accounts for the
logarithmic growth on the number of \emph{unique} weak hypothesis $h_t
\in \EHypo \subset \Hypo$ output by $\mathbf{WeakLearn}(D,w_t)$ at
each round $t$ of AdaBoost. Denote the \emph{set of (effective
  representative) hypotheses actually selected by AdaBoost during
  execution} by $\Ucal \equiv \{h_1^*,\ldots,h_{\THat}^* \} \equiv
\bigcup_{t=1}^T \{ h_t \mid h_t = \mathbf{WeakLearn}(D,w_t)\} \subset
\EHypo$, and by $\THat \equiv |\Ucal|$ the \emph{number of (unique)
  hypotheses} AdaBoost actually selects from the set of representative
hypotheses $\Hypo$. Note that $\THat \leq |\EHypo| \leq T^*$.
Note also that the sample $D$
determines 
$\THat$, $\Ucal$, and 
\[ \EHypo_{\textrm{mix}}  \equiv \left\{ \sign{\sum_{t=1}^{\THat} c_t h_t^*(x) } \left| c_t \geq 0, \text{ and } h_t^* \in \Ucal \text{ for all } t=1,\ldots,\THat \right. \right\}\]
(i.e., they
are functions of the training dataset $D$, thus also random variables
with respect to the corresponding probability space
$(\D,\Sigma,P)$). The next theorem only exploits the
empirically-observed logarithmic dependency of $\THat$ on $T$. 
\begin{theorem}
\label{the:databnd}
The following holds with probability $1-\delta$ over the choice of the
training dataset $D$ of i.i.d. samples drawn according to the probability space $(\D,\Sigma,P)$: for all $H \in \EHypo_{\textrm{mix}}$,
\[
\textrm{Err}(H) \leq \widehat{\textrm{Err}}(H) + O\left( \sqrt{ \frac{\left(  (\THat \ln{\THat})
    \VC{\Hypo} \right) \left( 1 + 
    \ln{\frac{m}{(\THat \ln{\THat}) \VC{\Hypo}}} \right) + \ln(1/\delta)}{m}} \right) .
\]
\end{theorem}
The following standard generalization result for AdaBoost is 
useful.
\begin{theorem}{\bf \citep{Freund97adecision-theoretic}}
\label{the:FSgb}
Let \[ \Hypo_{\textrm{mix}} \equiv \left\{ \sign{\sum_{t=1}^T c_t
    h_t^*(x) } \; \mid \,  c_t \geq 0, \text{ and } h_t^* \in \Hypo
  \text{ for all } t=1,\ldots,T \right\}.\] The following holds with probability $1-\delta$ over the choice of the
training dataset $D$ of i.i.d. samples drawn according to the probability space $(\D,\Sigma,P)$: for all $H \in \Hypo_{\textrm{mix}}$,
\[
\textrm{Err}(H) \leq \widehat{\textrm{Err}}(H) + O\left( \sqrt{\frac{\left( (T \ln{T})
    \VC{\Hypo} \right) \left( 1 + \ln{\frac{m}{(T \ln{T})
    \VC{\Hypo}}} \right) + \ln(1/\delta)  }{m}} \right) .
\]
\end{theorem}
We are now ready to prove our data-dependent bound given in Theorem~\ref{the:databnd}. We will first provide a proof sketch, followed by the formal proof.
\begin{proof}
\emph{(Sketch)} Let $\Tmin \equiv \min(T,T^*)$.
The basic idea is to apply the previous theorem
(Theorem~\ref{the:FSgb}) over the number of rounds/base-classifiers $t
=1,\ldots,\Tmin$, using a specific weighting/distribution over $t$.

\noindent \emph{(Formal Proof)} We now present the formal proof. Suppose $\Hypo^{(t)}_{\textrm{mix}} \equiv \left\{ \left. \sum_{s=1}^t c_s \bar{h}_s \right| c_s \geq 0, \bar{h}_s \in \Hypo \right\}$. Using Theorem~\ref{the:FSgb} above, we have, for all $t=1,\ldots,\Tmin$,
\begin{align*}
&\Pbf\left(
  \textrm{Err}(H_t) > \widehat{\textrm{Err}}(H_t) +
\Omega\left( \sqrt{\frac{
    ((t \ln{t}) \VC{\Hypo}) \left( 1 + \ln{\frac{m}{(t \ln{t})
  \VC{\Hypo}}} \right) + \ln{(1/\delta_t)}}{m}} \right), \right. \\ 
& \left. \text{ for some } H_t \in \Hypo^{(t)}_{\textrm{mix}} \right) < \delta_t\; .
\end{align*}
Let $p(t) \equiv \left( \Tmin \right)^{-t} / Z$, where $Z \equiv \sum_{t=1}^{\Tmin} \left(\Tmin \right)^{-t}$ is the normalizing constant. Set $\delta_t = p(t)\, \delta$. Let $K$ be a positive integer and debote by $[K] \equiv \{1,\ldots,K\}$ the set of all positive integers up to and including $K$. 
Applying the Union Bound, substituting the expression for $\delta_t$,
and using some algebra, we obtain
\begin{align*}
& \Pbf\left(\textrm{Err}(H_t) >
  \widehat{\textrm{Err}}(H_t) + \right. \\
& \Omega\left( \sqrt{\frac{
    ((t \ln{t}) \VC{\Hypo}) \left( 1 + \ln{\frac{m}{(t \ln{t})
  \VC{\Hypo}}} \right) + \ln{(1/\delta_t)}}{m}} \right), \\ 
& \left. \text{for some } t \in [\Tmin] \text{ and } H_t \in
  \Hypo^{(t)}_{\textrm{mix}} \right) \\
\leq & \sum_{t=1}^{\Tmin} \Pbf\left( \textrm{Err}(H_t) > \widehat{\textrm{Err}}(H_t) + \Omega\left( \sqrt{\frac{
    ((t \ln{t}) \VC{\Hypo}) \left( 1 + \ln{\frac{m}{(t \ln{t})
       \VC{\Hypo}}} \right) + \ln{(1/\delta_t)}}{m}} \right),
       \right.\\
& \left. \text{
       for some } H_t \in \Hypo^{(t)}_{\textrm{mix}}  \right) \\
< & \sum_{t=1}^{\Tmin} \delta_t 
=  \sum_{t=1}^{\Tmin} p(t) \, \delta 
=  \delta \; .
\end{align*}
Now, turning to the bound on the generalization error, and in particular to the term $\ln{(1/\delta_t)}$, we obtain
\begin{align*}
\ln{(1/\delta_t)} = & \ln{(1/(\delta \, p(t)))}\\
= &  - \ln{p(t)} + \ln{(1/\delta)} \\
= & \ln{Z} + t \; \ln{\Tmin} + \ln{(1/\delta)} \\
\leq & \ln{Z} + t \; \VC{\Hypo} \ln{2} + \ln{(1/\delta)} \\
\leq & \ln{2}  + t \; \VC{\Hypo} \ln{2} + \ln{(1/\delta)} \\
=  & O( t \; \VC{\Hypo} + \ln{(1/\delta)}) \; .
\end{align*}
The result follows by substitution. \qed
\end{proof}

\begin{remark}
For the typical application of Optimal AdaBoost which uses decision stumps as the class of weak/base-classifiers, perhaps the most-commonly used instantiation in practice, the dependence on the number of rounds is reduced to the number of
effective representative
decision stumps \emph{induced} by the data using the \emph{midpoint rule}. 
When $\Hypo$ is the set of half-spaces,
we have 
$|\EHypo| \leq |\Hrep| \leq  2 ( d (m-1) + 1)$, 
where $d$ is the number of features. Lets us denote \emph{the set of
  decision stumps induced by dataset $D$ using the midpoint rule} by
$\EHypo^{\textrm{dstump}}$ and \emph{the set of all $T$,
  non-necessarily unique positively-weighted combination of decision
  stumps in} $\EHypo^{\textrm{dstump}}$ by
$\EHypo_{\textrm{mix}}^{\textrm{dstump}}$. A corollary of
Theorem~\ref{the:databnd} for decision stumps follows by replacing
$\EHypo_{\textrm{mix}}$ and $\VC{\Hypo}$ in the statement of the theorem with
$\EHypo_{\textrm{mix}}^{\textrm{dstump}}$ and $\log_2(d m)$, respectively.

Recall from Table~\ref{tab:eff} that, in practice, $\THat$ could be considerably smaller than $|\EHypo^{\textrm{dstump}}| \subset |\Hypo^{\textrm{dstump}}|$.
Indeed, our empirical
results, represented here in part by the plots in the center column of Figure~\ref{fig:all}, suggest that the expected value of $\THat$ is $\E{\THat} \approx O\left((\log{T})^{3/2}\right)$,
for $T$ ``large enough'' (i.e., after a few initial rounds). Hence the
dependency of the generalization error of AdaBoost on the number of rounds is significantly reduced from $O(\sqrt{T \log{T}})$ to roughly $O\left((\log{T})^{3/4} \sqrt{\log\log{T}} \right)$, a considerable and certainly non-trivial reduction. 

\end{remark}

\subsubsection{Exploiting the number of effective representative weak
  hypotheses too}

We can similarly derive another data-dependent PAC-style bound that tries to exploit the number of \emph{effective} representative classifiers $\EHypo$ too. The statement is slightly more complex.
The uniform-convergence probabilistic bound still accounts for the
logarithmic growth on the number of \emph{unique} weak hypothesis $h_t
\in \EHypo$ output by $\mathbf{WeakLearn}(D,w_t)$ at each round $t$ of
AdaBoost. 
\begin{theorem}
\label{the:databndeff}
The following holds with probability $1-\delta$ over the choice of the
training dataset $D$ of i.i.d. samples drawn according to the
probability space $(\D,\Sigma,P)$: for all $H \in
\EHypo_{\textrm{mix}}$, $\textrm{Err}(H) \leq \widehat{\textrm{Err}}(H) +$
\begin{align*}
O\left( \sqrt{\frac{(\THat
    \ln{|\EHypo|}) \left( 1 + \ln{\THat} \left( 1 + 
    \ln{\frac{m}{(\THat \ln{\THat})
    \ln{|\EHypo|}}} \right)\right) + (|\EHypo| +1) \min(\VC{\Hypo},m) \ln{2} + 
\ln{\frac{1}{\delta}}}{m}} \right) .
\end{align*}
\end{theorem}
\begin{proof}
\emph{(Sketch)}
The basic idea is to apply the previous theorem (Theorem~\ref{the:FSgb}) over a size-induced hierarchy of
base-classifier hypothesis spaces, for sizes $k=1,\ldots,
T^*$, and the number of rounds/base-classifiers $t =1,\ldots,\Tmin$, using a specific weighting/distribution over the hierarchy.

\noindent \emph{(Formal Proof)} We now present the formal proof. Let
$k \in \N$. Denote by $\HypoRep_k \equiv \{ \Hypo_k \in \HypoRep \,
\mid \, |\Hypo_k| = k \}$ the \emph{set of all possible sets of
  exactly $k$
  representative hypotheses}, where $\HypoRep$ is as defined in
Equation~\ref{eqn:hyporep}. Note that $|\HypoRep_k| \leq
{T^* \choose k}$. Consider $\Hypo_k \in \HypoRep_k$. 
Suppose $\Hypo^{(t,k,\Hypo_k)}_{\textrm{mix}} \equiv \left\{ \left. \sum_{s=1}^t c_s \bar{h}_s \right| c_s \geq 0, \bar{h}_s \in \Hypo_k \right\}$. Using Theorem~\ref{the:FSgb} above, we have that for any $H_{t,k,\Hypo_k} \in \Hypo^{(t,k,\Hypo_k)}_{\textrm{mix}}$,
\begin{align*}
& \Pbf\left( \textrm{Err}(H_{t,k,\Hypo_k}) >
  \widehat{\textrm{Err}}(H_{t,k,\Hypo_k}) + \right. \\
& \left. \Omega\left( \sqrt{\frac{
    ((t \ln{t}) \ln{k}) 
\left( 1 + \ln{\frac{m}{(t \ln{t}) \ln{k}}} \right)+ \ln{(1/\delta_{t,k,\Hypo_k})}  }{m}} \right) \right) < \delta_{t,k,\Hypo_k}\; .
\end{align*}
Let 
$p(t,k,\Hypo_k) \equiv k^{-t} |\HypoRep_k|^{-1}/Z$, 
where $Z \equiv \sum_{k =1}^{T^*} \sum_{\Hypo_k \in
       \HypoRep_k} \sum_{t=1}^{\min(k,T)} k^{-t}
|\HypoRep_k|^{-1}$ is the normalizing constant. Set
$\delta_{t,k,\Hypo_k} = p(t,k,\Hypo_k) \, \delta$. 
Applying the Union Bound, substituting the expression for $\delta_{t,k,\Hypo_k}$, and some algebra, we obtain
\begin{align*}
& \Pbf\left( \textrm{Err}(H_{t,k,\Hypo_k}) > \widehat{\textrm{Err}}(H_{t,k,\Hypo_k}) + \Omega\left( \sqrt{\frac{
    ((t \ln{t}) \ln{k}) \left( 1 + \ln{\frac{m}{(t \ln{t}) \ln{k}}} \right) + \ln{(1/\delta_{t,k,\Hypo_k})}  }{m}} \right), \right.\\
& \left. \text{ for some } k \in [T^*], \Hypo_k \in \HypoRep_k, \text{ and } t \in [\min(k,T)]  \right) \\
\leq & \sum_{k =1}^{T^*} \sum_{\Hypo_k \in
       \HypoRep_k} \sum_{t=1}^{\min(k,T)} \Pbf\left( \textrm{Err}(H_{t,k,\Hypo_k}) >
       \widehat{\textrm{Err}}(H_{t,k,\Hypo_k}) + \right. \\ 
& \left. \Omega\left( \sqrt{\frac{
     ((t \ln{t}) \ln{k}) \left( 1 + \ln{\frac{m}{(t \ln{t}) \ln{k}}} \right) + \ln{(1/\delta_{t,k,\Hypo_k})}  }{m}} \right) \right) \\
< & \sum_{k =1}^{T^*} \sum_{\Hypo_k \in
       \HypoRep_k} \sum_{t=1}^{\min(k,T)} \delta_{t,k,\Hypo_k} 
= \sum_{k =1}^{T^*} \sum_{\Hypo_k \in
       \HypoRep_k} \sum_{t=1}^{\min(k,T)} p(t,k,\Hypo_k) \, \delta
= \delta \; .
\end{align*}
Now, turning to the bound on the generalization error, and in particular to the term $\ln{(1/\delta_{t,k,\Hypo_k})}$, we obtain
\begin{align*}
\ln{(1/\delta_{t,k,\Hypo_k})} = & \ln{(1/(\delta \, p(t,k,\Hypo_k)))} \\
= &  - \ln{p(t,k,\Hypo_k)} + \ln{(1/\delta)} \\
= & \ln{Z} + t \ln{k} + \ln{|\HypoRep_k|} + \ln{(1/\delta)} \\
=  & \ln{Z} + t \ln{k} + \ln{T^* \choose k} + \ln{(1/\delta)} \\
\leq  & \ln{Z} + t \ln{k} + k \; \ln{T^*} + \ln{(1/\delta)} \\
=  &\ln{Z} + t \ln{k} + k \; \min(\VC{\Hypo},m) \ln{2} + \ln{(1/\delta)} \; ,
\end{align*}
where
\begin{align*}
Z =  & \sum_{k =1}^{T^*} \sum_{\Hypo_k \in
       \HypoRep_k} \sum_{t=1}^{\min(k,T)} k^{-t} {T^* \choose k}^{-1} \\
=  & \sum_{k =1}^{T^*} \sum_{t=1}^{\min(k,T)} k^{-t} \\
=  & 1 +  \sum_{k =2}^{T^*} \sum_{t=1}^{\min(k,T)} k^{-t} \\
=  & 1 +  \sum_{k =2}^{T^*} \left( \left(\sum_{t=0}^{\min(k,T)} k^{-t}\right) -
     1\right)\\
\leq & 1 +  \sum_{k =2}^{T^*} \left( 2 -
     1\right)\\
= & 1 +  \sum_{k =2}^{T^*} 1\\
= & \sum_{k =1}^{T^*} 1\\
= & T^* \; .
\end{align*}
The result follows by substitution. \qed
\end{proof}

Similarly, in the case of half-spaces/decision stumps, the dependence on the number of rounds decreases to the number of
\emph{effective} representative
decision stumps induced by the data using the midpoint rule. 
A corollary of
Theorem~\ref{the:databndeff} for decision stumps follows by replacing
$\EHypo_{\textrm{mix}}$, $|\EHypo|$, and $\VC{\Hypo}$ in the statement of the theorem with
$\EHypo_{\textrm{mix}}^{\textrm{dstump}}$, $\EHypo^{\textrm{dstump}}$, and $\log_2(d m)$, respectively.
As previously mentioned, recall from Table~\ref{tab:eff} that, in practice, $|\EHypo^{\textrm{dstump}}|$ could be considerably smaller than $\VC{\Hypo}$, and similarly, $\THat$ could be considerably smaller than $|\EHypo^{\textrm{dstump}}|$.

Note that our data-dependent PAC bounds on the generalization error do not increase with $T$ without bound. This is because $|\EHypo| \leq 2^{\min(\VC{\Hypo},m)}$

\end{document}